%% file: ArXivBook.tex
\documentclass[graybox,envcountchap,sectrefs,12pt]{svmono} 
\UseRawInputEncoding

\include{frontmatter/preambleArXiv}

\makeindex             

\begin{document}

\author{Michele Colledanchise and Petter \"Ogren }
\title{Behavior Trees in \\Robotics and AI}
\subtitle{An Introduction}
\date{}

\maketitle

\frontmatter
\pagenumbering{Roman}

\tableofcontents

\mainmatter

\include{bts/bts}
\include{design/design}
\include{extensions/extensions}

\include{properties/properties}

\include{btsvsothers/btsvsothers}
\include{planning/planning}
\include{learning/learning}

\include{stochastic/stochastic}

\include{conclusions/conclusions}

\bibliographystyle{plain}
\bibliography{bookReferences}
\backmatter
\include{glossary}
\include{solutions}
\printindex


\end{document}

%% file: frontmatter/preambleArXiv.tex
\usepackage{amssymb}
\usepackage{amsmath}
\usepackage{graphicx}
\usepackage{makeidx}
\usepackage{multicol}
\usepackage{lscape} 

\usepackage{mathptmx}
\usepackage{helvet}
\usepackage{courier}
\usepackage{type1cm}         

\usepackage{graphicx}        
\usepackage{multicol}        
\usepackage[bottom]{footmisc}
\usepackage{fancybox} 

\usepackage[font=small,labelfont=bf]{caption} 
\usepackage{subcaption}
\usepackage{rotating}

\usepackage{dirtytalk} 
\usepackage{listings} 
\usepackage{url}
 \usepackage{array}

\usepackage{amssymb} 
\usepackage{amsmath} 

\DeclareMathAlphabet\mathbfcal{OMS}{cmsy}{b}{n}

\newcommand{\minus}[1]{{#1}{-}} 
\newcommand{\plus}[1]{{#1}{+}}  

\usepackage[scientific-notation=true]{siunitx} 

\usepackage[linesnumbered,ruled,vlined,algo2e,resetcount]{algorithm2e}
\SetKwRepeat{Do}{do}{while} 

\newtheorem{assumption}{Assumption}[chapter]
\newtheorem{experiment}{Experiment}[chapter]

\newcommand{\bt}{\mathcal{T}} 

\usepackage{xargs}                      
\usepackage[pdftex,dvipsnames]{xcolor}  
\usepackage[colorinlistoftodos,prependcaption,textsize=tiny]{todonotes}
\newcommandx{\unsure}[2][1=]{\todo[linecolor=red,backgroundcolor=red!25,bordercolor=red,#1]{#2}}
\newcommandx{\change}[2][1=]{\todo[linecolor=blue,backgroundcolor=blue!25,bordercolor=blue,#1]{#2}}
\newcommandx{\info}[2][1=]{\todo[linecolor=OliveGreen,backgroundcolor=OliveGreen!25,bordercolor=OliveGreen,#1]{#2}}
\newcommandx{\improvement}[2][1=]{\todo[linecolor=Plum,backgroundcolor=Plum!25,bordercolor=Plum,#1]{#2}}
\newcommandx{\addition}[2][1=]{\todo[linecolor=cyan,backgroundcolor=cyan!25,bordercolor=cyan,#1]{#2}}
\newcommandx{\thiswillnotshow}[2][1=]{\todo[disable,#1]{#2}}

\newcommandx{\todomichele}[2][1=]{\todo[linecolor=blue,backgroundcolor=blue!25,bordercolor=blue,#1]{MICHELE: #2}}
\newcommandx{\todopetter}[2][1=]{\todo[linecolor=OliveGreen,backgroundcolor=OliveGreen!25,bordercolor=OliveGreen,#1]{PETTER: #2}}

%% file: bts/bts.tex
\graphicspath{{bts/figures/}}




\section*{Quotes on Behavior Trees}
\thispagestyle{empty}

\say{I'm often asked why I chose to build the SDK with behavior trees instead of finite state machines. The answer is that behavior trees are a far more expressive tool to model behavior and control flow of autonomous agents.}
\footnote{\url{https://developers.jibo.com/blog/the-jibo-sdk-} \url{reaching-out-beyond-the-screen}}

\begin{flushright}
Jonathan Ross \\Head of Jibo SDK 
\end{flushright}

\say{There are a lot of different ways to create AI's, and I feel like I've tried pretty much all of them at one point or another, but ever since I started using behavior trees, I wouldn't want to do it any other way. I wish I could go back in time with this information and do some things differently.}
\footnote{\url{http://www.gamasutra.com/blogs/ChrisSimpson/20140717/221339/Behavior_trees_for_AI_How_they_work.php}}

\begin{flushright}
Mike Weldon \\Disney, Pixar  
\end{flushright}

\say{[...]. Sure you could build the very same behaviors with a finite state machine (FSM). But anyone who has worked with this kind of technology in industry knows how fragile such logic gets as it grows. A finely tuned hierarchical FSM before a game ships is often a temperamental work of art not to be messed with!}
\footnote{\url{http://aigamedev.com/open/article/fsm-age-is-over/}}

\begin{flushright}
Alex J. Champandard \\Editor in Chief  \& Founder AiGameDev.com, \\Senior AI Programmer Rockstar Games
\end{flushright}

\say{Behavior trees offer a good balance of supporting goal-oriented behaviors and reactivity.}
\footnote{\url{https://forums.unrealengine.com/showthread.php?6016-Behavior-Trees-What-and-Why}}

\begin{flushright}
Daniel Broder \\Unreal Engine developer
\end{flushright}

 \say{The main advantage [of Behavior Trees] is that individual behaviors can easily be reused in the context of another higher-level behavior, without needing to specify how they relate to subsequent behaviors}, \cite{Bagnell2012b}.
\begin{flushright}
Andrew Bagnell et al. \\ Carnegie Mellon University.
\end{flushright}

\chapter{What are Behavior Trees?}
\label{chap:bts}

A Behavior Tree (BT) is a way to structure the switching between different tasks\footnote{assuming that an activity can somehow be broken down into reusable sub-activities called \emph{tasks} sometimes also denoted \emph{actions} or \emph{control modes}}  in an autonomous agent, such as a robot or a virtual entity 
in a computer game. An example of a BT performing a pick and place task can be seen in Fig.~\ref{introduction.fig.BTsmall}. 
As will be explained, BTs are a very efficient way of creating complex systems that are both \emph{modular} and \emph{reactive}.
These properties are crucial in many applications, which has led to the spread of BT
 from computer game programming to many branches of AI and Robotics.

\begin{figure}[ht!]
    \centering

    \begin{subfigure}[t]{\columnwidth}
        \centering
\includegraphics[width = 0.7\columnwidth]{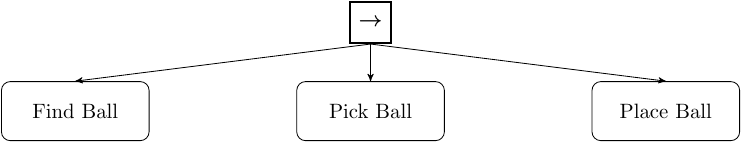}
   \caption{A high level BT carrying out a task consisting of first finding, then picking and finally placing a ball.  }
   \label{introduction.fig.BTsmall}
    \end{subfigure}
    \vspace*{0.5cm}

      \begin{subfigure}[t]{\columnwidth}
        \centering
\includegraphics[width = 0.7\columnwidth]{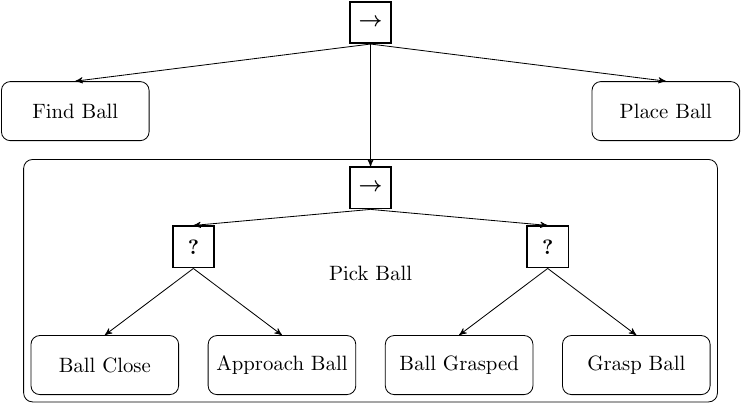}
   \caption{The Action Pick Ball from the BT in Fig.~\ref{introduction.fig.BTsmall} is expanded into a sub-BT. The Ball is approached until it is considered close, and then the Action grasp is executed until the ball is securely grasped.}
      \label{introduction.fig.BTextended}
    \end{subfigure}
    \vspace{1em}
    \caption{Illustrations of a BT carrying out a pick and place task with different degrees of detail. The execution of a BT will be described in Section~\ref{sec:classicalBT}.}
    \label{introduction.fig.BT}
\end{figure}


In this book, we will first give an introduction to BTs, in the present chapter. Then, in Chapter~\ref{ch:earlier_ideas} we describe how BTs relate to, and in many cases generalize, earlier switching structures, or control architectures as they are often called. These ideas are then used as a foundation for a set of efficient and easy to use design principles described in Chapter~\ref{ch:design_principles}.
Then, in Chapter~\ref{ch:extensions} we describe a set of important extensions to BTs.
Properties such as safety, robustness, and efficiency are important for an autonomous system, and in Chapter~\ref{ch:properties} we describe a set of tools for formally analyzing these using a state space formulation of BTs. 
With the new analysis tools, we can formalize the descriptions of how BTs generalize earlier approaches in Chapter~\ref{ch:btsvsothers}.
Then, we see how BTs can be automatically generated using planning, in Chapter~\ref{ch:planning} and learning, in Chapter~\ref{ch:learning}.  
Finally, we describe an extended set of tools to capture the behavior of Stochastic BTs, where the outcomes of actions are described by probabilities, in Chapter~\ref{ch:stochastic}. These tools enable the computation of both success probabilities and time to completion.

In this chapter, we will first tell a brief history of BTs in Section~\ref{sec:history}, 
and explain the core benefits of BTs,  in Section~\ref{sec:modularity},
then in Section~\ref{sec:classicalBT} we will describe how a BT works.
Then, we will create a simple BT for the computer game Pac-Man in Section~\ref{sec:pacman} and a more sophisticated BT for a mobile manipulator in Section~\ref{sec:youbot}. We finally describe the usage of BT in a number of applications in Section~\ref{sec:use_of_BTs}.

\section{A Short History and Motivation of BTs}
\label{sec:history}
BTs were developed in the computer game industry, as a tool to increase
modularity in the control structures of Non-Player Characters (NPCs)\label{definition:NPC}
\cite{isla2005handling,champandard2007understanding,mateas2002abl,isla2008halo,millington2009artificial,rabin2014gameAiPro}.
In this billion-dollar industry, modularity is a key property that enables reuse of code, incremental design of functionality, and efficient testing.

In games, the control structures of NPCs were often formulated in terms of Finite State Machines (FSMs). However, just as Petri Nets \cite{murata1989petri} provide an alternative to FSMs that supports design of \emph{concurrent} systems, BTs provide an alternative view of FSMs that supports design of \emph{modular} systems.

Following the development in the industry, BTs have now also started to receive attention in academia \cite{lim2010evolving,Nicolau2016,shoulson2011parameterizing,bojic2011extending,ogren,Marzinotto14,Bagnell2012b,klockner2013,Colledanchise14, bart,hu2015ablation,guerin2015manufacturing}.

At Carnegie Mellon University, BTs have been used extensively to do robotic manipulation \cite{Bagnell2012b, bart}. 
The fact that modularity is the key reason for using BTs is clear from the following quote from~\cite{Bagnell2012b}:
 \say{The main advantage is that individual behaviors can easily be reused in the context of another higher-level behavior, without needing to specify how they relate to subsequent behaviors}.

BTs have also been used to enable non-experts to do robot programming of pick and place operations,
due to their \say{modular, adaptable representation of a robotic task}
\cite{guerin2015manufacturing} and allowed \say{end-users to visually create programs
with the same amount of complexity and power as traditionally-written programs}~\cite{paxton2016costar}. Furthermore, BTs have been proposed as a key component in brain surgery robotics due to their \say{flexibility, reusability, and simple syntax} \cite{hu2015ablation}.

%
%

\section{What is wrong with FSMs? The Need for Reactiveness and Modularity}
\label{sec:modularity}

Many autonomous agents need to be both reactive and modular.
By reactive we mean the ability to quickly and efficiently react to changes.
We want a robot to slow down and avoid a collision if a human enters into its planned trajectory and we want a virtual game character to hide, flee, or fight, if made aware of an approaching enemy.
By modular, we mean the degree to which a system's components may be separated into building blocks, and recombined \cite{gershenson2003product}. 
We want the agent to be modular, to enable components to be developed, tested, and reused independently of one another. Since complexity grows with size, it is  beneficial to be able to work with components one at a time, rather than the combined system.

FSMs have long been the standard choice when designing a task switching structure \cite{powers2012sting,montemerlo2008junior}, and
 will be discussed in detail in Chapter~\ref{sec:FSM}, but here we make a short description of the unfortunate tradeoff between reactivity and modularity that is inherent in FSMs.
This tradeoff can be understood in terms of the classical Goto-statement that was used in early programming languages.
The Goto statement is an example of a so-called
\emph{one-way control transfer}, where the execution of a program jumps to another part of the code and continue executing from there.
Instead of \emph{one-way control transfers}, modern programming languages tend to rely on 
\emph{two-way control transfers} embodied in e.g. function calls. Here, execution jumps to a particular part of the code, executes it, and then returns to where the function call was made. The drawbacks of \emph{one-way control transfers} were made explicit by 
 Edsgar Dijkstra in his paper \emph{Goto statement considered harmful} \cite{Dijkstra:1968:LEG:362929.362947}, where he states that
``The Goto statement as it stands is just too primitive; it is too much an invitation to make a mess of one's program".
Looking back at the state transitions in  FSMs, we note that they  are indeed \emph{one-way control transfers}. 
This is where the tradeoff between reactivity and modularity is created.
For the system to be reactive, there needs to be many transitions between components, and many transitions means many  \emph{one-way control transfers}
which, just as Dijkstra noted, harms modularity by being an ``invitation to make a mess of one's program". 
If, for example, one component is removed, every transition to that component needs to be revised.
As will be seen, BTs use \emph{two-way control transfers}, governed by the internal nodes of the trees.

 Using BTs instead of FSMs 
to implement the task switching, allows us to describe the desired behavior in modules as depicted in Figure~\ref{introduction.fig.BTsmall}. 
Note that in the next section we will describe how BTs work in detail, so these figures are just meant to give a first glimpse of BTs, rather than the whole picture.

 A behavior is often composed of a sequence of sub-behaviors that are task independent, meaning that while creating one sub-behavior the designer does not need to know which sub-behavior will be performed next. Sub-behaviors can be designed recursively, adding more details as in Figure~\ref{introduction.fig.BTextended}. BTs are executed in a particular way, which will be described in the following section, that allows the behavior to be carried out reactively. For example, the BT in Figure~\ref{introduction.fig.BT} executes the sub-behavior \emph{Place Ball}, but also verifies that the ball is still at a known location and securely grasped. If, due to an external event, the ball slips out out of the grasp, then the robot will abort the sub-behavior \emph{Place Ball} and will re-execute the sub-behavior \emph{Pick Ball} or \emph{Find Ball} according to the current situation. 

\section{Classical Formulation of BTs}
\label{sec:classicalBT}

At the core, BTs are built from a small set of simple components, just as many other powerful concepts, but throughout this book, we will see how this simple formalism can be used to create very rich structures, in terms of both applications and theory.

Formally speaking, a BT is a directed rooted tree where the internal nodes  are called \emph{control flow nodes} and leaf nodes are called \emph{execution nodes}. For each connected node we use the common terminology of  \emph{parent} and \emph{child}. The root is the node without parents; all other nodes have one parent. The control flow nodes have at least one child. Graphically, the children of a node are placed below it, as shown in Figures~\ref{bts.fig.seq}-\ref{bts.fig.par}.

A BT starts its execution from the root node that generates signals that allow the execution of a node called \emph{ticks} with a given frequency, which are sent to its children. A node is executed if and only if it receives ticks. The child immediately returns \emph{Running} to the parent,   if its execution is under way, \emph{Success} if it has achieved its goal, or \emph{Failure} otherwise.

In the classical formulation, there exist four categories of control flow nodes (Sequence, Fallback, Parallel, and Decorator) and two categories of execution nodes (Action and Condition). They  are all explained below and summarized in Table~\ref{bts:tab:nodeTable}.

The Sequence node executes Algorithm \ref{bts:alg:sequence},
which corresponds to routing the ticks to its children from the left until it finds a child that returns either \emph{Failure} or \emph{Running}, then it returns \emph{Failure} or \emph{Running} accordingly to its own parent. It returns \emph{Success} if and only if all its children return \emph{Success}. Note that when a child returns \emph{Running} or \emph{Failure}, the Sequence node does not route the ticks to the next child (if any). The symbol of the Sequence node is a box containing the label \say{$\rightarrow$}, shown in Figure~\ref{bts.fig.seq}.
\begin{figure}[h]
\centering
\includegraphics[width=0.6\columnwidth]{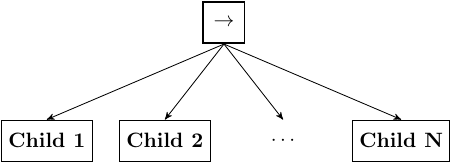}
\caption{Graphical representation of a Sequence node with $N$ children.}
\label{bts.fig.seq}
\end{figure}

\begin{algorithm2e}[h]
  \For{$i \gets 1$ \KwSty{to} $N$}
  {
    \ArgSty{childStatus} $\gets$ \FuncSty{Tick(\ArgSty{child($i$)})}\\
    \uIf{\ArgSty{childStatus} $=$ \ArgSty{Running}}
    {
      \Return{Running}
    }
    \ElseIf{\ArgSty{childStatus} $=$ \ArgSty{Failure}}
    {
      \Return{Failure}
    }
  }
  \Return{Success}
  \caption{Pseudocode of a Sequence node with $N$ children}
  \label{bts:alg:sequence}
\end{algorithm2e}

The Fallback node\footnote{Fallback nodes are sometimes also called \emph{selector} or \emph{priority selector} nodes.} executes Algorithm \ref{bts:alg:fallback},
which corresponds to routing the ticks to its children from the left until it finds a child that returns either \emph{Success} or \emph{Running}, then it returns \emph{Success} or \emph{Running} accordingly to its own parent. It returns \emph{Failure} if and only if all its children return \emph{Failure}. Note that when a child returns \emph{Running} or \emph{Success}, the Fallback node does not route the ticks to the next child (if any).
The symbol of the the Fallback node is a box containing the label \say{$?$}, shown in Figure~\ref{bts.fig.sel}.
\begin{figure}[h]
\centering
\includegraphics[width=0.6\columnwidth]{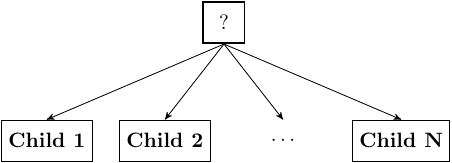}
\caption{Graphical representation of a Fallback node with $N$ children.}
\label{bts.fig.sel}
\end{figure}

\begin{algorithm2e}[h]
  \For{$i \gets 1$ \KwSty{to} $N$}
  {
    \ArgSty{childStatus} $\gets$ \FuncSty{Tick(\ArgSty{child($i$)})}\\
    \uIf{\ArgSty{childStatus} $=$ \ArgSty{Running}}
    {
      \Return{Running}
    }
    \ElseIf{\ArgSty{childStatus} $=$ \ArgSty{Success}}
    {
      \Return{Success}
    }
  }
  \Return{Failure}
  \caption{Pseudocode of a Fallback node with $N$ children}
    \label{bts:alg:fallback}
\end{algorithm2e}

The Parallel node executes Algorithm \ref{bts:alg:parallel},
which corresponds to routing the ticks to all its children and it returns \emph{Success} if $M$ children return \emph{Success}, it returns \emph{Failure} if $N-M+1$ children return \emph{Failure}, and it returns \emph{Running} otherwise, where $N$ is the number of children and $M\leq N$ is a user defined threshold.
The symbol of the the Parallel node is  a box containing the label \say{$\rightrightarrows$}, shown in Figure~\ref{bts.fig.par}.
\begin{figure}[h]
\centering
\includegraphics[width=0.6\columnwidth]{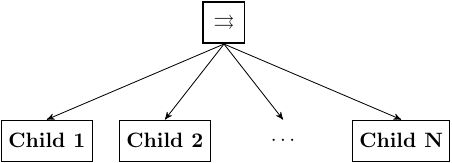}
\caption{Graphical representation of a Parallel node with $N$ children.}
\label{bts.fig.par}
\end{figure}

\begin{algorithm2e}[h]
  \For{$i \gets 1$ \KwSty{to} $N$}
  {
    \ArgSty{childStatus}(i) $\gets$ \FuncSty{Tick(\ArgSty{child($i$)})}\\
    }
    \uIf{$\Sigma_{i: \ArgSty{childStatus}(i)=Success}1\geq M$}
    {
      \Return{Success}
    }
    \ElseIf{$\Sigma_{i: \ArgSty{childStatus}(i)=Failure}1 > N-M$}
    {
      \Return{Failure}
    
  }
  \Return{Running}
    \caption{Pseudocode of a Parallel node with $N$ children and success threshold $M$}
  \label{bts:alg:parallel}
\end{algorithm2e}

\begin{figure}[h]
        \centering
        \begin{subfigure}[b]{0.3\columnwidth}
                \centering
                \includegraphics[width=0.4\columnwidth]{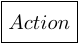}
                \caption{Action node. The label describes the action performed.}
                \label{bts.fig.act}              
        \end{subfigure}
        ~
                \begin{subfigure}[b]{0.3\columnwidth}
                \centering
                \includegraphics[width=0.7\columnwidth]{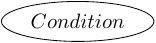}
                \caption{Condition node. The label describes the condition verified.}
                \label{bts.fig.cond}
        \end{subfigure}
        ~
         \begin{subfigure}[b]{0.3\columnwidth}
                \centering
                \includegraphics[width=0.5\columnwidth]{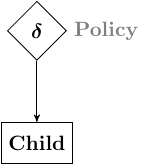}
                \caption{Decorator node. The label describes the user defined policy.}
                \label{bts.fig.dec}
        \end{subfigure}%
        \caption{Graphical representation of  Action (a), Condition (b), and Decorator (c) node.}
\end{figure}

When it receives ticks, an Action node executes a command. It returns \emph{Success} if the action is correctly completed or \emph{Failure} if the action has failed. While the action is ongoing it returns \emph{Running}. An Action node is shown in Figure~\ref{bts.fig.act}.

When it receives ticks, a Condition node checks a proposition. It returns \emph{Success} or \emph{Failure} depending on if the proposition holds or not. Note that a Condition node never returns a status of \emph{Running}. A Condition node is shown in Figure~\ref{bts.fig.cond}.

The Decorator node is a control flow node with a single child that manipulates the return status of its child according to a user-defined rule and also selectively ticks the child according to some predefined rule. For example, an \emph{invert} decorator inverts the \emph{Success}/\emph{Failure} status of the child; a \emph{max-N-tries} decorator only lets its child fail $N$ times, then always returns \emph{Failure} without ticking the child; a \emph{max-T-sec} decorator lets the child run for $T$ seconds then, if the child is still Running, the Decorator returns \emph{Failure} without ticking the child. The symbol of the Decorator is a rhombus, as in Figure~\ref{bts.fig.dec}.

\begin{table*}[htp]
\scriptsize
\begin{center}
\begin{tabular}{|c|c|c|c|c|c|}
\hline
 \bf{Node type} & \bf{Symbol}& \bf{Succeeds} & \bf{Fails} & \bf{Running} \cr

\hline
 Fallback  &\fbox{?} & If one child succeeds & If all children fail &If one child returns Running \cr
\hline
Sequence &\fbox{$\rightarrow$} &If all children succeed & If one child fails &If one child returns Running \cr
\hline
Parallel &\fbox{$\rightrightarrows$} & If $\geq M$ children succeed & If $>N-M$ children fail &else \cr
\hline
Action & \fbox{text}& Upon completion & If impossible to complete & During completion \cr
\hline
\cornersize{.9} 
Condition &\ovalbox{text} & If true & If false & Never  \cr
 \hline
Decorator & $\Diamond$& Custom  & Custom & Custom \cr
\hline
\end{tabular}
\end{center}
\caption{The node types of a BT.}
\label{bts:tab:nodeTable}
\end{table*}%

\subsection{Execution Example of a BT}
\label{bts.ee}
Consider the BT in Figure~\ref{bts.fig.btmot} 
designed to make an agent look for a ball, approach it, grasp it, proceed to a bin, and place the ball in the bin.
This example will illustrate the execution of the BT, including the reactivity when another (external) agent takes the ball from the first agent, 
making it switch to looking for the ball and approaching it again.
When the execution starts, the ticks traverse the BT reaching the condition node \emph{Ball Found}. The agent does not know the ball position hence the condition node returns  \emph{Failure} and the ticks reach the Action \emph{Find Ball}, which returns  \emph{Running} (see Figure~\ref{bts.fig.btmot1}). While executing this action, the agent  sees the ball with the camera.  In this new situation the agent knows the ball position. Hence the condition node \emph{Ball Found} now returns  \emph{Success} resulting in the ticks no longer reaching the Action node \emph{Find Ball} and the action is preempted. The ticks continue exploring the tree, and reach the condition node \emph{Ball Close}, which returns \emph{Failure} (the ball is far away) and then reach the Action node \emph{Approach Ball}, which returns  \emph{Running} (see Figure~\ref{bts.fig.btmot2}). Then the agent eventually reaches the ball, picks it up and goes towards the bin (see Figure~\ref{bts.fig.btmot3}). When an external agent moves the ball from the hand of the first agent to the floor (where the ball is visible), the condition node \emph{Ball Found} returns  \emph{Success} while the condition node \emph{Ball Close} returns  \emph{Failure}. In this situation the ticks no longer reach the Action \emph{Approach Bin} (which is preempted) and they instead reach the Action \emph{Approach Ball} (see Figure~\ref{bts.fig.btmot4}).


\begin{figure}[h!]
\centering
\includegraphics[width=\columnwidth]{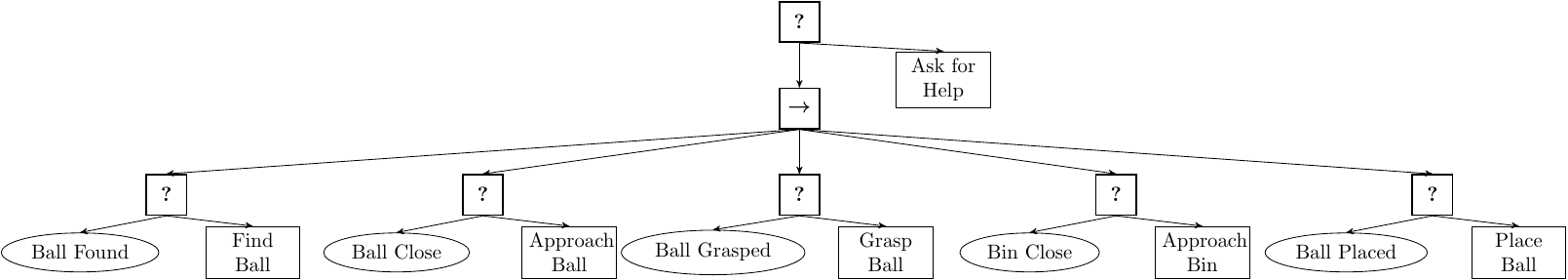}
\caption{BT encoding the behavior of Example~\ref{Introduction.ex.motivating}.}
\label{bts.fig.btmot}
\end{figure}

\begin{figure}[h!]
\centering
\begin{subfigure}[b]{\columnwidth}
\includegraphics[width=\columnwidth]{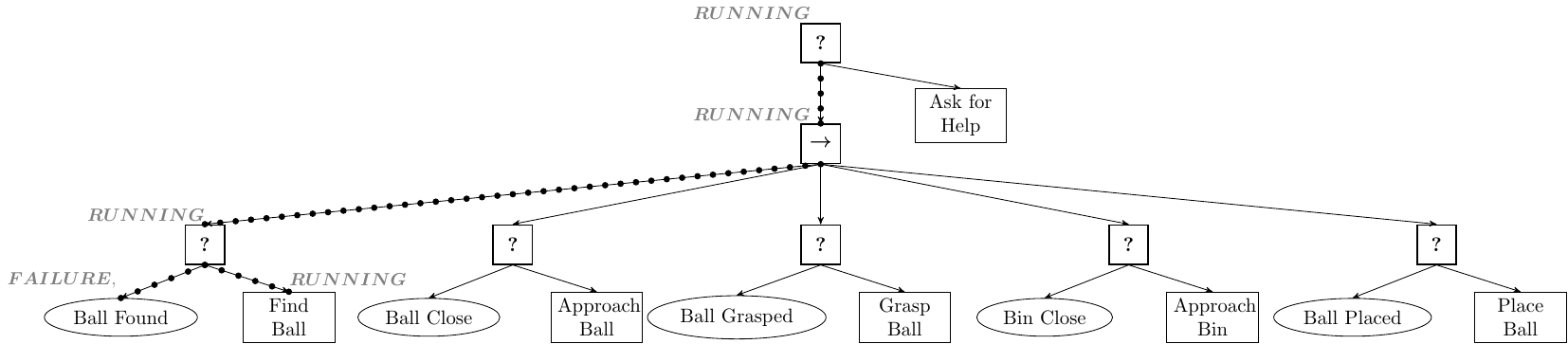}
\caption{Ticks' traversal when the robot is searching the ball.}
\label{bts.fig.btmot1}
\end{subfigure}

\begin{subfigure}[b]{\columnwidth}
\includegraphics[width=\columnwidth]{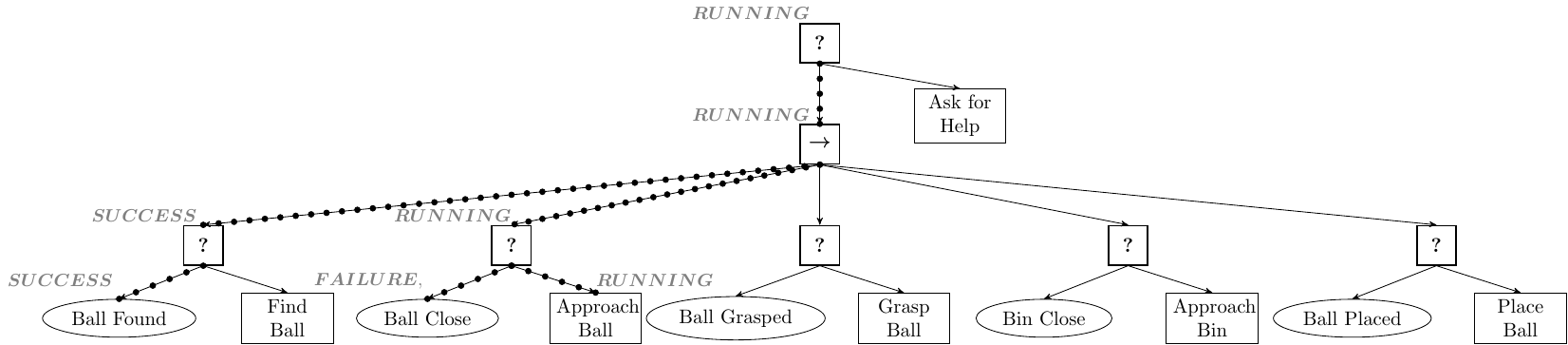}
\caption{Ticks'  traversal while the robot is approaching the ball.}
\label{bts.fig.btmot2}
\end{subfigure}

\begin{subfigure}[b]{\columnwidth}
\includegraphics[width=\columnwidth]{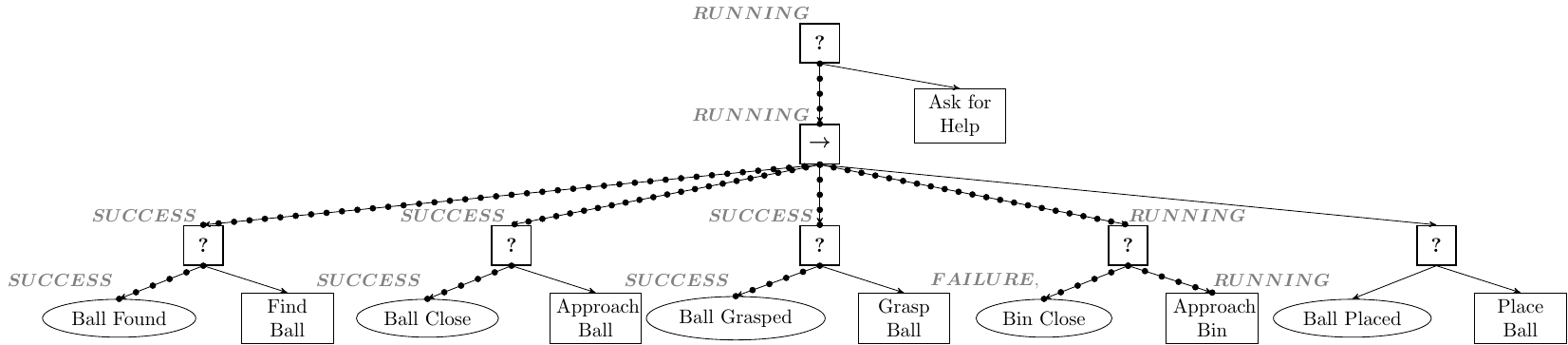}
\caption{Ticks'  traversal while the robot is approaching the bin.}
\label{bts.fig.btmot3}
\end{subfigure}

\begin{subfigure}[b]{\columnwidth}
\includegraphics[width=\columnwidth]{BTExample1Tick2}
\caption{Ticks'  traversal while the robot is approaching the ball again (because it was removed from the hand).}
\label{bts.fig.btmot4}
\end{subfigure}
\caption{Visualization of the ticks' traversal in the different situations, as explained in Section~\ref{bts.ee}.}
\label{bts.fig.btmot1to4}
\end{figure}

\subsection{Control Flow Nodes with Memory}
\label{bt:sec:mem}
As seen in the example above, to provide reactivity
the control flow nodes Sequence and Fallback keep sending ticks to the children to the left of a running child, in order to verify whether a child has to be re-executed and the current one has to be preempted. 
However, sometimes the user knows that a child, once executed, does not need to be re-executed.

Nodes with memory~\cite{millington2009artificial} have been introduced to enable the designer to avoid the unwanted re-execution of some nodes. Control flow nodes with memory always remember whether a child has returned  \emph{Success} or \emph{Failure}, avoiding the re-execution of the child until the whole Sequence or Fallback finishes in either \emph{Success} or \emph{Failure}. In this book, nodes with memory are graphically represented with the addition of the symbol  \say{$*$} (e.g. a Sequence node with memory is graphically represented by a box with a \say{$\rightarrow^*$}).
The memory is cleared when the parent node returns either  \emph{Success} or \emph{Failure}, so that at the next activation all children are considered. Note however that every execution of a control flow node with memory can be obtained with a non-memory BT using some auxiliary conditions as shown in Figure~\ref{bts.fig.mem}. Hence nodes with memory can be considered to be syntactic sugar.

\begin{figure}[h]

        \begin{subfigure}[b]{0.3\columnwidth}
                \centering
                \includegraphics[width=\columnwidth]{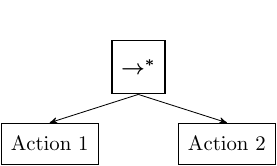}
                \caption{Sequence composition with memory.}
                \label{bts.fig.starnonreac}
        \end{subfigure}          
        ~ 
        \begin{subfigure}[b]{0.6\columnwidth}
                \centering
                \includegraphics[width=\columnwidth]{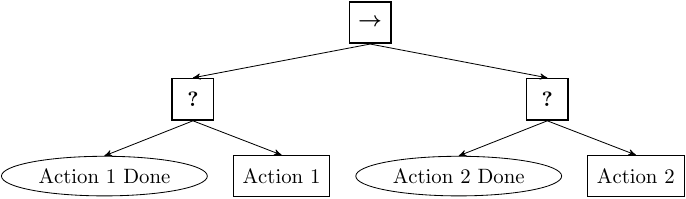}
                \caption{BT that emulates the execution of the Sequence composition with memory using nodes without memory.}
                \label{bts.fig.starreac}              
        \end{subfigure}
        \caption{Relation between memory and memory-less BT nodes.}
                        \label{bts.fig.mem}              
\end{figure}

\begin{remark}
 Some BT implementations, such as the one described in~\cite{millington2009artificial},  do not include the \emph{Running} return status. Instead, they let each Action run until it returns \emph{Failure} or \emph{Success}. We denote these BTs as  \emph{non-reactive}, since they do not allow  actions other than the currently active one to react to changes. This is a significant limitation on non-reactive BTs, which was also noted in \cite{millington2009artificial}. A non-reactive BT can be seen as a BT with only memory nodes.
 
As reactivity is one of the key strengths of BTs, the non-reactive BTs are of limited use. 
\end{remark}

\newpage	
\section{Creating a BT for Pac-Man from Scratch}
\label{sec:pacman}
In this section we create a set of BTs of increasing complexity for playing the game Pac-Man.
The source code of all the examples is publicly available and editable.\footnote{\url{https://btirai.github.io/}}
We use  a clone of the Namco's Pac-Man computer game depicted in Figure~\ref{bts.fig.ScenarioPacMan}\footnote{The software was developed at UC Berkeley for educational purposes. More information available at: \url{http://ai.berkeley.edu/project_overview.html}}.    

In the testbed, a BT controls the agent, Pac-Man, through a maze containing two ghosts, a large number of  pills, including two so-called power pills. 
The goal of the game is to consume all the pills, without being eaten by the ghosts. The power pills are such that, if eaten, Pac-Man receives temporary super powers, and is able to eat the ghosts. After a given time the effect of the power pill wears off, and the ghosts can again eat Pac-Man.
 When a ghost is eaten, it returns to the center box where it is regenerated and becomes dangerous again. Edible ghosts change color, and then flash to signal when they are about to become dangerous again.
\begin{figure}[h]
\centering
\includegraphics[width=\columnwidth]{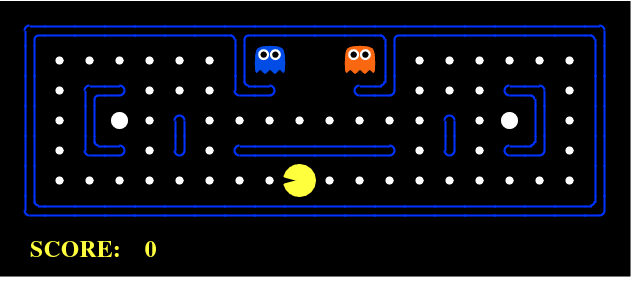}
\caption{The game Pac-Man for which we will design a BT. There exists maps of different complexity.}
\label{bts.fig.ScenarioPacMan}
\end{figure}

The simplest behavior is to let Pac-Man ignore the ghosts and just focus on eating pills.
This is done using a greedy action
 \emph{Eat Pills} as in Figure~\ref{bts.fig.greedy}.

\begin{figure}[h]
\centering
\includegraphics[width=0.2\columnwidth]{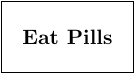}
\caption{BT for the simplest non-random behavior, \emph{Eat Pills}, which maximizes the number of pills eaten in the next time step.}
\label{bts.fig.greedy}
\end{figure}

The simple behavior described above ignores the ghosts. 
To take them into account,  we can extend the previous behavior by adding an \emph{Avoid Ghosts} Action to be executed whenever the condition \emph{Ghost Close} is true. This Action will 
greedily maximize the distance to all ghosts.
The new Action and condition can be added to the BT as depicted in 
 Fig.~\ref{bts.fig.avoid}.
The resulting BT will switch between Eat Pills and Avoid Ghost depending on whether Ghost Close returns Success or Failure.

\begin{figure}[h]
\centering
\includegraphics[width=0.5\columnwidth]{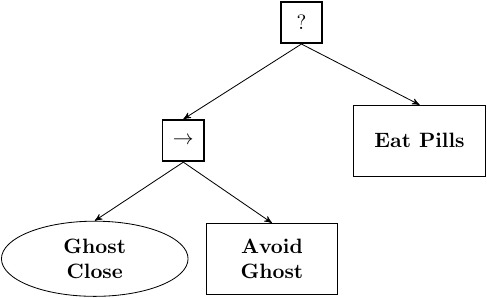}
\caption{If a Ghost is Close, the BT will execute the Action Avoid Ghost, else it will run Eat Pills.}
\label{bts.fig.avoid}
\end{figure}

The next extension we make is to take the power pills into account.
When Pac-Man eats a Power pill, the ghosts are edible, and we would like to chase them, instead of avoiding them.
To do this we add the condition \emph{Ghost  Scared} and the Action \emph{Chase Ghost} to the BT, as shown in Fig.~\ref{bts.fig.eat}.
\emph{Chase Ghost} greedily minimizes the distance to the closest edible ghost.
Note that we only start chasing the ghost if it is close, otherwise we continue eating pills.
Note also that all extensions are modular, without the need to rewire the previous BT.

\begin{figure}[h]
\centering
\includegraphics[width=0.5\columnwidth]{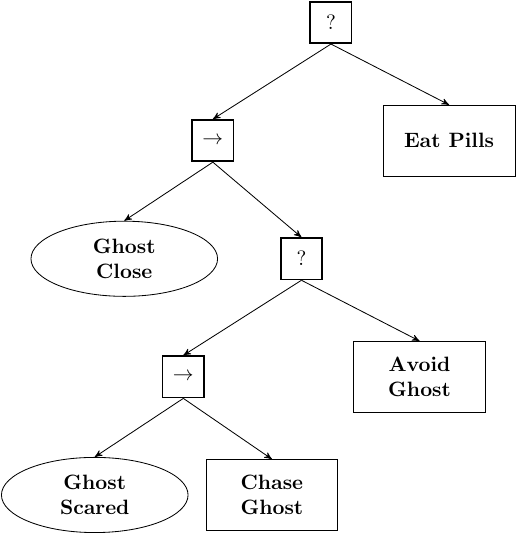}
\caption{BT for the Combative Behavior}
\label{bts.fig.eat}
\end{figure}

With this incremental design, we have created a basic AI for playing Pac-Man, but what if we want to make a world class Pac-Man AI?
You could add additional nodes to the BT, such as moving towards the Power pills when being chased, and stop chasing ghosts when they are blinking and soon
will transform into normal ghosts. However, much of the fine details of Pac-Man lies in considerations of the Maze geometry, choosing paths that avoid dead ends and 
possible capture by multiple ghosts.
Such spatial analysis is probably best done inside the actions, e.g., making Avoid Ghosts take dead ends and ghost positions into account.
The question of what functionality to address in the BT structure, and what to take care of inside the actions is open, and must be decided on a case to case basis,
as discussed in Section~\ref{design:sec:granularity}.

\section{Creating a BT for a Mobile Manipulator Robot}
\label{sec:youbot}
\begin{figure}[h]
\centering
\includegraphics[width=\columnwidth]{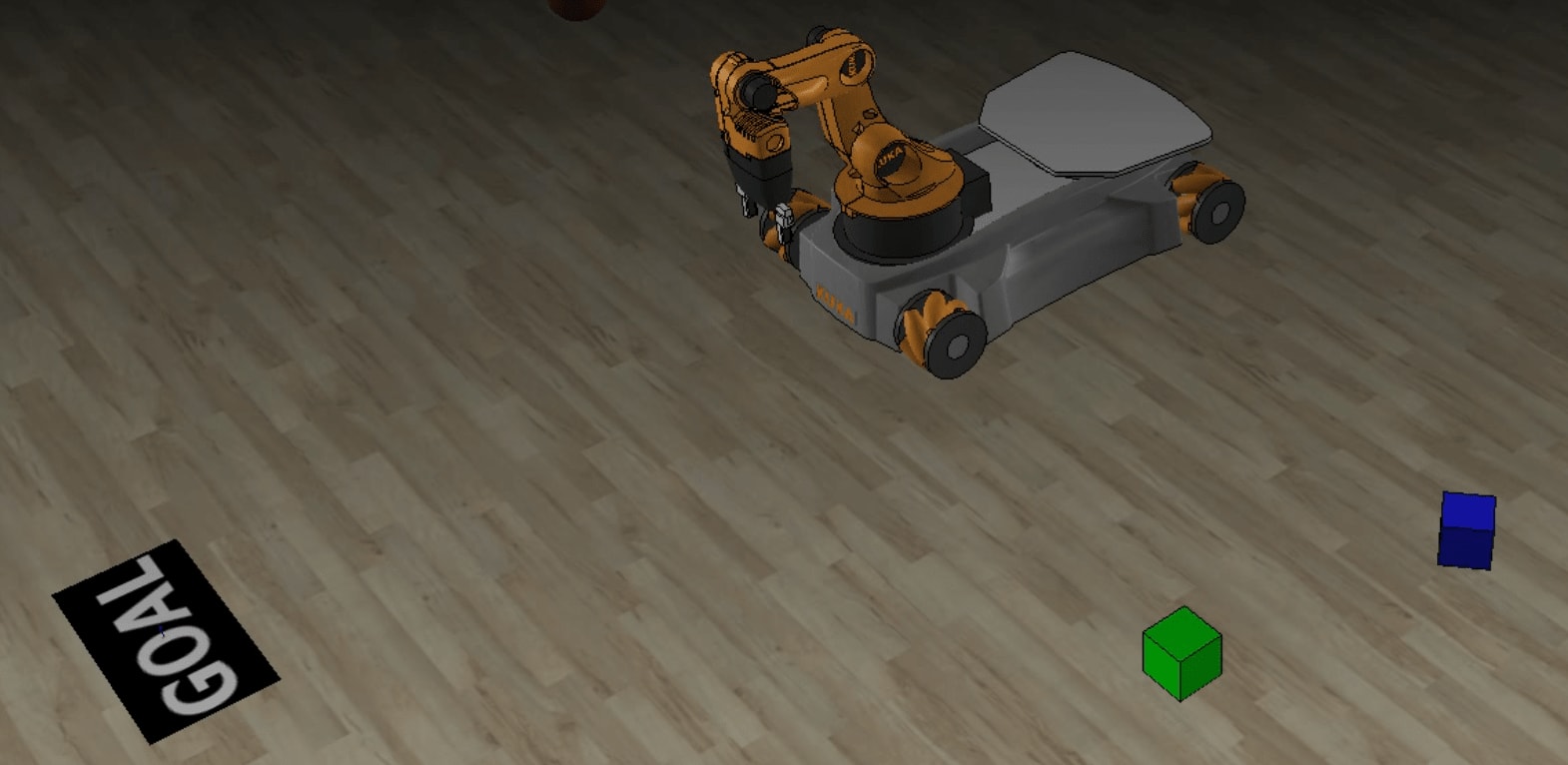}
\caption{The Mobile Manipulator for which we will design a BT.}
\label{bts.fig.ScenarioYouBot}
\end{figure}

In this section, we create a set of BTs of increasing complexity for controlling a mobile manipulator.
The source code of all the examples is publicly available and editable.\footnote{\url{https://btirai.github.io/}}
We use a custom-made testbed created in the V-REP robot simulator depicted in Figure~\ref{bts.fig.ScenarioYouBot}.

In the testbed, a BT controls a mobile manipulator robot, a youBot, on a flat surface. In the scenario, several colored cubes are lying on a flat surface. The goal is to move the green cube to the goal area while avoiding the other cubes. The youBot's grippers are such that the robot is able to pick and place the cubes if the robot is close enough.

The simplest possible BT is to check the goal condition \emph{Green Cube on Goal}. If this condition is satisfied (i.e. the cube is on the goal) the task is done, if it is not satisfied the robot needs to \emph{place the cube} onto the goal area. To correctly execute the Action \emph{Place Cube}, two conditions need to hold: the robot \emph{is holding the green cube} and the robot \emph{is close to the goal area}. The behavior described so far can be encoded in the BT in Figure~\ref{bts.fig.youbotsimple}. This BT is able to 
place the green cube on the goal area if and only if the robot is close to the goal area with the green cube grasped.

\begin{figure}[h]
\centering
\includegraphics[width=0.5\columnwidth]{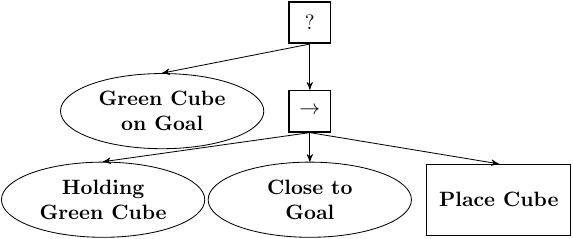}
\caption{BT for the simple Scenario.}
\label{bts.fig.youbotsimple}
\end{figure}

Now, thanks to the modularity of BTs, we can separately design the BTs needed to satisfy the two lower conditions in Fig.~\ref{bts.fig.youbotsimple}, i.e., the BT needed to grasp the green cube and the BT needed to reach the goal area. To grasp the green cube, the robot needs to have the \emph{hand free} and be \emph{close to the cube}. If it is not close, it approaches as long as a collision free trajectory exists. This behavior is encoded in the BT in Figure~\ref{bts.fig.youbotpick}. To reach the goal area we let the robot simply \emph{Move To the Goal}  as long as a \emph{collision free trajectory exists}. This behavior is encoded in the BT in Figure~\ref{bts.fig.youbotmove}.

\begin{figure}[ht!]
    \centering

    \begin{subfigure}[t]{\columnwidth}
        \centering
\includegraphics[width = 0.45\columnwidth]{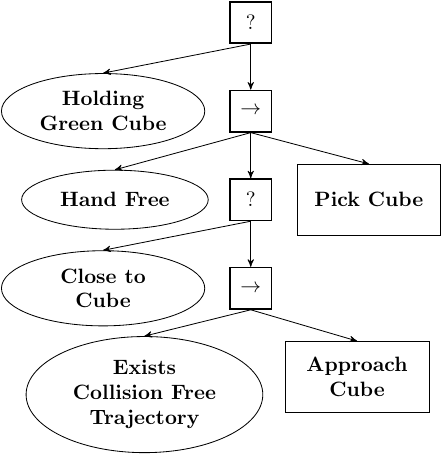}
   \caption{A BT that picks the green cube. }
   \label{bts.fig.youbotpick}
    \end{subfigure}
    
      \begin{subfigure}[t]{\columnwidth}
        \centering
\includegraphics[width = 0.45\columnwidth]{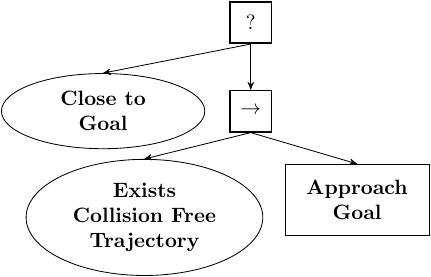}
   \caption{A BT that reaches the goal region.}
      \label{bts.fig.youbotmove}
    \end{subfigure}
    \vspace{1em}
    \caption{Illustrations of a BT carrying out the subtasks of picking the green cube and reaching the goal area}
    \label{bts.fig.youbotpickmove}
\end{figure}

Now we can extend the simple BT in Fig.~\ref{bts.fig.youbotsimple} above by replacing the two lower conditions in Fig.~\ref{bts.fig.youbotsimple}
with the two BTs in Fig.~\ref{bts.fig.youbotpickmove}.
The result can be seen  in Fig.~\ref{bts.fig.youbotfinal}. Using this design, the robot is able to place the green cube in the goal area as long as there exists a collision free trajectory to the green cube and to the goal area.

\begin{figure}[h]
\centering
\includegraphics[width=0.7\columnwidth]{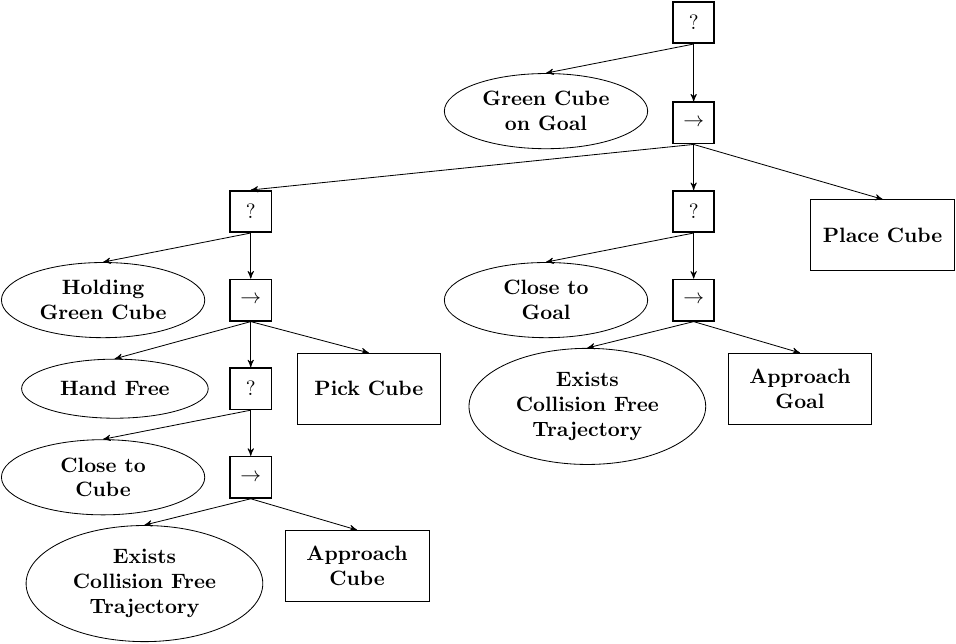}
\caption{Final BT resulting from the aggregation of the BTs in Figs.~\ref{bts.fig.youbotsimple}-\ref{bts.fig.youbotpickmove}}
\label{bts.fig.youbotfinal}
\end{figure}

We can continue to incrementally build the BT in this way to handle more situations, for instance removing obstructing objects to ensure that a \emph{collision free trajectory exists}, and dropping things in the hand to be able to pick the green cube up. 

\section{Use of BTs in Robotics and AI}
\label{sec:use_of_BTs}
In this section we describe the use  of BTs in a set of real robot applications and projects, spanning from autonomous driving to industrial robotics. 

\subsection{BTs in autonomous vehicles}

There is no standard control architecture for autonomous vehicles, however reviewing the architectures used to address the DARPA Grand Challenge, a competition for autonomous vehicles, we note that most teams employed FSMs  designed and developed exactly for that challenge \cite{urmson2008autonomous, urmson2007tartan}. Some of them used a HFSM\cite{montemerlo2008junior} decomposing the mission task in multiple subtasks in a  hierarchy. 
As discussed in Section~\ref{sec:modularity} there is reason to believe that using BTs instead of FSMs would be beneficial for autonomous driving applications.

\begin{figure}[h]
\centering
\includegraphics[width=\columnwidth]{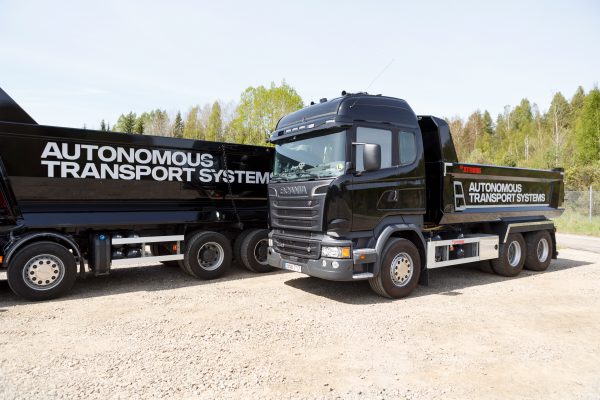}
\caption[Trucks running the Scania iQmatic's software.]{Trucks running the Scania iQmatic's software.\footnotemark}
\label{bts.fig.iQmatic}
\end{figure}

iQmatic is a Scania-led project that aims at developing a fully autonomous heavy truck for goods transport, mining, and other industrial applications. The vehicle's software has to be reusable, maintainable and easy to develop. For these reasons, the iQmatic's developers chose BTs as the  control architecture  for the project. BTs are appreciated in iQmatic for their human readability,  supporting the design and development of early prototypes; and their maintainability,  making the editing task easier. Figure~\ref{bts.fig.iQmatic} shows two trucks used in the iQmatic project.\footnotetext{Picture courtesy of \url{Scania.com}}

\subsection{BTs in industrial robotics}
Industrial robots usually operate in structured environments and their  control architecture  is designed for a single specific task. Hence classical architectures such as FSMs or Petri Nets~\cite{murata1989petri} have found successful applications in the last decades. However, future generations of collaborative industrial robots, so-called cobots, will operate in less structured environments and  collaborate closely with humans. Several research projects explore this research direction. 

\begin{figure}[h]
\centering
\includegraphics[width=\columnwidth]{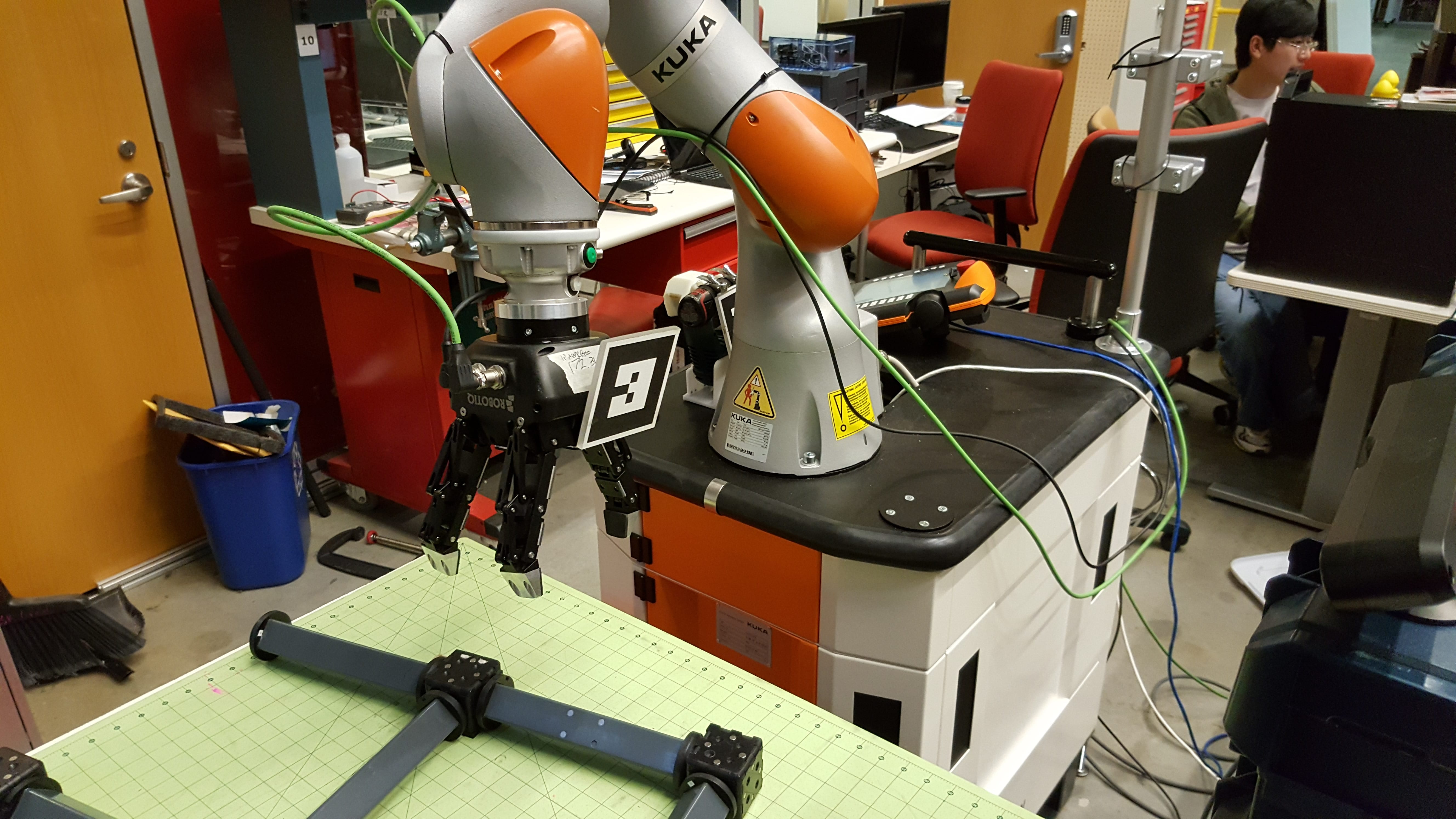}
\caption[Experimental platform of the CoSTAR project.]{Experimental platform of the CoSTAR project.\footnotemark}
\label{bts.fig.CoSTAR}
\end{figure}

CoSTAR~\cite{paxton2016costar} is a project that aims at developing a software framework that contains tools for  industrial applications that involve human cooperation.\footnotetext{Picture courtesy of \url{http://cpaxton.github.io/}} The use cases include non trained operators composing task plans, and training 
 robots  to perform complex behaviors.
BTs have found successful applications in this project as they simplify the composition of subtasks. The order in which the subtasks are executed is independent from the subtask implementation; this enables  easy composition of trees and the iterative composition of larger and larger trees. Figure~\ref{bts.fig.CoSTAR} shows one of the robotic platforms of the project.


SARAFun\footnote{\url{http://sarafun.eu}} is a project that aims at developing a robot-programming framework that enables a non-expert user to program an assembly task from scratch on a robot in less than a day. It takes advantages of state of the art techniques in sensory and cognitive abilities, robot control, and planning.
\begin{figure}[h]
\centering
\includegraphics[trim={0 2cm 0 15cm},clip,width=\columnwidth]{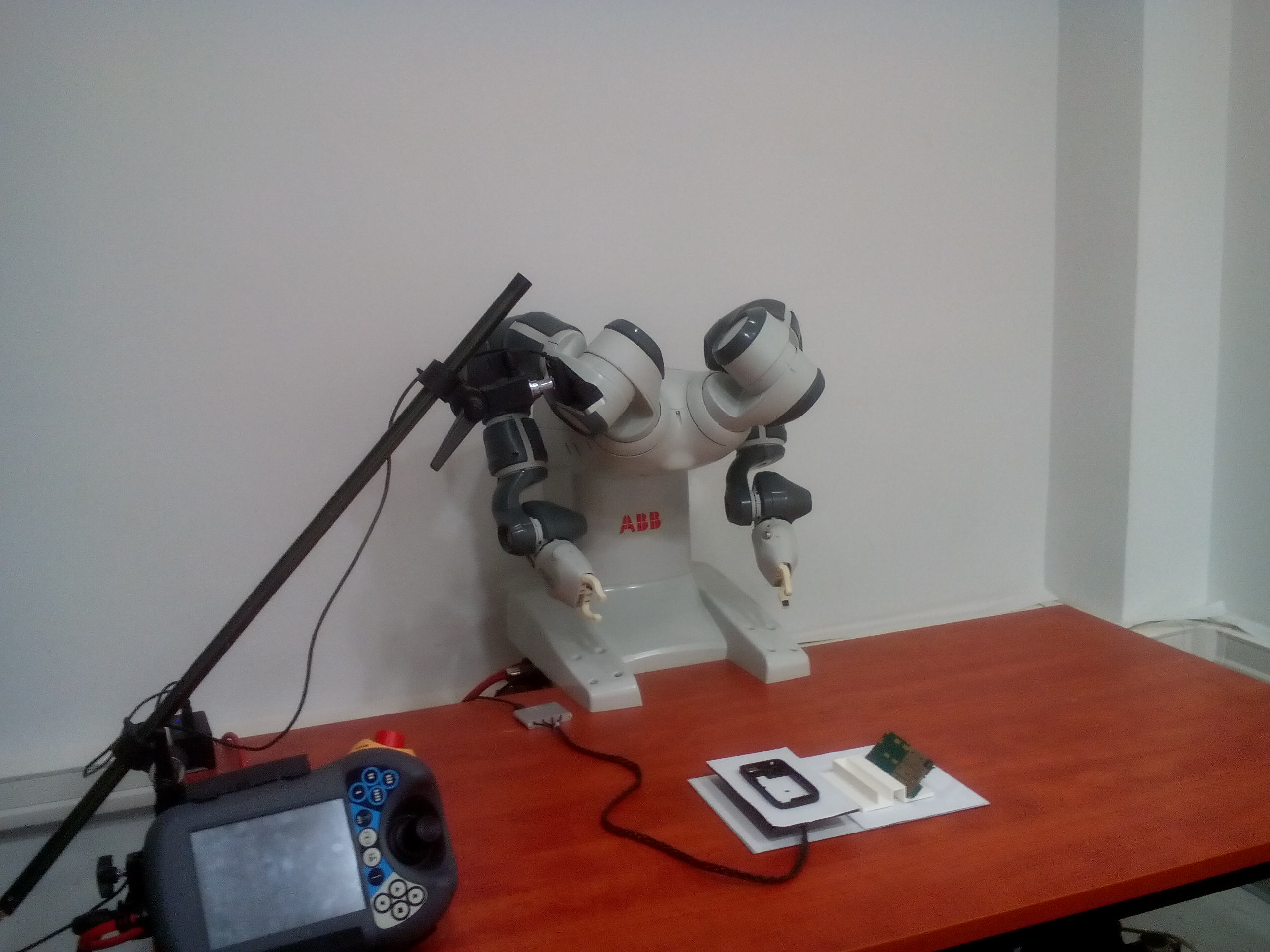}

\caption[Experimental platform of the SARAFun projectE.]{Experimental platform of the SARAFun project.\footnotemark}
\label{bts.fig.SARAFun}
\end{figure}

BTs are used to execute the generic actions learned or planned. For the purpose of this project, the  control architecture  must be human readable, enable code reuse, and modular.
 BTs have  created advantages also during the  development stage, when the code written by different partners had to be integrated. Figure~\ref{bts.fig.SARAFun} shows an ABB Yumi robot used in the SARAFun testbed.\footnotetext{Setup located at CERTH, Thessaloniki, Greece. Picture courtesy of Angeliki Topalidou-Kyniazopoulou.}



\begin{figure}[h]
  \includegraphics[width=0.9\textwidth]{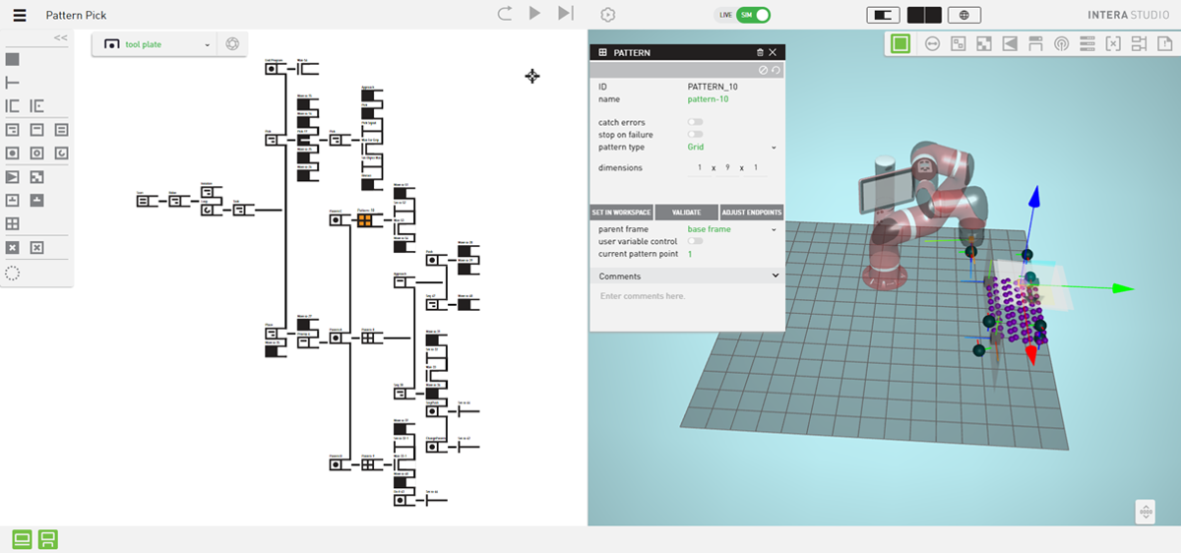}
  \caption[Intera's BT (left) and simulation environment (right).]{Intera's BT (left) and simulation environment (right).\footnotemark}
  \label{intera}
\end{figure}
  \footnotetext{Picture courtesy of \url{http://www.rethinkrobotics.com/intera/}}

Rethink Robotics released its software platform Intera in 2017, with
BTs  at the \say{heart of the design}. Intera claims to be a \say{first-of-its-kind software platform that connects everything from a single robot controller, extending the smart, flexible power of Rethink Robotics' Sawyer to the entire work cell and simplifying automation with unparalleled ease of deployment.}\footnote{\url{http://www.rethinkrobotics.com/news-item/rethink-robotics-releases-} \url{intera-5-new-approach-automation/}} It is designed with the goal of creating the world's fastest-to-deploy robot and fundamentally changing the concepts of integration, making it drastically easier and affordable.


Intera's BT defines the Sequence of tasks the robot will perform. The tree can be created manually or trained by demonstration. Users can inspect any portion of the BT and make adjustments. The Intera interface (see Figure~\ref{intera}) also includes a simulated robot, so a user can run simulations while  the program executes the BT. BTs are appreciated in this context because the train-by-demonstration framework builds a BT that is easily inspectable and modifiable.\footnote{\url{http://twimage.net/rodney-brooks-743452002}}

\subsection{BTs in the Amazon Picking Challenge}

\begin{figure}[h!]
\centering
  \includegraphics[width=0.9\textwidth]{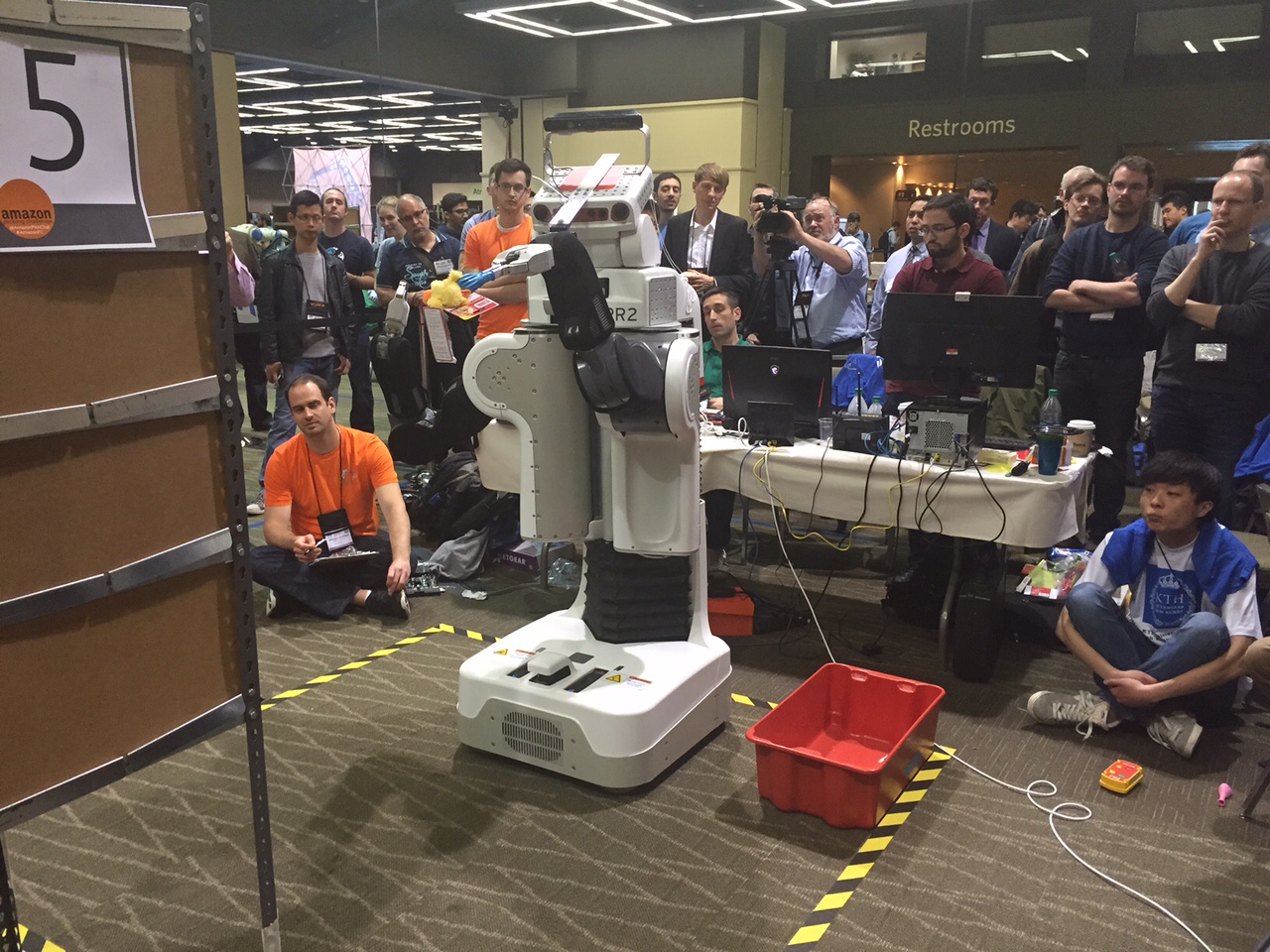}
  \caption{The KTH entry in the Amazon Picking Challenge at ICRA 2015.}
  \label{amazon}
\end{figure}

The Amazon Picking Challenge (APC) is an international robot competition. Robots need to autonomously retrieve a wide range of products from a shelf and put them into a container. The challenge was conceived with the purpose of strengthening the ties between  industrial and academic robotic research, promoting shared solutions to some open problems in unstructured automation. Over thirty companies and research laboratories from different continents competed in the APC's preliminary phases. The best performing teams earned the right to compete at the finals and the source codes of the finalists were made publicly available. \footnote{\url{https://github.com/amazon-picking-challenge}}

The KTH entry in the final challenge used BTs in both  2015 and 2016. BTs were appreciated for their modularity and code reusability, which allowed the integration of different functionalities developed by programmers with different background and coding styles. In  2015, the KTH entry got the best result out of the four teams competing with PR2 robots.

\subsection{BTs inside the social robot JIBO}
JIBO is a social robot that can recognize faces and voices, tell jokes, play games, and share information.
It is intended to be used in homes, providing the functionality of a tablet, but with an interface relying on speech and video instead of a touch screen. JIBO has been featured in Time Magazine's Best Inventions of 2017.\footnote{\url{http://time.com/5023212/best-inventions-of-2017/}} BTs are a fundamental part of the software architecture of JIBO\footnote{\url{https://developers.jibo.com/docs/behavior-trees.html}},
including an open SDK inviting external contributors to develop new skills for the robot.

\begin{figure}[ht]
\centering
  \includegraphics[width=0.9\textwidth]{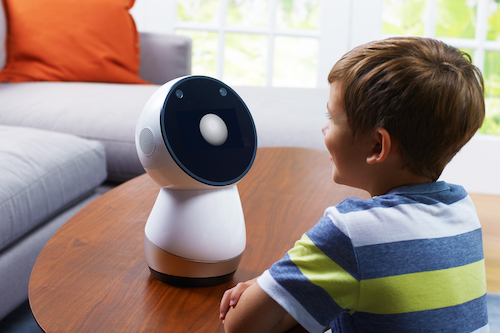}
  \caption{The JIBO social robot has an SDK based on BTs.}
  \label{jibo}
\end{figure}

%
%

\chapter{How Behavior Trees Generalize and Relate to Earlier Ideas}
\label{ch:earlier_ideas}

In this chapter, we describe how BTs relate to, and often generalize, a number of well known control architectures including FSMs (Section~\ref{sec:FSM}),
the Subsumption Architecture (Section~\ref{sec:SA}), the Teleo-Reactive Approach (Section~\ref{sec:TR}) and Decision Trees (Section~\ref{sec:DT}). We also present advantages and disadvantages of each approach.
Finally, we list a set of advantages and disadvantages of BTs 
(\ref{btasca.sec.properties}).
Some of the results of this chapter were previously published in the journal paper \cite{colledanchise2017behavior}.

\label{btasca.sec.ca}

\section{Finite State Machines}
\label{sec:FSM}
A FSM is one of the most basic mathematical models of computation.
The FSM consists of a set of states, transitions and events, as illustrated in Fig.~\ref{Arch.fig.FSM} showing an example of a FSM designed to carry out a grab-and-throw task. 
Note that the discussion here is valid for all  control architectures  based on FSMs, including Mealy~\cite{moore1956gedanken} and Moore~\cite{mealy1955method} machines.
\begin{figure}[h]
\centering
\includegraphics[width = \columnwidth]{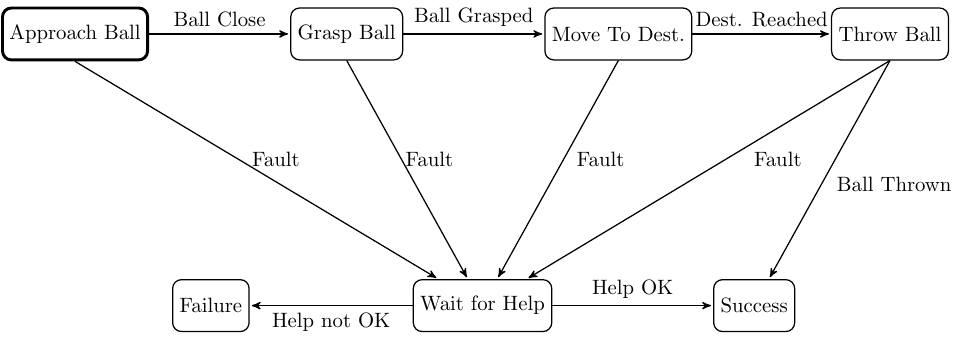} 
   \caption{Graphical representation of a FSM designed to carry out a simple grab-and-throw task. The initial state has a thicker border, and events names are given next to the corresponding transition arrows.}
    \label{Arch.fig.FSM}
\end{figure}


\subsection{Advantages and disadvantages}
FSMs are widely used due to their three main advantages:

\begin{itemize}
\item Very common structure, used in many different parts of computer science.
\item  Intuitive and easy to understand.
\item Easy to implement. 
\end{itemize}

However, the drawbacks of FSMs gives rise to problems when the system modelled grows in complexity and number of states, as described briefly in Section~\ref{sec:modularity}.
In particular we have the following drawbacks

\begin{itemize}
\item Reactivity/Modularity tradeoff.
A reactive system needs many transitions, and every transition corresponds to a Goto statement, see Section~\ref{sec:modularity}. In particular, the transitions give rise to the problems below:
\begin{itemize}
\item Maintainability: Adding or removing states requires the re-evaluation a potentially large number of transitions and internal states of the FSM. This makes FSMs highly susceptible to human design errors and 
impractical from an automated design perspective.
\item Scalability: FSMs with many states and many transitions between them are hard to modify, for both humans and computers.
\item Reusability: The transitions between states may depend on internal variables, making it unpractical to reuse the same sub-FSM  in multiple projects.
\end{itemize}
\end{itemize}

\section{Hierarchical Finite State Machines}
Hierarchical FSMs (HFSMs)\label{definition:HFSM}, also known as State Charts~\cite{Harel87statecharts}, where developed to alleviate some of the disadvantages of FSMs. In a HFSM, a state can in turn contain one or more substates. A state containing two or more states is called a \emph{superstate}. In a HFSM, a \emph{generalized transition} is a transition between superstates. Generalized transitions can reduce the number of transitions by connecting two superstates rather than connecting a larger number of substates individually. 
Each superstate has one substate identified as the starting state, executing whenever a transition to the superstate occurs.
Figure~\ref{Arch.fig.HFSM} shows an example of a HFSM for a computer game character.

\begin{figure}[h]
\centering
\includegraphics[width = \columnwidth]{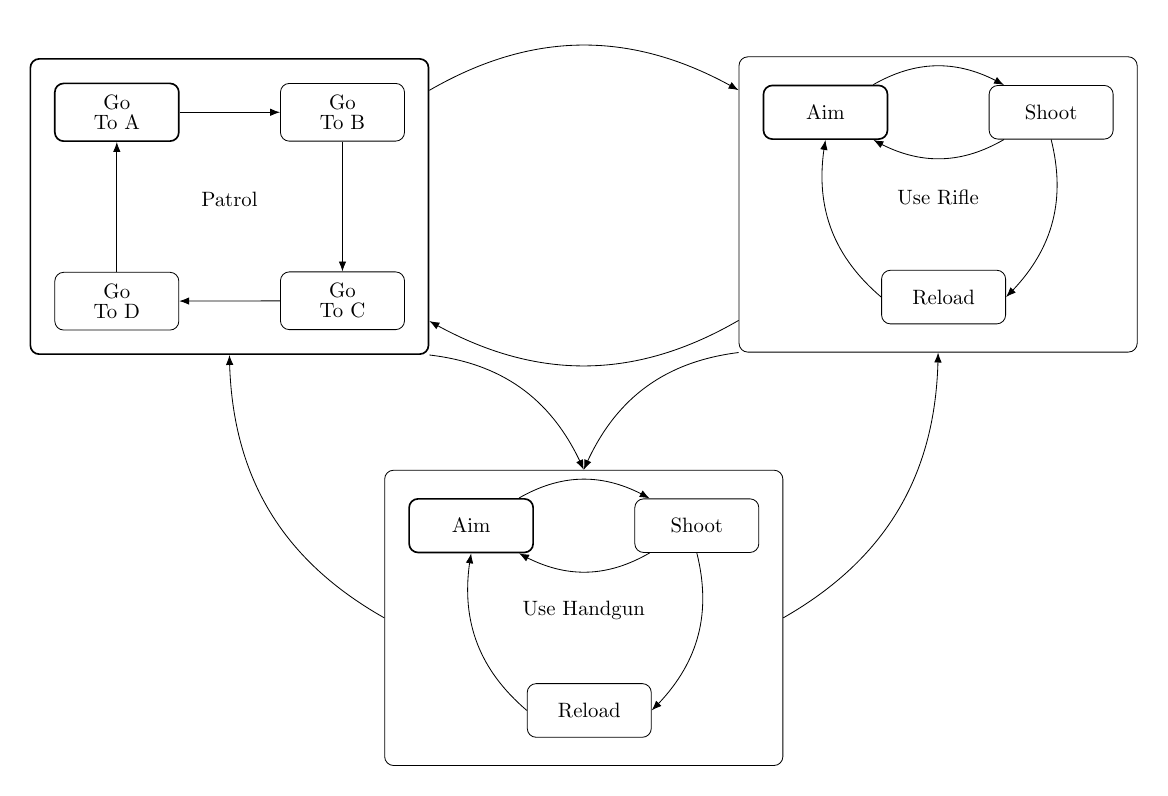} 
   \caption{Example of a HFSM controlling a NPC of a combat game. \emph{Patrol}, \emph{Use Rifle}, and \emph{Use Handgun} are superstates.}
    \label{Arch.fig.HFSM}
\end{figure}

\subsection{Advantages and disadvantages}

The main advantages of HFSMs are:

\begin{itemize}
\item Increased Modularity: it is possible to separate the tasks in subtasks. However these subtasks often still depend on each other through state-dependent transitions.
\item Behavior inheritance: The state nesting in HFSMs allows so-called \emph{behavior inheritance}. Behavior inheritance allows substates to inherit behaviors from the superstate; for example,  in the HFSM depicted in Figure~\ref{Arch.fig.HFSM}, while in the substates inside \emph{Use Handgun}, the character holds the weapon using one hand whereas while in the substates inside \emph{Use Rifle}, the character holds the weapon using two hands. Thus, there is no need for the sub states to specify this property, instead, it is inherited from the superstate.
\end{itemize}
The main disadvantages of HFSMs are:
\begin{itemize}
\item Maintainability: Adding or removing states is still hard. A long sequence of actions, with the possibility of going back in the sequence and re-execute a task that was undone by external agents (e.g. the environment), still requires a fully connected subgraph.
\item Manually created hierarchy: Although HFSMs were conceived as a hierarchical version of FSMs, the hierarchy has to be user defined and editing such a hierarchy  can be difficult. The hierarchy resolves some problems, but a reactive HFSM still results in some sub graphs being fully connected with many possible transitions, see Fig.~\ref{fig:hugeHFSM}.
\end{itemize}

From a theoretical standpoint, every execution described by a BT can be described by a FSM and vice-versa~\cite{ogren,Marzinotto14}. However, due to the number of transitions, using a FSM as a  control architecture  is unpractical for some applications as shown in Chapter~\ref{chap:bts}. Moreover, 
a potential problem is that
a FSM does not assume that the conditions triggering the outgoing transitions from the same state are mutually exclusive. When implemented, the conditions are checked regularly in discrete time, hence there exists a non-zero probability that two or more conditions hold simultaneously after one cycle. To solve this problem we need to redefine some transitions, as done in the FSM in Figure~\ref{Introduction.fig.FSMreactive}, making the propositions of the outgoing transitions mutually exclusive. A FSM of this format is impractical to design for both humans and computers. Manually adding and removing behaviors  is prone to errors. After adding a new state, each existing transition must be re-evaluated (possibly removed or replaced) and new transitions from/to the new state must be evaluated as well. A high number of transitions make any automated process to analyze or synthesize FSMs computationally expensive.  

HFSMs is the most similar  control architecture  to BTs in terms of purpose and use. To compare BTs with HFSMs we use the following complex example. Consider the HFSM shown in Figure~\ref{fig:hugeHFSM} describing the behavior of a humanoid robot. We can describe the same functionality using the BT shown in Figure~\ref{fig:hugeBT}. Note that we have used the standard notation~\cite{Harel87statecharts} of HFSMs to denote two activities running in parallel with a dashed line as separation. One important difference is that, in HFSMs, each layer in the hierarchy needs to be added explicitly, whereas in BTs every subtree can be seen as a module of its own, with the same interface as an atomic action.

In the HFSM shown in Figure~\ref{fig:hugeHFSM}, a proposition needs to be given for each transition, and to improve readability we have numbered these propositions from $C1$ to $C10$.
In the top layer of the HFSM we have the sub-HFSMs of \emph{Self Protection} and  \emph{Perform Activites}.
Inside the latter we have two parallel sub-HFSMs. One is handling the user interaction, while the larger one contains 
a complete directed graph handling the switching between the different activities. Finally, \emph{Play Ball Game} is yet another  sub-HFSM with the 
ball tracking running in parallel with another complete directed graph, handling the reactive switching between  \emph{Approach Ball}, \emph{Grasp Ball}, and \emph{Throw Ball}. 

It is clear from the two figures how modularity is handled by the HFSM. 
The explicitly defined sub-HFSM encapsulates  \emph{Self Protection},  \emph{Perform Activities} and  \emph{Play Ball Game}.
However, inside these sub-HFSMs, the transition structure is a complete directed graph, with $n(n-1)$ transitions that need to be maintained ($n$ being the number of nodes).

\begin{landscape}
 \begin{figure}
 \centering
  \includegraphics[width=\textwidth]{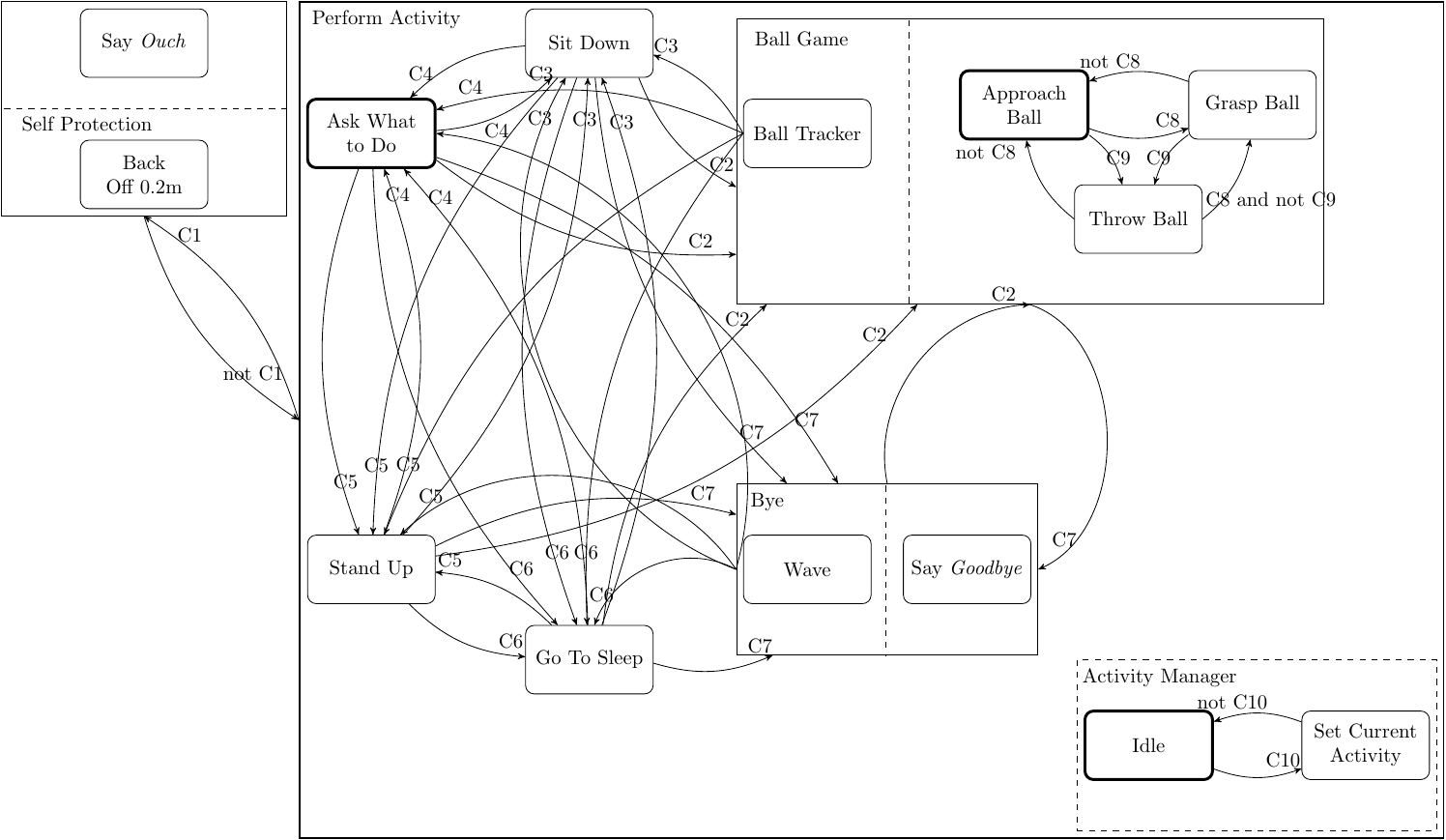}
  \caption{A HFSM description of the BT in Figure \ref{fig:hugeBT}. The transition conditions are shown at the end of each arrow to indicate the direction of the transition. Note how the complexity of the transitions \emph{within} each layer of the HFSM grows with the number of nodes. The conditions labels are: $C1= \mbox {Bumper Pressed}$, $C2= \mbox {Activity Ball Game}  $, $C3= \mbox {Activity Sit} $, $C4=\mbox {Not Know What to Do} $, $C5= \mbox {Activity Stand Up} $, $C6= \mbox {Activity Sleep} $, $C7=\mbox {Say Goodbye} $, $C8= \mbox {Ball Close} $, $C9= \mbox {Ball Grasped}$, $C10= \mbox {New User Suggestion}$. }
  \label{fig:hugeHFSM}
 \end{figure}
\end{landscape}

\begin{landscape}
 \begin{figure}\centering
  \includegraphics[width=\textwidth]{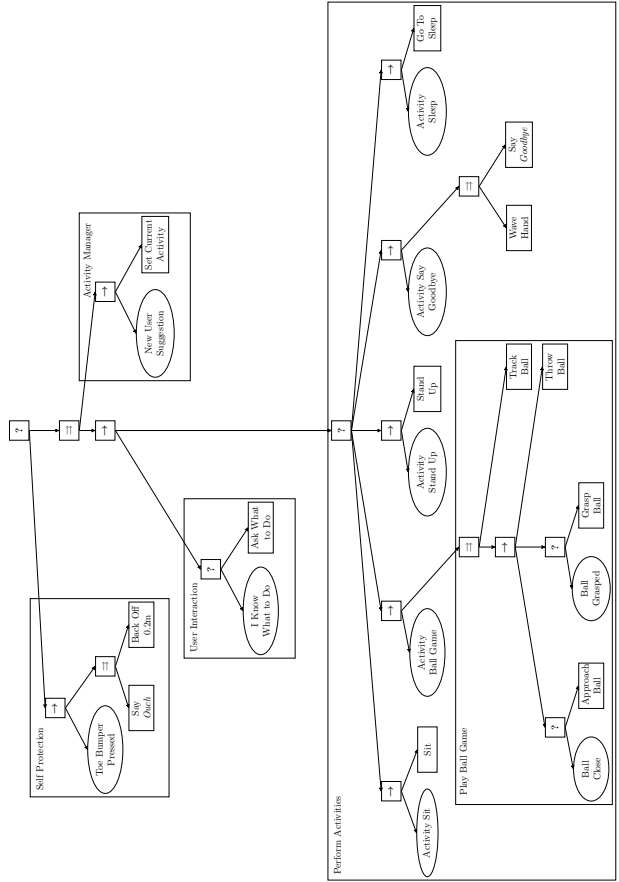}
  \caption{A BT that combines some capabilities  of a humanoid robot in an interactive and modular way. Note how atomic actions can easily be replaced by more complex sub-BTs.}
  \label{fig:hugeBT}
 \end{figure}
\end{landscape}


\newpage
\subsection{Creating a FSM that works like a BTs }
\label{btsvsothers:sec:FSMandBTs}

As described in Chapter~\ref{chap:bts}, each BT returns \emph{Success}, \emph{Running} or \emph{Failure}.
Imagine we have a state in a FSM that has 3 transitions, corresponding to these 3 return statements.
Adding a Tick source that collect the return transitions and transfer the execution back into the state, as depicted in Figure \ref{btsvsothers:FSMticksource}, we have a structure that resembles a BT.

 \begin{figure}[h]
\begin{center}
\includegraphics[width=5cm]{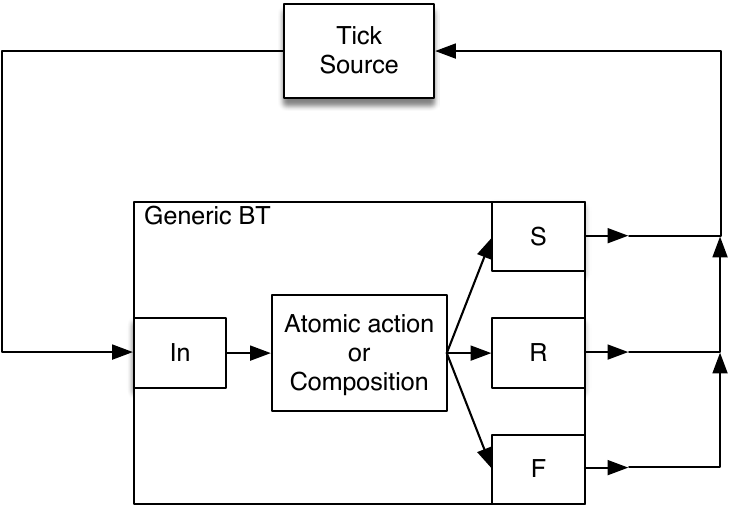}
\caption{An FSM behaving like a BT, made up of a single normal state, three out transitions Success (S), Running (R) and Failure (F), and a Tick source. }
\label{btsvsothers:FSMticksource}
\end{center}
\end{figure}

We can now compose such FSM states using both Fallback and Sequence constructs. The FSM corresponding to the Fallback example in Figure~\ref{btsvsothers:subsumption} would then look like the one shown in Figure \ref{btsvsothers:FSMfallback}.

\begin{figure}[h]
\begin{center}
\includegraphics[width=3.5cm]{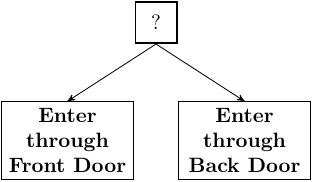}
\caption{
A Fallback is used to create an \emph{Enter Building} BT.
The back door option is only tried if the front door option fails. }
\label{btsvsothers:subsumption}
\end{center}
\end{figure}

 \begin{figure}[h]
\begin{center}
\centering
\includegraphics[width=0.8\columnwidth]{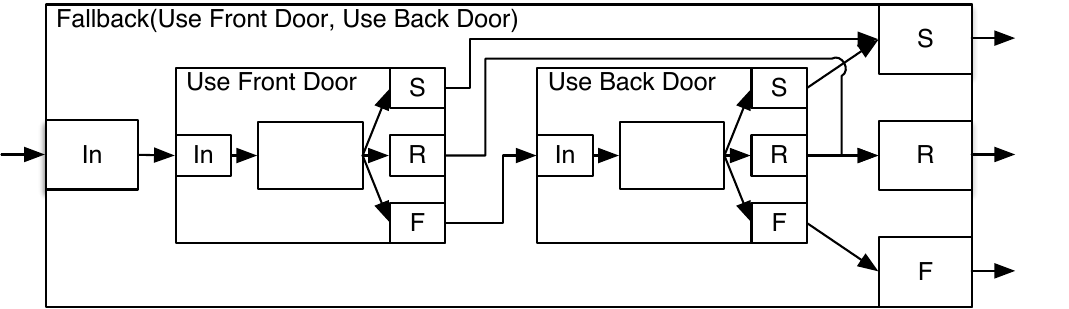}
\caption{A FSM corresponding to the Fallback BT in Figure~\ref{btsvsothers:subsumption}.
Note how the second state is only executed if the first fails.}
\label{btsvsothers:FSMfallback}
\end{center}
\end{figure}

Similarly, the FSM corresponding to the sequence example in Figure~\ref{btsvsothers:recipe} would then look like the one shown in Figure \ref{btsvsothers:FSMsequence}, and a two level BT, such as the one in  Figure~\ref{btsvsothers:comb1} would look like Figure~\ref{btsvsothers:FSMcombo}.

\begin{figure}[h]
\begin{center}
\includegraphics[width=3.5cm]{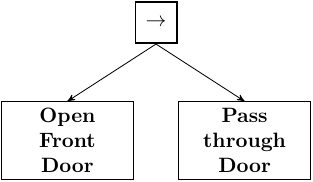}
\caption{
A Sequence is used to to create an \emph{Enter Through Front Door} BT.
Passing the door  is only tried if the opening action succeeds.}
\label{btsvsothers:recipe}
\end{center}
\end{figure}

 \begin{figure}[h]
\begin{center}
\centering
\includegraphics[width=0.8\columnwidth]{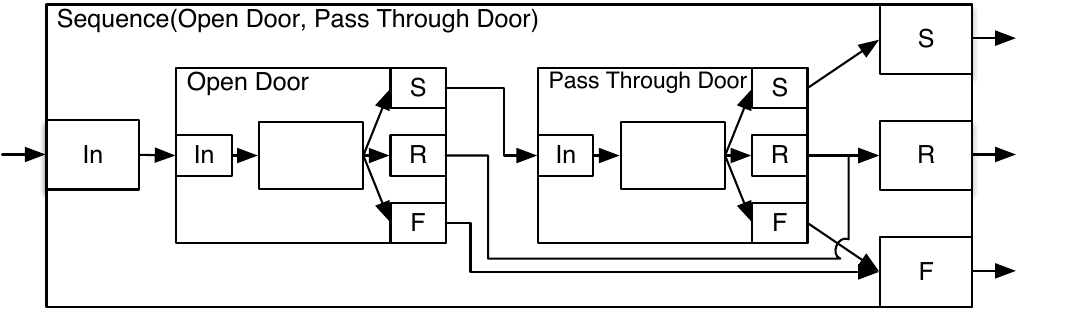}
\caption{An FSM corresponding to the Sequence BT in Figure~\ref{btsvsothers:recipe}.  
Note how the second state is only executed if the first succeeds.
}
\label{btsvsothers:FSMsequence}
\end{center}
\end{figure}

\begin{figure}[htbp]
\begin{center}
\includegraphics[width=6cm]{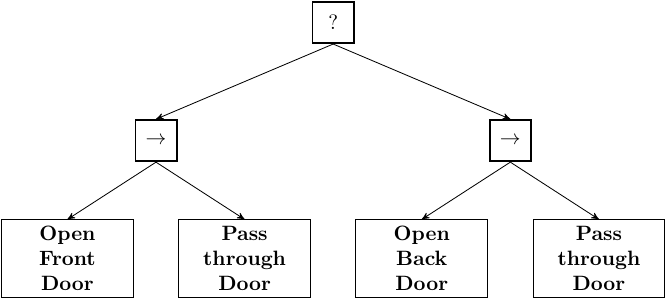}
\caption{The two BTs in Figures \ref{btsvsothers:subsumption}
 and \ref{btsvsothers:recipe} are combined to larger BT. If e.g. the robot opens the front door, but does not manage to pass through it, it will try the back door.}
\label{btsvsothers:comb1}
\end{center}
\end{figure}

 \begin{figure*}[htbp]
\begin{center}
\includegraphics[width=\columnwidth]{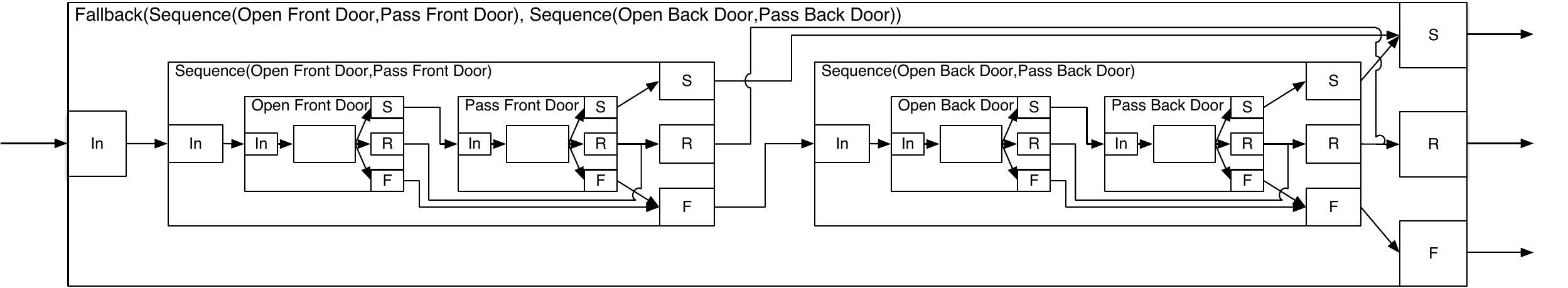}
\caption{An FSM corresponding to the BT in Figure~\ref{btsvsothers:comb1}. 
 }
\label{btsvsothers:FSMcombo}
\end{center}
\end{figure*}

\begin{figure}[htbp]
\begin{center}
\includegraphics[width=7cm]{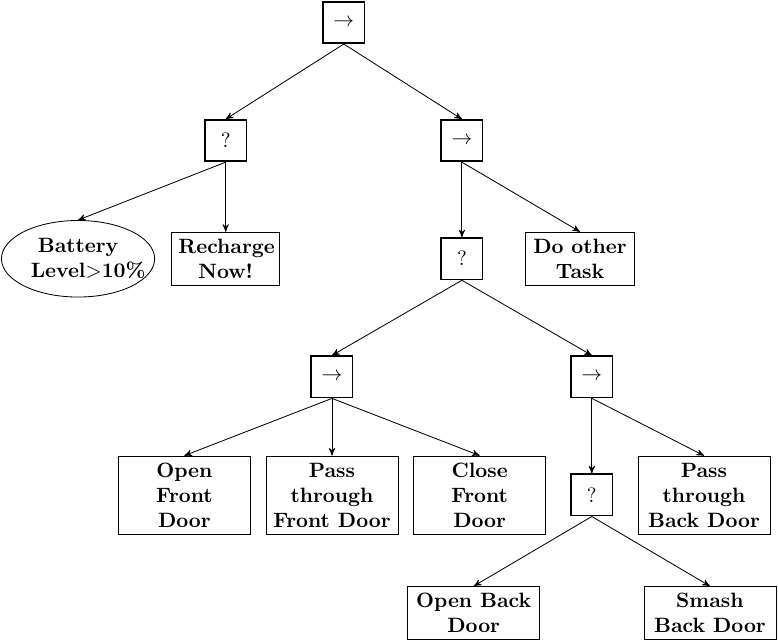}
\caption{Combining the BTs above and some additional Actions, we get a flexible BT for entering a building and performing some task.}
\label{btsvsothers:comb2}
\end{center}
\end{figure}

A few observations can be made from the above examples.
First, it is perfectly possible to design FSMs with a structure taken from BTs.
Second, considering that a BT with 2 levels corresponds to the FSM in Figure~\ref{btsvsothers:FSMcombo}, a BT with 5 levels, such as the one in Figure~\ref{btsvsothers:comb2} would correspond to a somewhat complex FSM. 

Third, and more importantly, the \emph{modularity} of the BT construct is illustrated 
in  Figures~\ref{btsvsothers:FSMticksource}-\ref{btsvsothers:FSMcombo}. 
Figure~\ref{btsvsothers:FSMcombo} might be complex, but that complexity is encapsulated in a box with a single in-transition and three out-transitions, just as the box in Figure~\ref{btsvsothers:FSMticksource}.

Fourth, as was mentioned in Section~\ref{sec:modularity}, the decision of what to do after a given sub-BT returns is always decided on the parent level of that BT. The sub-BT is ticked, and returns \emph{Success}, \emph{Running} or \emph{Failure} and the parent level decides whether to tick the next child, or return something to its own parent. Thus, the BT ticking and returning of a sub-BT is similar to a \emph{function call} in a piece of source code, just as described in Section~\ref{sec:modularity}. A function call in Java, C++, or Python moves execution to another piece of the source code, but then returns the execution to the line right below the function call. What to do next is decided by the piece of code that made the function call, not the function itself.
As discussed, this is quite different from standard FSMs where the decision of what to do next is decided by the  state being transitioned to, in a way that resembles the Goto statement.


\subsection{Creating a BT that works like a FSM }
\label{btsvsothers:sec:FSMandBTs_part2}

If you have a FSM design and want to convert it to a BT, the most straight forward way is to create a \emph{State Variable} available to all parts of the BT and then
list all the states of the FSM and their corresponding transitions and actions as shown in Figure~\ref{Arch.fig.FSMasBT}.

\begin{figure}[h!]
\centering
\includegraphics[width = 0.7\columnwidth]{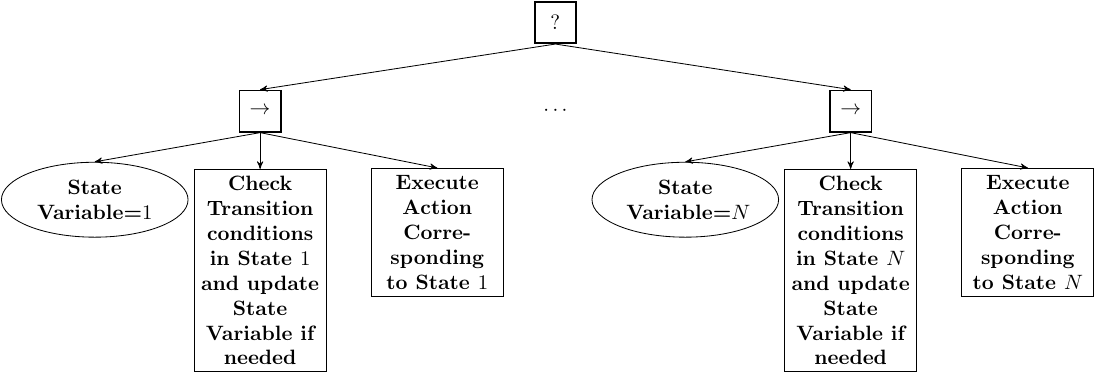} 
   \caption{Example of a straightforward translation of a FSM to a BT using a global \emph{State Variable}.}
    \label{Arch.fig.FSMasBT}
\end{figure}

\section{Subsumption Architecture}
\label{sec:SA}
The Subsumption Architecture~\cite{brooks1986robust} is heavily associated with the behavior-based robotic architecture, which was very popular in the late 1980s and 90s. 
This architecture has been widely influential in autonomous robotics and elsewhere in real-time AI and  has found a number of successful applications~\cite{brooks1990elephants}.
The basic idea of the Subsumpion Architecture is to have several controllers, each one implementing a task, running in parallel. Each controller is allowed to output both its actuation commands and a binary value that signifies if it wants to control the robot or not. The controllers are ordered according to some priority (usually user defined), and the highest priority controller, out of the ones that want to control the robot, is given access to the actuators. Thus, a controller with a higher priority is able to subsume a lower level one. Figure~\ref{Arch.fig.subsump} shows an example of a Subsumption Architecture.

\begin{figure}[h]
\centering
\includegraphics[width = 0.7\columnwidth]{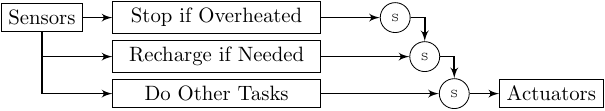} 
   \caption{Example of Subsumption Architecture composed by three controllers. The controller \emph{Stop if Overheated} subsumes the controller \emph{Recharge if Needed}, which subsumes the controller \emph{Do Other Tasks}.}
    \label{Arch.fig.subsump}
\end{figure}

\subsection{Advantages and disadvantages}

The Subsumption Architecture has many practical advantages, in particular:

\begin{itemize}
\item Easy development: The Subsumption Architecture is naturally well suited for iterative development and testing.
\item Modularity: The Subsumption Architecture connects limited, task-specific actions.
\item Hierarchy: The controllers are hierarchically ordered, which makes it possible to define high priority behaviors (e.g. safety guarantees) that override others.
\end{itemize}
The main disadvantages of the Subsumption Architecture are:
\begin{itemize}
\item Scalability:  Designing complex action selection through a distributed system of inhibition and suppression can be hard.
\item Maintainability: Due to the lack of structure, the consequences of adding or removing controllers can be hard to estimate. 
\end{itemize}

\subsection{How BTs Generalize the Subsumption Architecture}
There is a straightforward mapping from a Subsumption Architecture design to a BT using a Fallback node.
If each controller in the Subsumption Architecture is turned into a BT Action, returning running if the binary output indicates that it wants to run and Failure the rest of the time, 
a standard Fallback composition will create an equivalent BT.
As an example we see that the structure in Fig.~\ref{Arch.fig.subsump} is represented by the BT in Fig.~\ref{btsvsothers:fig:subsBT_v2}. A more formal argument using a state space representation of BTs will be given in Section \ref{btsvsothers:sec:analogySA}.


\begin{figure}[htbp]
\begin{center}
\includegraphics[width=6cm]{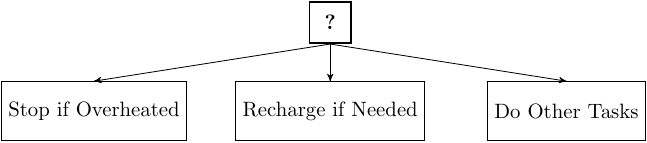}
\caption{A BT version of the subsumption example in Figure  \ref{Arch.fig.subsump}. }
\label{btsvsothers:fig:subsBT_v2}
\end{center}
\end{figure}

\section{Teleo-Reactive programs}
\label{sec:TR}
\label{architectures.tr}
Teleo-Reactive\label{definition:TR} (TR) programs were introduced by Nils Nilsson~\cite{nilsson1994teleo} at Stanford University in 1994 to allow engineers to define the behavior of a robotics system that had to achieve specific goals while being responsive to changes in the environment. A TR program is composed of a set of prioritized condition-action rules that directs the agent towards a goal state (hence the term \emph{teleo}) while monitoring the environmental changes (hence the term \emph{reactive}). In its simplest form, a TR program is described by a list of condition-action rules as the following:
\begin{eqnarray*}
c_1 &\rightarrow& a_1 \\
 c_2  &\rightarrow& a_2 \\
 &\cdots \\
  c_m  &\rightarrow& a_m
\end{eqnarray*}
where the $c_i$ are conditions and $a_i$ are actions. The condition-action rules list is scanned from the top until it finds a condition that holds, then the corresponding action is executed. In a TR program, actions are usually \emph{durative} rather than discrete. A durative action is one that continues indefinitely in time, e.g. the Action \emph{move forwards} is a durative action, whereas the action \emph{take one step} is discrete. In a TR program, a durative action is executed as long as its corresponding condition remains the one with the highest priority among the ones that hold. When the highest priority condition that holds changes, the action executed changes accordingly. Thus, the conditions must be evaluated continuously so that the action associated with the current highest priority condition that holds, is always the one
being executed. A running action terminates when its corresponding condition ceases to hold or when another condition with higher priority takes precedence. Figure~\ref{Arch.fig.tr} shows an example of a TR program for navigating in a obstacle free environment.

\begin{figure}[h]
\centering
\begin{eqnarray*}
 \mbox{Equal(pos,goal)} &\rightarrow& \mbox{Idle} \\
 \mbox{Heading Towards (goal)} &\rightarrow& \mbox{Go Forwards} \\
 \mbox{(else)} &\rightarrow& \mbox{Rotate} 
\end{eqnarray*}
   \caption{Example of teleoreactive program carrying out a navigation task. If the robot is in the goal position, the action performed is \emph{Idle} (no actions executed). Otherwise if it is heading towards the goal, the action performed is \emph{Go Forwards}. Otherwise, the robot performs the action \emph{Rotate}.}
    \label{Arch.fig.tr}
\end{figure}
 
TR programs have been extended in several directions, including integrating TR programs with automatic planning and machine learning~\cite{benson1993reacting,vargas2008solving}, removing redundant parts of a TR program~\cite{mousavi2003simplification}, and using TR programs to play
robot soccer~\cite{gubisch2008teleo}.

\subsection{Advantages and disadvantages}
The main advantages of a TR program are:
\begin{itemize}
\item Reactive execution: TR programs enable  reactive executions by continually monitoring the conditions and aborting actions when needed.
\item Intuitive structure: The list of condition-action rules is intuitive to design for small problems.
\end{itemize}
The main disadvantages of a TR program are:
\begin{itemize}
\item Maintainability: Due to its structure (a long list of rules), adding or removing condition-action rules is prone to cause errors when a TR program has to encode a complex system. In those cases, a TR program takes the shape of a long list.
\item Failure handling: To enable failure handling, a TR program needs to have a condition that checks if an action fails.
\end{itemize}

\subsection{How BTs Generalize Teleo-Reactive Programs}
\label{sec:BT_generalize_TR}

\begin{figure}[htbp]
\begin{center}
\includegraphics[width=0.8\columnwidth]{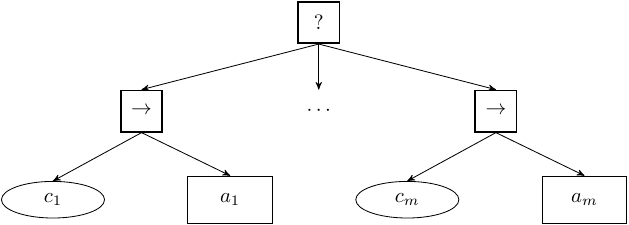}
\caption{The BT that is analogous to a given TR.}
\label{btsvsothers:tr2bt_v2}
\end{center}
\end{figure}

The core idea of continuously checking conditions and applying the corresponding rules can be captured using a Fallback node and pairs of conditions and actions. Thus, a general TR program can be represented in the BT of Fig.~\ref{btsvsothers:tr2bt_v2}.
A more formal argument using a state space representation of BTs will be given in Section~\ref{btsvsothers:sec:analogyTRs}.

\section{Decision Trees}
\label{sec:DT}
A Decision Tree is a directed tree that represents a list of nested if-then clauses used to derive decisions~\cite{sammut20027}. Leaf nodes describe decisions, conclusions, or actions to be carried out, whereas non-leaf nodes describe predicates to be evaluated. Figure~\ref{Arch.fig.dt} shows a Decision Tree where according to some conditions, a robot will decide what to do.

\begin{figure}[h]
\centering
\includegraphics[width = 0.8\columnwidth]{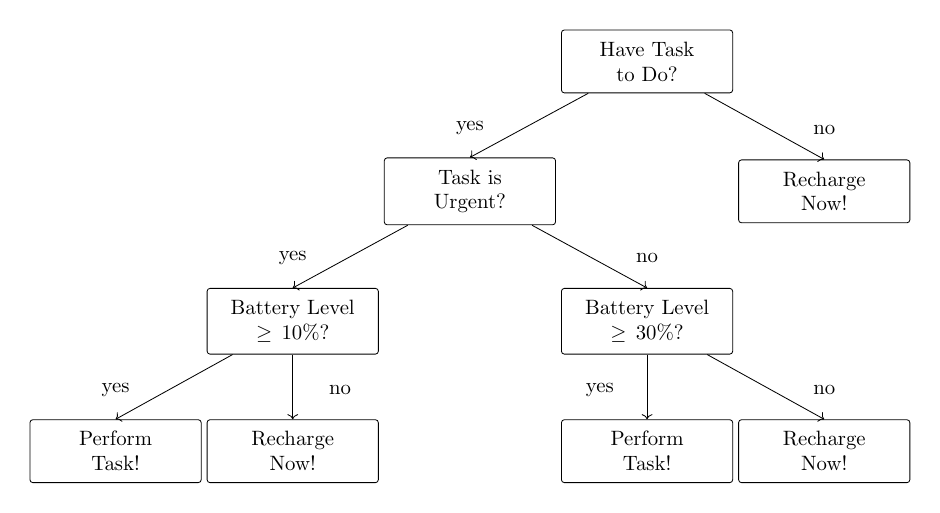} 
   \caption{Example of a Decision Tree executing a generic robotic task. The predicate are evaluated traversing the tree in a top-down fashion.}
    \label{Arch.fig.dt}
\end{figure}

\subsection{Advantages and disadvantages}

The main advantages of a Decision Tree are:
\begin{itemize}
\item Modularity: The Decision Tree structure is modular, in the sense that a subtree can be developed independently from the rest of the Decision Tree, and added where suitable.
\item Hierarchy: Decision Tree's structure is hierarchical, in the sense that predicates are evaluated in a top-down fashion.
\item Intuitive structure: It is straightforward to design and understand Decision Trees.
\end{itemize}
The main disadvantages of a Decision Tree are:
\begin{itemize}
\item No information flow out from the nodes, making failure handling very difficult

\end{itemize}

\subsection{How BTs Generalize Decision Trees}
\label{sec:BT_generalize_DT}

\begin{figure}[htbp]
\begin{center}
\includegraphics[width=\columnwidth]{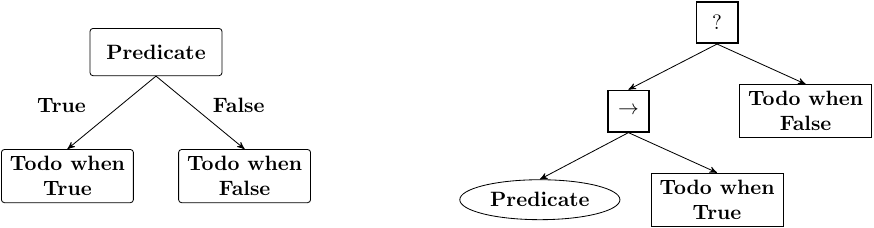}
\caption{The basic building blocks of Decision Trees are `If ... then ... else ...' statements (left), and those can be created in BTs as illustrated above (right). }
\label{btsvsothers:fig:decisionTreeEq_v2}
\end{center}
\end{figure}


\begin{figure}[htbp]
\begin{center}
\includegraphics[width=0.6\columnwidth]{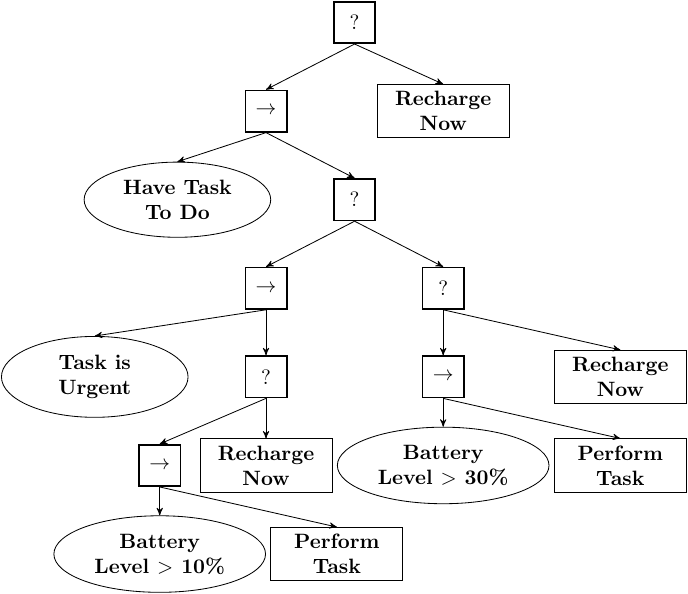}
\caption{A BT that is equivalent to the Decision Tree in Figure  \ref{Arch.fig.dt}. 
}
\label{btsvsothers:fig:decisionTreeBTeq_v2}
\end{center}
\end{figure}

A general Decision Tree can be converted into a BT using the mapping shown in Fig.~\ref{btsvsothers:fig:decisionTreeEq_v2}. By converting the predicate to a condition, letting the leaves be Action nodes always returning Running, we can map each decision node of the Decision Tree to a small BT. Applying the mapping to the Decision Tree of Fig.~\ref{Arch.fig.dt} we get the BT of Fig.~\ref{btsvsothers:fig:decisionTreeBTeq_v2}.
A more formal argument using a state space representation of BTs will be given in Section \ref{btsvsothers:sec:analogyDTs}.
Note that this structure requires actions always returning Running, reflecting the drawback of Decision Trees that no information flows out of the actions.

\section{Advantages and Disadvantages of Behavior Trees}
\label{btasca.sec.properties} 
Having looked at how BTs relate to a set of existing control architectures we will now take a step back and list a number of advantages and disadvantages 
of BTs.

\subsection{Advantages}

As described in Section~\ref{sec:modularity} many advantages stem from BTs being both modular and reactive.
Below we list a set of advantages of BTs.


\begin{description}
\item[Modular:] 
By modular, we mean the degree to which a system's components may be separated into building blocks, and recombined \cite{gershenson2003product}. 
A modular system can be designed, implemented, tested and reused one module at a time.
The benefits of modularity thus increases, the more complex a system is, by 
enabling a divide and conquer approach when designing, implementing and testing.

BTs are modular, since each subtree of a BT can be seen as a module in the above sense, with a standard interface given by the return statuses. 
Thus,  BTs are modular on all scales ranging from the topmost subtrees to all the leaves of the tree.


\item[Hierarchical organization:] 
If a control architecture contains several levels of decision making it is hierarchical.
The possibility of designing and analyzing structures on different hierarchical levels
 is important for both humans and computers, as it enables e.g., iterative refinement and extensions of a plan, see Section~\ref{design:sec:back_chaining}. BTs are hierarchical, since each level of a BT automatically defines a  level in the hierarchy.

 \item[Reusable code:] 
Having reusable code is very important in any large, complex, long-term project. The ability to reuse designs relies  on the ability to build larger things from smaller parts, and on the independence of the input and output of those parts from their use in the project. 
 To enable reuse of code, each module must interface the  control architecture  in a clear and well-defined fashion.

BTs enable  reusable code, since given the proper implementation, any subtree can be reused in multiple places of a BT.
Furthermore, when writing the code of a leaf node, the developer needs to just take care of returning the correct return status which is universally predefined as either \emph{Running}, \emph{Success}, or \emph{Failure}. Unlike FSMs and HFSMs, where the outgoing transitions require  knowledge about the next state, in BTs leaf nodes are developed disregarding which node is going to be executed next. Hence, the BT logic is independent from the leaf node executions and viceversa.

\item[Reactivity:] By reactive we mean the ability to quickly and efficiently react to changes.
 For unstructured environments, where outcomes of actions are not certain and the state of the world is constantly changed by external actors,
 plans that were created offline and then executed in an open loop fashion are often likely to fail.

BTs are reactive, since
the continual generation of ticks and their tree traversal result in a closed loop execution. Actions are executed and aborted according to the ticks' traversal, which depends on the leaf nodes' return statuses. Leaf nodes are tightly connected with the environment (e.g. condition nodes evaluate the overall system properties and Action nodes return  \emph{Failure}/\emph{Success} if the action failed/succeeded). 
Thus, BTs are highly responsive to changes in the environment.

\item[Human readable:] A readable structure is desirable for reducing the cost of development and debugging, especially when the task is human designed. The structure should remain readable even for large systems. Human readability requires a coherent and compact structure.

BTs are human readable due to their tree structure and modularity.

\item[Expressive:] A control architecture must be sufficiently expressive to encode a large variety of behaviors.

BTs are at least as expressive as FSMs, see Section~\ref{sec:FSM}, the Subsumption Architecture, see Section~\ref{sec:SA}, Teleo-Reactive programs, see Section~\ref{sec:TR}, and Decision Trees, see Section~\ref{sec:DT}.

\item[Suitable for analysis:] Safety critical robot applications often require an analysis of qualitative and quantitative system properties. These properties include: safety, in the sense of avoiding irreversible undesired behaviors;  robustness, in the sense of a large domain of operation; efficiency, in the sense of time to completion; reliability, in the sense of success probability; and composability, in the sense of analyzing whether properties are preserved over compositions of subtasks. 

BTs have tools available to evaluate such system properties, see Chapters \ref{ch:properties} and \ref{ch:stochastic}.

\item[Suitable for automatic synthesis:] In some problem instances, it is preferable that the action ordering of a task, or a policy, is automatically synthesized using task-planning or machine learning techniques. The  control architecture  can influence the efficiency of such synthesis techniques (e.g. a FSM with a large number of transitions can drastically deteriorate the speed of an algorithm that has to consider all the possible paths in the FSMs).

BTs are suitable for automatic synthesis in terms of both planning, see Section~\ref{design:sec:back_chaining} and in more detail Chapter~\ref{ch:planning} and learning, see Chapter~\ref{ch:learning}.

\end{description}

To  illustrate the advantages listed above, we consider the following simple example. 

\begin{example}
\label{Introduction.ex.motivating}
A robot is tasked to find a ball, pick it up, and place it into a bin. If the robot fails to complete the task, it should go to a safe position and wait for a human operator. After picking up the ball (Figure~\ref{IN.fig.FSMEx1}), the robot moves towards the bin (Figure~\ref{IN.fig.FSMEx2}). While moving towards the bin, an external entity takes the ball from the robot's gripper (Figure~\ref{IN.fig.FSMEx3}) and immediately throws it in front of the robot, where it can be seen (Figure~\ref{IN.fig.FSMEx4}). The robot aborts the execution of moving and it starts to approach the ball again.  
\end{example}

\begin{figure}[h!]
    \centering
    \begin{subfigure}[t]{0.45\columnwidth}
        \centering
\includegraphics[width = \columnwidth, trim={18cm 1cm 8cm 14cm},clip]{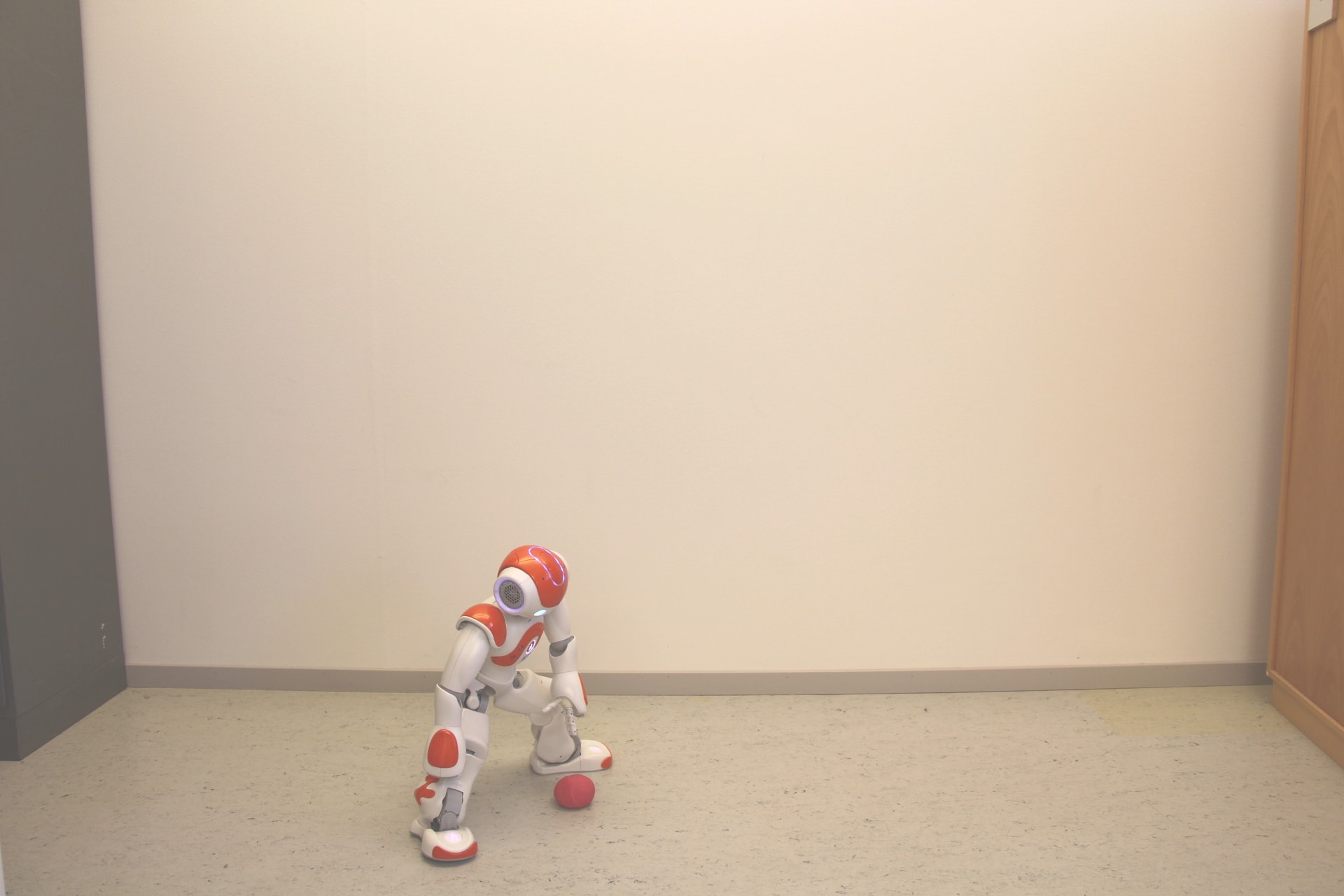} 
   \caption{The robot is picking up the ball.}
       \label{IN.fig.FSMEx1}
    \end{subfigure}
    ~  \begin{subfigure}[t]{0.45\columnwidth}
        \centering
\includegraphics[width = \columnwidth, trim={18cm 1cm 8cm 14cm},clip]{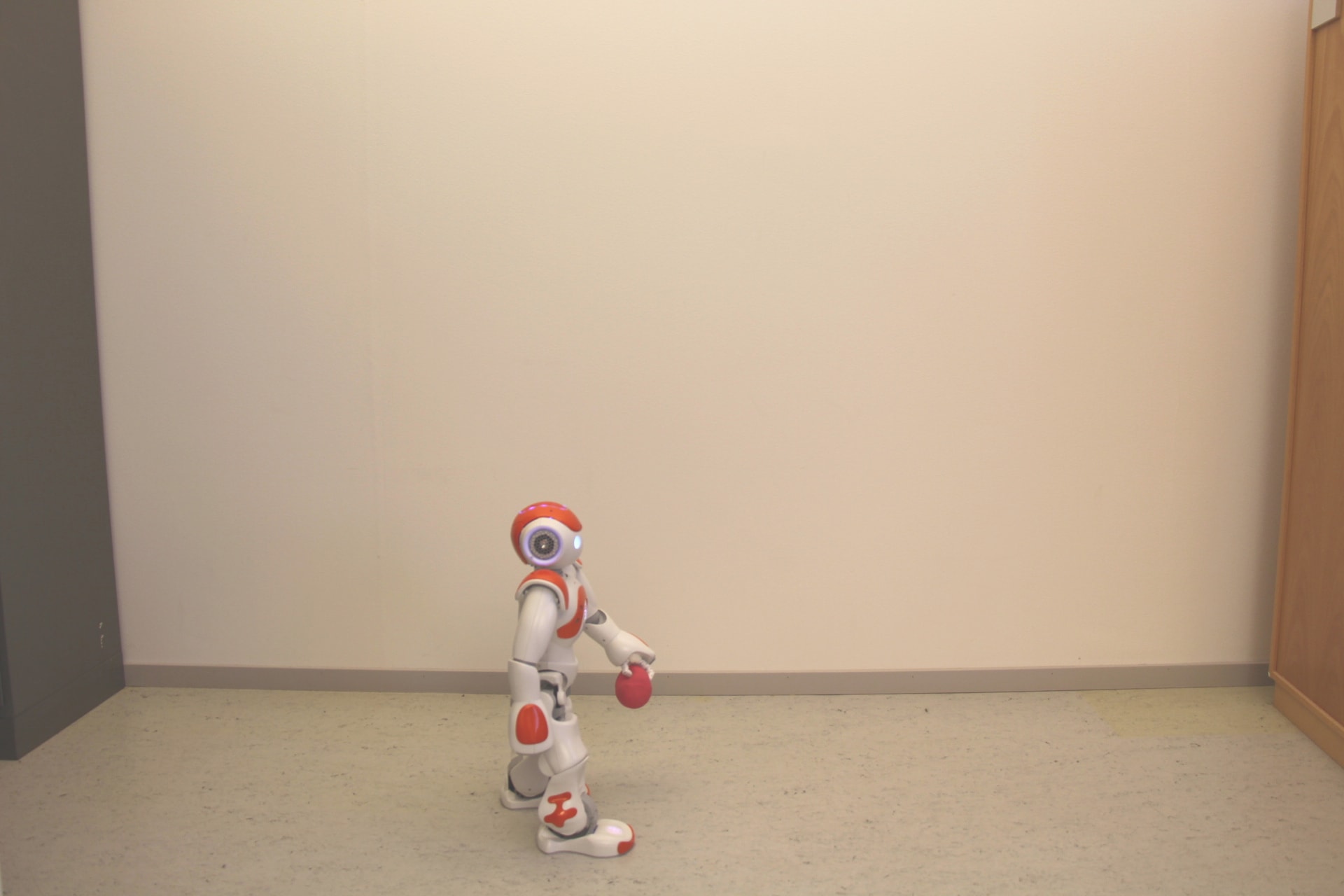} 
   \caption{The robot moves toward the bin (far away from the robot) with the ball in the hand.}
       \label{IN.fig.FSMEx2}
    \end{subfigure}

        \begin{subfigure}[b]{0.45\columnwidth}
        \centering
\includegraphics[width = \columnwidth, trim={19cm 2cm 7cm 13cm},clip]{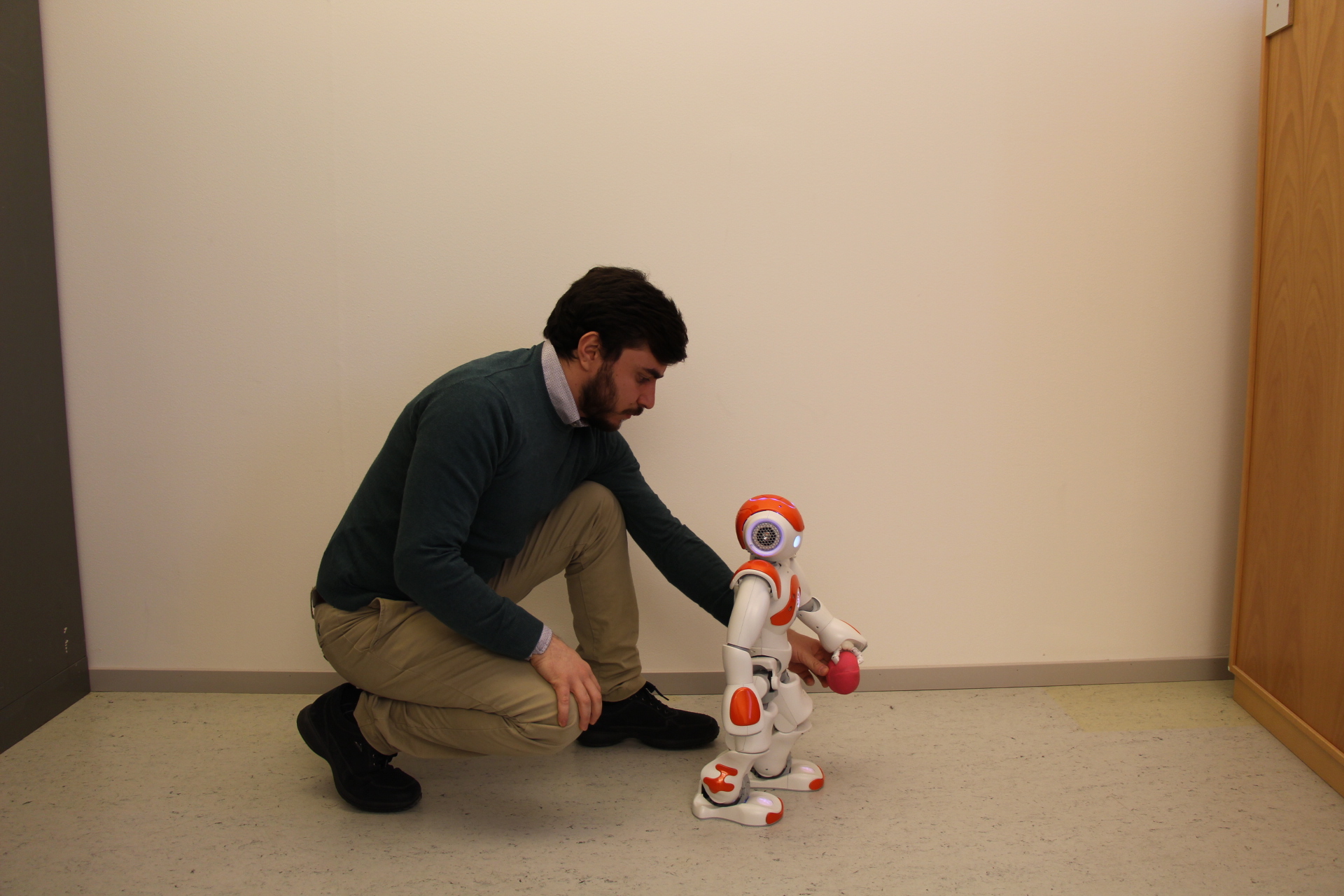} 
   \caption{An external entity (a human) takes the ball from the robot gripper.}
       \label{IN.fig.FSMEx3}
    \end{subfigure}
    ~  \begin{subfigure}[b]{0.45\columnwidth}
        \centering
\includegraphics[width = \columnwidth, trim={18cm 1cm 8cm 14cm},clip]{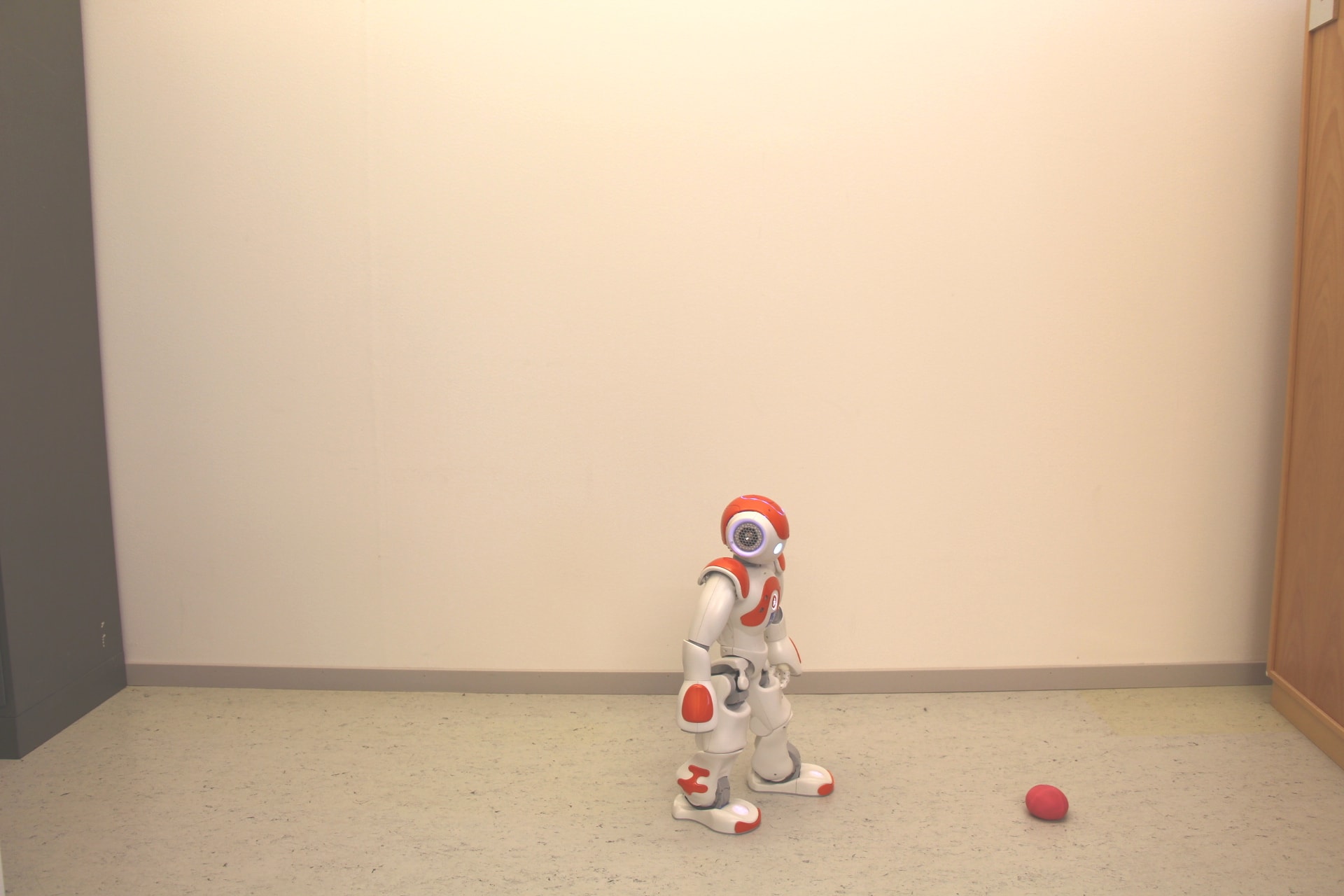} 
   \caption{The robot approaches the ball in the new location.}
       \label{IN.fig.FSMEx4}
    \end{subfigure}
    \caption{Execution stages of Example~\ref{Introduction.ex.motivating}.}
    \label{IN.fig.FSMEx}
\end{figure}

%
%
\begin{landscape}
 \begin{figure}
\includegraphics[width = \columnwidth]{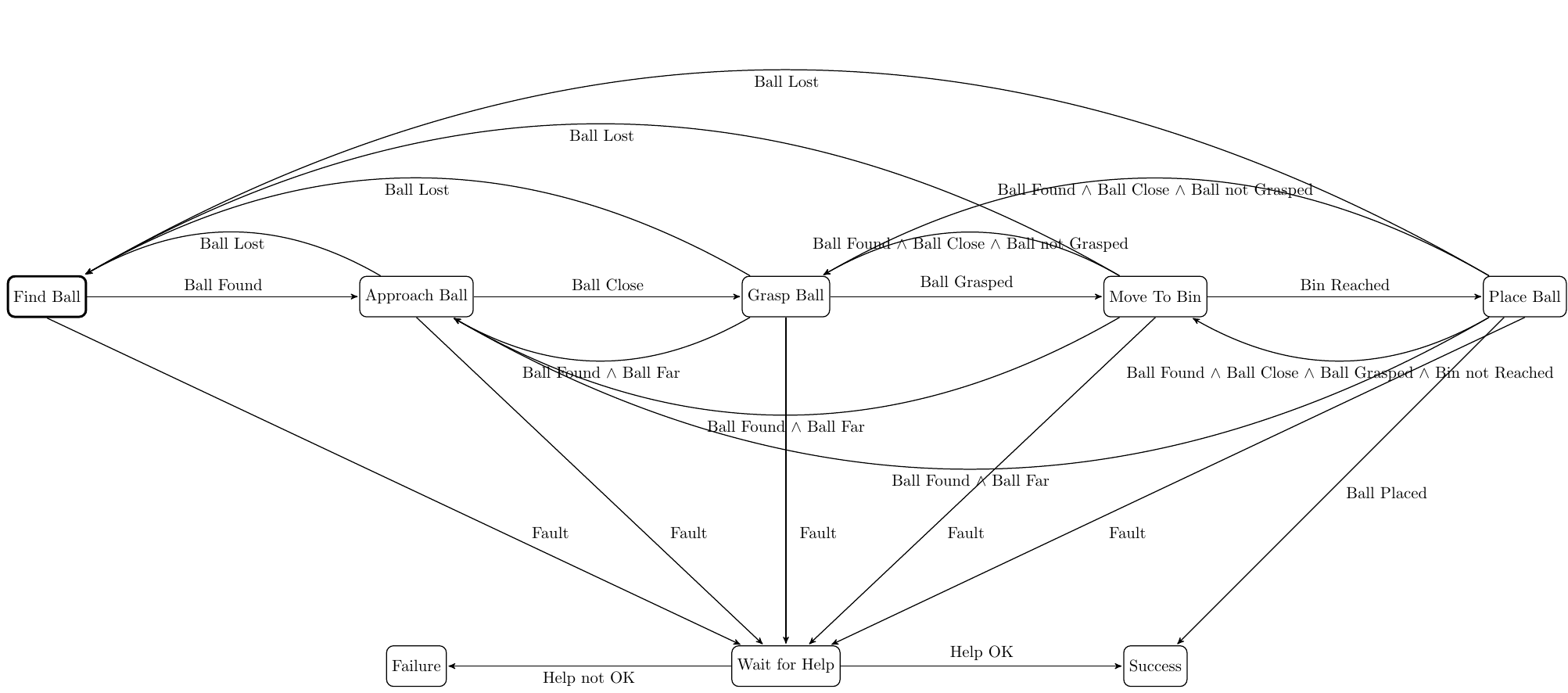} 
   \caption{FSM modeling the robot's behavior in Example~\ref{Introduction.ex.motivating}. The initial state has a thicker border.}
    \label{Introduction.fig.FSMreactive}
 \end{figure}
\end{landscape}
In this example, the robot does not simply execute a pick-and-place task. It \emph{continually} monitors the progress of the actions, stops whenever needed, skips planned actions, decides the actions to execute, and responds to exogenous events. In order to execute some actions, the robot might need to inject new actions into the plan (e.g. the robot might need to empty the bin before placing the ball). Hence the task requires a  control architecture  suitable for extensions. These extensions might be human made (e.g. the robot asks the operator to update the current action policy) requiring an architecture to be \emph{human readable}, or automated (e.g. using model-based reasoning) requiring an architecture to be \emph{suitable for automatic synthesis}. In either case, to be able to easily extend and modify the action policy, its representation must be \emph{modular}. In addition, new actions may subsume existing ones whenever needed (e.g. \emph{empty the bin if it is full} must be executed before \emph{place the ball)}. This requires a \emph{hierarchical} representation of the policy. Moreover there might be multiple different ways of carrying out a task (e.g. picking the ball using the left hand or the right hand). The robot must be able to decide which option is the best, requiring the architecture to be \emph{suitable for analysis}. Finally, once the policy is designed, it is desirable that it can be \emph{reused} in other contexts.  

Most  control architectures lack one or more of   the properties described above. Take as an example a FSM modeling the behavior of the robot in Example~\ref{Introduction.ex.motivating}, depicted in Figure~\ref{Introduction.fig.FSMreactive}. As can be seen, even for this simple example the FSM gets fairly complex with many transitions.

\subsection{Disadvantages}
In this section we describe some disadvantages of BTs. 

\begin{description}
\item [The BT engine can be complex to implement.] The implementation of the BT engine can get complicated using single threaded sequential programming. To guarantee the full functionality of BTs, the tick's generation and traversal should be executed in parallel with the action execution. However the
BT engine only needs to be implemented once, it can be reused, and several BT engines are available as off the shelf software libraries.\footnote{C++ library: \url{https://github.com/miccol/Behavior-Tree} \\ ROS library: \url{http://wiki.ros.org/behavior_tree} \\ 
python library: \url{https://github.com/futureneer/beetree}}
\item [Checking all the conditions can be expensive.] A BT needs to check several conditions to implement the closed-loop task execution. In some applications this checking is expensive or even infeasible. In those cases a closed-loop execution (using any architecture) presents more costs than advantages. However, it is still possible to design an open-loop task execution using BTs with memory nodes, see Section~\ref{bt:sec:mem}. 
\item [Sometimes a feed-forward execution is just fine.] In  applications where the robot operates in a very structured environment, predictable in space and time, BTs do not have any advantages over simpler architectures.
\item [BTs are different from FSMs.] BTs, despite being easy to understand, require a new mindset when designing a solution. The execution of BTs is not focused on states but on conditions and the switching is not event driven but tick driven. 
The ideas presented in this book, and in particular the design principles of Chapter~\ref{ch:design_principles}, are intended to support the design of efficient BTs.
\item [BT tools are less mature.] Although there is software for developing BTs, it is still far behind the amount and maturity of the software available for e.g. FSMs.
\end{description}

%% file: design/design.tex

\chapter{Design principles}
\label{ch:design_principles}
\graphicspath{{design/}}

BTs are fairly easy to understand and use, but to make full use of their potential it can be good to be aware of a set of design principles that can be used in different situations. In this chapter, we will describe these principles using a number of examples.
First, in Section~\ref{sec:explicit_conditions}, we will describe the benefit of using explicit success conditions in sequences,
then, in Section~\ref{sec:implicit_sequences}, we describe how the reactivity of a BT can be increased by creating implicit sequences, using Fallback nodes.
In Section~\ref{sec:dt}, we show how BTs can  be designed in a way that is similar to Decision Trees.
Then, in Section~\ref{sec:safety}, we show how safety can be improved using sequences.
Backchaining is an idea used in automated planning, and in Section~\ref{design:sec:back_chaining} we show how it can be used to create
deliberative, goal directed, BTs.
Memory nodes and granularity of BTs is discussed in 
 Sections~\ref{sec:design:memory} and \ref{design:sec:granularity}.
Finally, we show how easily all these principles can be combined at different levels of a BT in Section~\ref{design:sec:combinations}.

\section{Improving Readability using Explicit Success Conditions}
\label{sec:explicit_conditions}

One advantage of BTs is that the switching structure is clearly shown in the graphical representation of the tree. However, one thing that is not shown is the details regarding when the individual actions return  Success and Failure.

\begin{figure}[h]
\centering
\includegraphics[width=0.5\columnwidth]{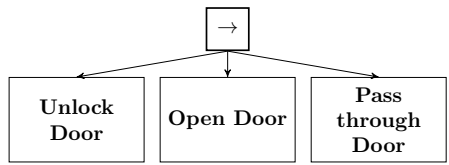}
\caption{Simple Sequence}
\label{design:fig:sequence}
\end{figure}
 Consider the sequence in Figure~\ref{design:fig:sequence}. One can assume that Unlock Door returns Success when it has unlocked the door, but what if it is called when the door is already unlocked? Depending on the implementation it might either return Success immediately, or actually try to unlock the door again, with the possibility of returning Failure if the key cannot be turned further. A similar uncertainty holds regarding the implementation of Open Door (what if the door is already open?) and Pass through Door.
 To address this problem, and remove uncertainties regarding the implementation, explicit Success conditions can be included in the BT.

\begin{figure}[h]
\centering
\includegraphics[width=0.9\columnwidth]{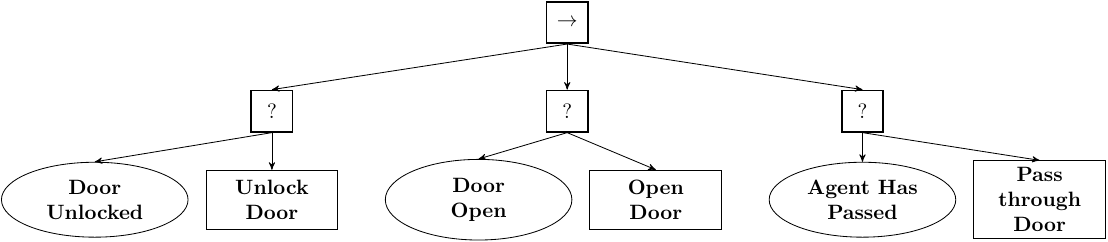}
\caption{Sequence with explicit success conditions. Note how each action is paired with a condition through a Fallback node, making the success condition of the pair explicit.}
\label{design:fig:explicit_sequence}
\end{figure}

In Figure~\ref{design:fig:explicit_sequence}, the BT from Figure~\ref{design:fig:sequence} has been extended to include explicit success conditions. These conditions are added in a pair with the corresponding action using a Fallback node. Now, if the door is already unlocked and open, the two first conditions of Figure~\ref{design:fig:explicit_sequence}
will return Success, the third will return Failure, and the agent will proceed to execute the action Pass through Door.

\section{Improving Reactivity using Implicit Sequences}
\label{sec:implicit_sequences}

It turns out that we can  improve the reactivity of the BT in Figure~\ref{design:fig:explicit_sequence} even further,
using the fact that BTs generalize the Teleo-Reactive approach, see Section~\ref{sec:BT_generalize_TR}.
Consider the case when the agent has already passed the door, but the door is closed behind it. The BT in Figure~\ref{design:fig:explicit_sequence}
would then proceed to unlock the door, open it, and then notice that it had already passed it and return Success.

\begin{figure}[h]
\centering
\includegraphics[width=0.9\columnwidth]{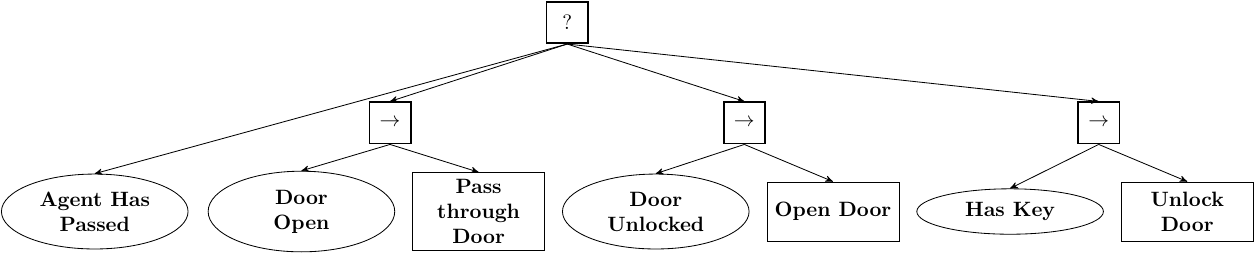}
\caption{An Implicit Sequence is constructed using a Fallback node, reversing the order of the actions and pairing them with appropriate preconditions.}
\label{design:fig:implicit_sequence}
\end{figure}

The key observation needed to improve reactivity is to realize that the goal is to get through the door, and that the other actions are just means to get to
that goal. In the BT in Figure~\ref{design:fig:implicit_sequence} we have reversed the order of the actions in order the check the goal state first.
We then changed fallbacks to sequences and vice versa, and finally changed the conditions. Now, instead of checking outcomes, or success conditions as we did in 
Figure~\ref{design:fig:explicit_sequence}, we check preconditions, conditions needed to execute the corresponding actions, in Figure~\ref{design:fig:implicit_sequence}.
First the BT checks if the agent has passed the door, if so it returns Success. If not, it proceeds to check if the door is open, and if so passes through it. If neither of the previous conditions are satisfied, it checks if the door is unlocked, and if so starts to open it. As a final check, if nothing else returns Success, it checks if it has the key to the door. If it does, it tries to open it, if not it returns Failure.

The use of implicit sequences is particularly important in cases where the agent needs to undo some of its own actions, such as closing a door after passing it.
A systematic way of creating implicit sequences is to use back chaining, as described in Section~\ref{design:sec:back_chaining}.

\section{Handling Different Cases using a Decision Tree Structure}
\label{sec:dt}
Sometimes, a reactive switching policy can be easily described in terms of a set of cases, much like a Decision Tree.
Then, the fact that BTs generalize Decision Trees can be exploited, see Section~\ref{sec:BT_generalize_DT}.

A simple Pac-Man example can be found in Figure~\ref{design:fig:dt}. The cases are separated by the two conditions \emph{Ghost Close} and \emph{Ghost Scared}.
If no ghost is close, Pac-Man continues to eat pills. If a ghost is close, the BT checks the second condition, \emph{Ghost Scared}, which turns true if Pac-Man eats a Power Pill. If the ghost is scared, Pac-Man chases it, if not, Pac-Man avoids the Ghost.

\begin{figure}[h]
\centering
\includegraphics[width=0.4\columnwidth]{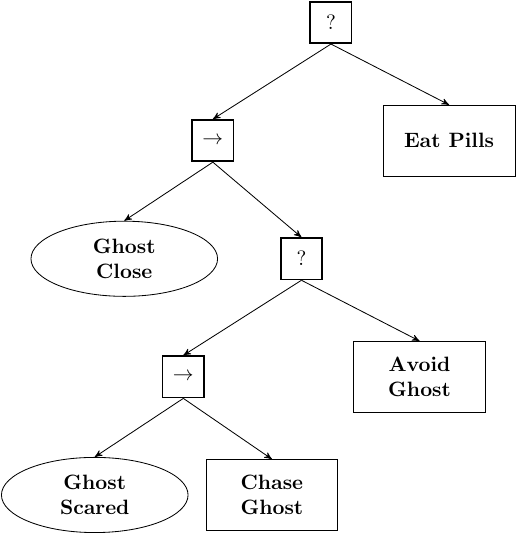}
\caption{Simple Pac-Man example using a Decision Tree structure.}
\label{design:fig:dt}
\end{figure}

\section{Improving Safety using Sequences}
\label{sec:safety}

In some agents, in particular robots capable of performing irreversible actions such as falling down stairs or damaging equipment, it is very important to be able to guarantee that some situations will never occur. These unwanted situations might be as simple as failing to reach the recharging station before running out of battery, or as serious as falling down a staircase and hurting someone.

\begin{figure}[h]
\centering
\includegraphics[height=3cm]{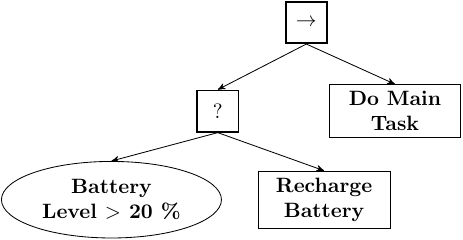}
\caption{A BT that is guaranteed not to run out of batteries, as long as Main Task keeps the robot close enough to the recharging station so that 20\% of battery will be enough to travel back.}
\label{design:fig:safety}
\end{figure}

A Sequence node can be used to guarantee safety, as shown in Figure~\ref{design:fig:safety}.
Looking closer at the BT in Figure~\ref{design:fig:safety} we see that it will probably lead to an unwanted chattering behavior. It will recharge until it reaches just over 20\%
and then start doing Main Task, but the stop as soon as the battery is back at 20\%,
and possibly end up chattering i.e. quickly switching between the two tasks.
The solution is to make sure that once recharging, the robot waits until the battery is back at 100\%. This can be achieved by the BT in Fig~\ref{design:fig:safety_hysteresis}.

\begin{figure}[h]
\centering
\includegraphics[height=4cm]{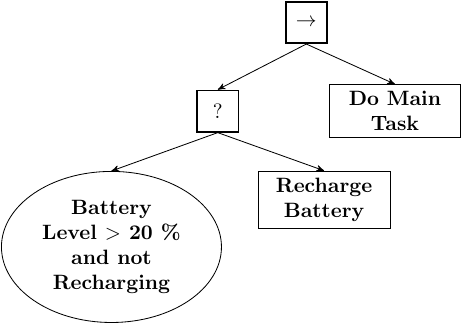}
\caption{By changing the condition in Fig.~\ref{design:fig:safety} the robot now keeps recharging  until the Battery level reaches 100\%. }
\label{design:fig:safety_hysteresis}
\end{figure}

\section{Creating Deliberative BTs using Backchaining}
\label{design:sec:back_chaining}

BTs can also be used to create deliberative agents, where the actions are carried out in order to reach a specific goal. 
We will use an example to see how this is done.
Imagine we want the agent to end up inside a house. To make that goal explicit, we create the trivial BT in Figure~\ref{design:fig:back_chaining_1},
with just a single condition checking if the goal is achieved or not.

\begin{figure}[h]
\centering
\includegraphics[width=0.2\columnwidth]{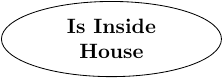}
\caption{A BT composed of a single condition checking if the goal is achieved.}
\label{design:fig:back_chaining_1}
\end{figure}

Now imagine we have a set of small BTs such as the ones in 
Figures~\ref{design:fig:back_chaining_2} and~\ref{design:fig:back_chaining_3}, each on the format of the general Postcondition-Precondition-Action (PPA) BT in 
Figure~\ref{design:fig:back_chaining_general}.

\begin{figure}[h]
\centering
\includegraphics[height=3cm]{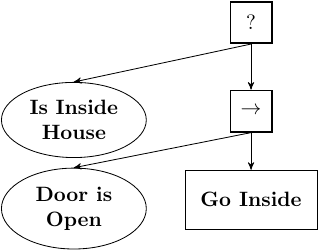}
\caption{PPA for achieving the postcondition Is Inside House. If the postcondition is not satisfied already, the BT checks the precondition Door is Open, if so it executes the action Go Inside.}
\label{design:fig:back_chaining_2}
\end{figure}

If we have such a set, be can work our way backwards from the goal (backchaining) by replacing preconditions with PPAs having the corresponding postcondition. Thus replacing the single condition in Figure~\ref{design:fig:back_chaining_1} with the PPA of Figure~\ref{design:fig:back_chaining_2} we get Figure~\ref{design:fig:back_chaining_2} again, since we started with a single condition.
More interestingly, if we replace the precondition Door is Open in Figure~\ref{design:fig:back_chaining_2} with the PPA of Figure~\ref{design:fig:back_chaining_3} we get the BT of Figure~\ref{design:fig:back_chaining_4}

\begin{figure}[h]
\centering
\includegraphics[width=0.9\columnwidth]{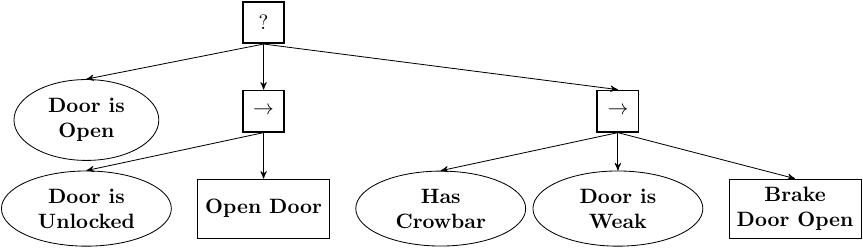}

\caption{PPA for achieving the postcondition Door is Open. If the postcondition is not satisfied, the BT checks the first precondition Door is Unlocked, if so it executes the action Open Door, if not it checks the second set of preconditions, starting with Has Crowbar, if so it checks Door is Weak, if both are satisfied it executes Brake Door Open.}

\label{design:fig:back_chaining_3}
\end{figure}

\begin{algorithm2e}[h]
\KwData{Set of Goal Conditions $C_i$, and a set of PPAs}
\KwResult{A reactive BT working to achieve the $C_i$s }
  Replace all $C_i$ with PPAs having $C_i$ as postcondition\;
 \While{the BT returns Failure when ticked}{
      replace one of the preconditions returning Failure (inside a PPA)
       with another complete PPA having the corresponding condition as postcondition, and therefore including at leaves one action to achieve the failing condition \;
    }    \caption{Pseudocode of Backchaining Algorithm}
  \label{design:alg:backchaining}
\end{algorithm2e}

\begin{figure}[h]
\centering
\includegraphics[width=0.9\columnwidth]{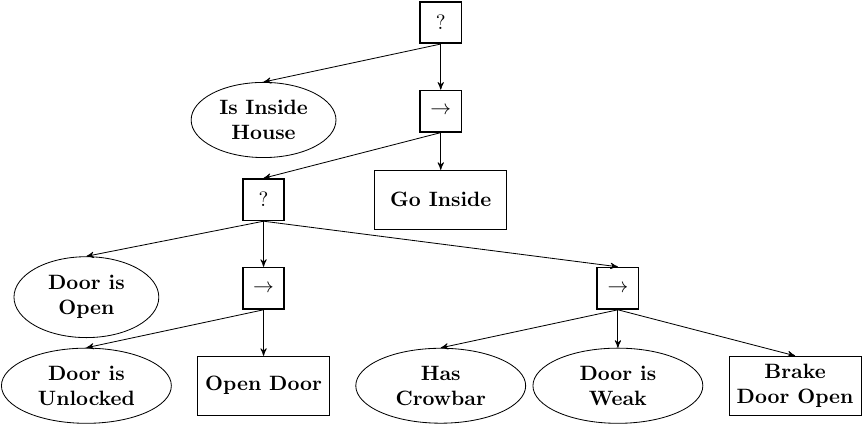}
\caption{The result of replacing \emph{Door is Open} in Figure~\ref{design:fig:back_chaining_2} with the PPA of Figure~\ref{design:fig:back_chaining_3}.}
\label{design:fig:back_chaining_4}
\end{figure}

Thus we can iteratively build a deliberative BT by applying   Algorithm~\ref{design:alg:backchaining}.
Looking at the BT in Figure~\ref{design:fig:back_chaining_4} we note that it first checks if the agent \emph{Is Inside House}, if so it returns Success. If not it checks if \emph{Door is Open}, and if it is, it proceeds to \emph{Go Inside}. If not it checks if \emph{Door is Unlocked} and correspondingly executes \emph{Open Door}. Else it checks if \emph{Door is Weak}, and it \emph{Has Crowbar} and proceeds to \emph{Brake Door Open} if that is the case. Else it returns Failure. If an action is executed it might either succeed, which will result in a new condition being satisfied and another action being executed until the task is finished, or it might fail. If \emph{Go Inside} fails, the whole BT returns Failure, but if \emph{Open Door} fails, the conditions \emph{Door is Weak} and \emph{Has Crowbar} are checked.

\begin{figure}[h]
\centering
\includegraphics[width=\columnwidth]{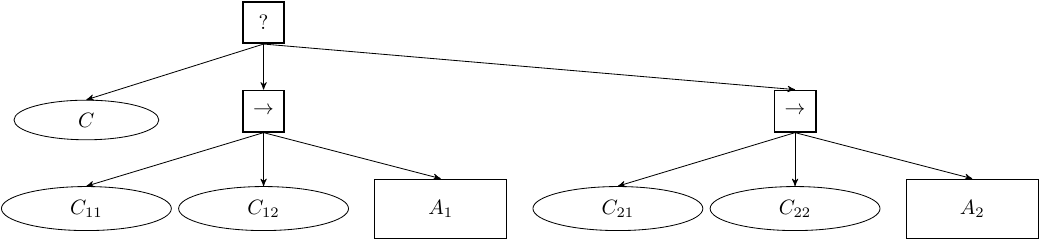}
\caption{General format of a PPA BT. The Postcondition $C$ can be achieved by either one of actions $A_1$ or $A_2$, which have Preconditions $C_{1i}$ and $C_{2i}$ respectively.}
\label{design:fig:back_chaining_general}
\end{figure}

In general, we let the PPA have the form of Figure~\ref{design:fig:back_chaining_general}, with one postcondition $C$ that can be achieved by either one of a set of actions $A_i$, each of these action are combined in a sequence with its corresponding list of preconditions $C_{ij}$, and these action precondition sequences are fallbacks for achieving the same objective. We see that from an efficiency point of view it makes sense to put actions that are most likely to succeed first (to avoid unnecessary failures) and check preconditions that are most likely to fail first (to quickly move on to the next fallback option).

\section{Creating Un-Reactive BTs using Memory Nodes}
\label{sec:design:memory}
As mentioned in Section~\ref{bt:sec:mem}, sometimes a child, once executed, does not need to be re-executed for the whole execution of a task. 
Control flow nodes with memory	are used to simplify the design of a BT avoiding the unwanted re-execution of some nodes. The use of nodes with memory is advised exclusively for those cases where there is no unexpected event that will undo the execution of the subtree in a composition with memory, as in the 
example below.

Consider the behavior of an industrial manipulator in a production line that has to \emph{pick}, \emph{move}, and \emph{place} objects. The robot's actions are carried out in a fixed workspace, with high precision. Human operators
make sure that nothing on the line changes. If they need a change in the line,
the software is manually updated accordingly. In this example the robot operates in a structured environment that is fully predictable in space and time. In this case we can disregard any unexpected change enabling us to describe the desired behavior
by a Sequence with memory of pick and place as in Figure~\ref{design:fig:mem}. In this scenario, after picking we can be sure that the object does not slips out of the robot's grippers. Hence while the robot is moving the object, the BT does not need to check if the object is still picked.  
\label{design:ex:mem}
\begin{figure}[h]
\centering
  \includegraphics[width=0.6\textwidth]{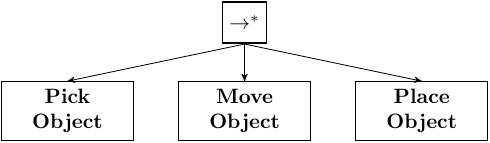}
\caption{Example of a Un-Reactive Sequence composition of the behaviors pick, move, and place.}
\label{design:fig:mem}
\end{figure}

\section{Choosing the Proper Granularity of a BT}
\label{design:sec:granularity}

In any modular design, 
 we need to decide the granularity of the modules. In a BT framework, this is translated into the choice of what to represent as a leaf node (single action or condition) and what to represent as a BT. 
The following two cases can be considered.

\begin{itemize}
\item 
It makes sense to encode the behavior in a single leaf when the potential
 subparts of the behavior are always used and executed in this particular combination.
 
\item
It makes sense to encode a behavior as a sub-BT, braking it up into conditions, actions and flow control nodes,
when the subparts are likely to be usable in other combinations in other parts of the BT,
and when the reactivity of BTs can be used to re-execute parts of the behavior when needed.
 
\end{itemize}

\label{design:ex:reused}
Consider the BT in Figure~\ref{design:fig:reused} describing the behavior of a humanoid robot. The actions \emph{sit} and \emph{stand} 
cannot be divided into meaningful sub-behaviors.

\begin{figure}[h]
\centering
  \includegraphics[width=\textwidth]{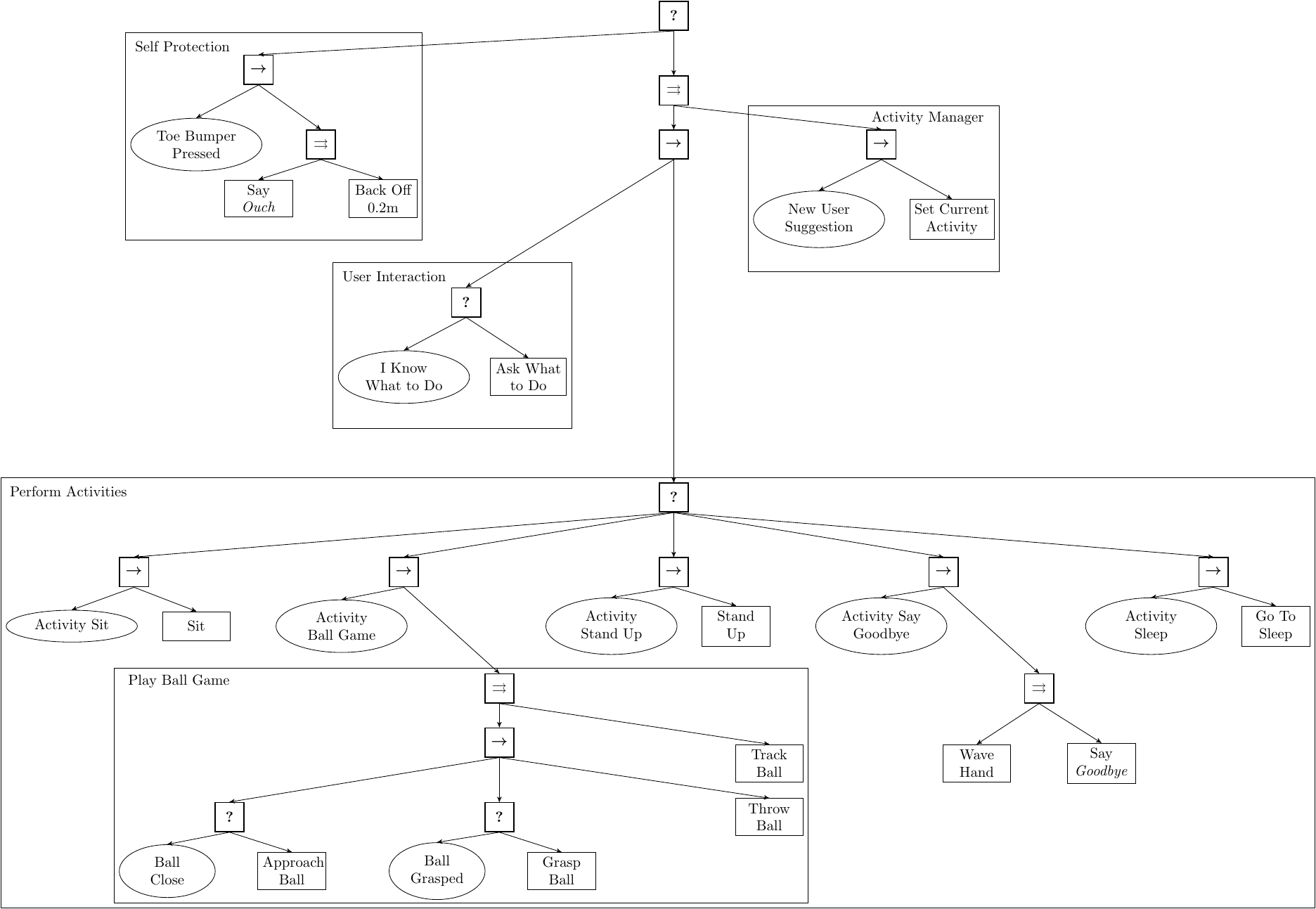}
\caption{Robot activity manager}

\label{design:fig:reused}
\end{figure}

\label{design:ex:closed}

Consider an assembly task for an industrial robot that coexists in a semi-structured environment with human workers. The tasks to perform are \emph{pick object}, \emph{assemble object}, and \emph{place object}. A closed-loop execution of this task can be represented with the BT in Figure~\ref{design:fig:closed}. Note that the BT can reactively handle  unexpected changes, possibly produced by the human worker in the line, such as when the worker picks up an object that the robot is trying to reach, or the object slipping out of the robot gripper while the robot is moving it, etc. If we had instead chosen to aggregate the actions \emph{pick object}, \emph{assemble object}, and \emph{place object} into a single action 
we would lose reactiveness when, for example, the robot has to re-pick an assembled object that slipped out from the robot's grippers. With a single action the robot would try to re-assemble an already assembled object. 
\begin{figure}[h]
\centering
  \includegraphics[width=\textwidth]{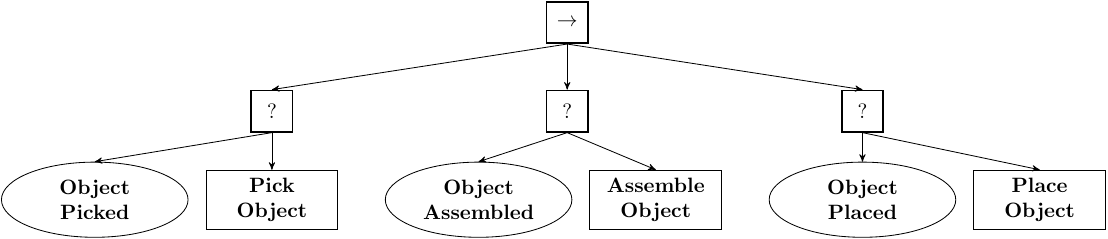}
\caption{Closed loop example}
\label{design:fig:closed}
\end{figure}
%

The advice above should give the designer an idea on how to reach a balanced BT that is neither too \emph{fine grained} nor too \emph{compact}. A fine grained BT might be unreasonably complex. While a compact BT may risk being not sufficiently reactive, by executing too many operations in a feed-forward fashion, losing one main advantage of BTs.

\section{Putting it all together}
\label{design:sec:combinations}

\begin{figure}[h]
\centering
  \includegraphics[width=10cm]{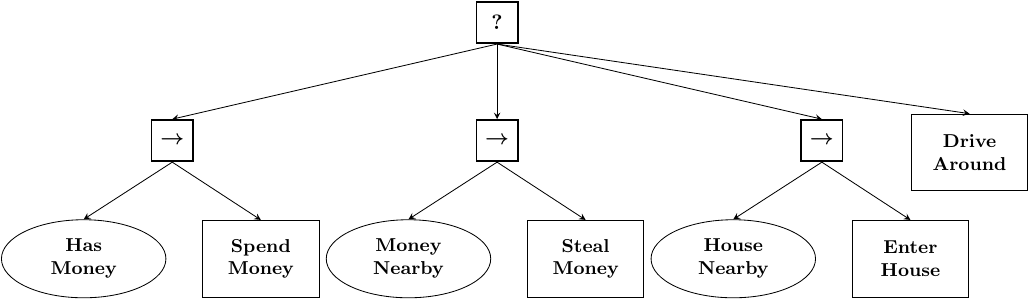}
\caption{Implicit sequence design of the activities of a burglar.}
\label{design:fig:burglar_implicit_sequence}
\end{figure}

In this section, we will show how the modularity of BTs make it very straightforward to combine the design principles described in this chapter at different levels of a BT.
Image we are designing the AI for a game character making a living as a burglar.
Its daily live could be filled with stealing and spending money, as described in the BT of Figure~\ref{design:fig:burglar_implicit_sequence}.
Note that we have used the Implicit Sequence design principle from Section~\ref{sec:implicit_sequences}.
The intended progression is driving around until a promising house is found, enter the house and find indications of money nearby,
steal the money and then leave the house to spend the money.

\begin{figure}[h]
\centering
  \includegraphics[width=7cm]{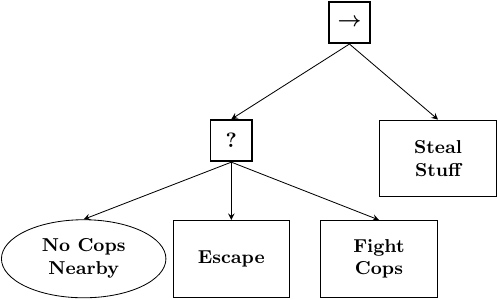}
\caption{If the escape (or fight) action is efficient enough, this sequence construction will guarantee that the burglar is never caught.}
\label{design:fig:burglar_safety}
\end{figure}

Performing the actions described above, the burglar is also interested in  not beeing captured by the police. Therefore we might design a BT handling when to escape, and when to fight the cops trying to catch it.
This might be considered a safety issue, and we can use the design principle for improving safety using sequences, as described in Section~\ref{sec:safety} above.
The result might look like the BT in Figure~\ref{design:fig:burglar_safety}.
If cops are nearby the burglar will first try to escape, and if that fails fight. If anytime during the fight, the escape option is viable, the burglar will switch to escaping.

\begin{figure}[h]
\centering
  \includegraphics[width=10cm]{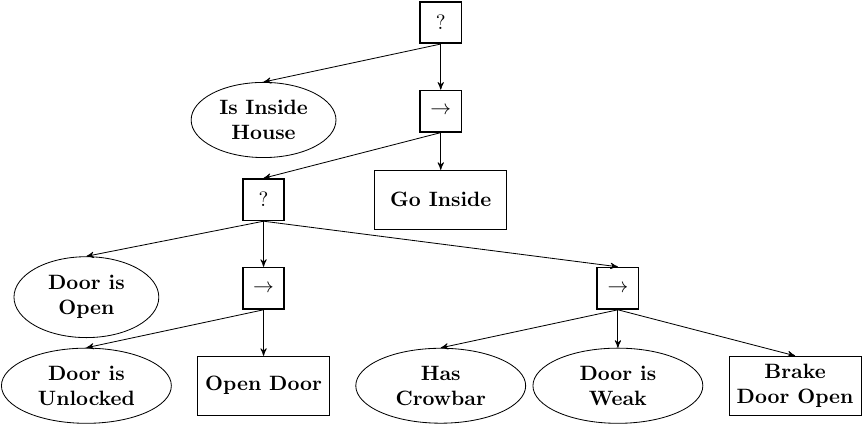}
\caption{Using backchaining, a BT of desired complexity can be created to get a burglar into a house, this is the same as Figure~\ref{design:fig:back_chaining_4}. }
\label{design:fig:burglar_into_house}
\end{figure}

We saw in Section~\ref{design:sec:back_chaining} how backchaining could be used to create a BT of the desired complexity for achieving a goal. That
same BT is shown here in Figure~\ref{design:fig:burglar_into_house} for reference.

Now, the modularity of BTs enable us to combine all these BTs, created with different design principles, into a single, more complex BT,
as shown in Figure~\ref{design:fig:burglar_combined}.
Note that the reactivity of all parts is maintained, and the switches between different sub-BTs happen just the way they should, for example from Drive Around, to Braking a Door Open (when finding a house), to Fighting Cops (when the police arrives and escape is impossible) and then Stealing Money (when police officers are defeated).
We will come back to this example in the next chapter on BT extensions.

\begin{figure}[h]
\centering
  \includegraphics[width=0.9\textwidth]{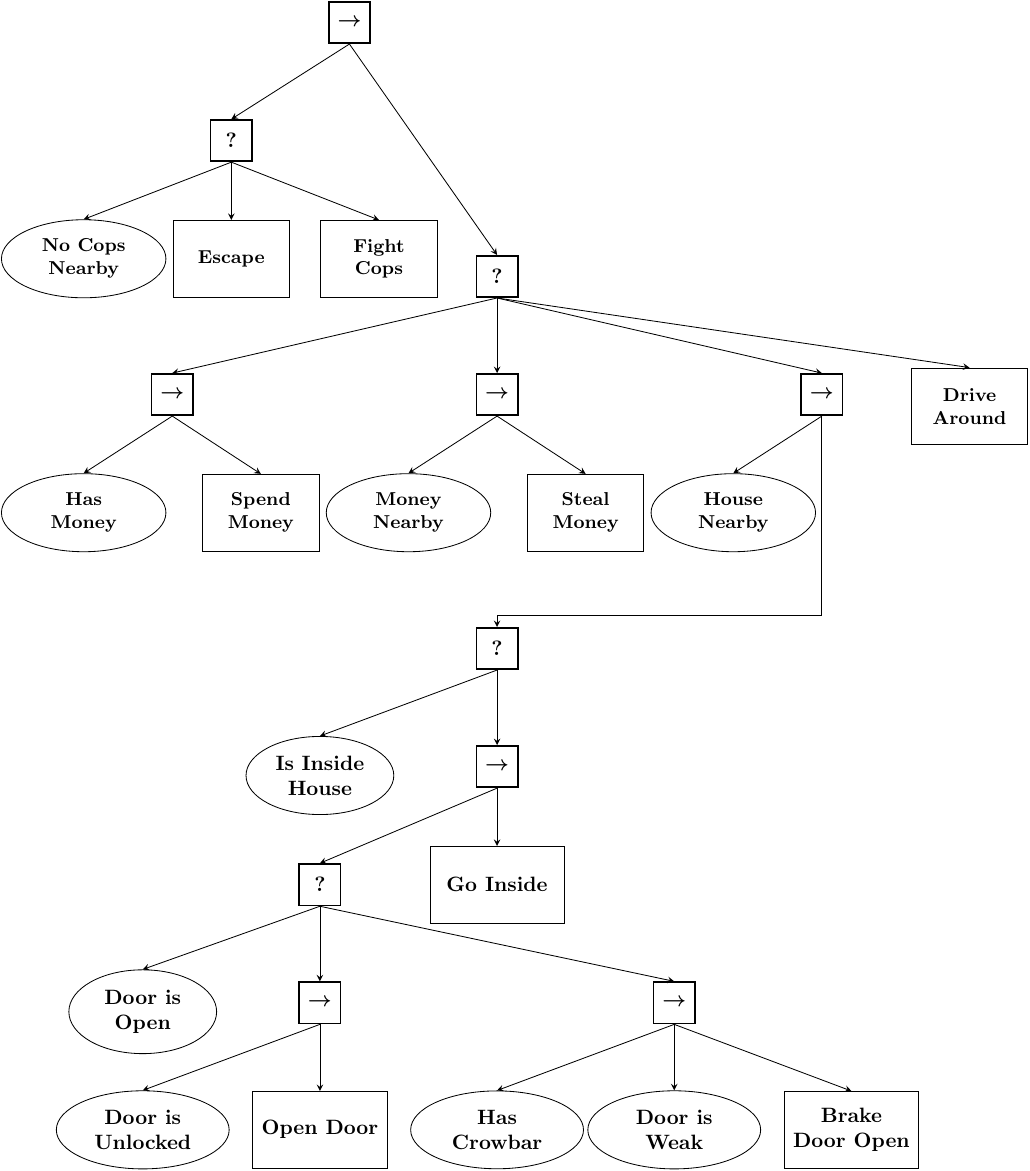}
\caption{A straightforward combination of the BTs in Figures~\ref{design:fig:burglar_implicit_sequence},  \ref{design:fig:burglar_safety}, and   \ref{design:fig:burglar_into_house}.}
\label{design:fig:burglar_combined}
\end{figure}

%% file: extensions/extensions.tex

\chapter{Extensions of Behavior Trees}
\label{ch:extensions}
\graphicspath{{extensions/}}

As the concept of BT has spread in the AI and robotics communities, a number of extensions have been proposed.
Many of them revolve around the Fallback node, and the observation that the ordering of a Fallback node is often somewhat arbitrary.
In the nominal case, the children of a Fallback node are different ways of achieving the same outcome,
which makes the ordering itself unimportant, but note that this is not the case when Fallbacks are used to increase reactivity with implicit sequences, as described in Section~\ref{sec:implicit_sequences}. 

In this chapter, we will describe a number of extensions of the BT concept that have been proposed.

\begin{figure}[h]
\centering
  \includegraphics[width=8cm]{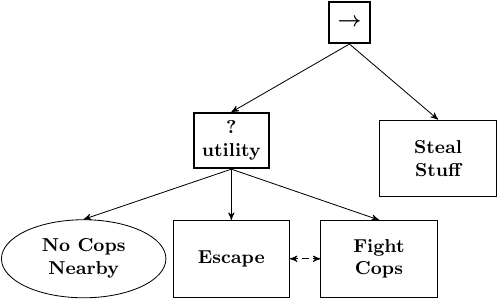}
\caption{The result of adding a utility Fallback in the BT controlling a burglar game character in Figure~\ref{design:fig:burglar_safety}.
Note how the Utility node enables a reactive re-ordering of the actions \emph{Escape} and \emph{Fight Cops}.}
\label{design:fig:burglar_utility}
\end{figure}

\section{Utility BTs}
Utility theory is the basic notion that if we can measure the utility of all potential decisions, it would make sense to 
choose the most useful one. In \cite{merrill2014building} it was suggested that a utility Fallback node would address what was described as the biggest drawback of BTs, i.e. having fixed priorities in the children of Fallback nodes.

A simple example can be seen in the burglar BT of Figure~\ref{design:fig:burglar_utility}.
How do we know that escaping is always better than fighting?
This is highly dependent on the circumstances, do we have a getaway vehicle, do we have a weapon, how many opponents are there, and what are their vehicles and weapons?

By letting the children of a utility Fallback node return their expected utility, the Fallback node can start with the node of highest utility. Enabling the burglar to escape when a getaway car is available, and fight when having a superior weapon at hand. 
In \cite{merrill2014building} it is suggested that all values are normalized to the interval $[0,1]$ to allow comparison between different actions.

Working with utilities is however not entirely straightforward. One of the core strengths of BTs is the modularity, how single actions are handled in the same way as a large tree.
But how do we compute utility for a tree?
Two possible solutions exist,
either we add Decorators computing utility below every utility Fallback node, 
or we add a utility estimate in all actions, and create  a way to propagate utility up the tree, passing both Fallbacks and Sequences. 
The former is a bit ad-hoc, while the latter presents some theoretical difficulties.

It is unclear how to aggregate and propagate utility in the tree. It is suggested in~\cite{merrill2014building} to use the max value in both Fallbacks and Sequences.
This is reasonable for Fallbacks, as the utility Fallback will prioritize the max utility child and execute it first, but one might also argue that a second Fallback child of almost as high utility should increase overall utility for the Fallback.
The max rule is less clear in the Sequence case, as there is no re-ordering, and a high utility child might not be executed due to a failure of another child before it.
These difficulties brings us to the next extension, the Stochastic BTs.


\section{Stochastic BTs}
\label{sec:stochastic_extension}
A natural variation of the idea of utilities above is to consider success probabilities, as suggested in  \cite{Colledanchise14,hannaford2016simulation}.
If something needs to be done, the action with the highest success probability might be a good candidate.
Before going into details, we note that both costs, execution times, and possible undesired outcomes also matters,
but defer this discussion to a later time.

One advantage of considering success probabilities is that the aggregation across both Sequences and Fallbacks
is theoretically straightforward. Let  $P^s_i$ be the success probability of a given tree, then
the probabilities can be aggregated  as follows~\cite{hannaford2016simulation}:
\begin{equation}
 P^s_{\mbox{Sequence}} = \Pi_i P^s_i, \quad P^s_{\mbox{Fallback}} = 1- \Pi_i (1- P^s_i),
\end{equation}
since Sequences need all children to succeed, while Fallbacks need only one, with probability equal to the complement of all failing.
This is theoretically appealing, but relies on the implicit assumption that each action is only tried once. In a reactive BT for a robot picking and placing items,
you could imagine the robot first picking an item, then accidentally dropping it halfway, and then picking it up again. 
Note that the formulas above do not account for this kind of events.

Now the question comes to how we compute or estimate $P^s_i$ for the individual actions. A natural idea
is to learn this from experience~\cite{hannaford2016simulation}.
It is reasonable to assume that the success probability of an action, $P^s_i$, is  a function of the world state, so it would make sense to try to learn
the success probability as a function of state. Ideally we can classify situations such that one action is known to work in some situations,
and another is known to work in others. The continuous maximization of success probabilities in a Fallback node would then make
the BT choose the correct action depending on the situation at hand.

There might still be some randomness to the outcomes, and then the following estimate is reasonable
\begin{equation}
  P^s_i = \frac{\mbox{\# successes}}{\mbox{\# trials}}. 
\end{equation}
However, this leads to a exploit/explore problem \cite{hannaford2016simulation}. What if both available actions of a Fallback have high success probability?
Initially we try one that works, yielding a good estimate for that action. Then the optimization might continue to favor (exploit) that action,
never trying (explore) the other one that might be even better. For the estimates to converge for all actions, even the ones with
lower success estimates needs to be executed sometimes. One can also note that having multiple similar  robots connected to a cloud service
enables much faster learning of both forms of success estimates described above.

It was mentioned above that it might also be relevant to 
 include costs and execution times in the decision of what tree to execute. A formal treatment of both success probabilities and execution times can be found in 
 Chapter~\ref{ch:stochastic}. A combination of cost and success probabilities might result in a utility system, as described above, but finding the right combination of all three is still an open problem.

\section{Temporary Modification of BTs}
Both in robotics and gaming there is sometimes a need to temporary modify the behavior of a BT.
In many robotics applications there is  an operator or collaborator that might want to temporarily 
influence the actions or priorities of a robot. For instance, convincing a service robot to set the table 
before doing the dirty dishes, or making a delivery drone complete the final mission even though
the battery is low enough to motivate an immediate recharge in normal circumstances.
In computer games, the AI is influenced by both level designers, responsible for the player experience,
and AI engineers, responsible for agents behaving rationally.
Thus, the level designers need a way of making some behaviors more likely, without causing
irrational side effects ruining the game experience.

\begin{figure}[h]
\centering
  \includegraphics[width=8cm]{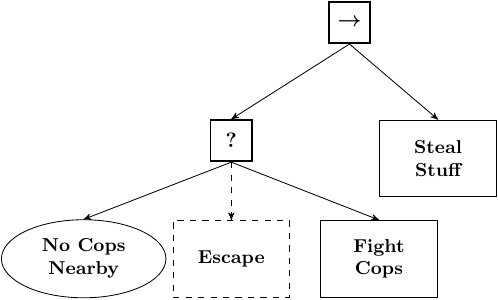}
\caption{The \emph{agressive burglar} style, resulting from disabling  \emph{Escape} in the BT controlling a burglar game character in Figure~\ref{design:fig:burglar_safety}.}
\label{design:fig:burglar_agressive}
\end{figure}

This problem was discussed in one of the first papers on BTs \cite{isla2005handling}, 
with the proposed solutions being \emph{styles},
with each style corresponding to disabling a subset of the BT. For instance, the style \emph{agressive burglar}
might simply have the actions \emph{Escape} disabled, making it disregard injuries and attack until defeated, see Figure~\ref{design:fig:burglar_agressive}.
Similarly, the  \emph{Fight} action can be disabled in the \emph{pacifist burglar} style, as shown in Figure~\ref{design:fig:burglar_pacifist}.
A more elaborate solution to the same problem can be found in the Hinted BTs described below.

\begin{figure}[h]
\centering

  \includegraphics[width=8cm]{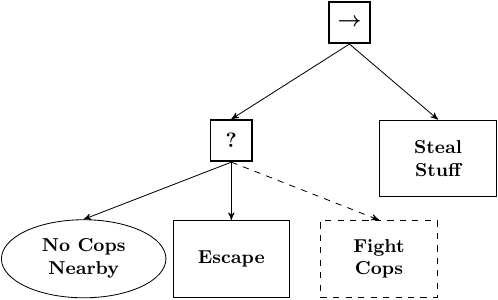}
\caption{The \emph{pacifist burglar} style, resulting from disabling  \emph{Fight} in the BT controlling a burglar game character in Figure~\ref{design:fig:burglar_safety}.}
\label{design:fig:burglar_pacifist}
\end{figure}

Hinted BTs were first introduced in \cite{ocio2010dynamic, ocio2012adapting}.
The key idea is to have an external entity, either human or machine, 
giving suggestions, so-called \emph{hints}, regarding actions to execute, to a  BT.
In robotics, the external entity might be an operator or user suggesting something, and in a computer game
it might be the level designer wanting to influence the behavior of a character  without having to edit the actual BT.

The hints can be both positive (+), in terms of suggested actions, and negative (-), actions to avoid,
and a somewhat complex example can be found in Figure~\ref{design:fig:burglar_hinted}.
 Multiple hints can be active simultaneously,
 each influencing the BT in one, or both, of two different ways.
  First they can effect the ordering of Fallback nodes. Actions or trees with positive hints are moved to the left, and ones with negative hints are moved to the right.
 Second, the BT is extended with additional conditions, checking if a specific hint is given.

\begin{figure}[h!]
\centering
  \includegraphics[width=\textwidth]{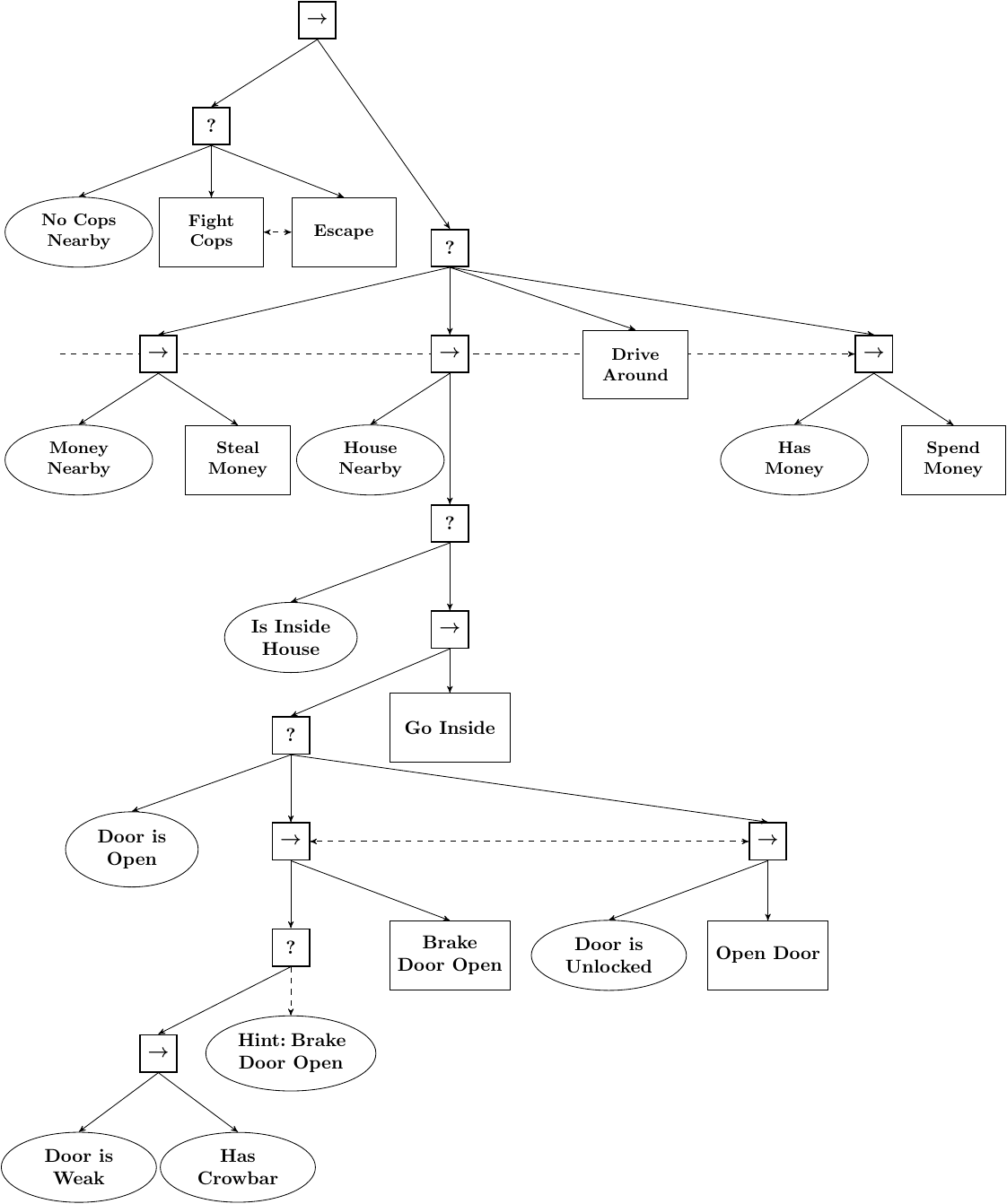}
\caption{The result of providing the hints \emph{Fight Cops+},  \emph{Brake Door Open+} and  \emph{Spend Money-} to the BT in Figure~\ref{design:fig:burglar_combined}. The dashed arrows indicated changes in the BT.}
\label{design:fig:burglar_hinted}
\end{figure}

In the BT of Figure~\ref{design:fig:burglar_hinted}, the following hints were given: \emph{Fight Cops+},  \emph{Brake Door Open+} and  \emph{Spend Money-}.
\emph{Fight Cops+} makes the burglar first considering the fight option, and only escaping when fighting fails.
 \emph{Brake Door Open+} makes the burglar try to brake the door, before seeing if it is open or not,
 and the new corresponding condition makes it ignore the requirements of having a weak door and a crowbar before attempting to brake the door.
Finally,   \emph{Spend Money-} makes the burglar prefer to drive around looking for promising houses rather than spending money.

\section{Other extensions of BTs}
In this section we will briefly describe a number of additional suggested extensions of BTs.
\subsection{Dynamic Expansion of BTs}
The concept of Dynamic Expansions was suggested in \cite{florez2008dynamic}.
Here, the basic idea is to let the BT designer leave some details of the BT to a run-time search.
To enable that search, some desired features of the action needed are specified, these include the
category, given a proposed behavior taxonomy, including \emph{Attack}, \emph{Defend}, \emph{Hunt}, and \emph{Move}. 
The benefit of the proposed approach is that newly created actions can be used in BTs that were created before the actions, 
as long as the BTs have specified the desired features that the new action should have.

%% file: properties/properties.tex

\graphicspath{{properties/figures/}}

\chapter{Analysis of Efficiency, Safety, and Robustness}
\label{chap:properties}
\label{ch:properties}

Autonomous agents will need to be efficient, robust, and reliable in order to be used on a large scale. 
 In this chapter, we present a mathematical framework for analyzing  these properties for  a BT~(Section~\ref{bts:sec:funcBT}).
The analysis includes efficiency (Section~\ref{properties:sec:propDef}), in terms of execution time bounds; robustness (Section~\ref{properties:sec:propDef}), in terms of capability to operate in large domains; and safety (Section~\ref{properties:sec:safety}), in terms of avoiding some particular parts of the state space. 
Some of the results of this chapter were previously published in the journal paper \cite{colledanchise2017behavior}.

\section{Statespace  Formulation of BTs}
\label{bts:sec:funcBT}

In this section, we present a new formulation of BTs.
The new formulation is more formal, and will allow us to analyze how properties are preserved over modular compositions of BTs.
In the functional version, the \emph{tick} is replaced by a recursive function call that includes both the return status, the system dynamics and the system state.

\begin{definition}[Behavior Tree]
\label{bts:thesis.thesis.def:BT}
A BT is a three-tuple 
\begin{equation}
 \bt_i=\{f_i,r_i, \Delta t\}, 
\end{equation}
where $i\in \mathbb{N}$ is the index of the tree, $f_i: \mathbb{R}^n \rightarrow  \mathbb{R}^n$ is the right hand side of an ordinary difference equation, $\Delta t$ is a time step and 
$r_i: \mathbb{R}^n \rightarrow  \{\mathcal{R},\mathcal{S},\mathcal{F}\}$ is the return status that can be equal to either 
\emph{Running} ($\mathcal{R}$),
\emph{Success} ($\mathcal{S}$), or
\emph{Failure} ($\mathcal{F}$).
Let the Running/Activation region ($R_i$),
Success region ($S_i$) and
Failure region ($F_i$) correspond to a partitioning of the state space,  defined as follows:
\begin{eqnarray}
 R_i&=&\{x: r_i(x)=\mathcal{R} \} \\
 S_i&=&\{x: r_i(x)=\mathcal{S} \} \\
 F_i&=&\{x: r_i(x)=\mathcal{F} \}. 
\end{eqnarray}
Finally,  let $x_k=x(t_k)$ be the system state at time $t_k$, then the execution of a BT $\bt_i$ is a standard ordinary difference equation
\begin{eqnarray}
 x_{k+1}&=&f_i( x_{k}),  \label{bts:eq:executionOfBT}\\
 t_{k+1}&=&t_{k}+\Delta t.
\end{eqnarray}
 \end{definition}
The return status $r_i$ will be used when combining BTs recursively, as explained below.

\begin{assumption}
\label{bts:ass:same}
 From now on we will assume that all BTs evolve 
in the same continuous space $\mathbb{R}^n$ using
 the same time step $\Delta t$. 
\end{assumption}

\begin{remark}
\label{bts:rem:differentStatespace}
It is often the case, that different BTs, controlling different vehicle subsystems evolving in different state spaces, need to be 
 combined into a single BT. Such cases can be accommodated in the assumption above by letting all systems evolve in a larger state space, that is the Cartesian product of the smaller state spaces.
\end{remark}


\begin{definition}[Sequence compositions of BTs]
\label{bts:def.seq}
 Two or more BTs  can be composed into a more complex BT using a Sequence operator,
 $$\bt_0=\mbox{Sequence}(\bt_1,\bt_2).$$ 
 Then $r_0,f_0$ are defined as follows
\begin{eqnarray}
   \mbox{If }x_k\in S_1&& \\
   r_0(x_k) &=&  r_2(x_k) \\
   f_0(x_k) &=&  f_2(x_k) \label{bts:eq:seq1}\\ 
   \mbox{ else }&& \nonumber \\
   r_0(x_k) &=&  r_1(x_k) \\
   f_0(x_k) &=&  f_1(x_k). \label{bts:eq:seq2}
 \end{eqnarray}
\end{definition}
$\bt_1$ and $\bt_2$ are called children of $\bt_0$. Note that when executing the new BT, $\bt_0$ first keeps executing its first child $\bt_1$ as long as it returns Running or Failure.  
The second child is executed only when the first returns Success, and $\bt_0$ returns Success only when all children have succeeded, hence the name Sequence,
just as the classical definition of Sequences in Algorithm~\ref{bts:alg:sequence} of Section~\ref{sec:classicalBT}.

For notational convenience, we write
\begin{equation}
 \mbox{Sequence}(\bt_1, \mbox{Sequence}(\bt_2,\bt_3))= \mbox{Sequence}(\bt_1,\bt_2, \bt_3),
\end{equation}
and similarly for arbitrarily long compositions.

\begin{definition}[Fallback compositions of BTs]
\label{bts:def.fallb}
 Two or more BTs  can be composed into a more complex BT using a Fallback operator,
 $$\bt_0=\mbox{Fallback}(\bt_1,\bt_2).$$ 
 Then $r_0,f_0$ are defined as follows
\begin{eqnarray}
   \mbox{If }x_k\in {F}_1&& \\
   r_0(x_k) &=&  r_2(x_k) \\
   f_0(x_k) &=&  f_2(x_k) \\ 
   \mbox{ else }&&\nonumber \\
   r_0(x_k) &=&  r_1(x_k) \\
   f_0(x_k) &=&  f_1(x_k).
 \end{eqnarray}
\end{definition}

Note that when executing the new BT, $\bt_0$  first keeps executing its first child $\bt_1$ as long as it returns Running or Success. 
The second child is executed only when the first returns Failure, and $\bt_0$ returns Failure only when all children have tried, but failed, hence the name Fallback,
just as the classical definition of Fallbacks in Algorithm~\ref{bts:alg:fallback} of Section~\ref{sec:classicalBT}.

 For notational convenience, we write
\begin{equation}
 \mbox{Fallback}(\bt_1, \mbox{Fallback}(\bt_2,\bt_3))= \mbox{Fallback}(\bt_1,\bt_2, \bt_3),
\end{equation}
and similarly for arbitrarily long compositions.

Parallel compositions only make sense if the BTs to be composed control separate parts of the state space, thus we make the following assumption.
\begin{assumption}
\label{bts:ass:parallel}
 Whenever two BTs $\bt_1,\bt_2$ are composed in parallel, we assume that there is a partition of the state space $x=(x_1,x_2)$ such that $f_1(x)=(f_{11}(x),f_{12}(x))$ implies $f_{12}(x)=x$ and $f_2(x)=(f_{21}(x),f_{22}(x))$
implies $f_{21}(x)=x$ (i.e. the two BTs control different parts of the system). 
\end{assumption}

 \begin{definition}[Parallel compositions of BTs]
 \label{bts:def:parallel}
 Two or more BTs  can be composed into a more complex BT using a Parallel operator,
 $$\bt_0=\mbox{Parallel}(\bt_1,\bt_2).$$ 
 Let $x=(x_1,x_2)$ be the partitioning of the state space described in Assumption~\ref{bts:ass:parallel},
 then $f_0(x)=(f_{11}(x),f_{22}(x))$ and $r_0$ is defined as follows
\begin{eqnarray}
   \mbox{If } M=1&&\nonumber \\
   r_0(x) &=&  \mathcal{S}  \mbox{ If } r_1(x)=\mathcal{S} \vee r_2(x)=\mathcal{S}\\
   r_0(x) &=&  \mathcal{F}  \mbox{ If } r_1(x)=\mathcal{F} \wedge r_2(x)=\mathcal{F}\\
   r_0(x) &=&  \mathcal{R}  \mbox{ else } \\
   \mbox{If } M=2&&\nonumber \\
   r_0(x) &=&  \mathcal{S}  \mbox{ If } r_1(x)=\mathcal{S} \wedge r_2(x)=\mathcal{S}\\
   r_0(x) &=&  \mathcal{F}  \mbox{ If } r_1(x)=\mathcal{F} \vee r_2(x)=\mathcal{F}\\
   r_0(x) &=&  \mathcal{R}  \mbox{ else } 
 \end{eqnarray}
\end{definition}

%

\section{Efficiency and Robustness }
\label{properties:sec:propDef}

In this section we will show how some aspects of time efficiency and robustness

carry across modular compositions of BTs. This result will then enable us to conclude, that if two BTs are efficient, then their composition will also be \emph{efficient}, if the right conditions are satisfied.
We also show how the Fallback composition can be used to increase the region of attraction of a BT, thereby making it more robust to uncertainties in the initial configuration.

Note that, as in \cite{burridge1999sequential}, by robustness we mean large regions of attraction. We do not investigate e.g.\ disturbance rejection, or other forms of robustness\footnote{Both meanings of robustness are  aligned with the IEEE standard glossary of software engineering terminology:
``The degree to which a system or component can function correctly in the presence of invalid inputs or stressful environmental conditions."}

Many control problems, in particular in robotics, can be formulated in terms of achieving a given goal configuration in a way that is  time efficient and robust with respect to the initial configuration. Since all BTs return either Success, Failure, or Running, the definitions below will include a finite time, at which Success must be returned.

In order to formalize the discussion above, we  say that
 \emph{efficiency} can be measured by the size of a time bound $\tau$ in Definition~\ref{properties:def:FTS} and
 \emph{robustness} can be measured by the size of the region of attraction $R'$ in the same definition.

\begin{definition}[Finite Time Successful]
\label{properties:def:FTS}
 A BT is Finite Time Successful (FTS) with region of attraction $R'$, if for all starting points $x(0)\in R'\subset R$, there is a time $\tau$, and a time $\tau'(x(0))$ such that $\tau'(x)\leq \tau$ for all starting points, and
 $x(t)\in R'  $ for
 all $t\in [0,\tau')$ 
 and $x(t)\in S$ for
  $t = \tau'$
\end{definition}
As noted in the following lemma, exponential stability implies FTS, given the right choices of the sets $S,F,R$.
\begin{lemma}[Exponential stability and FTS]
 A BT for which $x_s$ is a globally exponentially stable equilibrium of the execution (\ref{bts:eq:executionOfBT}),
 and $S \supset \{x: ||x-x_s||\leq \epsilon\}$, $\epsilon>0$, $F=\emptyset$,  $R=\mathbb{R}^n \setminus S$, is FTS.
\end{lemma}
\begin{proof}
Global exponential stability implies that there exists  $a>0$ such that $||x(k)-x_s|| \leq e^{-ak}$ for all $k$.
Then, for each $\epsilon$ there is a time $\tau$ such that $||x(k)-x_s|| \leq e^{-a\tau}<\epsilon$, which implies that there is a $\tau'<\tau$ such that $x(\tau')\in S$ and the BT is FTS.
\end{proof}

We are now ready to look at how these properties extend across compositions of BTs.

\begin{lemma}
\label{properties:lem:robustnessSequence}
\emph{(Robustness and Efficiency of Sequence Compositions)}
If $\bt_1,\bt_2$ are FTS, with $S_1=R_2' \cup S_2 $,
then $\bt_0=\mbox{Sequence}(\bt_1,\bt_2)$ is FTS with 
$\tau_0 = \tau_1+\tau_2$, 
$R_0'= R_1' \cup R_2'$ and 
$S_0=S_1 \cap S_2 = S_2$.
\end{lemma}
\begin{proof}
 First we consider the case when $x(0)\in R_1'$. Then, as $\bt_1$ is FTS, the state will reach $S_1$ in a time   $k_1 <\tau_1$, without leaving $R_1'$.
 Then $\bt_2$ starts executing, and will keep the state inside $S_1$, since $S_1=R_2' \cup S_2$.
 $\bt_2$ will then bring the state into $S_2$, in a time $k_2 <\tau_2$, and $\bt_0$ will return Success.
 Thus we have the combined time $k_1+k_2 <\tau_1+\tau_1$.
 
 If $x(0)\in R_2'$, $\bt_1$ immediately returns Success, and $\bt_2$ starts executing as above.
\end{proof}

The lemma above is illustrated in Figure \ref{properties:fig:seq1rob}, 
and  Example~\ref{properties:ex:seq}
 below.

\begin{figure}[h]
\begin{center}
\includegraphics[width=4cm]{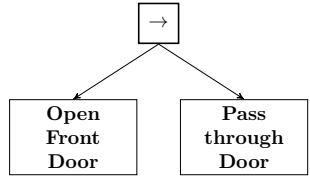}
\caption{
A Sequence is used to to create an \emph{Enter Through Front Door} BT.
Passing the door  is only tried if the opening action succeeds.
Sequences are denoted by a white box with an arrow.}
\label{properties:fig:recipe}
\end{center}
\end{figure}

 \begin{figure}[h]
\begin{center}
\includegraphics[width=0.8\columnwidth]{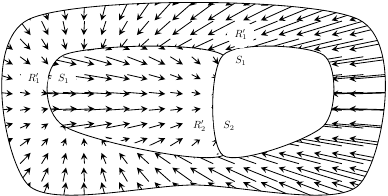}
\caption{The sets $R_1', S_1, R_2', S_2$ of Example \ref{properties:ex:seq} and Lemma \ref{properties:lem:robustnessSequence}. }
\label{properties:fig:seq1rob}
\end{center}
\end{figure}

\begin{example}
\label{properties:ex:seq}
Consider the BT in Figure \ref{properties:fig:recipe}.
If we know that \emph{Open Front Door} is FTS and will finish in less than $\tau_1$ seconds, and that \emph{Pass through Door} is FTS and will finish in less than $\tau_2$ seconds. Then, as long as $S_1=R_2' \cup S_2$,
Lemma \ref{properties:lem:robustnessSequence} states that the combined BT in Figure  \ref{properties:fig:recipe} is also FTS, with an upper bound on the execution time of $\tau_1+\tau_2$.
Note that the condition $S_1=R_2' \cup S_2$ implies that
 the action  \emph{Pass through Door} will not make the system leave $S_1$, by e.g. accidentally colliding with the door and thereby closing it without having passed through it.
\end{example}

The result for Fallback compositions is related, but with a slightly different condition on $S_i$ and $R_j'$. 
Note that this is the theoretical underpinning of the design principle \emph{Implicit Sequences} described in Section~\ref{sec:implicit_sequences}.

\begin{lemma}
\label{properties:lem:robustnessSelector}
\emph{(Robustness and Efficiency of Fallback Compositions)}
If $\bt_1,\bt_2$ are FTS, with $S_2 \subset  R_1'$ and $R_1«=R_1$,
then $\bt_0=\mbox{Fallback}(\bt_1,\bt_2)$ is FTS with 
$\tau_0 = \tau_1+\tau_2$, 
$R_0'= R_1' \cup R_2'$ and 
$S_0=S_1$.
\end{lemma}
\begin{proof}
 First we consider the case when $x(0)\in R_1'$. Then, as $\bt_1$ is FTS, the state will reach $S_1$ before $k=\tau_1<\tau_0$, without leaving $R_1'$.
 If $x(0)\in R_2'\setminus R_1'$, $\bt_2$ will execute, and the state will progress towards $S_2$.  But as $S_2 \subset R_1'$, $x(k_1)\in R_1'$ at some time $k_1<\tau_2$. Then, we have the case above, reaching $x(k_2) \in S_1$ in a total time of $k_2<\tau_1+k_1<\tau_1+\tau_2$.
\end{proof}

The Lemma above is illustrated in Figure \ref{properties:fig:sel1}, and Example~\ref{properties:ex:sel}
 below.

 \begin{figure}[htbp]
\begin{center}
\includegraphics[width=0.6\columnwidth]{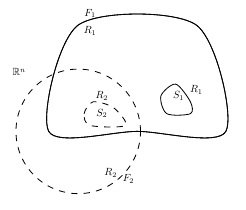}
\caption{The sets $S_1, F_1, R_1$ (solid boundaries) and $S_2, F_2, R_2$ (dashed boundaries) of Example \ref{properties:ex:sel} and Lemma \ref{properties:lem:robustnessSelector}. }
\label{properties:fig:sel1}
\end{center}
\end{figure}
 \begin{figure}[htbp]
\begin{center}
\includegraphics[width=0.4\columnwidth]{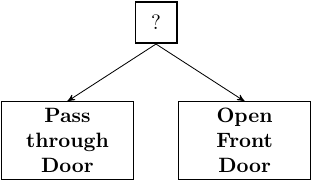}
\caption{An Implicit Sequence created  using a Fallback, as described in Example
 \ref{properties:ex:sel} and Lemma \ref{properties:lem:robustnessSelector}. }
\label{properties:fig:implicit}
\end{center}
\end{figure}

\begin{remark}
 As can be noted, the necessary conditions in Lemma \ref{properties:lem:robustnessSequence},
including~$S_1=R_2' \cup S_2 $ might be harder to satisfy than the conditions of 
Lemma \ref{properties:lem:robustnessSelector}, including~$S_2 \subset  R_1'$.
Therefore,  Lemma \ref{properties:lem:robustnessSelector} is often preferable from a practical point of view,
e.g. using implicit sequences as shown below.
\end{remark}

\begin{example}
\label{properties:ex:sel}
This example will illustrate the design principle \emph{Implicit sequences} of Section~\ref{sec:implicit_sequences}.
Consider the BT in Figure \ref{properties:fig:implicit}. During execution, if  the door is closed, then \emph{Pass through Door} will fail and \emph{Open Front Door} will start to execute. Now, right before \emph{Open Front Door} returns Success, the first action  \emph{Pass through Door} (with higher priority) will realize that the state of the world has now changed enough to enable a possible success and starts to execute, i.e.\ return Running instead of Failure. The combined action of this BT will thus make the robot open the door (if necessary) and then pass through if.

Thus, even though a Fallback composition is used, the result   is sometimes a sequential execution of the children in reverse order (from right to left). Hence the name Implicit sequence.
\end{example}

The example above illustrates how we can increase the robustness of a BT.
If we want to be able to handle more diverse situations, such as a closed door, we do not have to make the door passing action more complex, instead we combine it with another BT that can handle the situation and move the system into a part of the statespace that the first BT can handle.
The sets $S_0, F_0, R_0$ and $f_0$ of the combined BT are shown in Figure~\ref{properties:fig:sel2}, together with the  vector field $f_0(x)-x$.
As can be seen, the combined BT can now move a larger set of initial conditions to the desired region $S_0=S_1$.

\begin{figure}[htbp]
\begin{center}
\includegraphics[width=0.6\columnwidth]{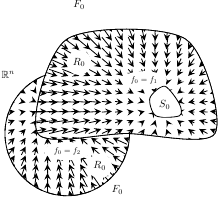}
\caption{The sets $S_0, F_0, R_0$ and the vector field $(f_0(x)-x)$ of Example~\ref{properties:ex:sel} and Lemma \ref{properties:lem:robustnessSelector}. }
\label{properties:fig:sel2}
\end{center}
\end{figure}

\begin{lemma}
\label{properties:lem:robustnessParallel}
\emph{(Robustness and Efficiency of Parallel Compositions)}
If $\bt_1,\bt_2$ are FTS,
then $\bt_0=\mbox{Parallel}(\bt_1,\bt_2)$ is FTS with 

\begin{eqnarray}
   \mbox{If } M=1&&\nonumber \\
	R_0' &=& \{R_1' \cup R_2'\} \setminus \{S_1 \cup S_2\}   \\ 
	S_0 &=& S_1 \cup S_2  \\ 		
	\tau_0 &=& \min(\tau_1,\tau_2) \\
   \mbox{If } M=2&&\nonumber \\
	R_0' &=&  \{R_1' \cap R_2'\} \setminus \{S_1 \cap S_2\}  \\ 
	S_0 &=& S_1 \cap S_2  \\ 
	\tau_0 &=& \max(\tau_1,\tau_2)
 \end{eqnarray}
\end{lemma}
\begin{proof}
The Parallel composition executes $\bt_1$ and $\bt_2$ independently. If $M = 1$ the Parallel composition returns Success if either  $\bt_1$ or $\bt_2$ returns Success, thus $\tau_0 = \min(\tau_1,\tau_2)$. It returns Running if either  $\bt_1$ or $\bt_2$ returns Running and the other does not return Success. If $M = 2$ the Parallel composition returns Success if and only if both  $\bt_1$ and $\bt_2$ return Success, thus $\tau_0 = \max(\tau_1,\tau_2)$.
It returns Running if either  $\bt_1$ or $\bt_2$ returns Running and the other does not return Failure.
\end{proof}

\section{Safety }
\label{properties:sec:safety}
In this section we will show how some aspects of safety carry across modular compositions of BTs. The results will enable us to design a BT to handle safety guarantees and a BT to handle the task execution separately.

In order to formalize the discussion above, we say that \emph{safety} can be measured by the ability to avoid a particular part of the statespace, which we for simplicity  denote the \emph{Obstacle Region}.

\begin{definition}[Safe]
\label{properties:def:Safe}
 A BT is safe, with respect to the obstacle region $O \subset \mathbb{R}^n$, and the initialization region $I \subset R$,
 if for all starting points $x(0)\in I$, we have that $x(t) \not \in O$, for all $t \geq 0$.
 \end{definition}

In order to make statements about the safety of composite BTs we also need the following definition.

\begin{definition}[Safeguarding]
\label{properties:def:Safeguarding}
 A BT is safeguarding, with respect to the step length $d$, the obstacle region $O \subset \mathbb{R}^n$, and the initialization region $I \subset R$,
 if it is safe, and FTS with region of attraction $R' \supset I$ and a success region $S$, such that $I$ surrounds $S$ in the following sense:
\begin{equation}
  \{x\in X \subset \mathbb{R}^n: \inf_{s\in S} || x-s  || \leq d \} \subset I,
\end{equation}
where $X$ is the reachable part of the state space $\mathbb{R}^n$. 
  \end{definition}
This implies that the system, under the control of another BT with maximal statespace steplength $d$, cannot leave $S$ without entering $I$, and thus avoiding $O$,  see Lemma~\ref{properties:lem:safety} below.

\begin{example}
\label{properties:ex:safe}
To illustrate how safety can be improved using a Sequence composition, we  consider the UAV control BT 
 in Figure~\ref{properties:fig:uav}.
The sets $S_i, F_i, R_i$ are shown in Figure~\ref{properties:fig:seq1}. As $\bt_1$ is \emph{Guarrantee altitude above 1000 ft}, its failure region $F_1$ is a small part of the state space (corresponding to a crash) surrounded by the running region $R_1$ that is supposed to move the UAV away from the ground, guaranteeing a minimum altitude of 1000 ft. The success region $S_1$ is large, every state sufficiently distant from $F_1$. The BT that performs the mission, $\bt_2$, has a smaller success region $S_2$, surrounded by a very large running region $R_2$, containing a small failure region $F_2$. The 
function $f_0$ is governed by Equations (\ref{bts:eq:seq1}) and (\ref{bts:eq:seq2}) and is depicted in form of the  vector field $(f_0(x)-x)$  in Figure~\ref{properties:fig:seq2}.
\end{example}

\begin{figure}[htbp]
\begin{center}
\includegraphics[width=4cm]{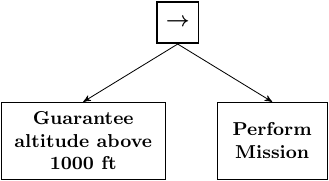}
\caption{The Safety of the UAV control BT is Guaranteed by the first Action.  }
\label{properties:fig:uav}
\end{center}
\end{figure}

 \begin{figure}[htbp]
\begin{center}
\includegraphics[width=0.8\columnwidth]{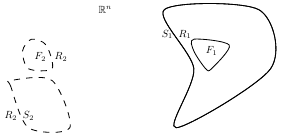}
\caption{The sets $S_1, F_1, R_1$ (solid boundaries) and $S_2, F_2, R_2$ (dashed boundaries) of Example \ref{properties:ex:safe} and Lemma \ref{properties:lem:safety}. }
\label{properties:fig:seq1}
\end{center}
\end{figure}

\begin{figure}[htbp]
\begin{center}
\includegraphics[width=0.8\columnwidth]{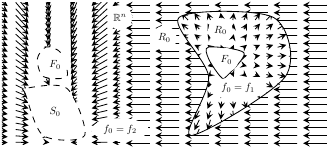}
\caption{The sets $S_0, F_0, R_0$ and  the vector field $(f_0(x)-x)$ of Example~\ref{properties:ex:safe} and Lemma \ref{properties:lem:safety}. }
\label{properties:fig:seq2}
\end{center}
\end{figure}

The discussion above is formalized in Lemma \ref{properties:lem:safety} below.

\begin{lemma}[Safety of Sequence Compositions]
\label{properties:lem:safety}
If $\bt_1$ is safeguarding, with respect to the obstacle $O_1$ initial region $I_1$, and margin $d$, and
 $\bt_2$ is an arbitrary BT with $\max_x ||x-f_2(x)||<d$,
then the composition $\bt_0=\mbox{Sequence}(\bt_1,\bt_2)$ is safe with respect to  $O_1$ and $I_1$.
\end{lemma}
\begin{proof}
 $\bt_1$ is safeguarding, which implies that $\bt_1$ is safe and thus any trajectory starting in $I_1$ will stay out of $O_1$ as long as $\bt_1$ is executing. But if the trajectory reaches $S_1$, $\bt_2$ will  execute until the trajectory leaves $S_1$. We must now show that the trajectory cannot reach $O_1$ without first entering $I_1$. But any trajectory leaving $S_1$ must immediately enter $I_1$, as the first state outside $S_1$ 
 must lie in the set $\{x\in \mathbb{R}^n: \inf_{s\in S_1} || x-s  || \leq d \} \subset I_1$ due to the fact that for $\bt_2$, $||x(k)-x(k+1) ||=||x(k)-f_2(x(k)) ||<d$.
\end{proof}


We conclude this section with a discussion about undesired chattering in switching systems.

The issue of undesired chattering, i.e., switching back and fourth between different subcontrollers,
is always an important concern when designing switched control systems, and BTs are no exception. 
As is suggested by the right part of Figure \ref{properties:fig:seq2}, chattering can be a problem when  vector fields meet at a switching surface. 

Although the efficiency of some compositions can be computed using Lemma \ref{properties:lem:robustnessSequence} and \ref{properties:lem:robustnessSelector}  above, chattering can significantly reduce the efficiency of others.
Inspired by \cite{filippov1988differential} the following result can give an indication of when chattering is to be expected.

Let $R_i$ and $R_j$ be  the running region of  $\bt_i$ and $\bt_j$ respectively. 
We want to study the behavior of the system when a composition of $\bt_i$ and $\bt_j$ is applied. In some cases the execution of a BT will lead to the running region of the other BT and vice-versa. Then, both BTs are alternatively executed and the state trajectory chatters on the boundary between $R_i$ and $R_j$. 
We formalize this discussion in  the following lemma.

\begin{lemma}
 Given a composition $\bt_0=\mbox{Sequence}(\bt_1,\bt_2)$, where $f_i$ depend on $\Delta t$ such
 that $||f_i(x)-x|| \to 0$ when $\Delta t \to 0$.
 Let  $s:\mathbb{R}^n\to \mathbb{R}$ be such that 
  $s(x)=0$ if $x\in \delta S_1 \cap R_2$,
$s(x)<0$ if $x\in \mbox{interior}( S_1 )\cap R_2$,
 $s(x)>0$ if $x\in \mbox{interior}(  \mathbb{R}^n \setminus S_1 ) \cap R_2$,
 and let 
 $$
 \lambda_i(x)=(\frac{\partial s}{\partial x})^T(f_i(x)-x).
 $$
 Then, $x \in \delta S_1$ is chatter free, i.e., avoids switching between $\bt_1$ and $\bt_2$ at every timestep, for small enough $\Delta t$, if $\lambda_1(x)<0$ or $\lambda_2(x)>0$.
\end{lemma}
 
\begin{proof}
When the condition holds, the vector field is pointing outwards on at least one side of the switching boundary.
\end{proof}

Note that this condition is not satisfied on the right hand side of Figure~\ref{properties:fig:seq2}.
This concludes our analysis of BT compositions.

\section{Examples}
\label{properties:sec:example}
In this section, we show some BTs of example and we analyze their properties.

Section~\ref{properties:sec:example:robeff} Illustrates how to analyze robustness and efficiency of a robot executing a generic task. Section~\ref{properties:sec:safety} illustrates to compute safety using the functional representation of Section~\ref{bts:sec:funcBT}. Section~\ref{properties:sec:example:reliability} illustrate how to compute the performance estimate of a given BT. Finally, Section~\ref{properties:sec:example:huge}  illustrate the properties above of a complex BT .

All BTs were implemented using the ROS BT library.\footnote{library available at \url{http://wiki.ros.org/behavior_tree}.}
A video showing the executions of the BTs used in Sections~\ref{properties:sec:example:safety}-\ref{properties:sec:example:robeff} is publicly available. \footnote{\url{https://youtu.be/fH7jx4ZsTG8}}

\subsection{Robustness and Efficiency}
\label{properties:sec:example:robeff}
To illustrate Lemma \ref{properties:lem:robustnessSelector} we look at the BT of Figure~\ref{properties:fig:robust3} controlling a humanoids robot.
The BT has three subtrees \emph{Walk Home}, which is first tried,
if that fails (the robot cannot walk if it is not standing up) it tries the subtree \emph{Sit to Stand},
and if that fails, it tries \emph{Lie down to Sit Up}.
Thus, each fallback action brings the system into the running region of the action to its left,
e.g., the result of \emph{Sit to Stand} is to enable the execution of  \emph{Walk Home}.

\begin{example}
\label{properties:ex:rob}
Let  $x=(x_{1},x_{2})\in \mathbb{R}^2$, where $x_{1}\in [0,0.5]$ is the horizontal position of the robot head and
$x_{2}\in [0,0.55]$ is vertical position  (height above the floor) of the robot head.
The objective of the robot is to get to the destination at $(0,0.48)$. 

\begin{figure}[htbp]
\begin{center}
\includegraphics[width=0.4 \columnwidth]{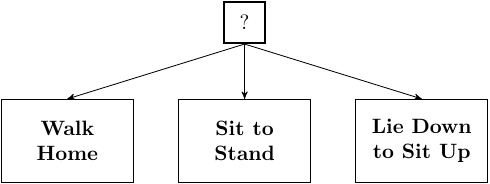}
\caption{The combination $\bt_3$=Fallback($\bt_4,\bt_5,\bt_6$) increases robustness by increasing the region of attraction. }
\label{properties:fig:robust3}
\end{center}
\end{figure}
\end{example}

First we describe the sets $S_i, F_i, R_i$ and the corresponding vector fields of the functional representation.
Then we apply Lemma~\ref{properties:lem:robustnessSelector} to see that the combination does indeed improve robustness. For this example $\Delta t=1s$.

For \emph{Walk Home}, $\bt_4$, we have that 
\begin{align}
S_4 &= \{x: x_1 \leq 0 \}  \label{properties:eq:startRobust} \\
R_4 &= \{x: x_1 \neq 0, x_2\geq 0.48 \}\\
F_4 &= \{x: x_1 \neq 0, x_2 < 0.48 \} \\
 f_4(x)&=\begin{pmatrix}
x_1-0.1 \\
x_2
\end{pmatrix}
\end{align}
that is, it runs as long as the vertical position of the robot head, $x_2$, is at least $0.48m$ above the floor, and moves towards the origin with a speed of $0.1m/s$. If the robot is not standing up $x_2<0.48m$ it returns Failure. A phase portrait of $f_4(x)-x$ is shown in Figure~\ref{properties:fig:goback}.
 Note that $\bt_4$ is FTS with the completion time bound $\tau_4 =0.5/0.1=10$ and region of attraction $R_4'=R_4$.

\begin{figure}[htbp]
\begin{center}
\includegraphics[width=0.6\columnwidth]{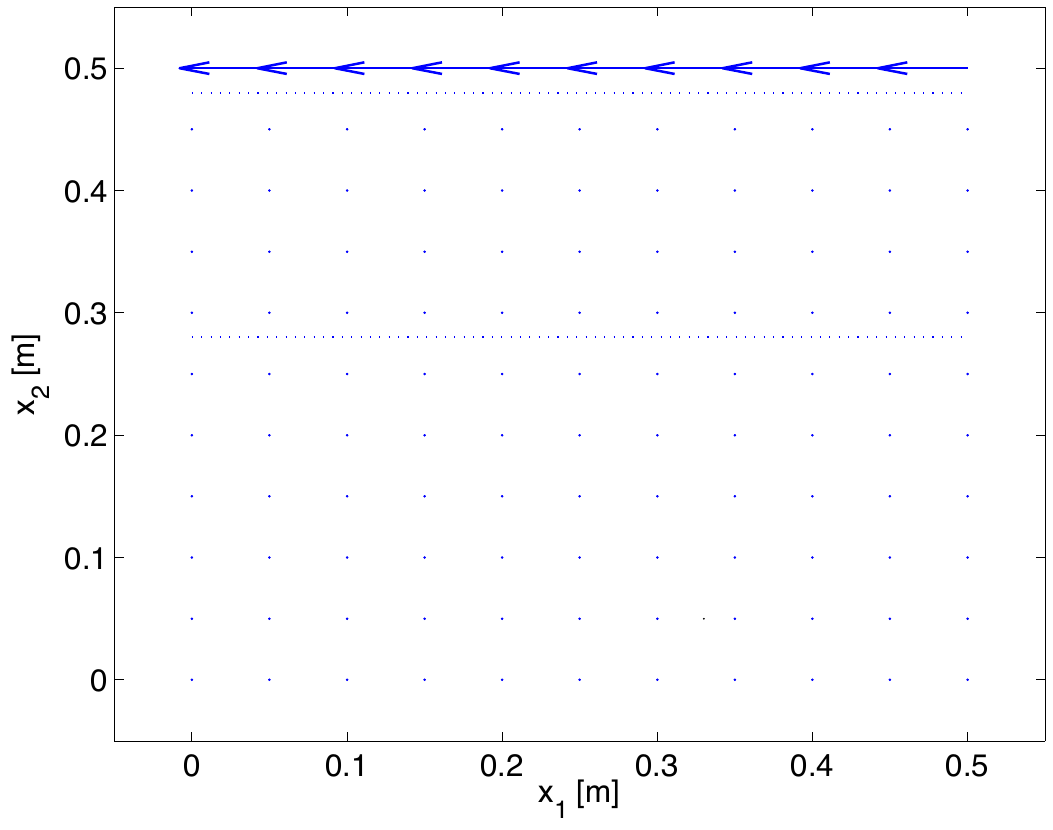}
\caption{The Action \emph{Walk Home}, keeps the head around $x_2=0.5$ and moves it towards the destination $x_1=0$.}
\label{properties:fig:goback}
\end{center}
\end{figure}

For \emph{Sit to Stand}, $\bt_5$, we have that 
\begin{align}
S_5 &= \{x: 0.48 \leq x_2 \}\\
R_5 &= \{x: 0.3 \leq x_2 <  0.48 \}\\
F_5 &= \{x: x_2 < 0.3 \} \\
 f_5(x)&=\begin{pmatrix}
x_1 \\
x_2+0.05
\end{pmatrix}
\end{align}
that is, it runs as long as the vertical position of the robot head, $x_2$, is in between $0.3m$ and $0.48m$ above the floor.
If $0.48 \leq x_2$ the robot is standing up, and it returns Success.
If $x_2 \leq 0.3$ the robot is lying down, and it returns Failure.
 A phase portrait of $f_5(x)-x$ is shown in Figure~\ref{properties:fig:standup}.
 Note that $\bt_5$ is FTS with the completion time bound $\tau_5 = \mbox{ceil}(0.18/0.05)=\mbox{ceil}(3.6)=4$ and region of attraction $R_5'=R_5$

\begin{figure}[htbp]
\begin{center}
\includegraphics[width=0.6\columnwidth]{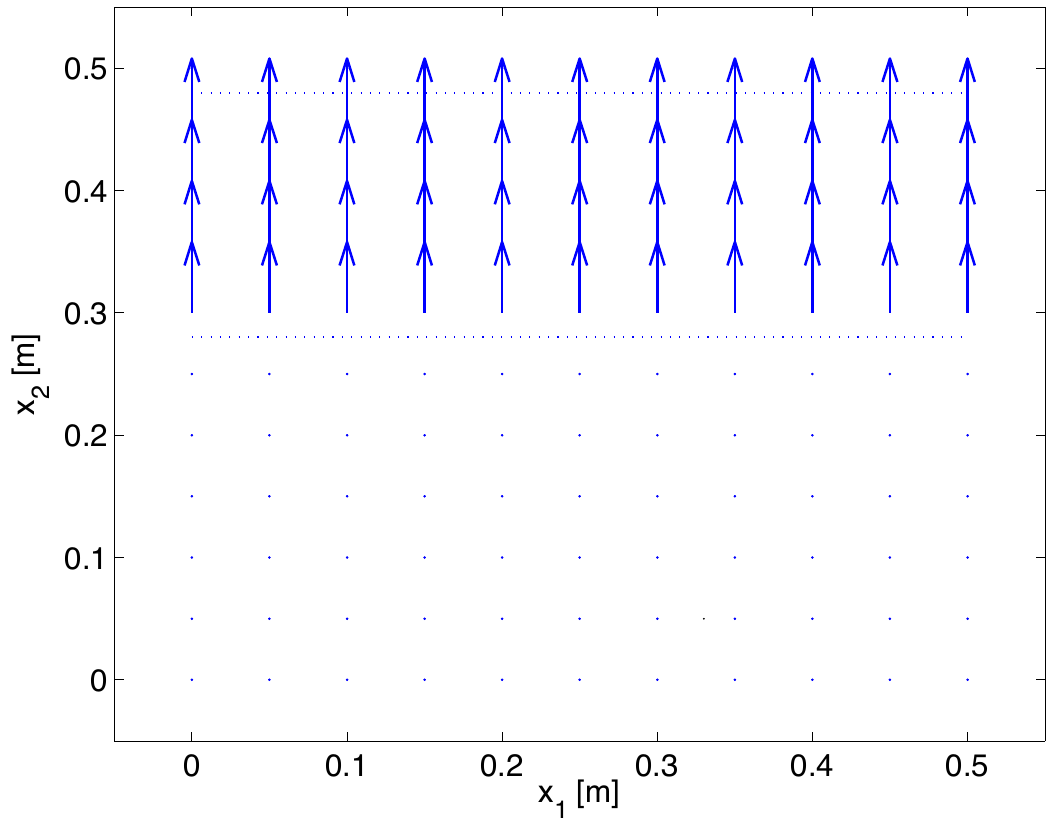}
\caption{The Action \emph{Sit to Stand} moves the head upwards in the vertical direction towards standing. }
\label{properties:fig:standup}
\end{center}
\end{figure}

For \emph{Lie down to Sit Up}, $\bt_6$, we have that 
\begin{align}
S_6 &= \{x: 0.3 \leq x_2 \}\\
R_6 &= \{x: 0 \leq x_2 <  0.3 \}\\
F_6 &= \emptyset \\
 f_6(x)&=\begin{pmatrix}
x_1 \\
x_2+0.03
\end{pmatrix} \label{properties:eq:endRobust}
\end{align}
that is, it runs as long as the vertical position of the robot head, $x_2$, is below $0.3m$  above the floor.
If $0.3 \leq x_2$ the robot is sitting up (or standing up), and it returns Success.
If $x_2 < 0.3$ the robot is lying down, and it returns Running.
 A phase portrait of $f_6(x)-x$ is shown in Figure~\ref{properties:fig:situp}.
 Note that $\bt_6$ is FTS with the completion time bound $\tau_6 =0.3/0.03=10$ and region of attraction $R_6'=R_6$

\begin{figure}[htbp]
\begin{center}
\includegraphics[width=0.6\columnwidth]{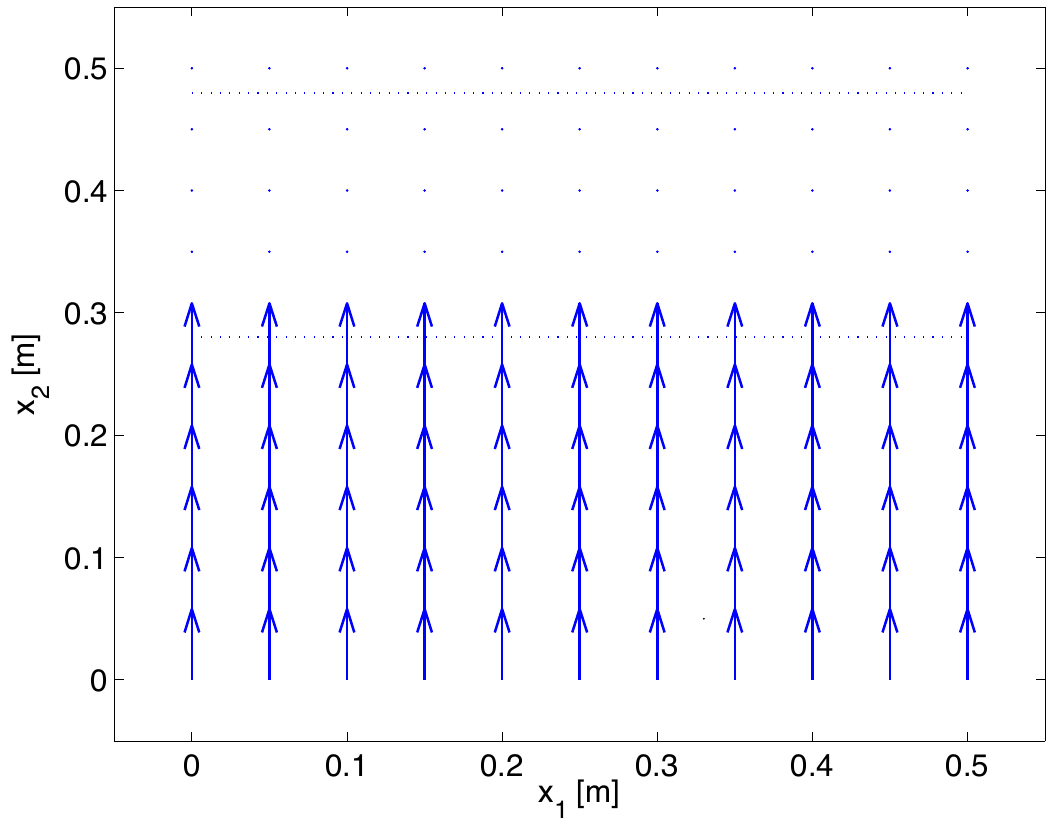}
\caption{The Action \emph{Lie down to Sit Up} moves the head upwards in the vertical direction towards sitting. }
\label{properties:fig:situp}
\end{center}
\end{figure}

Informally, we can look at the phase portrait in Figure~\ref{properties:fig:getback} to get a feeling for what is going on.
As can be seen the Fallbacks make sure that the robot gets on its feet and walks back,
independently of where it started in $\{x:0< x_1 \leq 0.5, 0\leq x_2 \leq 0.55\}$.

\begin{figure}[htbp]
\begin{center}
\includegraphics[width=0.6\columnwidth]{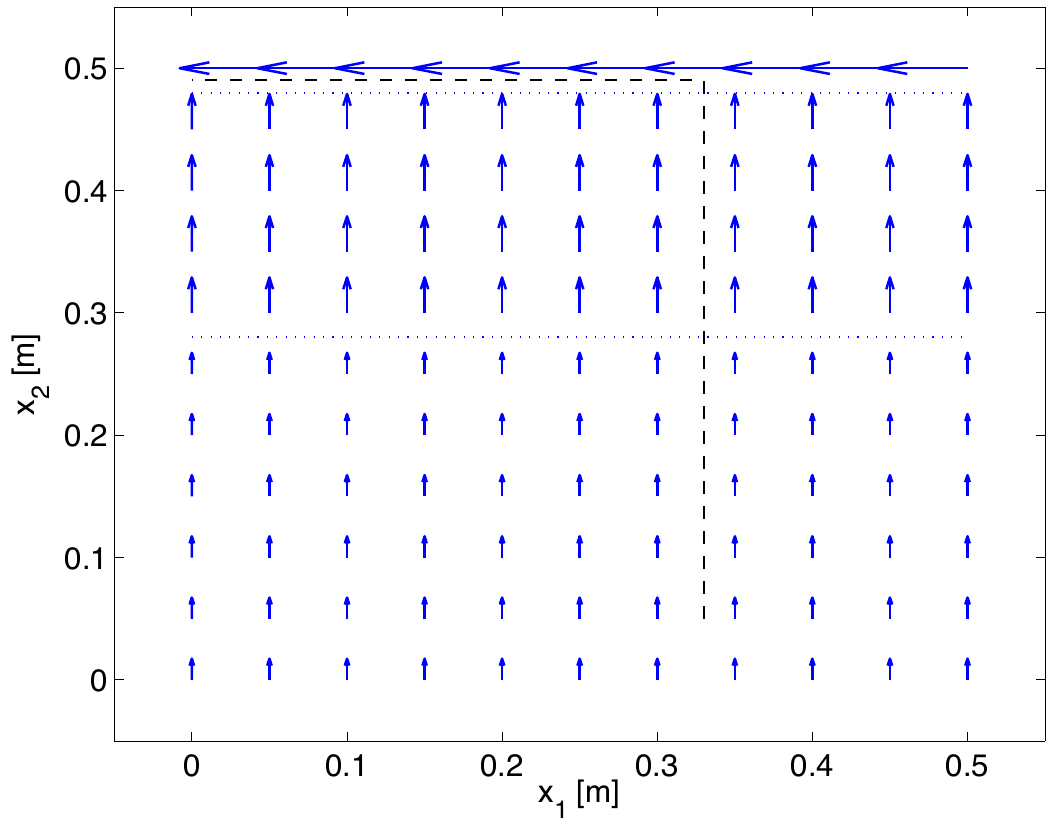}
\caption{The combination Fallback($\bt_4,\bt_5,\bt_6$) first gets up, and then walks home.}
\label{properties:fig:getback}
\end{center}
\end{figure}

Formally, we can use Lemma \ref{properties:lem:robustnessSelector} to compute 
robustness in terms of the region of attraction $R_3'$,  and efficiency in terms of bounds on completion time $\tau_3$.
The results are described in the following Lemma.

\begin{lemma}
\label{properties:lem:robustExample}
 Given $\bt_4,\bt_5,\bt_6$ defined in Equations (\ref{properties:eq:startRobust})-(\ref{properties:eq:endRobust}).
 
 The combined BT $\bt_3=\mbox{Fallback}(\bt_4,\bt_5,\bt_6)$ is FTS, with 
  region of attraction $R_3'=\{x:0< x_1 \leq 0.5, 0\leq x_2 \leq 0.55\}$,
 completion time bound $\tau_3 =24$.
\end{lemma}
\begin{proof}
We note that $\bt_4,\bt_5,\bt_6$ are FTS with
$\tau_4 =10$, $\tau_5 =4$, $\tau_6 =10$
and regions of attractions equal to the running regions $R_i'=R_i$.
Thus we have that 
$S_6 \subset R_5=R_5'$ and
$S_5 \subset R_4=R_4'$.
Applying Lemma \ref{properties:lem:robustnessSelector} twice now gives the desired results,
$R_3'=R_4' \cup R_5'\cup R_6'=\{x:0\leq x_1 \leq 0.5, 0\leq x_2 \leq 0.55\}$ and
$\tau_3 = \tau_4+\tau_5+\tau_6=10+4+10=24$.
\end{proof}


\subsection{Safety}
\label{properties:sec:example:safety}

To illustrate Lemma~\ref{properties:lem:safety} we choose the BT of Figure \ref{properties:fig:safe3}. 
The idea is that the first subtree in the Sequence (named \emph{Guarantee Power Supply}) is to guarantee that the combination does not run out of power,
under very general assumptions about what is going on in the second BT.

First we describe the sets $S_i, F_i, R_i$ and the corresponding vector fields of the functional representation.
Then we apply Lemma~\ref{properties:lem:safety} to see that the combination does indeed guarantee against running out of batteries.
\begin{figure}[htbp]
\begin{center}
\includegraphics[width=0.4 \columnwidth]{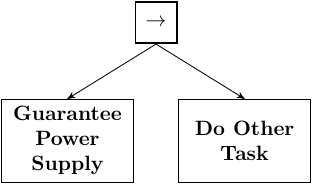}
\caption{A BT where the first action guarantees that the combination does not run out of battery. }
\label{properties:fig:safe3}
\end{center}
\end{figure}

\begin{example}

Let $\bt_1$ be \emph{Guarantee Power Supply} and $\bt_2$ be \emph{Do other tasks}.
Let furthermore $x=(x_{1},x_{2})\in \mathbb{R}^2$, where $x_{1}\in [0,100]$ is the distance from the current position to the recharging station
 and
$x_{2}\in [0,100]$ is the battery level. For this example $\Delta t=10s$.

For \emph{Guarantee Power Supply}, $\bt_1$, we have that 
\begin{align}
S_1 &= \{x:  100 \leq x_2 \mbox{ or } (0.1 \leq x_1, 20 < x_2) \} \label{properties:bt1Start} \\
R_1 &= \{x: x_2 \leq 20 \mbox{ or } (x_2 < 100  \mbox{ and }   x_1 < 0.1) \}\\
F_1 &= \emptyset \\
 f_1(x)&=\begin{pmatrix}
x_1 \\
x_2+1
\end{pmatrix}  \mbox{if } x_1<0.1, x_2<100\\
&=\begin{pmatrix}
x_1-1 \\
x_2-0.1
\end{pmatrix}  \mbox{else} \label{properties:bt1End}
\end{align}
that is, when running, the robot moves  to $x_1<0.1$ and recharges. While moving, the battery level decreases and while charging the battery level increases.
If at the recharge position, it returns Success only after reaching $x_2 \geq 100$.
Outside of the recharge area, it returns Success as long as the battery level is above 20\%.
 A phase portrait of $f_1(x)-x$ is shown in Figure~\ref{properties:fig:rechargeR}.

\begin{figure}[htbp]
\begin{center}
\includegraphics[width=0.6\columnwidth]{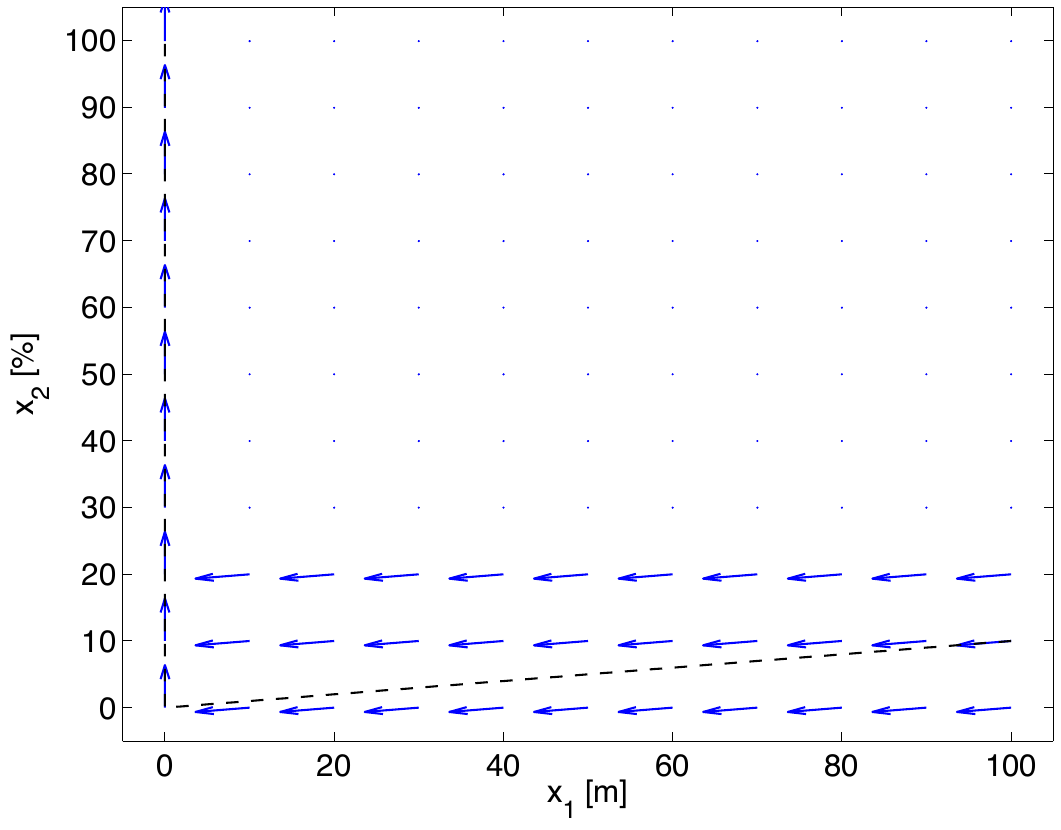}
\caption{The \emph{Guarantee Power Supply} Action }
\label{properties:fig:rechargeR}
\end{center}
\end{figure}

For \emph{Do Other Task}, $\bt_2$, we have that 
\begin{align}
S_2 &=  \emptyset \\
R_2 &= \mathbb{R}^2 \\
F_2 &= \emptyset \\
 f_2(x)&=\begin{pmatrix}
x_1 + (50-x_1)/50\\
x_2-0.1
\end{pmatrix} 
\end{align}
that is, when running, the robot moves towards $x_1=50$ and does some important task,
while the  battery  level keeps on decreasing. 
 A phase portrait of $f_2(x)-x$ is shown in Figure~\ref{properties:fig:rechargeR}.

\begin{figure}[htbp]
\begin{center}
\includegraphics[width=0.6\columnwidth]{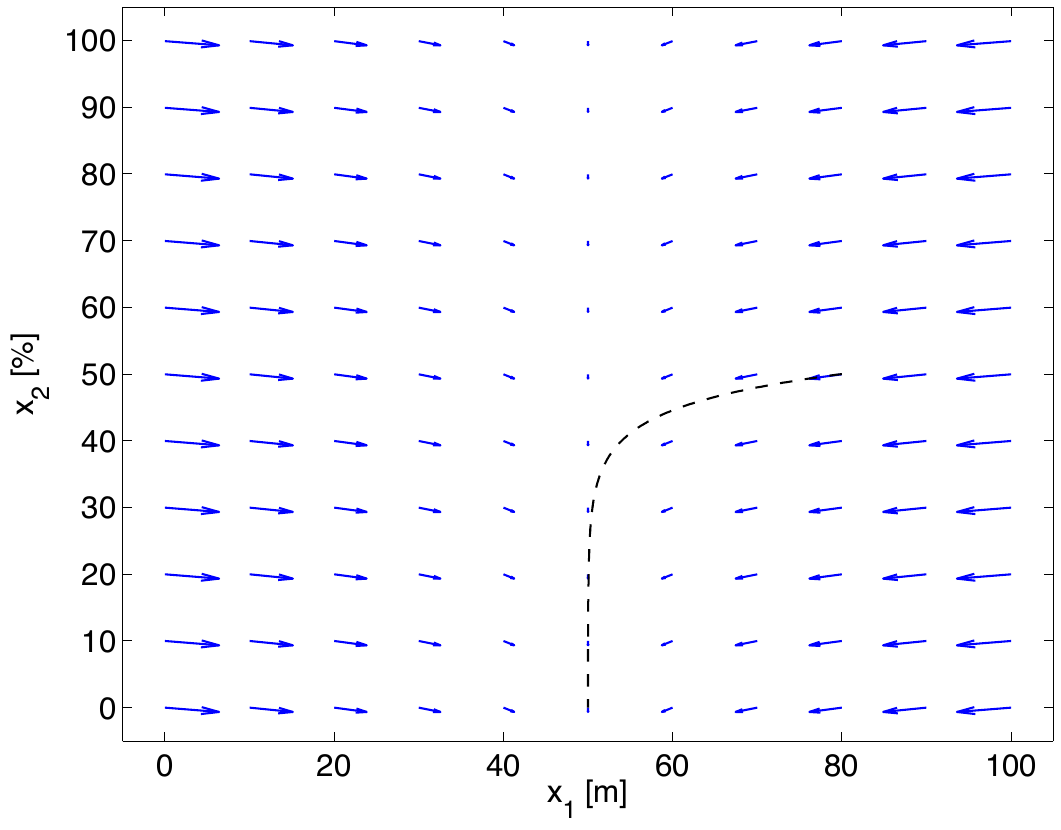}
\caption{The \emph{Do Other Task} Action  }
\label{properties:fig:rechargeDo}
\end{center}
\end{figure}
\end{example}

Given $\bt_1$ and $\bt_2$, the composition $\bt_0$=Sequence($\bt_1,\bt_2$) is created to improve the safety of $\bt_2$, as described below.

Informally, we can look at the phase portrait in Figure~\ref{properties:fig:rechargeTot} to get a feeling for what is going on.
The obstacle to be avoided is the Empty Battery state $O=\{x: x_2=0 \}$, and
$\bt_0$ makes sure that this state is never reached, since the \emph{Guarantee Power Supply} action
starts executing as soon as \emph{Do Other Task} brings the battery level below 20\%.
The remaining battery level is also enough for the robot to move back to the recharging station,
given that the robot position is limited by the reachable space, i.e., $x_{1k}\in [0,100]$.

\begin{figure}[htbp]
\begin{center}
\includegraphics[width=0.6\columnwidth]{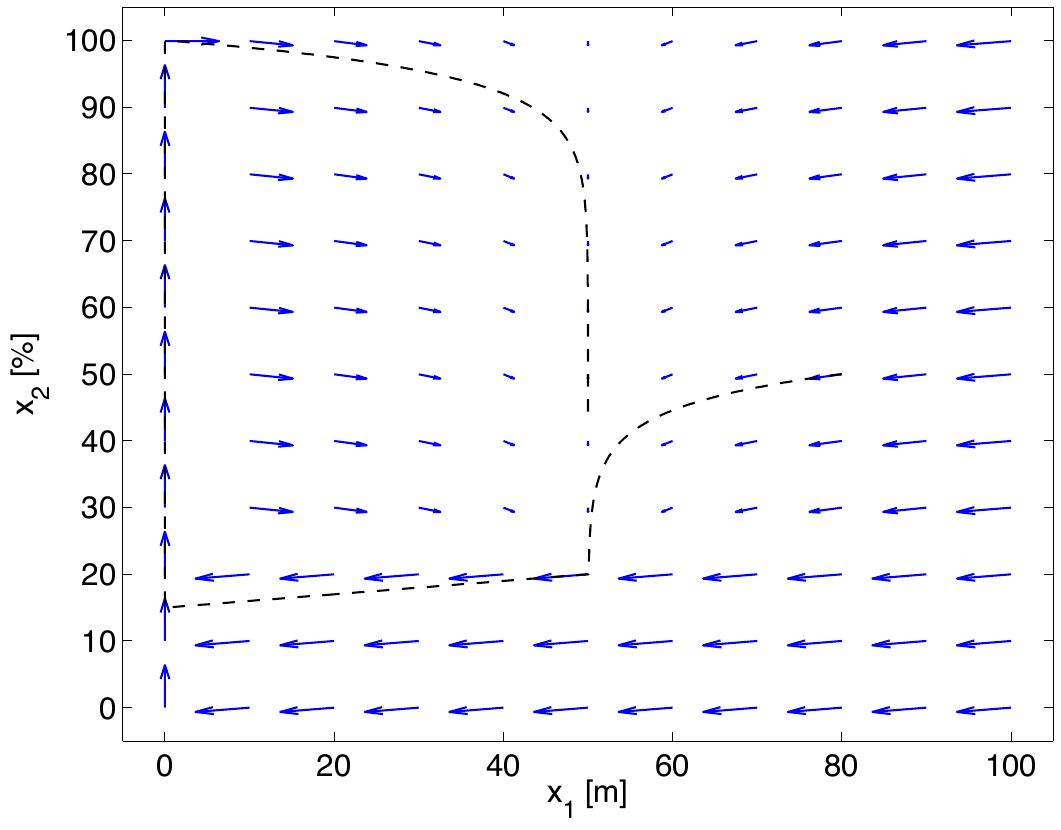}
\caption{Phase portrait of $\bt_0$=Sequence($\bt_1,\bt_2$). Note that $\bt_1$ guarantees that the combination does not run out of battery. 
The dashed line is a simulated execution, starting at $(80,50)$.}
\label{properties:fig:rechargeTot}
\end{center}
\end{figure}

Formally, we state the following Lemma

%
\begin{lemma}
\label{properties:lem:recharge}
 Let the obstacle region be $O=\{x: x_2=0 \}$ and the initialization region be $I=\{x: x_1 \in [0,100], x_2\geq 15 \}$.

 Furthermore, let $\bt_1$ be given by (\ref{properties:bt1Start})-(\ref{properties:bt1End}) and $\bt_2$ be an arbitrary BT satisfying $\max_x ||x-f_2(x)||<d=5$, then  $\bt_0$=Sequence($\bt_1,\bt_2$) is safe 
 with respect to $I$ and $O$, i.e. if $x(0)\in I$, then $x(t) \not \in O$, for all $t > 0$. 
\end{lemma}
\begin{proof}
 First we see that $\bt_1$ is safe with respect to $O$ and $I$.
 Then we notice that $\bt_1$ is safeguarding with margin $d=10$ for the reachable set $X=\{x: x_1 \in [0,100], x_2 \in [0,100] \}$.
 Finally we conclude that $\bt_0$ is Safe, according to Lemma~\ref{properties:lem:safety}.
\end{proof}
Note that if we did not constraint the robot to move in some reachable set $X=\{x: x_1 \in [0,100], x_2 \in [0,100] \}$, it would be able to 
move so far away from the recharging station that the battery would not be sufficient to bring it back again before reaching $x_2=0$.


\clearpage
\subsection{A More Complex BT}
\label{properties:sec:example:huge}
Below we will use a larger BT to illustrate modularity, as well as the applicability of the proposed analysis tools to more complex problems.

\begin{figure*}[t]
  \includegraphics[width=\textwidth]{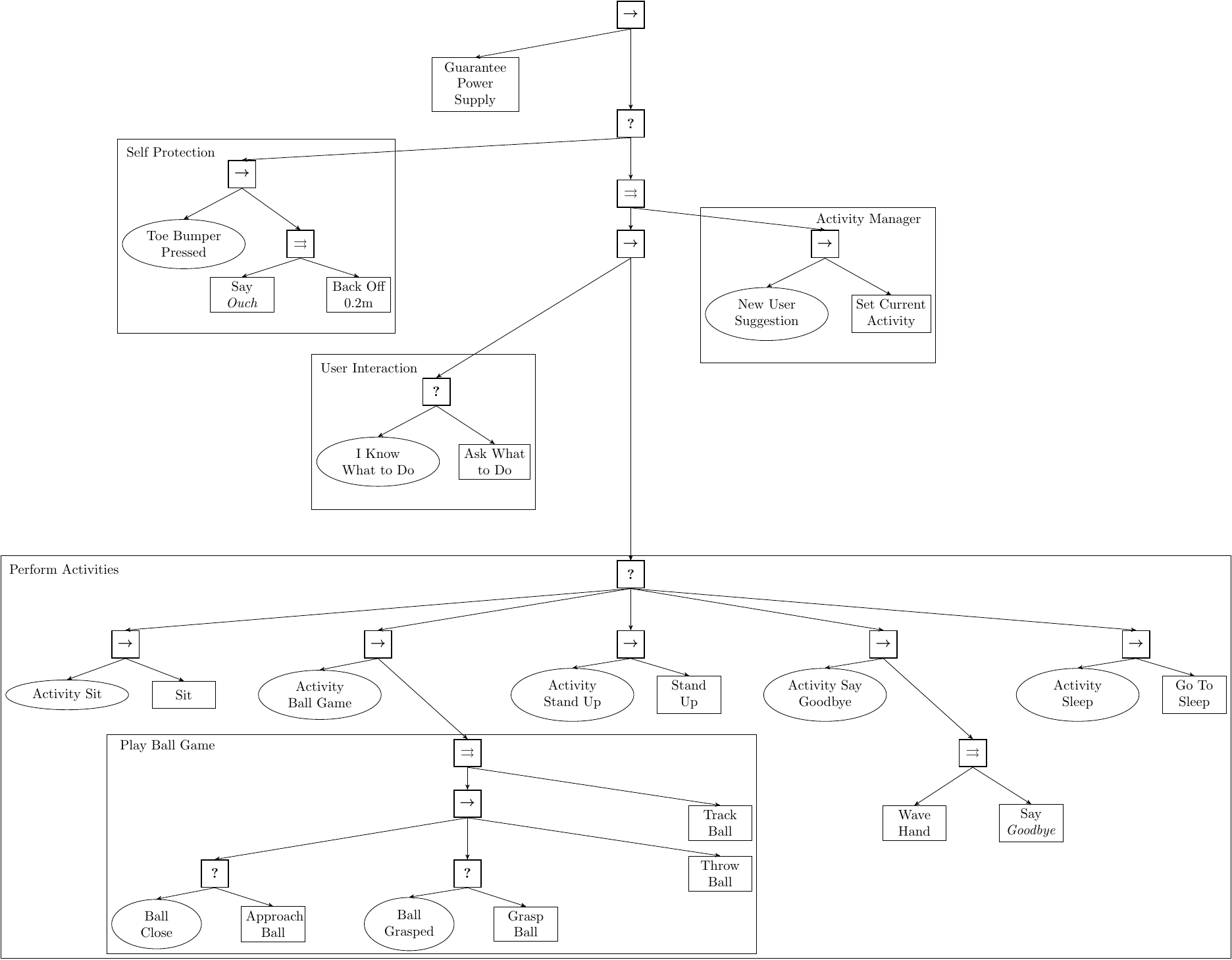}
  \caption{A BT that combines some capabilities  of the humanoid robot in an interactive and modular way. Note how atomic actions can easily be replaced by more complex sub-BTs.}
  \label{properties:fig:hugeExample}
\end{figure*}
\begin{example}

The BT in Figure~\ref{properties:fig:hugeExample} is 
designed for controlling a humanoid robot in an interactive capability demo,
and includes the BTs of Figures~\ref{properties:fig:safe3} and \ref{properties:fig:robust3} as subtrees, as discussed below.

The top left part of the tree includes some exception handling, in terms of battery management, and backing up and complaining in case  the toe bumpers are pressed.
The top right part of the tree is a Parallel node, listening for new user commands, along with a request for such commands if none are given and an execution of the corresponding activities if a command has been received. 

The subtree \emph{Perform Activities} is composed of  checking of
what activity to do, and execution of the corresponding command.
Since the activities are mutually exclusive, we let the Current Activity hold only the latest command and no ambiguities of control commands will occur.

The subtree \emph{Play Ball Game} runs the ball tracker, in parallel with moving closer to the ball, grasping it, and throwing it. 

As can be seen, the design is quite modular. A HDS implementation of the same functionality would need an extensive amount of transition arrows going in between the different actions.  

We will now apply the analysis tools of the paper to the example, initially assuming that all atomic actions are FTS, as described in Definition~\ref{properties:def:FTS}.

Comparing Figures~\ref{properties:fig:safe3} and \ref{properties:fig:hugeExample} we see that they are identical, if we let \emph{Do Other Task} correspond to the whole right part of the larger BT. Thus, according to Lemma \ref{properties:lem:recharge}, the complete BT is safe, i.e.\ it will not run out of batteries, as long as the reachable state space is bounded by 100 distance units from the recharging station and the time steps are small enough so that $\max_x ||x-f_2(x)||<d=5$, i.e.\ the battery does not decrease more than  5\%  in a single time step.

The design of the right subtree in \emph{Play Ball Game} is made to satisfy Lemma~\ref{properties:lem:robustnessSequence},
with the condition $S_1=  R_2'\cup S_2$. Let 
$\bt_1=\mbox{Fallback(Ball Close?, Approach Ball)}$, 
$\bt_2=\mbox{Fallback(Ball Grasped?, Grasp Ball)}$,
$\bt_3=\mbox{Throw Ball}$.
Note that the use of condition-action pairs makes the success regions explicit. 
Thus $S_1=  R_2'\cup S_2$, i.e.\ Ball Close is designed to describe the Region of Attraction of Grasp Ball, and
$S_2=  R_3'\cup S_3$, i.e.\ Ball Grasped is designed to describe the Region of Attraction of Throw Ball.
Finally, applying Lemma~\ref{properties:lem:robustnessSequence} twice
we conclude that the right part of \emph{Play Ball Game} is FTS with completion time bound $\tau_1+\tau_2+\tau_3$,
region of attraction $R_1'\cup R_2'\cup R_3'$ and success region $S_1 \cap S_2 \cap S_3$.

The Parallel composition at the top of 
\emph{Play Ball Game}  combines \emph{Ball Tracker} which always returns Running, with the subtree discussed above.
The Parallel node has $M=1$, 
i.e.\ it only needs the Success of one child to return Success.
Thus, it is clear from Definition~\ref{bts:def:parallel} that the whole BT \emph{Play Ball Game} has the same
properties regarding FTS as the right subtree.

Finally, we note that  \emph{Play Ball Game} fails if the robot is not standing up.
Therefore, we improve the robustness of that subtree in a way similar to 
  Example~\ref{properties:ex:rob}  in  Figure~\ref{properties:fig:robust3}.
Thus we create the composition Fallback(Play Ball Game, $\bt_5$, $\bt_6$),
with 
$\bt_5=\mbox{Sit to Stand}$, 
$\bt_6=\mbox{Lie Down to Sit Up}$.

Assuming that that high dimensional dynamics of  \emph{Play Ball Game}
is somehow captured in the $x_1$ dimension we can apply an argument similar to 
Lemma~\ref{properties:lem:robustExample} to conclude that the combined BT is indeed also FTS
with 
completion time bound $\tau_1+\tau_2+\tau_3+\tau_5+\tau_6$,
region of attraction $R_1'\cup R_2'\cup R_3' \cup R_5' \cup R_6'$
and success region $S_1 \cap S_2 \cap S_3$.

The rest of the BT concerns user interaction and is thus not suitable for doing performance analysis.
\end{example}

Note that the assumption on all atomic actions being FTS is fairly strong. For example, the humanoid's grasping capabilities are somewhat unreliable. 
A deterministic analysis such as this one
is still  useful for making good design choices, but 
in order to capture the stochastic properties of a BT, we need the tools of Chapter~\ref{ch:stochastic}.

But first we will use the tools developed in this chapter to formally investigate how BTs relate to other control architectures. 

%% file: btsvsothers/btsvsothers.tex

\chapter{Formal Analysis of How Behavior Trees  Generalize Earlier Ideas}
\label{chap:btsvsothers}
\label{ch:btsvsothers}
\graphicspath{{btsvsothers/figures/}}
In this chapter, we will formalize the arguments of Chapter~\ref{ch:earlier_ideas}, using the tools developed in Chapter~\ref{ch:properties}.
In particular, we prove that BTs generalize Decision Trees (\ref{btsvsothers:sec:analogyDTs}), the Subsumptions Architecture (\ref{btsvsothers:sec:analogySA}), Sequential Behavior Compositions (\ref{btsvsothers:sec:analogySBCs}) and the Teleo-Reactive Approach (\ref{btsvsothers:sec:analogyTRs}).
Some of the results of this chapter were previously published in the journal paper \cite{colledanchise2017behavior}.

\section{How BTs Generalize Decision Trees}
\label{btsvsothers:sec:analogyDTs}

\begin{figure}[htbp]
\begin{center}
\includegraphics[width=10cm]{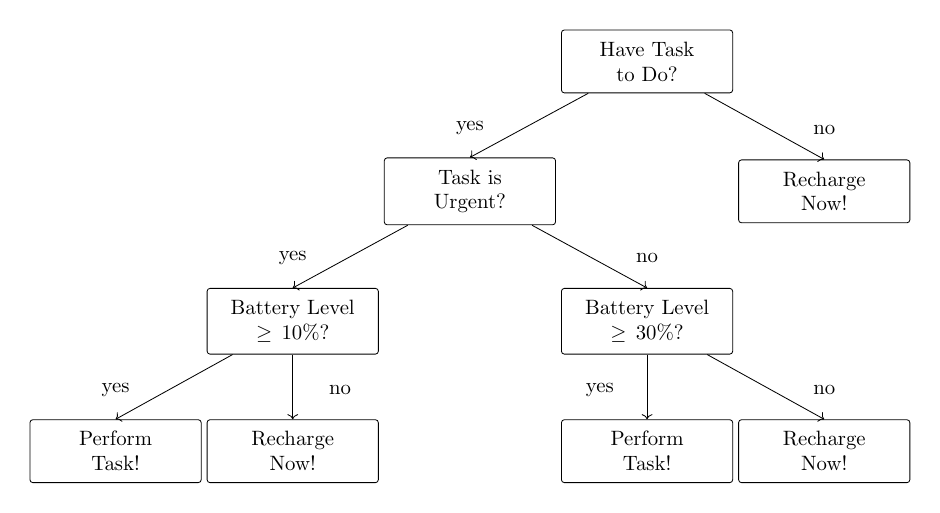}
\caption{The Decision Tree of a robot control system. The decisions are interior nodes, and the actions are leaves. }
\label{btsvsothers:fig:decisionTreeEx}
\end{center}
\end{figure}


Consider the Decision Tree of Figure  \ref{btsvsothers:fig:decisionTreeEx}, the robot has to decide whether to perform a given task or recharge its batteries. This decision is taken based upon the urgency of the task, and the current battery level.
The following Lemma shows how to create an equivalent BT from a given Decision Tree.

\begin{lemma}
\label{btsvsothers:lem:DTBT}
 Given a Decision Tree as follows
 \begin{equation}
 \label{btsvsothers:eq:DT}
DT_i = \begin{cases}
DT_{i1} & \mbox{if predicate $P_i$ is true} \\ 
DT_{i2} & \mbox{if predicate $P_i$ is false} \\ 
\end{cases}
\end{equation}
where 
$DT_{i1}$, $DT_{i2}$ are either atomic actions, or subtrees with identical structure, we can create an equivalent BT by setting
 \begin{equation}
 \bt_i=\mbox{Fallback}( \mbox{Sequence}(P_i,\bt_{i1}),\bt_{i2})
\end{equation}
for non-atomic actions, $ \bt_i = DT_{i}$ for atomic actions
 and requiring all actions to return Running all the time.
 
 The original Decision Tree and the new BT are equivalent in the sense that the same
 values for $P_i$ will always lead to the same atomic action being executed.
 The lemma is illustrated in Figure~\ref{btsvsothers:fig:decisionTreeEq}.
\end{lemma}
\begin{proof}
The BT equivalent of the Decision Tree is given by

\begin{equation*}
 \bt_i=\mbox{Fallback}( \mbox{Sequence}(P_i,\bt_{i1}),\bt_{i2})
\end{equation*}
For the atomic actions always returning Running we have  $r_i=R$,
for the actions being predicates we have that $r_i=P_i$.
This, together with Definitions \ref{bts:def.seq}-\ref{bts:def.fallb} gives that
\begin{equation}
f_i(x) = \begin{cases}
f_{i1} & \mbox{if predicate $P_i$ is true} \\ 
f_{i2} & \mbox{if predicate $P_i$ is false} \\ 
\end{cases}
\end{equation}
which is equivalent to (\ref{btsvsothers:eq:DT}).
\end{proof}

Informally, first we note that by requiring all actions to return Running, we basically disable the feedback functionality that is built into the BT. Instead whatever action that is ticked will be the one that executes, just as the Decision Tree.
Second the result is a direct consequence of the fact that the predicates of the Decision Trees are essentially `If ... then ... else ...' statements, that can be captured by BTs as shown in Figure \ref{btsvsothers:fig:decisionTreeEq}.

\begin{figure}[htbp]
\begin{center}
\includegraphics[width=\columnwidth]{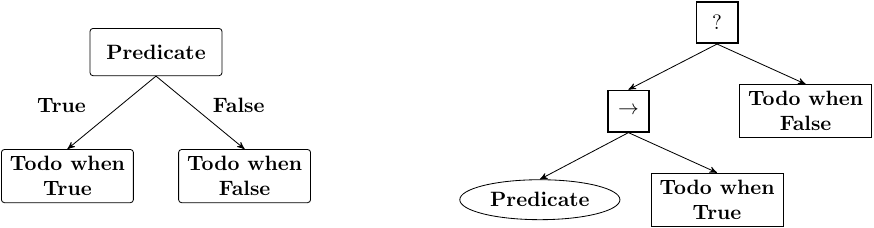}
\caption{The basic building blocks of Decision Trees are `If ... then ... else ...' statements (left), and those can be created in BTs as illustrated above (right). }
\label{btsvsothers:fig:decisionTreeEq}
\end{center}
\end{figure}

Note that this observation opens possibilities of using the extensive literature on learning Decision Trees from human operators, see e.g.\ \cite{sammut20027}, to create BTs. These learned BTs can then be extended with safety or robustness  features, as described in Section~\ref{properties:sec:propDef}.

\begin{figure}[htbp]
\begin{center}
\includegraphics[width=0.45\columnwidth]{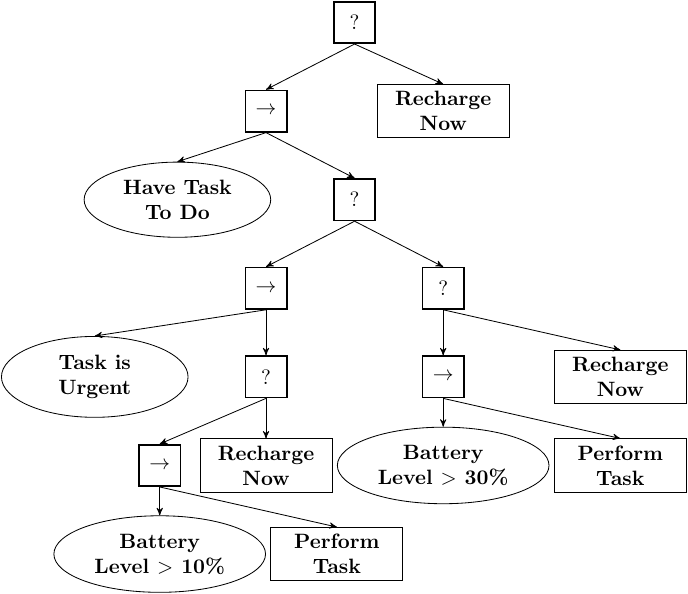}
\caption{A BT that is equivalent to the Decision Tree in Figure  \ref{btsvsothers:fig:decisionTreeEx}. 
A compact version of the same tree can be found in Figure \ref{btsvsothers:fig:decisionTreeBTopt}.
}
\label{btsvsothers:fig:decisionTreeBTeq}
\end{center}
\end{figure}

\begin{figure}[htbp]
\begin{center}
\includegraphics[width=0.45\columnwidth]{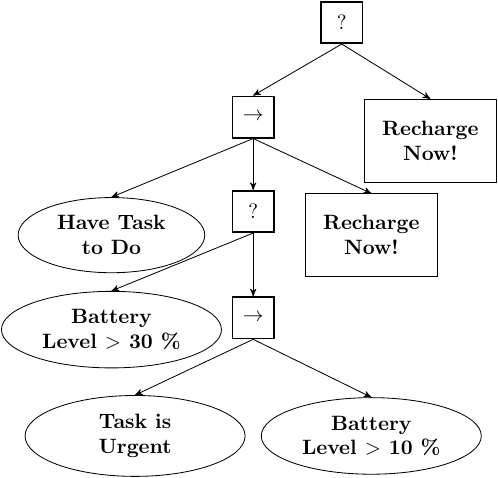}
\caption{A compact formulation of the BT in Figure  \ref{btsvsothers:fig:decisionTreeBTeq}. }
\label{btsvsothers:fig:decisionTreeBTopt}
\end{center}
\end{figure}

We finish this section with an example of how BTs generalize Decision Trees.
Consider the Decision Tree in Figure \ref{btsvsothers:fig:decisionTreeEx}. Applying Lemma \ref{btsvsothers:lem:DTBT} we get the equivalent BT of Figure~\ref{btsvsothers:fig:decisionTreeBTeq}. However the direct mapping does not always take full advantage of  the features of BTs. Thus a more compact, and still equivalent, BT can be found in Figure~\ref{btsvsothers:fig:decisionTreeBTopt}, where again, we assume that all actions always return \emph{Running}.

\section{How BTs Generalize the Subsumption Architecture}
\label{btsvsothers:sec:analogySA}
In this section, we will see how the
 Subsumption Architecture, proposed by Brooks~\cite{brooks1986robust}, can be realized using a Fallback composition.
The basic idea in  \cite{brooks1986robust} was to have a number of controllers set up in parallel and each controller was allowed to output both actuator commands, and a binary value, signaling if it wanted to control the robot or not. The controllers were then ordered according to some priority, and the highest priority controller, out of the ones signaling for action, was allowed to control the robot. Thus, a higher level controller was able to \emph{subsume} the actions of a lower level one.

\begin{figure}[htbp]
\begin{center}
\includegraphics[width=6cm]{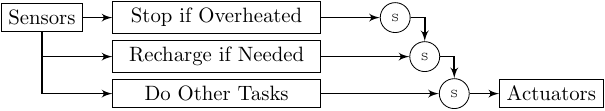}
\caption{The Subsumption Architecture. A higher level behavior can subsume (or suppress) a lower level one. }
\label{btsvsothers:fig:subsump}
\end{center}
\end{figure}

An example of a Subsumption architecture can be found in Figure  \ref{btsvsothers:fig:subsump}. Here, the basic level controller \emph{Do Other Tasks} is assumed to be controlling the robot for most of the time. However, when the battery level is low enough, the \emph{Recharge if Needed} controller will signal that it needs to command the robot, subsume the lower level controller, and guide the robot towards the recharging station. Similarly, if there is risk for overheating, the top level controller \emph{Stop if Overheated} will subsume both of the lower level ones, and stop the robot until it has cooled down.

\begin{lemma}
\label{btsvsothers:lem:sub}
Given a Subsumption architecture, we can create an equivalent BT by arranging the controllers as actions under a Fallback composition, in order from higher to lower priority. Furthermore, we let the return status of the actions be Failure if they do not need to execute, and Running if they do. They never return Success.
Formally, a subsumption architecture composition $S_i(x)=\mbox{Sub}(S_{i1}(x),S_{i2}(x))$ can be defined by
\begin{equation}
 \label{btsvsothers:eq:subEquiv}
S_i(x) = \begin{cases}
S_{i1}(x) & \mbox{if  $S_{i1}$ needs to execute} \\ 
S_{i2}(x) & \mbox{else} \\ 
\end{cases}
\end{equation}
Then we write an equivalent BT as follows
 \begin{equation}
 \bt_i=\mbox{Fallback}( \bt_{i1},\bt_{i2})
\end{equation}
where $ \bt_{ij}$ is defined by $f_{ij}(x)=S_{ij}(x)$ and
\begin{equation}
r_{ij}(x) = \begin{cases}
\mathcal{R} & \mbox{if  $S_{ij}$ needs to execute} \\ 
\mathcal{F} & \mbox{else}. \\ 
\end{cases}
\end{equation}

\end{lemma}
\begin{proof}
By the above arrangement, and Definition~\ref{bts:def.fallb} we have that
\begin{equation}
f_i(x) = \begin{cases}
f_{i1}(x) & \mbox{if  $S_{i1}$ needs to execute} \\ 
f_{i2}(x) & \mbox{else}, \\ 
\end{cases}
\end{equation}
which is equivalent to (\ref{btsvsothers:eq:subEquiv}) above.
In other words, actions will be checked in order of priority, until one that returns Running is found. 
\end{proof}

A BT version of the example in Figure  \ref{btsvsothers:fig:subsump} can be found in Figure  \ref{btsvsothers:fig:subsBT}. Table \ref{btsvsothers:subsumptionTable} illustrates how the two control structures are equivalent, listing all the $2^3$ possible return status combinations. Note that no action is executed if all actions return Failure.
\begin{figure}[htbp]
\begin{center}
\includegraphics[width=\columnwidth]{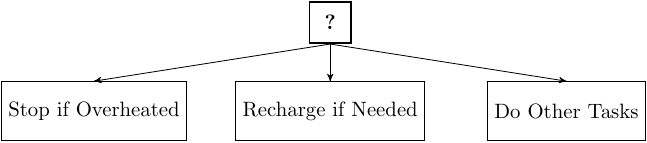}
\caption{A BT version of the Subsumption example in Figure  \ref{btsvsothers:fig:subsump}. }
\label{btsvsothers:fig:subsBT}
\end{center}
\end{figure}

\begin{table}[htp]
\begin{center}
\begin{tabular}{|p{1.8cm} |p{1.6cm}|p{1.6cm}|p{2.5cm}|p{2.5cm}|}
\hline
 \bf{Stop if overheated} & \bf{Recharge if Needed} & \bf{Do Other Tasks} & \bf{Action Executed} \cr
\hline
Running & Running & Running & Stop ...  \cr
 \hline
 Running & Running & Failure & Stop ...  \cr
 \hline
 Running & Failure & Running & Stop ...  \cr
 \hline
 Running & Failure & Failure & Stop ...  \cr
 \hline
 Failure & Running & Running & Recharge ...  \cr
 \hline
 Failure & Running & Failure & Recharge ...  \cr
 \hline
 Failure & Failure & Running & Do other ...  \cr
 \hline
 Failure & Failure & Failure & -  \cr \hline
\end{tabular}
\end{center}
\caption{Possible outcomes of Subsumption-BT example.}
\label{btsvsothers:subsumptionTable}
\end{table}%

\section{How BTs Generalize Sequential Behavior Compositions}
\label{btsvsothers:sec:analogySBCs}

In this section, we will see how the Fallback composition, and Lemma \ref{properties:lem:robustnessSelector}, can also be used to implement the Sequential Behavior Compositions proposed in   \cite{burridge1999sequential}.

The basic idea proposed by  \cite{burridge1999sequential}
is to extend the region of attraction by using a family of controllers,
where the asymptotically stable equilibrium of each controller was either the goal state, or inside the region of attraction of another controller, positioned earlier in the sequence. 


 We will now describe the construction of   \cite{burridge1999sequential} in some detail, and then see how this concept is captured in the BT framework. Given a family of controllers $U=\{\Phi_i\}$, we say that $\Phi_i$ \emph{prepares} $\Phi_j$ if the goal $G(\Phi_i)$ is inside the domain $D(\Phi_j)$. Assume the overall goal is located at $G(\Phi_1)$. A set of execution regions $C(\Phi_i)$ for each controller was then calculated according to the following scheme:

\begin{enumerate}
 \item Let a Queue contain $\Phi_1$. Let $C(\Phi_1)=D(\Phi_1)$, $N=1$, $D_1=D(\Phi_1)$.
 \item Remove the first element of the queue and append all controllers that \emph{prepare} it to the back of the queue.
 \item Remove all elements in the queue that already has a defined $C(\Phi_i)$.
 \item Let $\Phi_j$ be the first element in the queue. Let $C(\Phi_j)=D(\Phi_j) \setminus D_N$, $D_{N+1}=D_N \cup D(\Phi_j)$ and $N \leftarrow N+1$.
 \item Repeat steps 2, 3 and 4 until the queue is empty.
\end{enumerate}
 
The combined controller is then executed by finding $j$ such that $x \in C(\Phi_j)$ and then  invoking controller $\Phi_j$.
 
Looking at the design of the Fallback operator in BTs, it turns out that it does exactly the job of the Burridge algorithm above, as long as the subtrees of the Fallback are ordered in the same fashion as the queue above. We formalize this in  Lemma \ref{btsvsothers:lem:seqBComp} below.

\begin{lemma}
\label{btsvsothers:lem:seqBComp}
Given a set of controllers $U=\{\Phi_i\}$ we define the corresponding regions $S_i=G(\Phi_i), R_i'=D(\Phi_i), F_i=\mbox{Complement}(D(\Phi_i))$, and consider the controllers as atomic BTs, $\bt_i=\Phi_i$.
Assume $S_1$ is the overall goal region.
 Iteratively create a larger BT $\bt_L$ as follows
\begin{itemize}
 \item[1. ] Let  $\bt_L=\bt_1$. 
 \item[2. ] Find a BT $\bt_*\in U$ such that $S_* \subset R_L'$
 \item[3. ] Let $\bt_L \leftarrow \mbox{Fallback}(\bt_L,\bt_*)$
 \item[4. ] Let $U \leftarrow U \setminus \bt_*$
 \item[5. ] Repeat steps 2, 3 and 4 until $U$ is empty.
\end{itemize}
If all $\bt_i$ are FTS, then so is $\bt_L$.
\end{lemma}
\begin{proof}
The statement is a direct consequence of iteratively applying Lemma \ref{properties:lem:robustnessSelector}.
\end{proof}

Thus, we see that BTs generalize the Sequential Behavior Compositions of \cite{burridge1999sequential},
with the execution region computations and controller switching replaced by the Fallback composition,
as long as the ordering is given by Lemma~\ref{btsvsothers:lem:seqBComp} above.

\section{How BTs Generalize the Teleo-Reactive approach}
\label{btsvsothers:sec:analogyTRs}
In this section, we use the following 
Lemma to show how to create a BT with the same execution as a given Teleo-Reactive program. The lemma is illustrated by Example~\ref{btsvsothers:ex:goto} and Figure~\ref{btsvsothers:tr2bt}.

\begin{lemma}[Teleo-Reactive BT analogy]
\label{btsvsothers:lem:tr2bt}
 Given a TR in terms of conditions $c_i$ and actions $a_i$, an equivalent BT can be constructed as follows
\begin{equation}
 \bt_{TR}=\mbox{Fallback}(\mbox{Sequence}(c_1,a_1), \ldots, \mbox{Sequence}(c_m,a_m)),
\end{equation}
where we convert the True/False of the conditions to Success/Failure, and let the actions only return Running.
\end{lemma}
\begin{proof}
 It is straightforward to see that the BT above executes the exact same $a_i$ as the original TR would have, depending on the values of the conditions $c_i$, i.e. it finds the first condition $c_i$ that returns Success, and executes the corresponding $a_i$.
\end{proof}

\begin{figure}[htbp]
\begin{center}
\includegraphics[width=0.8\columnwidth]{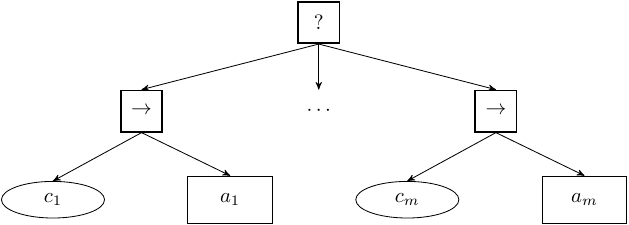}
\caption{The BT that is analogous to a given TR.}
\label{btsvsothers:tr2bt}
\end{center}
\end{figure}

We will now illustrate the lemma with an example from Nilssons original paper~\cite{nilsson1994teleo}.

\begin{example}
\label{btsvsothers:ex:goto}
 The Teleo-Reactive program \emph{Goto(loc)} is described as follows, with conditions on the left and corresponding actions to the right:
\begin{eqnarray}
 \mbox{Equal(pos,loc)} &\rightarrow& \mbox{Idle} \\
 \mbox{Heading Towards (loc)} &\rightarrow& \mbox{Go Forwards} \\
 \mbox{(else)} &\rightarrow& \mbox{Rotate} 
\end{eqnarray}
where \emph{pos} is the current robot position and \emph{loc} is the current destination.

Executing this Teleo-Reactive program, we get the following behavior. If the robot is at the destination it does nothing. If it is heading the right way it moves forward, 
and else it rotates on the spot. In a perfect world without obstacles, this will get the robot to the goal, just as predicted in Lemma~\ref{btsvsothers:lem:TR}.
Applying Lemma~\ref{btsvsothers:lem:tr2bt}, the Teleo-Reactive program Goto is translated to a BT in Figure~\ref{btsvsothers:fig:goto}.

\begin{figure}[htbp]
\begin{center}
\includegraphics[width=0.9\columnwidth]{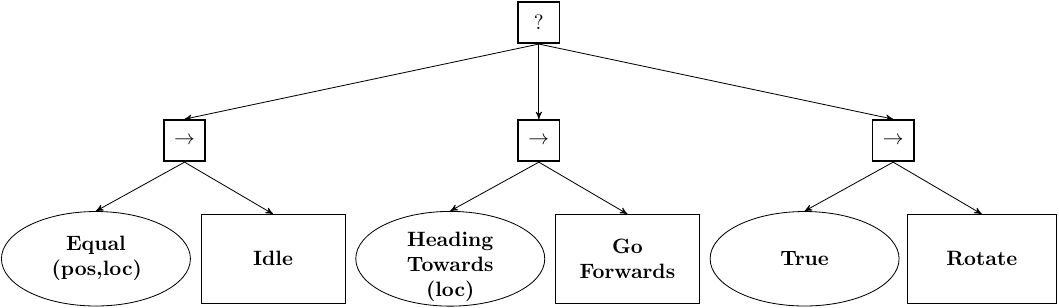}
\caption{The BT version of the Teleo-Reactive program Goto. }
\label{btsvsothers:fig:goto}
\end{center}
\end{figure}

The example continues in  \cite{nilsson1994teleo} with 
 a higher level recursive Teleo-Reactive program, called \emph{Amble(loc)}, designed to add a basic obstacle avoidance behavior
\begin{eqnarray}
 \mbox{Equal(pos,loc)} &\rightarrow& \mbox{Idle} \\
 \mbox{Clear Path(pos,loc)} &\rightarrow& \mbox{GoTo(loc)} \\
 \mbox{(else)} &\rightarrow& \mbox{Amble(new point(pos,loc))} 
\end{eqnarray}
where \emph{new point} picks a new random point in the vicinity of \emph{pos} and \emph{loc}.

Again,
if the robot is at the destination it does nothing. If the path to goal is clear it executes the Teleo-Reactive program Goto. 
Else it picks a new point relative to its current position and destination (loc) and recursively executes a new copy of Amble with that destination.
Applying Lemma~\ref{btsvsothers:lem:tr2bt}, the Amble TR is translated to a BT in Figure~\ref{btsvsothers:fig:amble}.

\begin{figure}[htbp]
\begin{center}
\includegraphics[width=0.9\columnwidth]{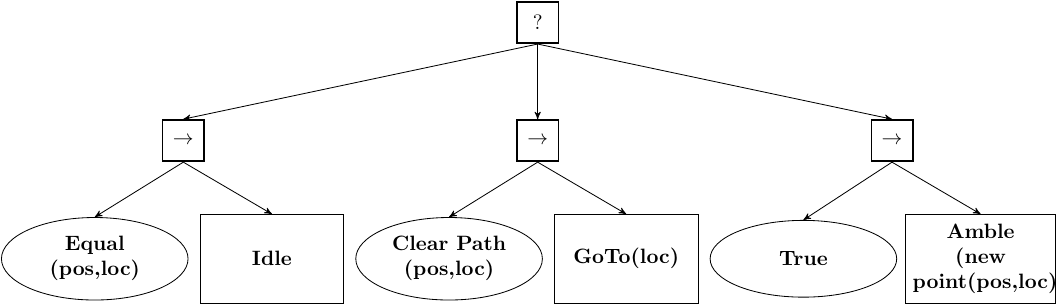}
\caption{The BT version of the TR Amble. }
\label{btsvsothers:fig:amble}
\end{center}
\end{figure}

\end{example}

\subsection{Universal Teleo-Reactive programs and FTS BTs}
\label{btsvsothers:sec:analogyTRs2}
Using the functional form of BTs introduced in~\ref{bts:sec:funcBT} we can show that Lemma~\ref{properties:lem:robustnessSelector} is a richer version of Lemma \ref{btsvsothers:lem:TR} below, and also fix one of its assumptions. Lemma~\ref{properties:lem:robustnessSelector} includes execution time, but more importantly builds on a finite difference equation system model over a continuous state space.
Thus control theory concepts can be used to include phenomena such as imperfect sensing and actuation into the analysis, that was
removed in the strong assumptions of Lemma~\ref{btsvsothers:lem:TR}.
Thus, the BT analogy provides a powerful tool for analyzing Teleo-Reactive designs.

\begin{lemma}[Nilsson 1994]
\label{btsvsothers:lem:TR}
 If a Teleo-Reactive program is Universal, and there are no sensing and execution errors, then the execution of the program will lead to
 the satisfaction of $c_1$.
\end{lemma}
\begin{proof}
In \cite{nilsson1994teleo} it is stated that it is easy to see that this is the case.
\end{proof}

The idea of the proof is indeed straight forward, but as we will see when we compare it to the BT results in Section~\ref{btsvsothers:sec:analogyTRs2} below, the proof is incomplete.

In Lemma~\ref{properties:lem:robustnessSelector},  $S_i,R_i,F_i$ correspond to Success, Running and Failure regions and $R'$ denotes the region of attraction.

Lemma~\ref{properties:lem:robustnessSelector} shows under what conditions we can guarantee that the Success region $S_0$ is reached in finite time.
If we for illustrative purposes assume that the regions of attraction are identical to the running regions $R_i=R_i'$,
the lemma states that as long as the system starts in $R_0'= R_1' \cup R_2'$ it will reach $S_0=S_1$ in less than $\tau_0=\tau_1+\tau_2$
time units. The condition analogous to the \emph{regression property} is that $S_2 \subset  R_1'$, i.e. that the Success region of the second BT
is a subset of the region of attraction $R_1'$ of the first BT. The regions of attraction, $R_1'$ and $R_2'$ are very important,
but  there is no corresponding concept in Lemma  \ref{btsvsothers:lem:TR}. 
In fact, we can construct a counter example showing that Lemma  \ref{btsvsothers:lem:TR} does not hold.

\begin{example}[Counter Example]
 Assume that a Teleo-Reactive program is Universal in the sense described above. 
 Thus, the execution of action $a_i$ eventually leads to the satisfaction of $c_j$ where $j<i$ for all $i\neq 1$.
 However, assume it is also the case that the execution of $a_i$, on its way towards satisfying $c_j$ actually leads to a violation of $c_i$.
 This would lead to the first true condition being some $c_m$, with $m>i$ and the execution of the corresponding action $a_m$.
 Thus, the chain of decreasing condition numbers is broken, and the goal condition $a_1$ might never be reached.
\end{example}

The fix is however quite straightforward, and amounts to using the following definition with a stronger assumption.

\begin{definition}[Stronger Regression property]
 For each $c_i, i>1$ there is $c_j, j<i$ such that the execution of action $a_i$ leads to the satisfaction of $c_j$, \emph{without ever violating} $c_i$.
\end{definition}

%% file: planning/planning.tex


\graphicspath{{./planning/figures/}}

\chapter{Behavior Trees and Automated Planning}
\label{ch:planning}

In this chapter, we describe how automatic planning can be used to create BTs, exploiting ideas from~\cite{florez2008dynamic, weber2011building, weber2010reactive, colledanchise2016towards}.
First, in Section~\ref{planning.robotics},  we present an extension of the Backchaining design principle, introduced in Section~\ref{design:sec:back_chaining}, including a robotics example. Then we present an alternative approach using A Behavior Language (ABL), in Section~\ref{planning.MS}, including a game example.

In classical planning research, the world is often assumed to be static and known, with
all changes occurring as a result of the actions executed by one controlled agent  \cite{ghallab2016automated}.
Therefore, most approaches return a static plan, i.e. a sequence of actions that brings the system from the initial state to the goal state, with a corresponding execution handled by a classical FSM. 

However, many agents, both real and virtual, act in an uncertain world populated by
 other agents, each with their own set of goals and objectives. 
  Thus, the effect of an action can be unexpected, diverging from the planned state trajectory, making the next planned action infeasible.
  A common way of handling this problem is to  re-plan from scratch on a regular basis, which can be expensive both in terms of time and computational load.
 To address these problems, the following two open challenges were
 identified within the planning community~\cite{Ghallab14}:
\begin{itemize}
\item \say{Hierarchically organized deliberation. This principle goes beyond existing hierarchical planning techniques; its requirements and scope are significantly different. The actor performs its deliberation online}

\item \say{Continual planning and deliberation. The actor monitors, refines, extends, updates, changes and repairs its plans throughout the acting process, using both descriptive and operational models of actions.}
\end{itemize}
Similarly, the recent book \cite{ghallab2016automated} describes the need for an actor that 
\say{reacts to events and extends, updates, and repairs its plan on the basis of its perception}.

Combining planning with BTs is one way of addressing these challenges.
 The reactivity of BTs
enables the agent to re-execute previous subplans without having
to replan at all, and the modularity enables extending the plan
recursively, without having to re-plan the whole task. 
Thus, using BTs as the control architecture in an automated planning algorithm  addresses the above challenges by enabling a reasoning process that is both hierarchical and modular in its deliberation, and can monitor, update, and extend its plans while acting.

In practice, and as will be seen in the examples below,
using BTs enables reactivity, in the sense that if an object slips out of a robot's gripper, the robot will automatically stop and pick it up again without the need to replan or change the BT, see Fig.~\ref{planning:SI.fig.yousimplescreen}.
Using BTs also enables iterative plan refinement, in the sense that if an object is moved to block the path, the original BT can be extended to include a removal of the blocking obstacle. Then, if the obstacle is removed by an external actor, the robot reactively skips the obstacle removal, and goes on to pick up the main object without having to change the BT, see Fig.~\ref{planning:DSI.fig.youscreen2bis}.

%
%

\section{The Planning and Acting (PA-BT) approach}
\label{planning.robotics}
In this section, we describe an extension of the Backchaining approach, called \emph{Planning and Acting using Behavior Trees} (PA-BT).

PA-BT was inspired by the Hybrid Backward-Forward (HBF) algorithm, a task planner for dealing with infinite state spaces~\cite{garrettbackward}. 
The HBF algorithm has been shown  to efficiently solve problems with large state spaces.
Using an HBF algorithm we can
 refine the acting process by mapping the \emph{descriptive} model of actions, which describes \emph{what} the actions do, onto an \emph{operational} model, which defines \emph{how} to perform an action in certain circumstances.

The PA-BT framework combines the planning capability in an infinite state space from HBF
with the advantages of BTs compared to FSMs in terms of \emph{reactivity} and \emph{modularity}.
Looking back at the example above, the \emph{reactivity} of BTs enables the robot to pick up a dropped object without 
having to replan at all. The \emph{modularity} enables extending the plan by adding actions for handling the blocking sphere, without having to
replan the whole task. Finally, when the sphere moves away, once again the \emph{reactivity} enables the correct execution without
changing the plan.
Thus, PA-BT is indeed hierarchical and modular in its deliberation, and it does monitor, update and extend its plans while acting,
addressing the needs described in~\cite{ghallab2016automated, Ghallab14}.
The interleaved plan-and-act process of PA-BT is similar to the one of \emph{Hierarchical Planning in the Now} (HPN)~\cite{kaelbling2011hierarchical} with the addition of improved reactivity, safety and fault-tolerance.

%

The core concept in BT Backchaining was to replace a condition by a small BT achieving that condition,
on the form of a PPA BT (see Section~\ref{design:sec:back_chaining}), as shown in Figure~\ref{planning:fig:back_chaining_general}.

\begin{figure}[h]
\centering
\includegraphics[width=\columnwidth]{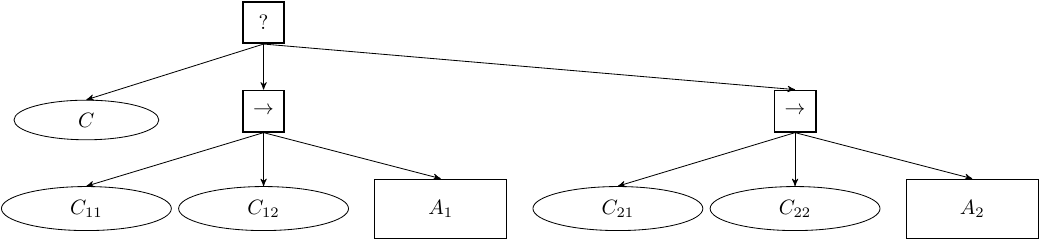}
\caption{Copy of Figure~\ref{design:fig:back_chaining_general}. The general format of a PPA BT. The Postcondition $C$ can be achieved by either one of actions $A_1$ or $A_2$, which have Preconditions $C_{1i}$ and $C_{2i}$ respectively.)}
\label{planning:fig:back_chaining_general}
\end{figure}


To get familiar with PA-BT, we look at a simple planning example. 
The planning algorithm iteratively creates the BTs in Figure~\ref{planning:PA.fig.it1to4} 
with the final BT used  in the example in            
 Figure~\ref{planning:PA.fig.it4}. The setup is shown in Figure~\ref{planning:IN.fig.front}.
 We can see how  each of the BTs in the figure is the result of applying the PPA expansion to a condition in the previous BT.
 However, a number of details needs to be taken care of, as explained in sections below.

\begin{figure}[t]

        \centering
          \begin{subfigure}[b]{0.25\columnwidth}
                \centering
\includegraphics[width=1\columnwidth]{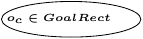}
                \caption{The initial BT.}
                \label{planning:PA.fig.it0}
        \end{subfigure}%
        
        \begin{subfigure}[b]{0.5\columnwidth}
                \centering
\includegraphics[width=1\columnwidth]{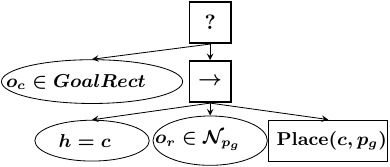}
                \caption{The BT after one iteration.}
                \label{planning:PA.fig.it1}
        \end{subfigure}%
       ~ 
        \begin{subfigure}[b]{0.5\columnwidth}
                \centering
\includegraphics[width=1\columnwidth]{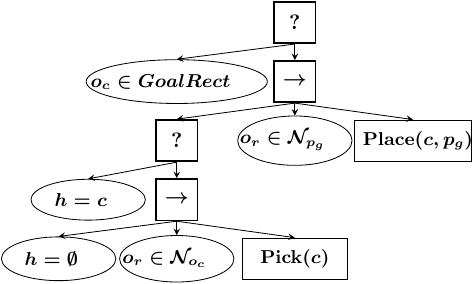}
                \caption{The BT after two iterations.}
                \label{planning:PA.fig.it2}              
        \end{subfigure}
        ~ 

        \centering
        \begin{subfigure}[b]{0.45\columnwidth}
                \centering
\includegraphics[width=1\columnwidth]{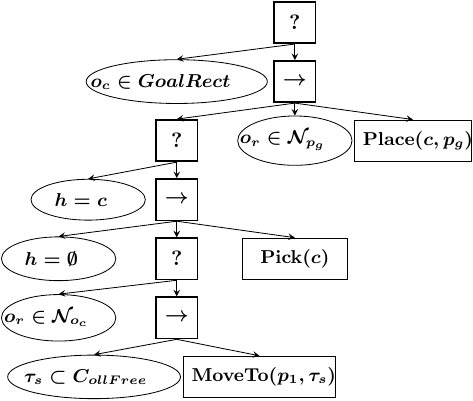}
                \caption{The BT after three iterations.}
                \label{planning:PA.fig.it3}
        \end{subfigure}%
       ~ 
        \begin{subfigure}[b]{0.55\columnwidth}
                \centering
\includegraphics[width=1\columnwidth]{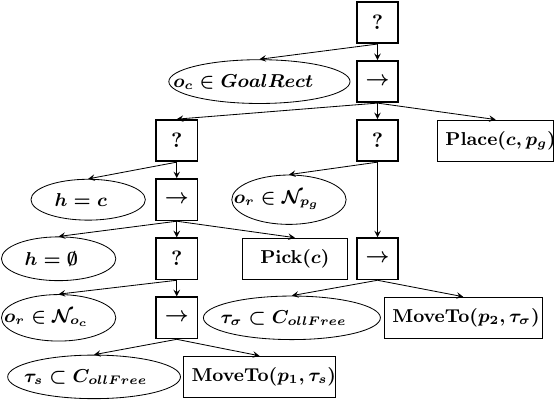}
                \caption{The BT after four iterations, the final version}
                \label{planning:PA.fig.it4}              
        \end{subfigure}
        
        \caption{BT updates during the execution.}
         \label{planning:PA.fig.it1to4}              
        ~ 

\end{figure}

\begin{figure}[t!]
    \centering
    \begin{subfigure}[t]{0.485\columnwidth}
        \centering
\includegraphics[width = \columnwidth]{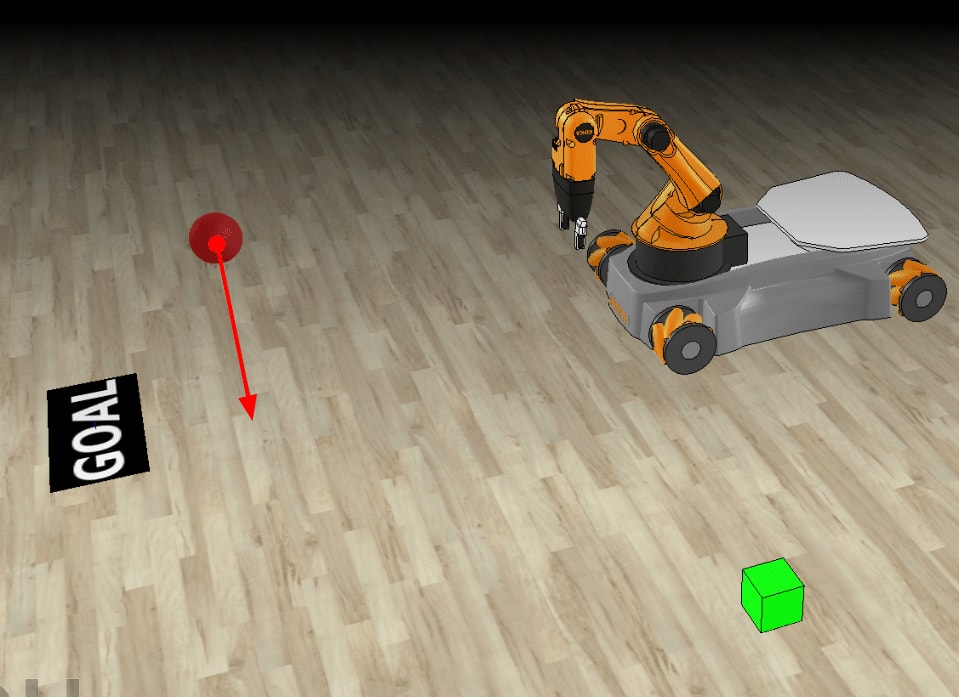}
         \caption{}
    \end{subfigure}%
    ~ 
    \begin{subfigure}[t]{0.485\columnwidth}
        \centering
\includegraphics[width = \columnwidth]{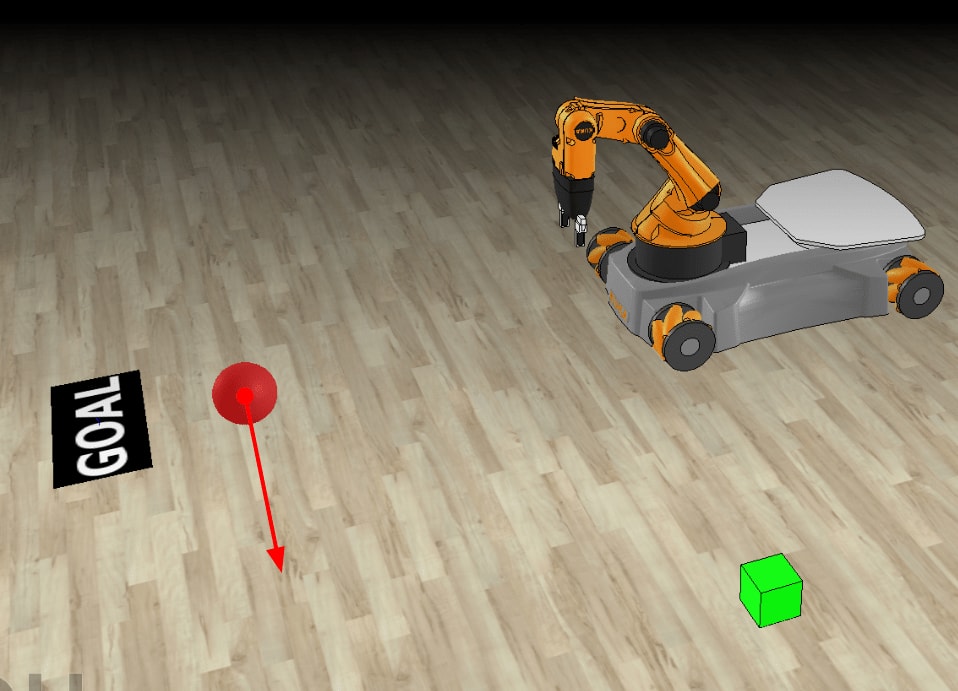}
   \caption{}
    \end{subfigure}
    \vspace{1em}
    \caption{
    A simple example scenario where the
    goal is to place the green cube $C$ onto the goal region $G$. The fact that the sphere $S$  is suddenly moved (red arrow) by an external agent to  block the path must be handled.
    In (a) the nominal plan is to \emph{MoveTo(c)$\to$Pick(c)$\to$MoveTo(g)$\to$Drop()} when the sphere suddenly moves to block the path. 
    In (b), after refining the plan, the extended plan is to \emph{MoveTo(s)$\to$Push(s)$\to$MoveTo(c)} 
        \emph{$\to$Pick(c)$\to$MoveTo(g)$\to$Drop()} when the sphere is again suddenly moved by another agent, before being pushed. Thus our agent must smoothly revert to the original set of actions. PA-BT does this without re-planning. Note that when $S$ appears, $\tau_\sigma \subset C_{collFree}$ returns Failure and the BT in Figure~\ref{planning:PA.fig.it4} is expanded further, see Example~\ref{planning:PA.exa.complex},  to push it out of the way. }
    \label{planning:IN.fig.front}
\end{figure}

\begin{example}
\label{planning:PA.exa.simple}

The robot in Figure~\ref{planning:IN.fig.front} is given the task to move the green cube into the rectangle marked GOAL
(the red sphere is handled in Example~\ref{planning:PA.exa.complex} below, in this inital example it is ignored).
The BT  in Figure~\ref{planning:PA.fig.it4} is executed, and in  each time step the root of the BT is ticked. The root is a Fallback node, which ticks is first child, the condition $o_c \in GoalRect$ (cube on goal). If the cube is indeed in the rectangle we are done, and the BT returns Success.

If not, the second child, a Sequence node, is ticked. The node ticks its first child, which is a Fallback, which again ticks its first child, the condition 
$h = c$ (object in hand is cube). If the cube is indeed in the hand, the Condition node returns Success, its parent, the Fallback node returns Success, and its parent, the Sequence node ticks its second child, which is a different Fallback, ticking its first child which is the condition $o_r \in \mathcal{N}_{p_g}$ (robot in the neighborhood of $p_g$). If the robot is in the neighborhood of the goal, the condition and its parent node (the Fallback) returns Success, followed by the sequence ticking its third child, the action $Place(c,p_g)$ (place cube in a position $p_g$ on the goal), and we are done.

If $o_r \in \mathcal{N}_{p_g}$ does not hold, the action $MoveTo(p_g,\tau_g)$ (move to position $p_g$ on the goal region using the trajectory $\tau_g$) is executed, given that the trajectory is free $\tau \subset C_{ollFree}$.
Similarly, if the cube is not in the hand, the robot does a $MoveTo(p_c,\tau_c)$ (move to cube, using the trajectory $\tau_c$) followed by a $Pick(c)$ after checking that the 
hand is empty, the robot is not in the neighborhood of $c$ and that the
corresponding trajectory is free.

We conclude the example by noting that the BT is ticked every timestep, e.g. every 0.1 second. Thus, when actions return Running (i.e. they are not finished yet) the return status of Running is progressed up the BT and the corresponding action is allowed to control the robot. However, if e.g., the cube slips out of the gripper, the condition $h = c$ instantly returns Failure, and the robot starts checking if it is in the neighborhood of the cube or if it has to move before picking it up again.
\end{example}

We are now ready to study PA-BT in detail. The approach is described in Algorithms~\ref{planning:PA.alg.main} (finding what condition to replace with a PPA) and~\ref{planning:PA.alg.update} (creating the PPA and adding it to the BT).
First we will give an overview of the algorithms and see how they are applied to the robot in Figure~\ref{planning:IN.fig.front},
to iteratively create the BTs of Figure~\ref{planning:PA.fig.it1to4}.
We will then discuss the key steps in more detail.

\begin{algorithm2e}[h!]
\caption{Main Loop, finding conditions to expand and resolving conflicts} 
 \label{planning:PA.alg.main}
\DontPrintSemicolon
%
\SetKwProg{myalg}{algorithm2e}{}{}
 $\mathcal{T} \gets \emptyset$ \\
\For{$c$ in $\mathcal{C}_{goal}$} {
 $\mathcal{T} \gets $SequenceNode($\mathcal{T}$, $c$)}
\While{True\label{planning:PA.alg.main.while}}{
$T\gets$\FuncSty{RefineActions($\bt$)} \label{planning:PA.alg.main.refine}\\
  \Do{\label{planning:PA.alg.main.do} $r \neq \mbox{\emph{Failure}}$ \label{planning:PA.alg.main.fail}} {
    $r \gets$ \FuncSty{Tick($T$)} \label{planning:PA.alg.main.fail2}}
 $c_f \gets$ \FuncSty{GetConditionToExpand($\bt$)}  \label{planning:PA.alg.expand.getcon}\\
 $\mathcal{T}, \mathcal{T}_{new\_subtree}\gets$ \FuncSty{ExpandTree($\bt$,$c_f$)}\label{planning:PA.alg.main.expand} \\
   \While{ $ \mbox{Conflict}(\bt)$  \label{planning:PA.alg.main.feas}} {  $\mathcal{T}\gets$ \FuncSty{IncreasePriority($\mathcal{T}_{new\_subtree}$)} \label{planning:PA.alg.main.incprio} } }
 
\end{algorithm2e}

\begin{algorithm2e}[h]
\caption{Behavior Tree Expansion, Creating the PPA}
      \label{planning:PA.alg.update}
  \SetKwFunction{algo}{ExpandTree}
  \SetKwFunction{proc}{proc}
  \SetKwProg{myalg}{Function}{}{}
  \myalg{\algo{$\mathcal{T}$, $c_f $}}{


 $A_T \gets $ \FuncSty {GetAllActTemplatesFor($c_f$) \label{planning:PA.alg.expand.getact}}  \\
 	$\mathcal{T}_{fall} \gets c_f$\\ 
\For{ $a$ in $A_T$} {
	$\mathcal{T}_{seq} \gets \emptyset$\\
\For{ $c_a$ in $a.con$}{
 $\mathcal{T}_{seq} \gets $ SequenceNode($\mathcal{T}_{seq}$,$c_a$)   \\  
}
 $\mathcal{T}_{seq} \gets $ SequenceNode($\mathcal{T}_{seq}$,$a$) \\ 
 $\mathcal{T}_{fall} \gets $ FallbackNodeWithMemory($\mathcal{T}_{fall}$,$\mathcal{T}_{seq}$) \\ 
}

 $\mathcal{T} \gets$  Substitute($\mathcal{T}$,$c_f$,$\mathcal{T}_{fall}$)\\
 \Return{$\mathcal{T}$, $\mathcal{T}_{fall}$}
 }
 
 \end{algorithm2e}

\begin{algorithm2e}[h]

\caption{Get Condition to Expand}
      \label{planning:PA.alg.getcond}
  \SetKwFunction{algo}{GetConditionToExpand}
  \SetKwFunction{proc}{proc}
  \SetKwProg{myalg}{Function}{}{}
  \myalg{\algo{$\mathcal{T}$}}{
	\For{$c_{next} $ in $\FuncSty{GetConditionsBFS()}$}
	{
		\If{$c_{next}.status =\mbox{\emph{Failure}}$ \textbf{and} $c_{next} \notin \mbox{ExpandedNodes}$ \label{planning:PA.alg.getcond.notexpanded}}
		{
		\mbox{ExpandedNodes}.\FuncSty{push\_back($c_{next}$)}
			\Return{$c_{next}$}		
		}
	} 
	\Return{$None$}\label{planning:PA.alg.getcond.nocond}
 }

\end{algorithm2e}

{
\begin{remark}
Note that the conditions of an action template can contain a disjunction of propositions. This can be encoded by a Fallback composition of the corresponding Condition nodes.
\end{remark}
}

\subsection{Algorithm Overview}
Running Algorithm~\ref{planning:PA.alg.main} we have the set of goal constraint  $\mathcal{C}_{goal} = \{o_c \in \{\mbox{GoalRect}\} \}$, thus the initial BT
is composed of a single condition
 $\bt = (o_c \in \{\mbox{GoalRect}\})$, as shown in Figure~\ref{planning:PA.fig.it0}.
 The first iteration of the loop starting on Line~\ref{planning:PA.alg.main.while} of Algorithm~\ref{planning:PA.alg.main} now produces the next BT shown in Figure~\ref{planning:PA.fig.it1}, the second iteration produces the BT in Figure \ref{planning:PA.fig.it2} and so on until the final BT in Figure \ref{planning:PA.fig.it4}.

In detail, at the initial state, running $\bt$ on Line 7 
returns a Failure, since the cube is not in the goal area. Trivially, the \emph{GetConditionToExpand} returns  $c_f=(o_c \in \{\mbox{GoalRect}\})$,
and a call to ExpandTree (Algorithm~\ref{planning:PA.alg.update}) is made on Line \ref{planning:PA.alg.main.expand}. 
On Line~2 of Algorithm~\ref{planning:PA.alg.update} we get  all action templates that satisfy $c_f$ i.e. $A_T=Place$.
Then on Line 7 and 8 a Sequence node $\bt_{seq}$ is created of the conditions of $Place$ (the hand holding the cube, $h = c$, and the robot being near the goal area, $o_r \in \mathcal{N}_{p_g}$) and $Place$ itself.
On Line 9 a Fallback node $\bt_{seq}$ is created of $c_f$ and the sequence above.
Finally, a BT is returned where this new sub-BT is replacing $c_f$.
The resulting BT is shown in Figure~\ref{planning:PA.fig.it1}. 

Note that Algorithm~\ref{planning:PA.alg.update} describes \emph{the core principle} of the PA-BT approach. \emph{A condition is replaced by the corresponding PPA, a check if the condition is met, and an action 
to meet it. The action is only executed if needed.} If there are several such actions, these are added in Fallbacks {with memory}. Finally, the action is preceded by conditions checking its own preconditions. If needed,
these conditions will be expanded in the same way in the next iteration.

Running the next iteration of Algorithm~\ref{planning:PA.alg.main}, a similar expansion of the condition $h = c$ transforms the BT in Figure~\ref{planning:PA.fig.it1} to the BT in Fig.~\ref{planning:PA.fig.it2}.
Then, an expansion of the condition $o_r \in \mathcal{N}_{o_c}$ transforms the BT in Figure~\ref{planning:PA.fig.it2} to the BT in Figure~\ref{planning:PA.fig.it3}.
Finally, an expansion of the condition $o_r \in \mathcal{N}_{p_g}$ transforms the BT in Figure~\ref{planning:PA.fig.it3} to the BT in Figure~\ref{planning:PA.fig.it4},
and this BT is able to solve the problem shown in Figure~\ref{planning:IN.fig.front}, until the red sphere shows up.
Then additional iterations are needed.

\subsection{The Algorithm Steps in Detail}

\paragraph*{\textbf{Refine Actions (Algorithm~\ref{planning:PA.alg.main} Line~\ref{planning:PA.alg.main.refine})}\\}

PA-BT is based on the definition of the \emph{action templates}, which contains the descriptive model of an action. An action template is characterized by conditions \emph{con} (sometimes called preconditions) and effects \emph{eff} (sometimes called postconditions) that are both constraints on the world (e.g. door open, robot in position). An action template is mapped online into an \emph{action primitive}, which contains the operational model of an action and is executable. Figure~\ref{planning.pabt.fig.acttemplate} shows an example of an action template and its corresponding action refinement.

To plan in infinite state space, PA-BT relies on a so-called Reachability Graph (RG) provided by the HBF algorithm, see \cite{garrettbackward} for details.
The RG provides efficient sampling for the actions in the BT, allowing us to map the descriptive model of an action into its operational model. 

%
%
%
%
%

\begin{figure}
\centering
          \begin{subfigure}[b]{0.35\columnwidth}
\begin{equation*}
\begin{aligned}
&\mbox{Pick}(i)   \\   
&\mbox{con}: o_r \in \mathcal{N}_{o_i} \hspace{2em} \\ 
&\hspace{2.2em}h  =  \emptyset \\	  
&\mbox{eff}:\hspace{0.4em}  h = i \hspace{1em} 
\end{aligned}
\end{equation*}
\caption{Action Template for picking a generic object denoted $i$.\\}
\end{subfigure}
~
\begin{subfigure}[b]{0.35\columnwidth}
\begin{equation*}
\begin{aligned}
&\mbox{Pick}(cube)   \\   
&\mbox{con}: o_r \in \mathcal{N}_{o_{cube}} \hspace{2em} \\ 
&\hspace{2.2em}h  =  \emptyset \\	  
&\mbox{eff}:\hspace{0.4em} h = cube \hspace{1em} 
\end{aligned}
\end{equation*}
\caption{Action primitive created from the Template in (a), where the object is given as $i= cube$.}
\end{subfigure}
\caption{\emph{Action Template} for Pick and its corresponding \emph{Action primitive}. $o_r$ is the robot's position, $\mathcal{N}_{o_i}$ is a set that defines a neighborhood of the object $o_i$, $h$ is the object currently in the robot's hand. The conditions are that the robot is in the neighborhood of the object, and that the robot hand is empty. The effect is that the object is in the robot hand.}
\label{planning.pabt.fig.acttemplate}
\end{figure}

\paragraph*{\textbf{Get Condition To Expand and Expand Tree (Algorithm~\ref{planning:PA.alg.main} Lines~9 and~10) }\\}

If the BT returns Failure, Line~\ref{planning:PA.alg.expand.getcon} of Algorithm~\ref{planning:PA.alg.main} invokes Algorithm~\ref{planning:PA.alg.getcond}, which finds the condition to expand by searching through the conditions returning Failure using a Breadth First Search~(BFS). If no such condition is found (Algorithm~\ref{planning:PA.alg.getcond} Line~\ref{planning:PA.alg.getcond.nocond}) that means that an action returned Failure due to an old refinement that is no longer valid. In that case, at the next loop of Algorithm~\ref{planning:PA.alg.main} a new refinement is found (Algorithm~\ref{planning:PA.alg.main} Line~\ref{planning:PA.alg.main.refine}), assuming that such a refinement always exists. If Algorithm~\ref{planning:PA.alg.getcond} returns a condition, this will be expanded (Algorithm~\ref{planning:PA.alg.main} Line 10), as shown in the example of Figure~\ref{planning:PA.fig.it1to4}. Example~\ref{planning:PA.ex.graph}  highlights the BFS expansion.
Thus, $\bt$ is expanded until it can perform an action (i.e. until $\bt$ contains an action template whose condition are satisfied in the current state). In~\cite{colledanchise2016towards} is proved that $\bt$ is expanded a finite number of times.
If there exists more than one valid action that satisfies a condition, their respective trees (sequence composition of the action and its conditions) are collected in a Fallback composition {with memory}, which implements the different options the agent has, to satisfy such a condition. {If needed, these options will be executed in turn.}
PA-BT does not investigate which action is the optimal one. As stressed in~\cite{ghallab2016automated} the cost of minor mistakes (e.g. non-optimal actions order) is often much lower than the cost of the extensive modeling, information gathering and thorough deliberation needed to achieve optimality.

\paragraph*{\textbf{Conflicts and Increases in Priority (Algorithm~\ref{planning:PA.alg.main} Lines~\ref{planning:PA.alg.main.feas} and~\ref{planning:PA.alg.main.incprio})}\\}
Similar to any STRIPS-style planner, adding a new action  in the plan can cause a \emph{conflict} (i.e. the execution of this new action reverses the effects of a previous action). 
In PA-BT, this possibility is checked in Algorithm~\ref{planning:PA.alg.main}  Line~\ref{planning:PA.alg.main.feas} by analyzing the conditions of the new action added with the effects of the actions that the subtree executes before executing the new action. If this effects/conditions pair is in conflict, the goal will not be reached.  An example of this situation is described in Example~\ref{planning:PA.exa.complex} below.


Again, following the approach used in STRIPS-style planners, we resolve this conflict by finding the correct action order. Exploiting the structure of BTs we can do so by moving the tree composed by the new action and its condition leftward (a BT executes its children from left to right, thus moving a subtree leftward implies executing the new action earlier). If it is the leftmost one, this means that it must be executed before its parent (i.e. it must be placed at the same depth of the parent but to its left). This operation is done in Algorithm~\ref{planning:PA.alg.main}  Line~\ref{planning:PA.alg.main.incprio}. PA-BT incrementally increases the priority of this subtree in this way, until it finds a feasible tree. In~\cite{colledanchise2016towards} it is proved that, under certain assumptions, a feasible tree always exists .

\paragraph*{\textbf{Get All Action Templates}\\}
Let's look again at Example~\ref{planning:PA.exa.simple} above and see how the BT in Figure~\ref{planning:PA.fig.it4} was created using the PA-BT approach. In this example, the action templates are summarized below with conditions and effect:
\begin{equation*}
\begin{aligned}
&\mbox{MoveTo}(p,\tau)   \\   
&\mbox{con}:\tau \subset C_{ollFree} \hspace{2em}  \\ 
& \\ 
&\mbox{eff}:o_r = p
\end{aligned}
\begin{aligned}
&\mbox{Pick}(i)   \\   
&\mbox{con}: o_r \in \mathcal{N}_{o_i} \hspace{2em} \\ 
	 &\hspace{2.3em}h  =  \emptyset \\	  
&\mbox{eff}: h = i \hspace{1em} 
\end{aligned}
\begin{aligned}
&\mbox{Place}(i,p)   \\   
&\mbox{con}: o_r \in \mathcal{N}_{p} \hspace{2em}  \\ 
	 &\hspace{2.3em}h = i  \\	  
&\mbox{eff}: o_i = p \hspace{1em} 
\end{aligned}
\end{equation*}
where $\tau$ is a trajectory, $C_{ollFree}$ is the set of all collision free trajectories,  $o_r$ is the robot pose, $p$ is a pose in the state space, $h$ is the object currently in the end effector, $i$ is the label of the $i$-th object in the scene, and  $\mathcal{N}_{x}$ is the set of all the poses near the pose $x$. 


The descriptive model of the action \emph{MoveTo}  is parametrized over the destination $p$ and the trajectory $\tau$. It requires that the trajectory is collision free ($\tau \subset C_{ollFree}$). As effect the action \emph{MoveTo} places the robot at $p$ (i.e. $o_r = p$), the descriptive model of the action \emph{Pick} is parametrized over object $i$. It requires having the end effector free (i.e. $h  =  \emptyset$) and the robot to be in a neighborhood $\mathcal{N}_{o_i}$ of the object $i$. (i.e. $o_r \in \mathcal{N}_{o_i}$). As effect the action Pick sets the object in the end effector to $i$ (i.e $h = i$); 
Finally, the  descriptive model of the action \emph{Place} is parametrized over object $i$ and final position $p$. It requires the robot to hold $i$ (i.e. $h = i$), and the robot to be in the neighborhood of the final position $p$. As effect the action \emph{Place} places the object $i$ at $p$ (i.e. $o_i = p$).

\begin{example}
\label{planning:PA.exa.complex}

Here we show a more complex example highlighting two main properties of PA-BT: the livelock freedom and the continual deliberative plan and act cycle. This example is an extension of Example~\ref{planning:PA.exa.simple} where, due to the dynamic environment, the robot needs to replan.


Consider the execution of the final BT in Figure~\ref{planning:PA.fig.it4} of Example \ref{planning:PA.exa.simple}, where the robot is carrying the desired object to the goal location. Suddenly, as in Figure~\ref{planning:IN.fig.front}~(b), an object $s$ obstructs the (only possible) path. Then the condition $\tau \subset C_{ollFree}$  returns Failure and Algorithm \ref{planning:PA.alg.main} expands the tree accordingly (Line \ref{planning:PA.alg.main.expand}) as in Figure~\ref{planning:PA.fig.it7}.

The new subtree has as condition $h = \emptyset$ (no objects in hand) but the effect of the left branch (i.e. the main part in Figure \ref{planning:PA.fig.it4}) of the BT is $h = c$ (cube in hand) (i.e. the new subtree will be executed if and only if $h = c$ holds). Clearly  the expanded tree has a conflict (Algorithm~\ref{planning:PA.alg.main} Line \ref{planning:PA.alg.main.feas}) and the priority of the new subtree is increased  (Line~\ref{planning:PA.alg.main.incprio}), until the expanded tree is in form of Figure~\ref{planning:PA.fig.it8}. Now the BT is free from conflicts as the first subtree has as effect $h = \emptyset$  and the second subtree has a condition $h = \emptyset$. Executing the tree the robot approaches the obstructing object, now the condition $h = \emptyset$ returns Failure and the tree is expanded accordingly, letting the robot drop the current object grasped, satisfying $h = \emptyset$, then it picks up the obstructing object and places it on the side of the path. Now the condition $\tau \subset C_{ollFree}$ finally returns Success. The robot can then again approach the desired object and move to the goal region and place the object in it.  

\end{example}

\begin{figure}[h!]
        \centering
        \begin{subfigure}[b]{0.5\columnwidth}
                \centering
\includegraphics[width=1\columnwidth]{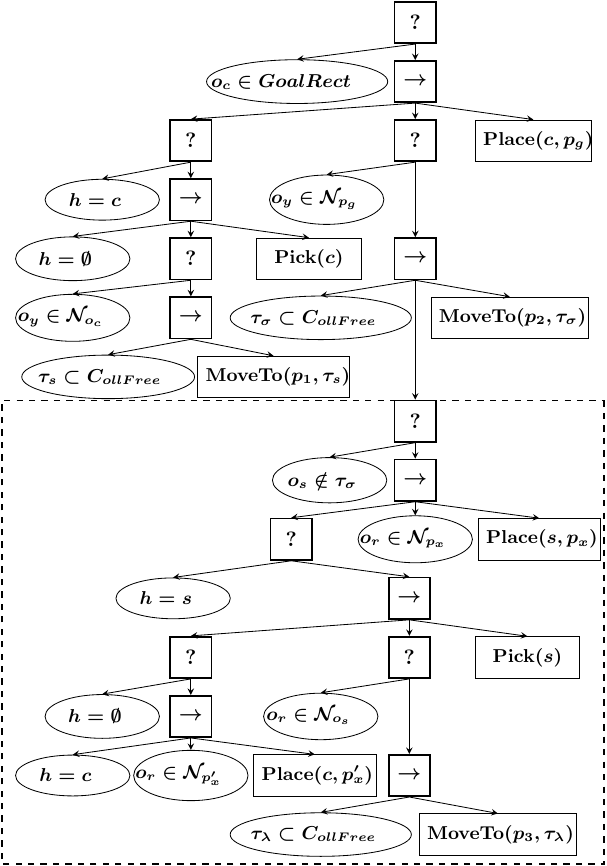}
                \caption{Unfeasible expanded tree. The new subtree is highlighted in red.}
                \label{planning:PA.fig.it7}
        \end{subfigure}\\
       ~ 
        \begin{subfigure}[b]{1\columnwidth}
                \centering
\includegraphics[width=1\columnwidth]{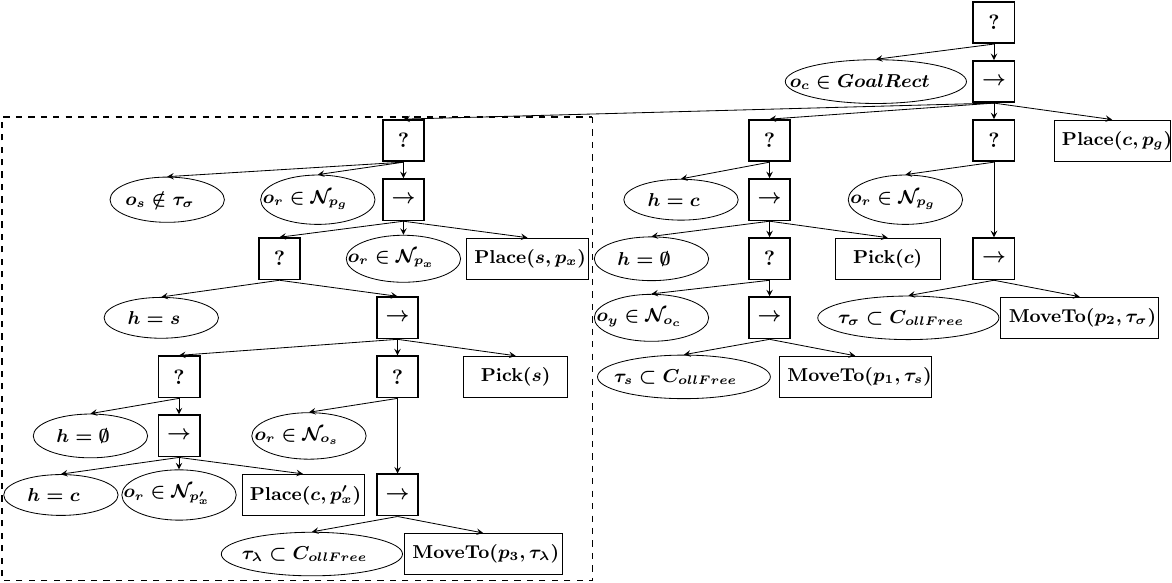}
                \caption{Expanded Feasible subtree.}
                \label{planning:PA.fig.it8}              
        \end{subfigure}
        ~ 
        \caption{Steps to increase the priority of the new subtree added in Example~\ref{planning:PA.exa.complex}.}
        ~ 
\end{figure}

\subsection{Comments on the Algorithm}

It is clear that this type of continual planning and acting exhibits both the important principles of deliberation stressed in \cite{ghallab2016automated, Ghallab14}: Hierarchically organized deliberation and continual online deliberation. For example, if the robot drops the object, then the condition $h = c$ is no longer satisfied and the BT will execute the corresponding subtree to pick the object, with no need for re-planning. This type of deliberative reactiveness is built into BTs.  On the other hand, if during its navigation a new object pops up obstructing the robot's path, the condition $\tau \subset C_{ollFree}$  will no longer return Success and the BT will be expanded accordingly. This case was  described  in Example~\ref{planning:PA.exa.complex}. Moreover, note that PA-BT refines the BT every time it returns Failure. This is to encompass the case where an older refinement is no longer valid. Is such cases an action will return Failure. This Failure is propagated up to the root. The function \emph{ExpandTree} (Algorithm~\ref{planning:PA.alg.main} Line~\ref{planning:PA.alg.main.expand}) will return the very same tree (the tree needs no extension as there is no failed condition of an action) which gets re-refined in the next loop (Algorithm~\ref{planning:PA.alg.main} Line~\ref{planning:PA.alg.main.refine}). For example, if the robot planned to place the object in a particular position on the desk but  this position was no longer feasible (e.g. another object was placed in that position by an external agent).

\subsection{Algorithm Execution on Graphs}
Here, for illustrative purposes, we show the result of PA-BT when applied to a standard shortest path problem in a graph.

\begin{example}
\label{planning:PA.ex.graph}
Consider an agent moving in different states modeled by the graph in Figure~\ref{planning:PA.fig.graph} where the initial state is $s_0$ and the goal state is $s_g$. Every arc represents an action that moves an agent from one state to another. The  action that moves the agent from a state $s_i$ to a state $s_j$ is denoted by $s_i \to s_j$. 
The initial tree, depicted in Figure~\ref{planning:PA.fig.graphit0}, is defined as a Condition node $s_g$ which returns Success if and only if the robot is at the state $s_g$ in the graph. The current state is $s_0$ (the initial state). Hence the BT returns a status of \emph{Failure}. Algorithm~\ref{planning:PA.alg.main} invokes the BT expansion routine. The state $s_g$ can be reached from the state $s_5$, through the action $s_5\to s_g$, or from the state $s_3$, through the action $s_3\to s_g$. The tree is expanded accordingly as depicted in Figure~\ref{planning:PA.fig.graphit1}. Now executing this tree, it returns a status of \emph{Failure}. Since the current state is neither $s_g$ nor $s_3$ nor $s_5$.
Now the tree is expanded in a BFS fashion, finding a subtree for condition $s_5$ as in Figure~\ref{planning:PA.fig.graphit2}. The process continues for two more iterations.
Note that at iteration 4 (See Figure~\ref{planning:PA.fig.graphit4}) Algorithm~\ref{planning:PA.alg.main} did not expand the condition $s_g$ as it was previously expanded (Algorithm~\ref{planning:PA.alg.getcond} line~\ref{planning:PA.alg.getcond.notexpanded}) this avoids infinite loops in the search. The same argument applies for conditions $s_4$ and $s_g$ in iteration 5 (See Figure~\ref{planning:PA.fig.graphit5}). The BT at iteration 5 includes the action $s_0 \to s_1$ whose precondition is satisfied (the current state is $s_0$). The action is then executed. Performing that action (and moving to $s_1$), the condition $s_1$ is satisfied. The BT executes the action  $s_1 \to s_3$ and then $s_3 \to s_g$, making the agent reach the goal state.

It is clear that the resulting execution is  the same as a BFS on the graph would have rendered. Note however that PA-BT is designed for more complex problems than graph search.

\end{example}

\begin{figure}[h]

                \centering
                \includegraphics[width=0.55\columnwidth]{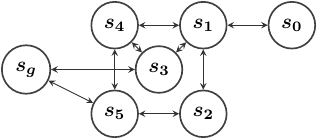}
        \caption{Graph representing the task of Example~\ref{planning:PA.ex.graph}.}
        \label{planning:PA.fig.graph}
\end{figure}

\begin{figure}[h]
                \centering

        \begin{subfigure}[b]{0.2\columnwidth}
                \centering
                \includegraphics[width=0.5\columnwidth]{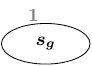}
                \caption{The initial BT. }
                \label{planning:PA.fig.graphit0}              
        \end{subfigure}          
  ~       
                   \centering

        \begin{subfigure}[b]{0.6\columnwidth}
                \centering
                \includegraphics[width=0.7\columnwidth]{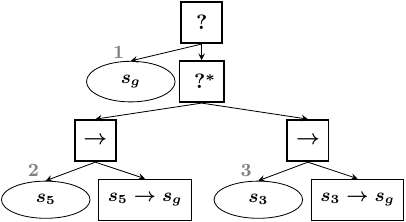}
                \caption{BT after one iteration.}
                \label{planning:PA.fig.graphit1}              
        \end{subfigure}
 ~         
        \begin{subfigure}[b]{0.6\columnwidth}
                \centering
                \includegraphics[width=1\columnwidth]{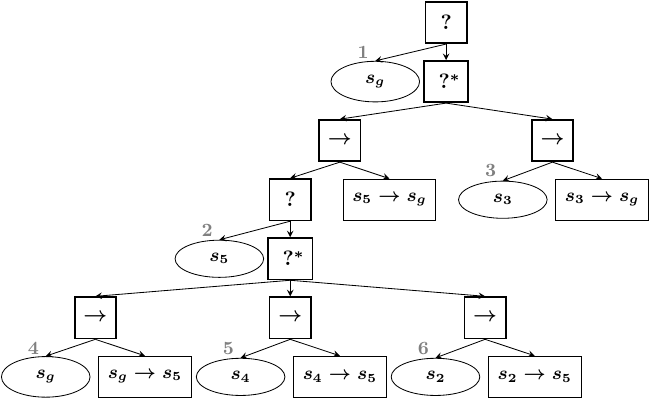}
                \caption{BT after two iterations.}
                \label{planning:PA.fig.graphit2}
        \end{subfigure}  
        \caption{First three BT updates during execution of Example~\ref{planning:PA.ex.graph}. The numbers represent the index of the BFS of Algorithm~\ref{planning:PA.alg.getcond}. {Note that the node labeled with $?^*$ is a Fallback node with memory.}}
     
  \end{figure}

\begin{figure}[h!]

  \centering    
        \begin{subfigure}[b]{1\columnwidth}
                \centering
                \includegraphics[width=1\columnwidth]{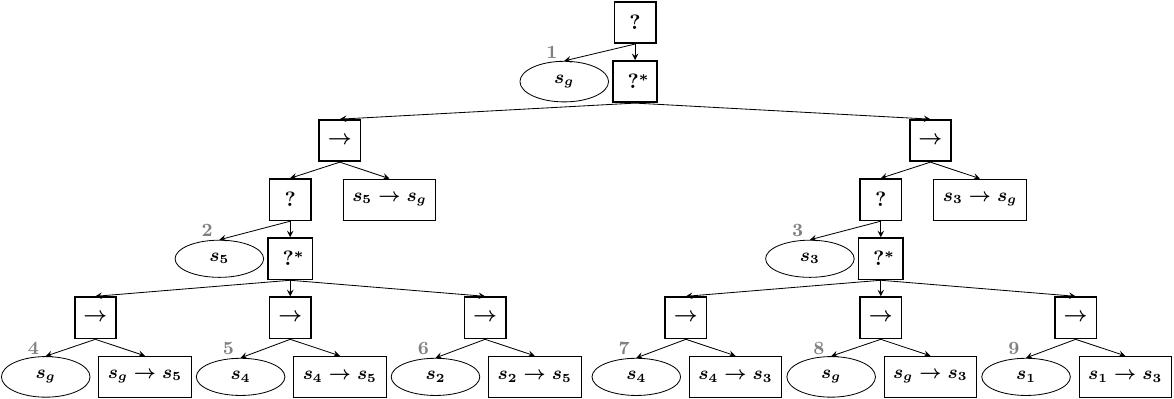}
                \caption{BT after three iterations.}
                \label{planning:PA.fig.graphit3}
        \end{subfigure}  
        
~
      \begin{subfigure}[b]{1\columnwidth}
                \centering
                \includegraphics[width=1\columnwidth]{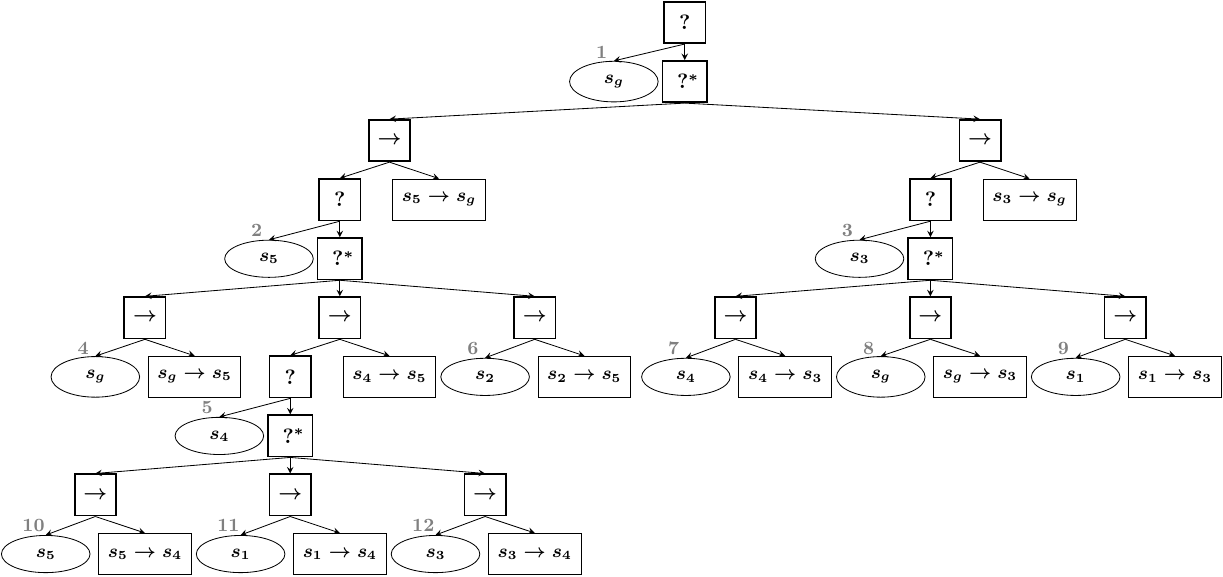}
                \caption{BT after four iterations.}
                \label{planning:PA.fig.graphit4}
        \end{subfigure}  
        
 ~       
              \begin{subfigure}[b]{1\columnwidth}
                \centering
                \includegraphics[width=1\columnwidth]{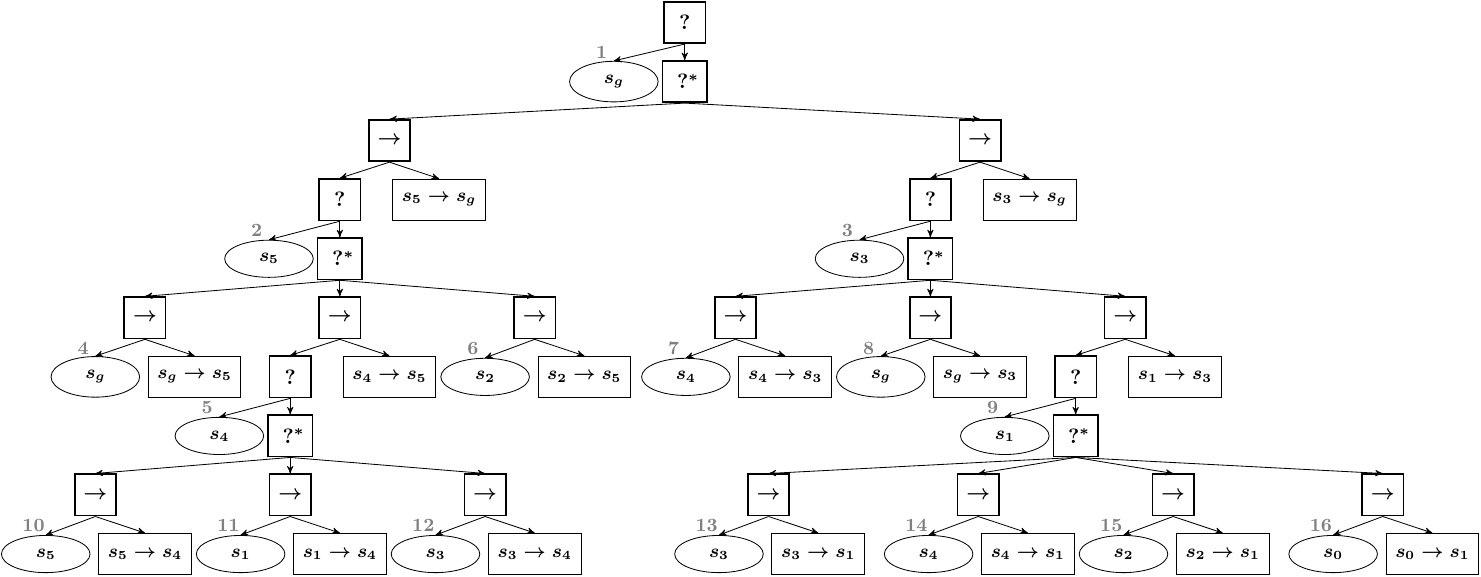}
                \caption{BT after five iterations.}
                \label{planning:PA.fig.graphit5}
        \end{subfigure}        
        \caption{Next BT updates during execution of Example~\ref{planning:PA.ex.graph}. The numbers represent the index of the BFS of Algorithm 3. {Note that the node labeled with $?^*$ is a Fallback node with memory.}}
\end{figure}

\subsection{Algorithm Execution on an existing Example}
In this section we apply the PA-BT approach in a more complex example adapted from~\cite{kaelbling2011hierarchical}.

\begin{figure}[h!]
    \centering
\includegraphics[width = \columnwidth,trim={1cm 3cm 1cm 3cm},clip]{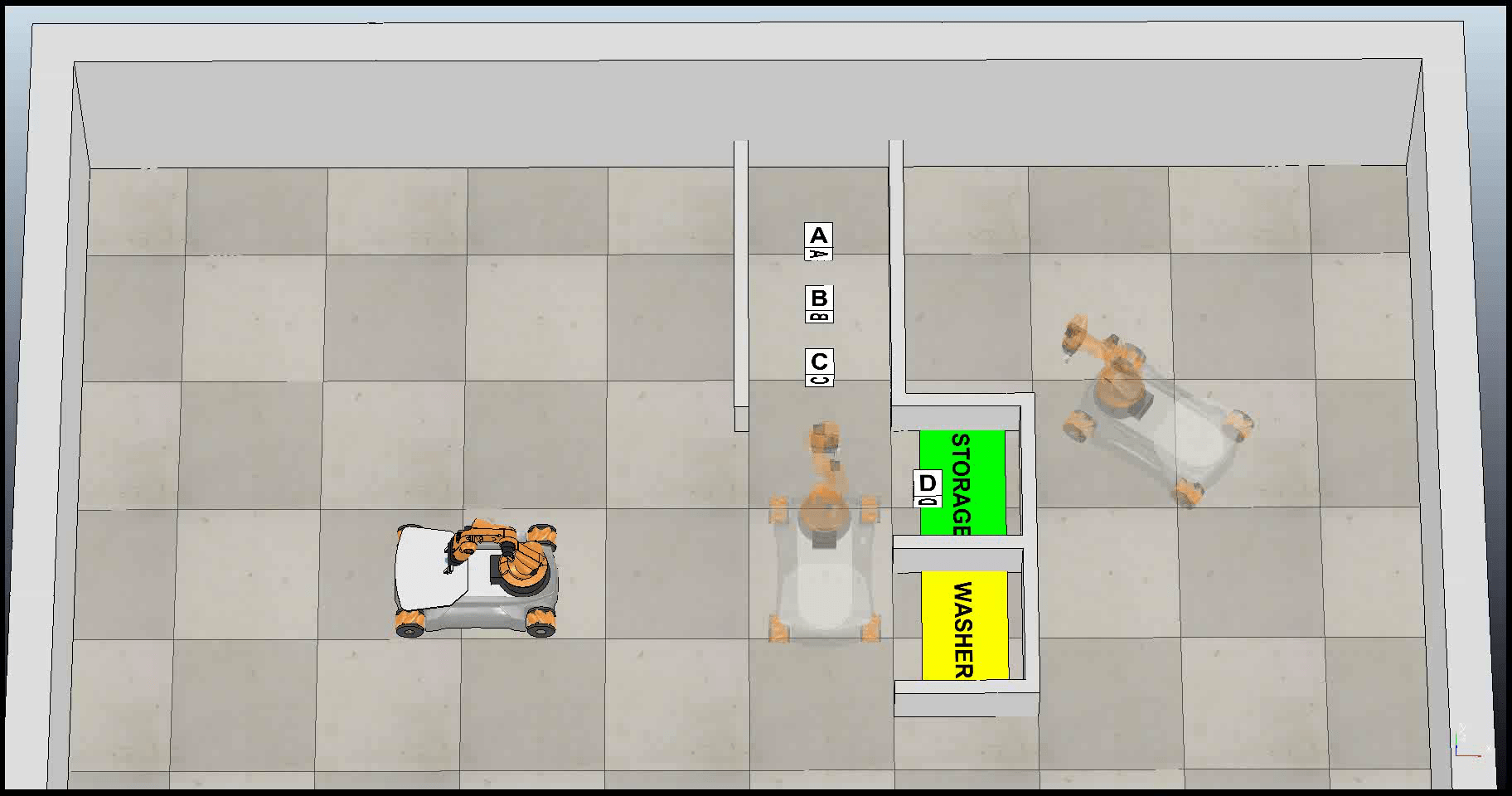} 
    \caption{
    An example scenario from~\cite{kaelbling2011hierarchical},
    with the addition of two externally controlled robots (semi-transparent) providing disturbances. The robot must wash the object "A" and then put in into the storage.
    }
    \label{planning:BG.fig.youbot}
\end{figure}
\begin{example}[From~\cite{kaelbling2011hierarchical}]
\label{planning:BG.ex.simple}
Consider a multipurpose robot that is asked to clean the object $A$ and then put it in the storage room as shown in Figure~\ref{planning:BG.fig.youbot} (in this first example we ignore the other robots as they are not in ~\cite{kaelbling2011hierarchical}). The goal is specified as a conjunction $Clean(A) \land In(A,storage)$. Using PA-BT, the initial BT is defined as a sequence composition of $Clean(A)$ with $In(A,storage)$ as in Figure~\ref{planning:BG.fig.it0}. At execution, the Condition node $Clean(A)$ returns \emph{Failure} and the tree is expanded accordingly, as in Figure~\ref{planning:BG.fig.it1}. Executing the expanded tree, the Condition node $In(A,Washer)$ returns \emph{Failure} and the BT is expanded again, as in Figure~\ref{planning:BG.fig.it2}. This iterative process of planning and acting continues until it creates a BT such that the robot is able to reach  object $C$ and remove it. After cleaning  object $A$, the approach constructs the tree to satisfy the condition $In(A,storage)$ as depicted in Figure~\ref{planning:BG.fig.it6}. This subtree requires  picking  object $A$ and then placing it into the storage. However after the BT is expanded to place $A$ into the storage it contains a conflict: in order to remove  object $D$ the robot needs to grasp it. But to let the ticks reach this tree, the condition $Holding()=a$ needs to returns \emph{Success}. Clearly the robot cannot hold both $A$ and $D$. The new subtree is moved in the BT to a position with a higher priority (See Algorithm~\ref{planning:PA.alg.main} Line~\ref{planning:PA.alg.main.incprio}) and the resulting BT is the one depicted in Figure~\ref{planning:BG.fig.it7}.
\end{example}
Note that the final BT depicted in Figure~\ref{planning:BG.fig.it7} is similar to the planning and execution tree of~\cite{kaelbling2011hierarchical} with the key difference that the BT enables the continual monitoring of the plan progress, as described in~\cite{Nau15Challenges}. For example, if  $A$ slips out of the robot gripper, the robot will automatically stop and pick it up again without the need to replan or change the BT. Moreover if $D$ moves away from the storage while the robot is approaching it to remove it, the robot aborts the plan to remove $D$ and continues with the plan to place $A$ in the storage room. Hence we can claim that the PA-BT is  reactive. The execution is exactly the same as~\cite{kaelbling2011hierarchical}: the robot removes the obstructing objects $B$ and $C$ then places $A$ into the washer. When $A$ is clean, the robot picks it up, but then it has to unpick it since it has to move  $D$ away from the storage. This is a small drawback of this type of planning algorithms. Again, as stressed in \cite{ghallab2016automated} the cost of a non-optimal plan is often much lower than the cost of extensive modeling, information gathering and thorough deliberation needed to achieve optimality.

\begin{figure}[h!]
        \centering
          \begin{subfigure}[t]{0.3\columnwidth}
                \centering
\includegraphics[width=1\columnwidth]{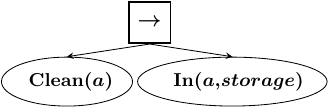}
                \caption{The initial BT.}
                \label{planning:BG.fig.it0}
        \end{subfigure}%

        \begin{subfigure}[t]{0.4\columnwidth}
\includegraphics[width=1\columnwidth]{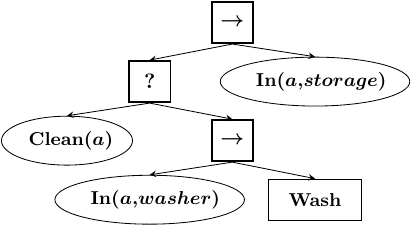}
                \caption{BT after one iteration.}
                \label{planning:BG.fig.it1}
        \end{subfigure}%
        
       ~ 
        \begin{subfigure}[t]{0.45\columnwidth}
                \centering
\includegraphics[width=1\columnwidth]{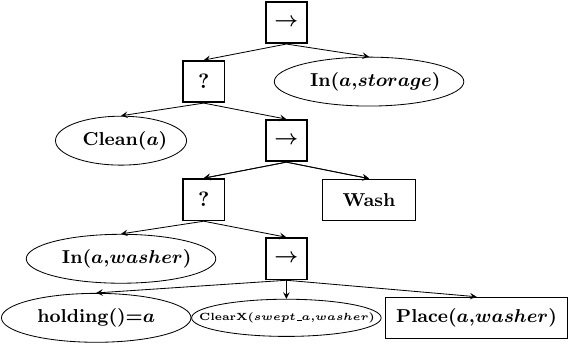}
                \caption{BT after two iterations.}
                \label{planning:BG.fig.it2}              
        \end{subfigure}
                \caption{BT updates during the execution of Example~\ref{planning:BG.ex.simple}.}
\end{figure}

\begin{figure}[h!]
        \centering
        \begin{subfigure}[b]{0.75\columnwidth}
                \centering
\includegraphics[width=1\columnwidth]{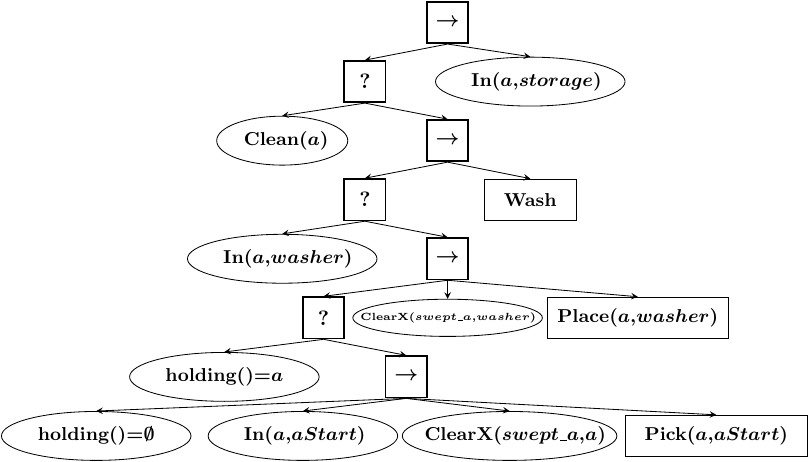}
                \caption{BT after three iterations.}
                \label{planning:BG.fig.it3}
        \end{subfigure}%
       ~ 
       
        \begin{subfigure}[b]{1\columnwidth}
                \centering
\includegraphics[width=0.85\columnwidth]{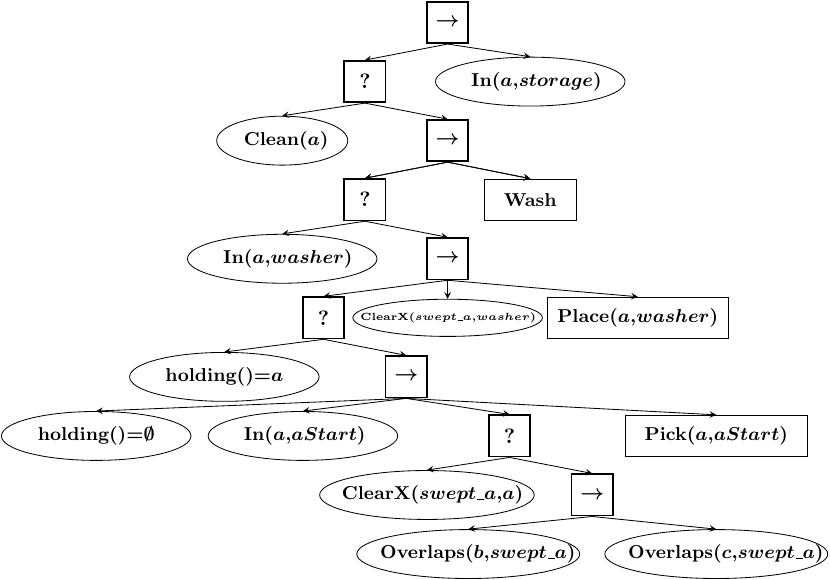}
                \caption{BT after four iterations.}
                \label{planning:BG.fig.it4}              
        \end{subfigure}
        
        \caption{BT updates during the execution of Example~\ref{planning:BG.ex.simple}.}
\end{figure}
\begin{landscape}
\begin{figure}[h!]
        \centering
\includegraphics[width=\columnwidth]{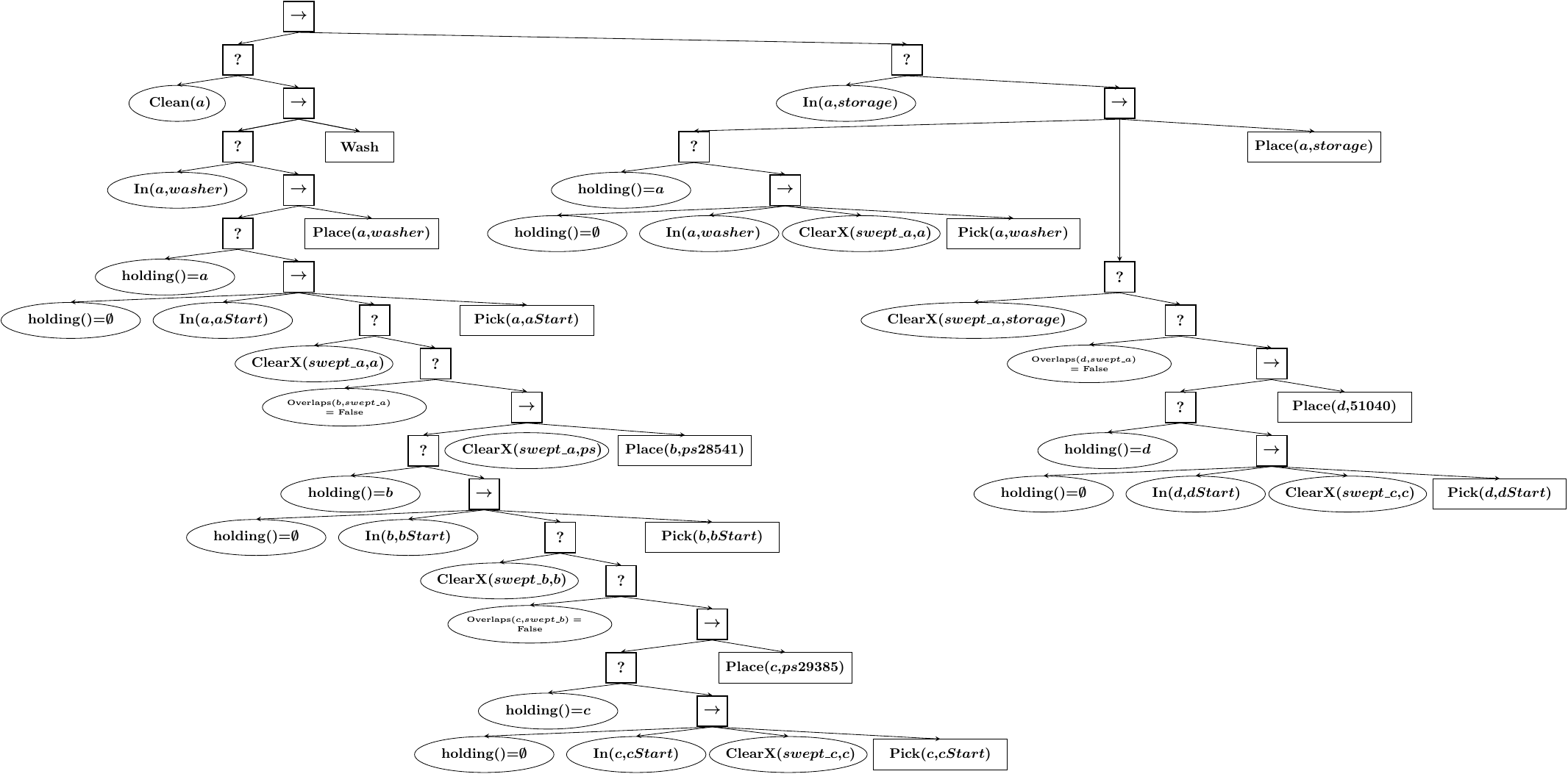}
        \caption{BT containing a conflict. The subtree created to achieve $ClearX(swept\_a,storage)$ is in conflict with the subtree created to achieve $holding()=a$. }
         \label{planning:BG.fig.it6}              
        ~ 
\end{figure}
\end{landscape}

\begin{landscape}

\begin{figure}[h]
        \centering
\includegraphics[width=\columnwidth]{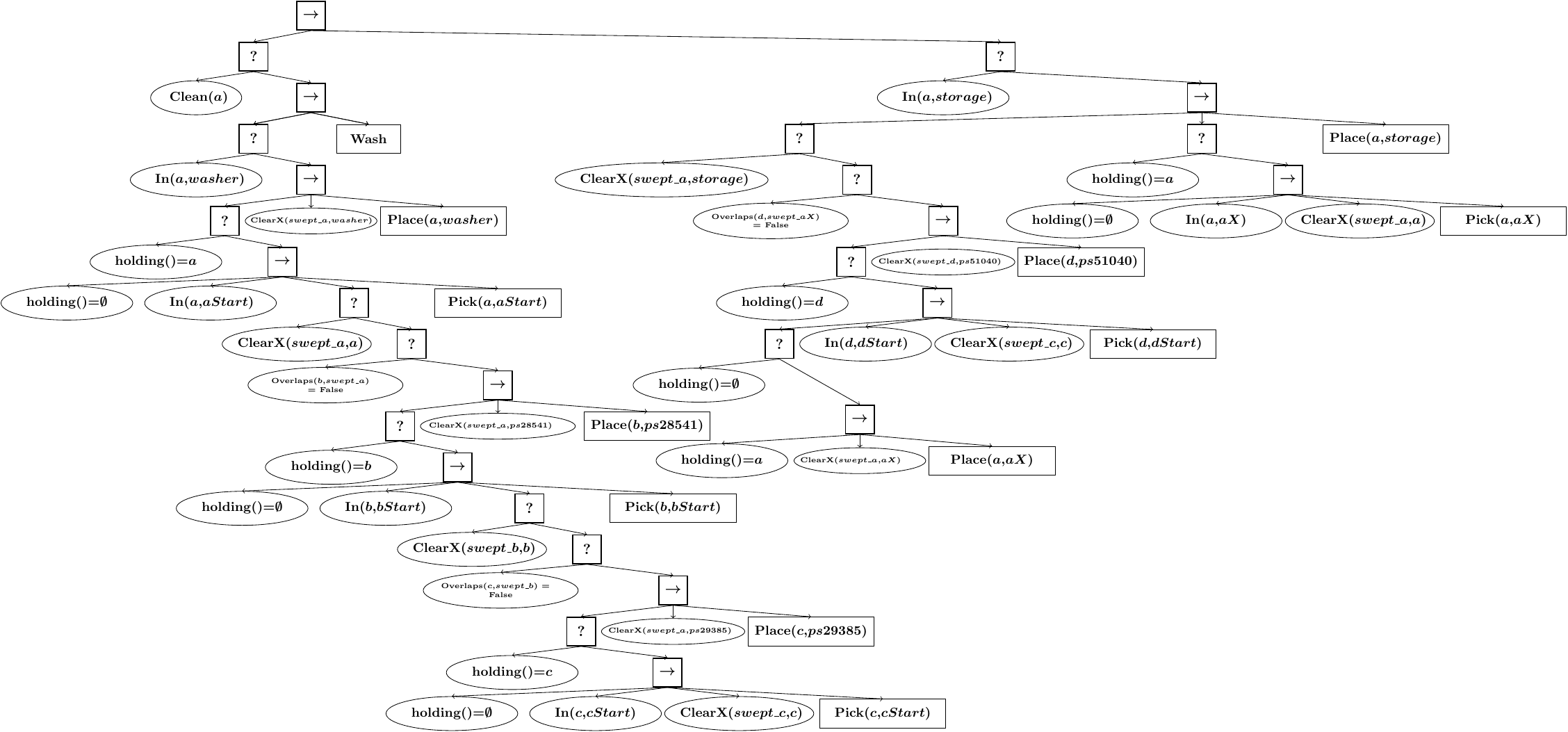}
        \caption{Conflict free BT obtained.}
         \label{planning:BG.fig.it7}              
        ~ 
\end{figure}
\end{landscape}

\subsection{Reactiveness}
\label{planning:RE}
In this section we show how BTs enable a \emph{reactive} blended acting and planning, providing concrete examples that highlight the importance of reactiveness in robotic applications.

Reactiveness is a key property for online deliberation. By reactiveness we mean \emph{the capability of dealing with drastic changes in the environment in short time}. The domains we consider are highly dynamic and unpredictable. To deal with such domains, a robot must be reactive in both planning and acting.

If an external event happens that the robot needs to react to, one or more conditions will change.
The next tick after the change will happen at most a time $\frac{1}{f_t}$ after the change. Then a subset of all
Conditions, Sequences and Fallbacks will be evaluated to either reach an Action,
or to invoke additional planning. This takes less than $\Delta t$.
The combined reaction time is thus bounded above by $\frac{1}{f_t}+\Delta t$.

\begin{remark}
Note that $\Delta t$ is strictly dependent on the real world implementation. Faster computers and better code implementation allows a smaller $\Delta t$.
\end{remark}

We are now ready to show three colloquial examples, one highlighting the reactive acting (preemption and re-execution of subplans if needed), and two highlighting the reactive planning, expanding the current plan as well as exploiting serendipity~\cite{levihn2013foresight}. 

\begin{example}[Reactive Acting]
\label{planning:RE.ex.acting}
Consider the robot in Figure~\ref{planning:BG.fig.youbot}, running the BT of Figure~\ref{planning:BG.fig.it7}. The object $A$ is not clean and it is not in the washer, however the robot is holding it. This results in the Condition nodes \emph{Clean($a$)} and \emph{In($a$,$washer$)} returning  \emph{Failure} and the Condition node \emph{holding()=$a$}  \emph{Success}.  According to the BT's logic, the ticks traverse the tree and reach the action \emph{Place($a$,$washer$)}. Now, due to vibrations during movements, the object slips out of the robot's grippers. In this new situation the Condition node \emph{holding()=$a$} now returns  \emph{Failure}. The ticks now traverse the tree and reach the action \emph{Pick($a$,$aCurrent$)} (i.e. pick $A$ from the current position, whose preconditions are satisfied). The robot then re-picks the object, making the Condition node  \emph{holding()=$a$} returning  \emph{Success} and letting the robot resume the execution of \emph{Place($a$,$washer$)}.
\end{example}

\begin{example}[Reactive Planning]
\label{planning:RE.ex.planning}
Consider the robot in Figure~\ref{planning:BG.fig.youbot}, running the BT of Figure~\ref{planning:BG.fig.it7}. The object $A$ is clean, it is not in the storage and the robot is not holding it. This results in the Condition nodes \emph{holding()=$a$}  and \emph{In($a$,$storage$)} returning  \emph{Failure} and the Condition nodes \emph{holding()=$\emptyset$}, \emph{Clean($a$)} and \\ \emph{ClearX($swept\_a$,$washer$)} returning  \emph{Success}.  According to the BT's logic,  the ticks traverse the tree and reach the action \emph{Pick($a$,$washer$)} that let the robot approach the object and then grasp it. While the robot is approaching $A$, an external uncontrolled robot places an object in front of $A$, obstructing the passage. In this new situation the Condition node \emph{ClearX($swept\_a$,$washer$)} now returns  \emph{Failure}. The action \emph{Pick($a$,$washer$)}  no longer receives ticks and it is then preempted. The BT is expanded accordingly finding a subtree to make \emph{ClearX($swept\_a$,$washer$)} return  \emph{Success}. This subtree will make the robot pick the obstructing object and remove it. Then the robot can finally reach $A$ and place it into the storage.
\end{example}

\begin{example}[Serendipity Exploitation]
\label{planning:RE.ex.serendipity}
Consider the robot in Figure~\ref{planning:BG.fig.youbot}, running the BT of Figure~\ref{planning:BG.fig.it7}. The object $A$ is not clean, it is not in the washer and the robot is not holding it. According to the BT logic, the ticks traverse the tree and reach the action \emph{Pick($b$,$bStart$)}. While the robot is reaching the object, an external uncontrolled agent picks $B$ and removes it. Now the condition \emph{Overlaps($b$,$swept\_a$) = False} returns  \emph{Success} and the BT preempts the execution of \emph{Pick($b$,$bStart$)} and skips the execution of \emph{Places($b$,$ps28541$)} going directly to execute \emph{Pick($a$,$aStart$)}.
\end{example}

Hence the BT encodes reactiveness, (re)satisfaction of subgoals whenever these are no longer achieved, and the exploitation of  external subgoal satisfaction after a change in the environment. Note that in the Examples~\ref{planning:RE.ex.acting} and~\ref{planning:RE.ex.serendipity} above, PA-BT did not replan.

\begin{figure}[b]
    \centering
\includegraphics[width = 0.5\columnwidth]{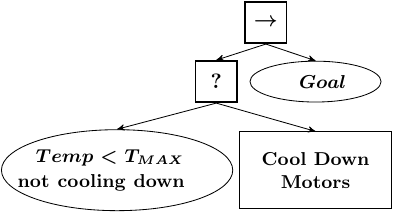}
    \caption{  A safe BT for Example \ref{planning:SA.ex.safe}. The BT guaranteeing safety is combined with the mission objective constraint.}
    \label{planning:SA.fig.tree}
\end{figure}

\subsection{Safety}
\label{planning:SA}
In this section we show how BTs allow  a \emph{safe} blended acting and planning, providing a concrete example that highlights the importance of safety in robotics applications.

Safety is a key property for robotics applications. By safety we mean \emph{the capability of avoiding undesired outcomes.} The domains we usually consider have few catastrophic outcomes of actions and,  as highlighted in~\cite{kaelbling2011hierarchical}, the result of an action can usually be undone by a finite sequence of actions. However there are some cases in which the outcome of the plan can damage the robot or its surroundings. 
These cases are assumed to be identified  in advance by human operators, who then add the corresponding sub-BT to guarantee the avoidance of them. Then, the rest of the BT is expanded using the algorithm described above.

We are now ready to show a colloquial example.
\begin{example}[Safe Execution]
\label{planning:SA.ex.safe}
Consider the multipurpose robot of Example~\ref{planning:BG.ex.simple}. Now, due to overheating, the robot has to stop whenever the motors' temperatures reach a given threshold, allowing them to cool down. This situation is relatively easy to model, and a subtree to avoid it can be designed as shown  in Figure~\ref{planning:SA.fig.tree}. When running this BT, the robot will preempt any action whenever the temperature is above the given threshold and stay inactive until the motors have cooled down to the temperature where Cool Down Motors return Success. Note that the Not Cooling Down part of the condition is needed to provide hystereses. The robot stops when $T_{\mbox{MAX}}$ is reached, and waits until Cool Down Motors return Success at some given temperature below $T_{\mbox{MAX}}$.
To perform the actual mission, the BT in Figure~\ref{planning:SA.fig.tree} is executed and expanded as explained above.
\end{example}

Thus, we first identify and handle the safety critical events separately, and then progress as above without 
jeopardizing the safety guarantees.
Note that the tree in Figure~\ref{planning:SA.fig.tree}, as well as all possible expansions of it using the PA-BT algorithm is \emph{safe} (see Section~\ref{properties:sec:safety}).

\subsection{Fault Tolerance}
\label{planning:FT}
In this section we show how BTs enable a \emph{fault tolerant} blended acting and planning, providing concrete examples that highlight the importance of fault tolerance in robotics application.

Fault tolerance is a key property for real world problem domains where no actions are guaranteed to succeed. By fault tolerant we mean \emph{the capability of operating properly in the event of Failures.} The robots we consider usually have more than a single way of carrying out a high level task. 
With multiple options, we can define a priority measure for these options. In PA-BT, the actions that achieve a given goal are collected in a Fallback composition (Algorithm~\ref{planning:PA.alg.update} Line 9). The BT execution is such that if, at execution time, one option fails (e.g. a motor breaks) the next option is chosen without replanning or extending the BT.
%
%
\begin{example}[Fault Tolerant Execution]
Consider a more advanced version the multipurpose robot of Example~\ref{planning:BG.ex.simple} having a second arm. Due to this redundancy, all the pick-and-place tasks can be carried out with either hands. In the approach (Algorithm~\ref{planning:PA.alg.update} Line 9) all the actions that can achieve a given subgoal are collected in a Fallback composition. Thus, whenever the robot fails to pick an object with a gripper, it can directly try to pick it with the other gripper.
\end{example}

\newpage

\subsection{Complex Execution on Realistic Robots}
\label{planning:sec:simulations}

In this section, we show how the PA-BT approach scales to complex problems using two different scenarios. First, a KUKA Youbot scenario, where we show the applicability of PA-BT on dynamic and unpredictable environments, highlighting the importance of continually planing and acting. Second, an  ABB Yumi industrial manipulator scenario, where we highlight the applicability of PA-BT to real world plans that require the execution of a long sequence of actions. The experiments were carried out using the physics simulator V-REP, in-house implementations of low level controllers for actions and conditions, and an open source BT library\footnote{\url{http://wiki.ros.org/behavior_tree}}.  
 Figures~\ref{planning:SI.fig.yousimplescreen} and~\ref{planning:DSI.fig.youscreen} show the execution of two KUKA youbot experiments and Figure~\ref{planning:SI.fig.yumiscreen} show the execution of one ABB Yumi robot experiment. A video showing the executions of all the experiments is publicly available.\footnote{\url{https://youtu.be/mhYuyB0uCLM}} {All experiments show the reactiveness of the PA-BT approach, one experiment is then extended to show safety maintenance and fault tolerance.}

%
%
%
%
\subsubsection{KUKA Youbot experiments}

In these scenarios, which are an extension of Examples~1 and~2, a KUKA Youbot has to place a green cube on a goal area, see Figures~\ref{planning:SI.fig.yousimplescreen} and~\ref{planning:DSI.fig.youscreen}. The robot is equipped with a single or dual arm with a simple parallel gripper. Additional objects may obstruct the feasible paths to the goal, and the robot has to plan when to pick and where to place  the obstructing objects. Moreover, external actors  co-exist in the scene and may force the robot to replan by modifying the environment, e.g. by picking and placing the objects.
{
\begin{experiment}[Static Environment]
\label{planning:results.ex.1}
In this experiment the single armed version of robot is asked to place the green cube in the goal area. First expansions of the BT allow the robot to pick up the desired object (see Figure~\ref{planning:SI.fig.youbotstep1}). Now the robot has to find a collision free path to the goal. Due to the shape of the floor and the position of the obstacle (a blue cube) the robot has to place the obstacle to the side. To do so the robot has to reach the obstacle and pick it up. Since this robot has a single arm, it needs to ungrasp the green cube (see Figure~\ref{planning:SI.fig.youbotstep2}) before placing the blue cube on the side (see Figure~\ref{planning:SI.fig.youbotstep3}). The robot then can re-grasp the green cube (without extending the plan) and approach the goal region. Now, due to vibrations, the green cube slips out of the robot gripper (see Figure~\ref{planning:SI.fig.youbotstep4}). The robot aborts its subplan to reach the goal and re-executes the subplan to grasp the object.  Finally the robot places the green cube in the desired area (see Figure~\ref{planning:SI.fig.youbotstep5}).
\end{experiment}
\begin{figure}[h]
        \centering
        \begin{subfigure}[b]{1\columnwidth}
                \centering
                \includegraphics[width=1\columnwidth,trim={1cm 1.5cm 10cm 9cm},clip]{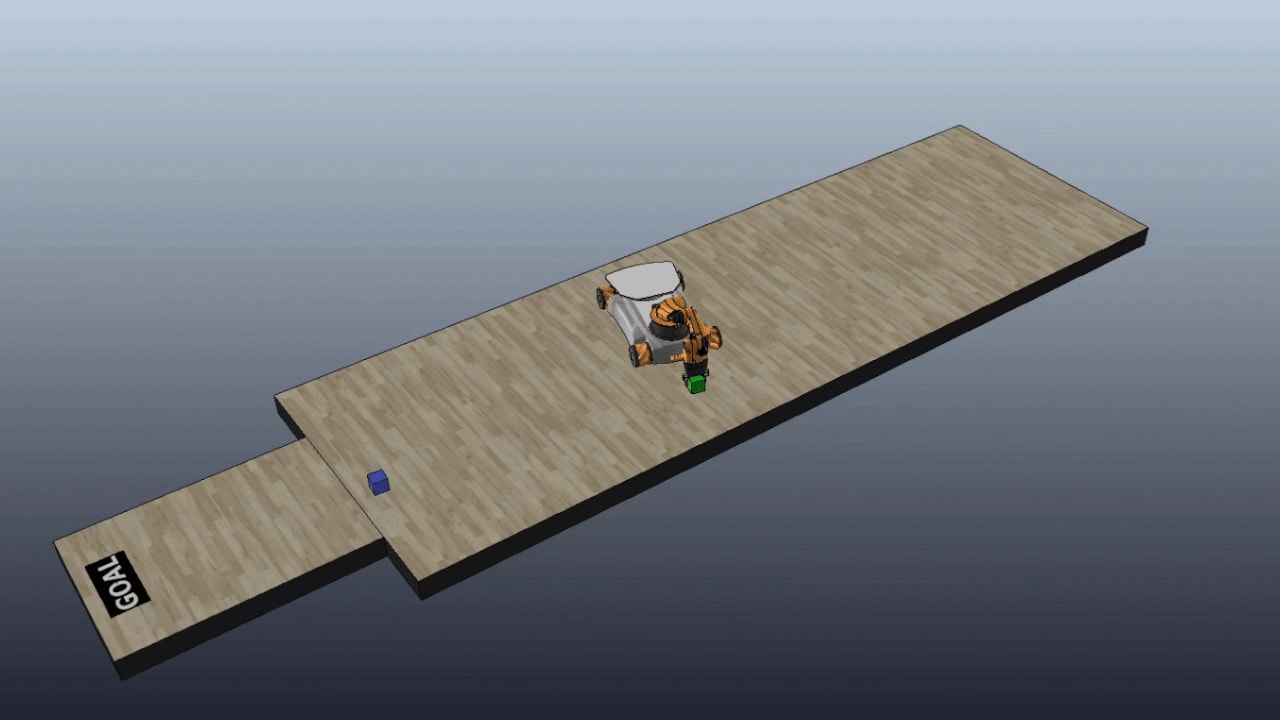}
                \caption{The robot picks the desired object: a green cube.  }
                \label{planning:SI.fig.youbotstep1}              
        \end{subfigure}       
        ~         
        \begin{subfigure}[b]{1\columnwidth}
                \centering
                \includegraphics[width=1\columnwidth,trim={1cm 1.5cm 10cm 9cm},clip]{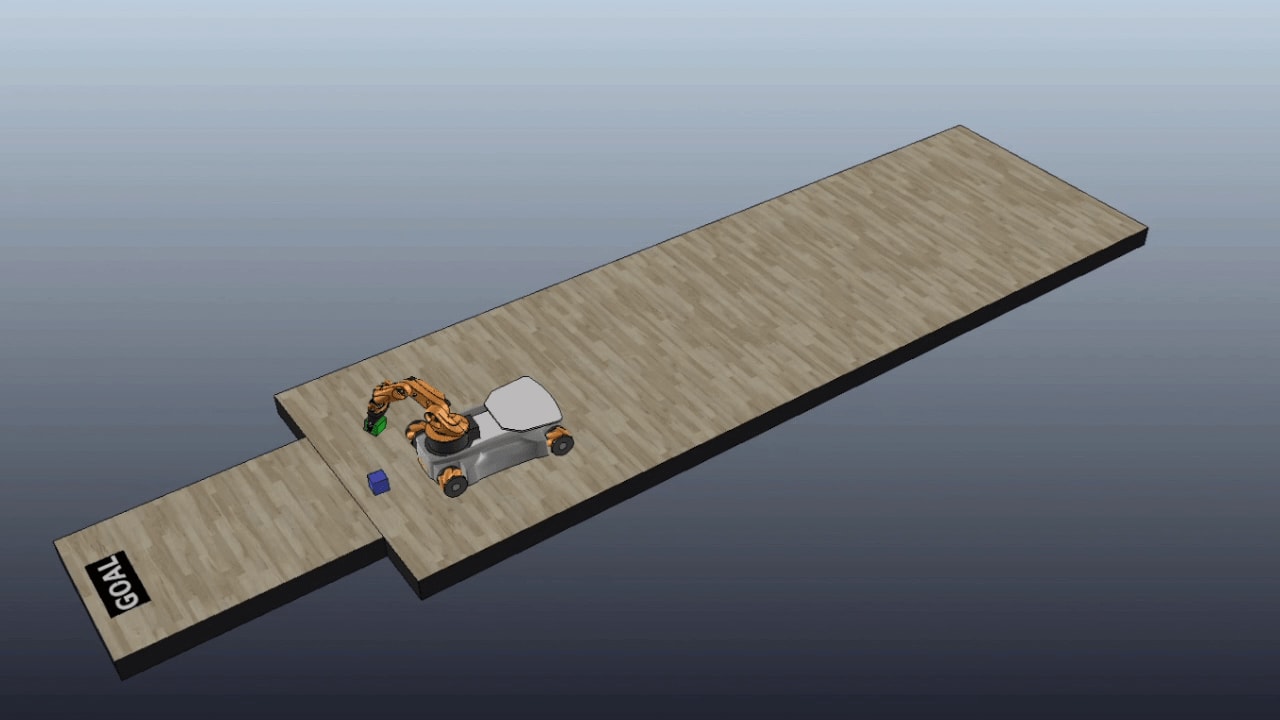}
                \caption{The robot has to move the blue cube away from the path to the goal. But the robot is currently grasping the green cube. Hence the subtree created to move the blue cube needs to have a higher priority. }
 			\label{planning:SI.fig.youbotstep2}  
         \end{subfigure}
~
        \begin{subfigure}[b]{1\columnwidth}
                \centering
                \includegraphics[width=1\columnwidth,trim={1cm 1.5cm 10cm 9cm},clip]{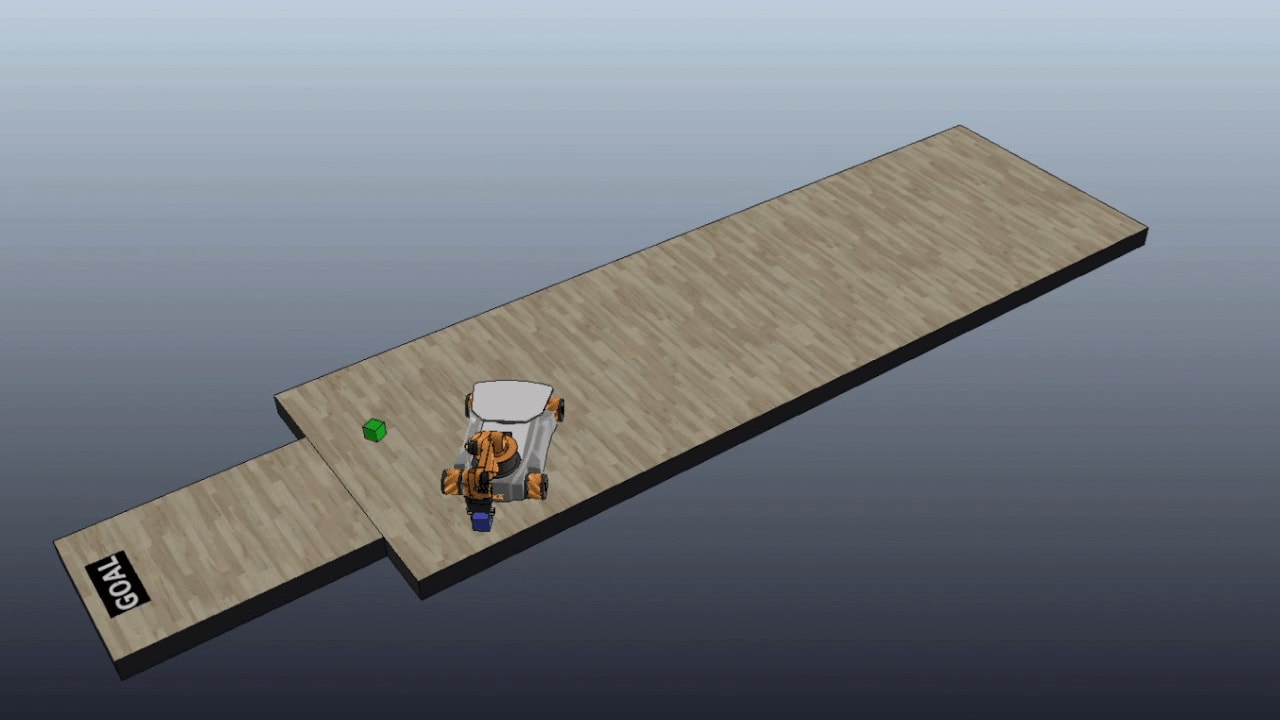}
                \caption{The blue cube is moved to the side. }
                 \label{planning:SI.fig.youbotstep3}  
        \end{subfigure}   
        \caption{Execution of a Simple KUKA Youbot experiment.}
        \label{planning:SI.fig.yousimplescreen2}
\end{figure}
\clearpage

\begin{figure}[t]
        \centering
        \begin{subfigure}[b]{1\columnwidth}
                \centering
                \includegraphics[width=1\columnwidth,trim={1cm 1.5cm 10cm 9cm},clip]{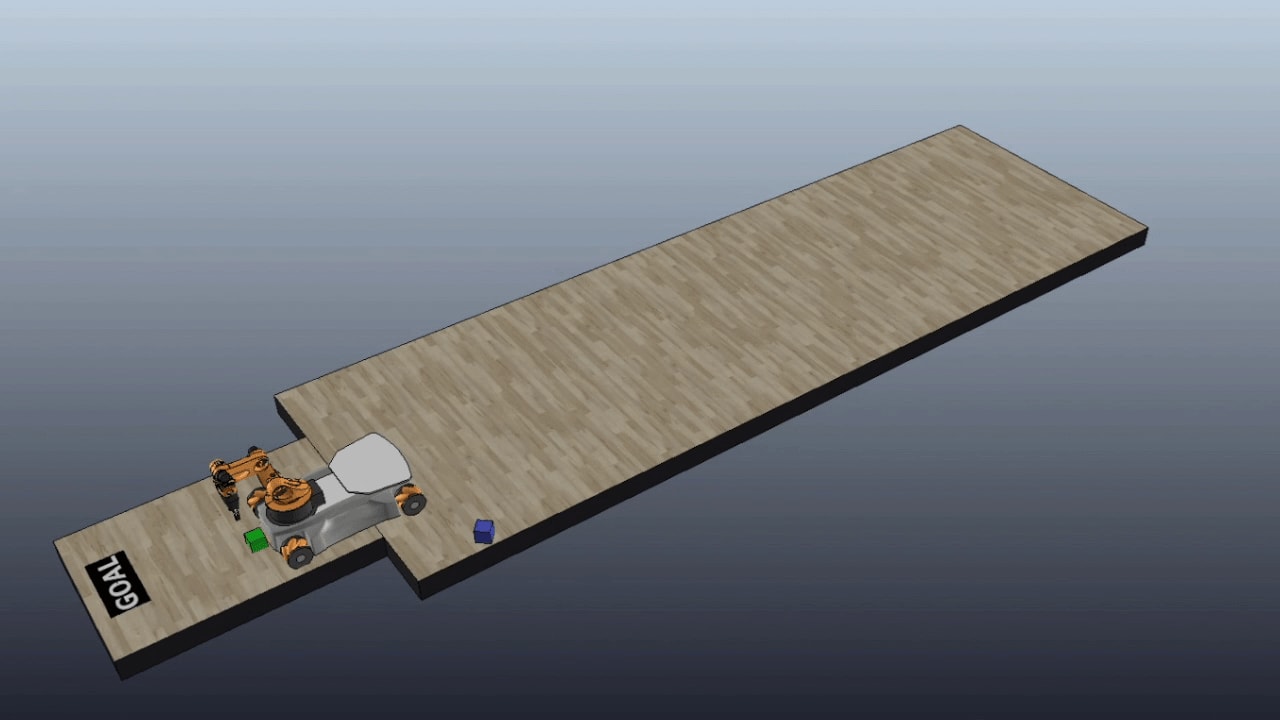}
                \caption{While the robot is moving towards the goal region, the green cube slips out of the gripper. The robot reactively preempts the subtree to move to the goal and re-executes the subtree to grasp the green cube. Without replanning. }
                 \label{planning:SI.fig.youbotstep4}  
        \end{subfigure} 
                ~            
        \begin{subfigure}[b]{1\columnwidth}
                \centering
                \includegraphics[width=1\columnwidth,trim={1cm 1.5cm 10cm 9cm},clip]{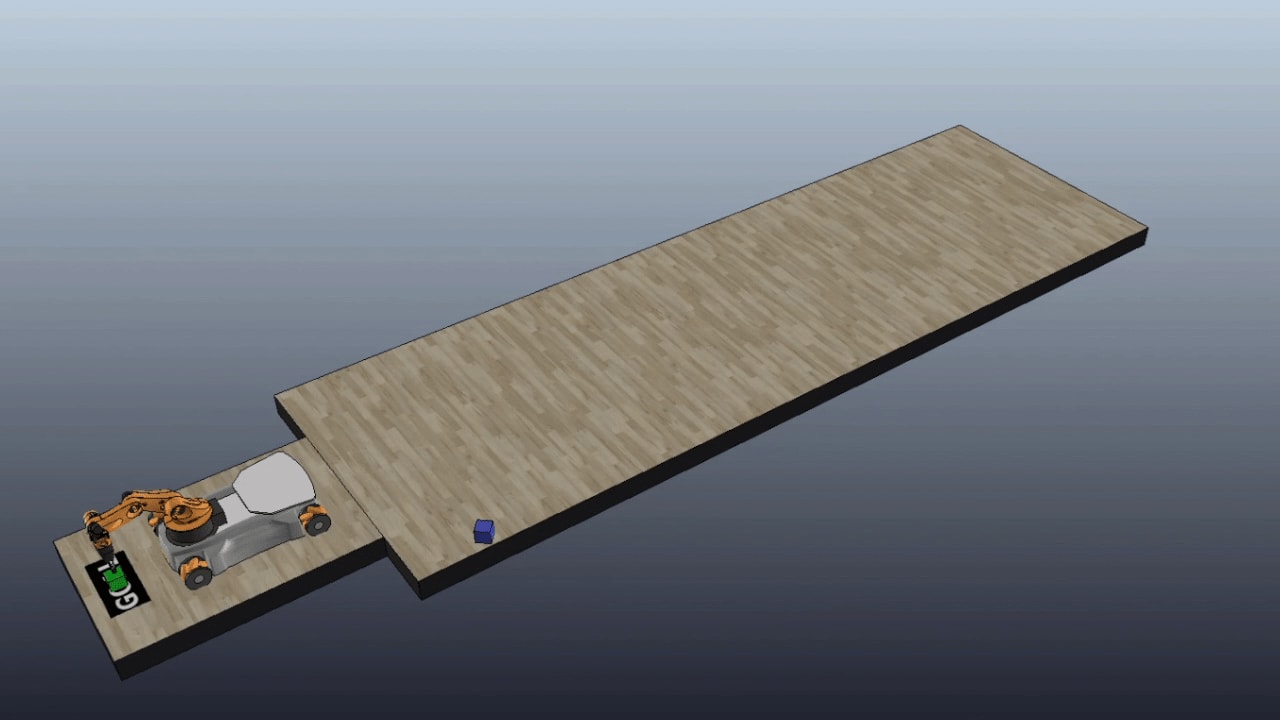}
                \caption{The robot places the object onto the desired location. }
                 \label{planning:SI.fig.youbotstep5}  
        \end{subfigure} 
        \caption{Execution of a Simple KUKA Youbot experiment.}
        \label{planning:SI.fig.yousimplescreen}
\end{figure}

\vspace*{\fill}

\clearpage

\begin{experiment}[Safety]
In this experiment the robot is asked to perform the same task as in Experiment~\ref{planning:results.ex.1} with the main difference that now the robot's battery can run out of power. To avoid this undesired irreversible outcome, the initial BT is manually created  in a way that is similar to the one in Figure~\ref{planning:SA.fig.tree}, managing the battery charging instead. As might be expected, the execution  is similar to the one described in Experiment~\ref{planning:results.ex.1} with the difference that the robot reaches the charging station whenever needed: The robot first reaches the green cube (see Figure~\ref{planning:SI.fig.youbotsafestep1}). Then, while the robot is approaching the blue cube, the battery level becomes low. Hence the subplan to reach the blue cube is aborted and the subplan to charge the battery takes over (see Figure~\ref{planning:SI.fig.youbotsafestep2}). When the battery is charged the robot can resume its plan (see Figure~\ref{planning:SI.fig.youbotsafestep3}) and complete the task (see Figure~\ref{planning:SI.fig.youbotsafestep4}).
\end{experiment}
\begin{figure}[h]
        \centering
        \begin{subfigure}[b]{	0.9\columnwidth}
                \centering
                \includegraphics[width=1\columnwidth,trim={5cm 5cm 28cm 15cm},clip]{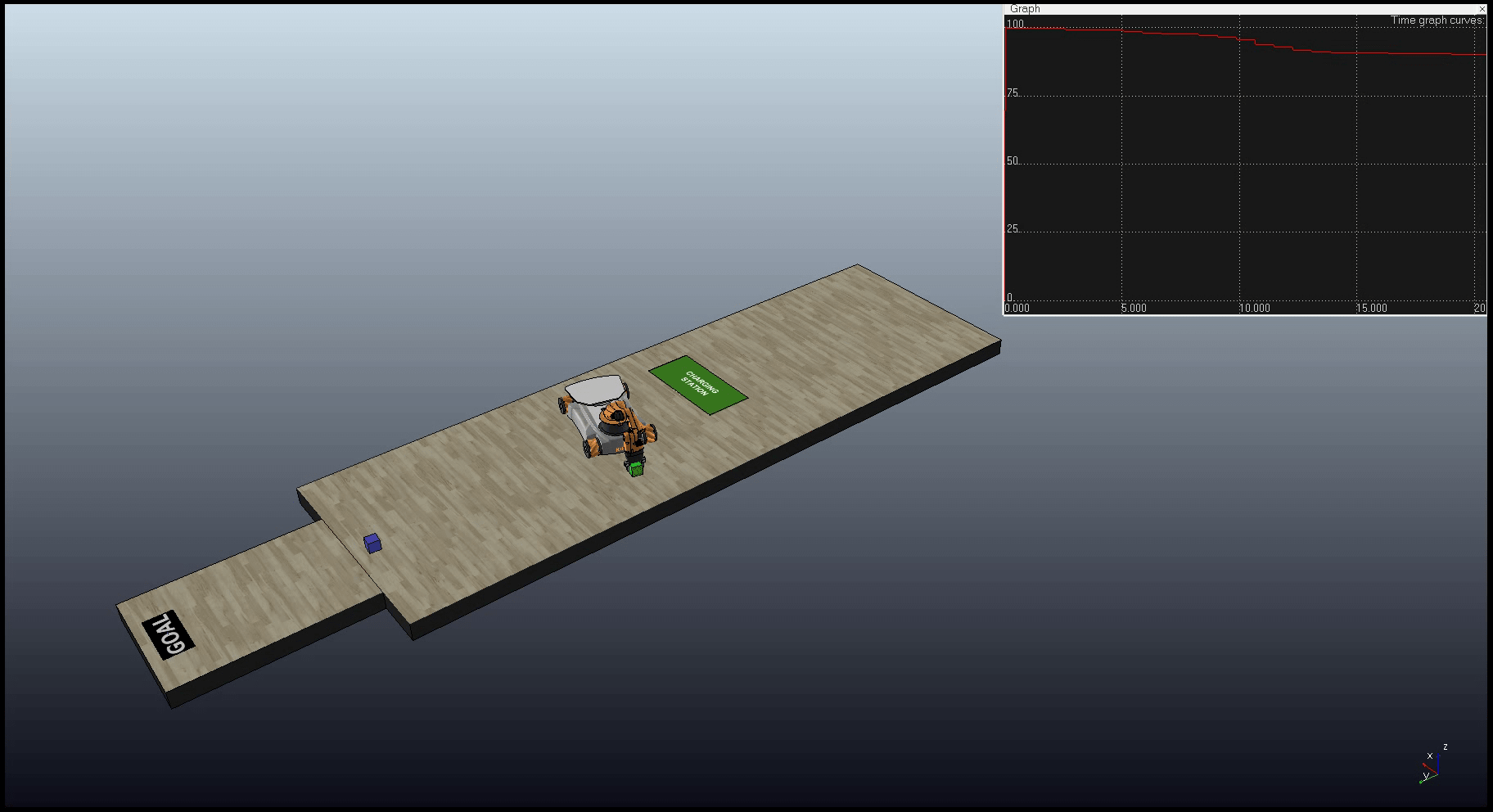}
                \caption{The robot picks the desired object, a green cube. }
                \label{planning:SI.fig.youbotsafestep1}              
        \end{subfigure}       
        ~         
        \begin{subfigure}[b]{0.9\columnwidth}
                \centering
                \includegraphics[width=1\columnwidth,trim={5cm 5cm 28cm 15cm},clip]{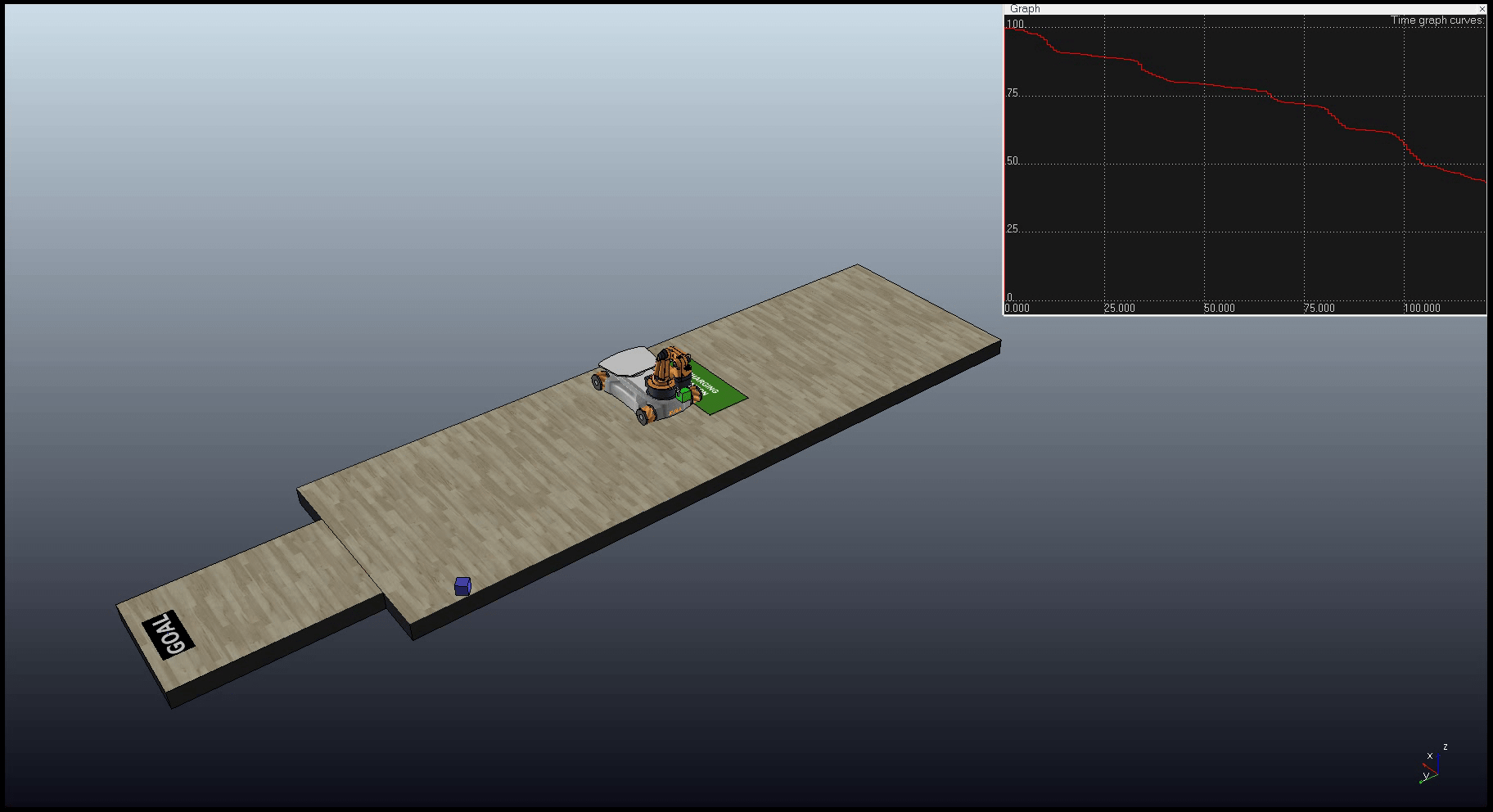}
                \caption{Due to the low battery level, the robot moves to the charging station.}
 			\label{planning:SI.fig.youbotsafestep2}  
         \end{subfigure}
                 \caption{Execution of a KUKA Youbot  experiment illustrating safety.}
        \label{planning:SI.fig.yousafescreen}
\end{figure}
\clearpage

         \begin{figure}[t]
\centering
        \begin{subfigure}[b]{0.9\columnwidth}
                \centering
                \includegraphics[width=1\columnwidth,trim={5cm 5cm 28cm 15cm},clip]{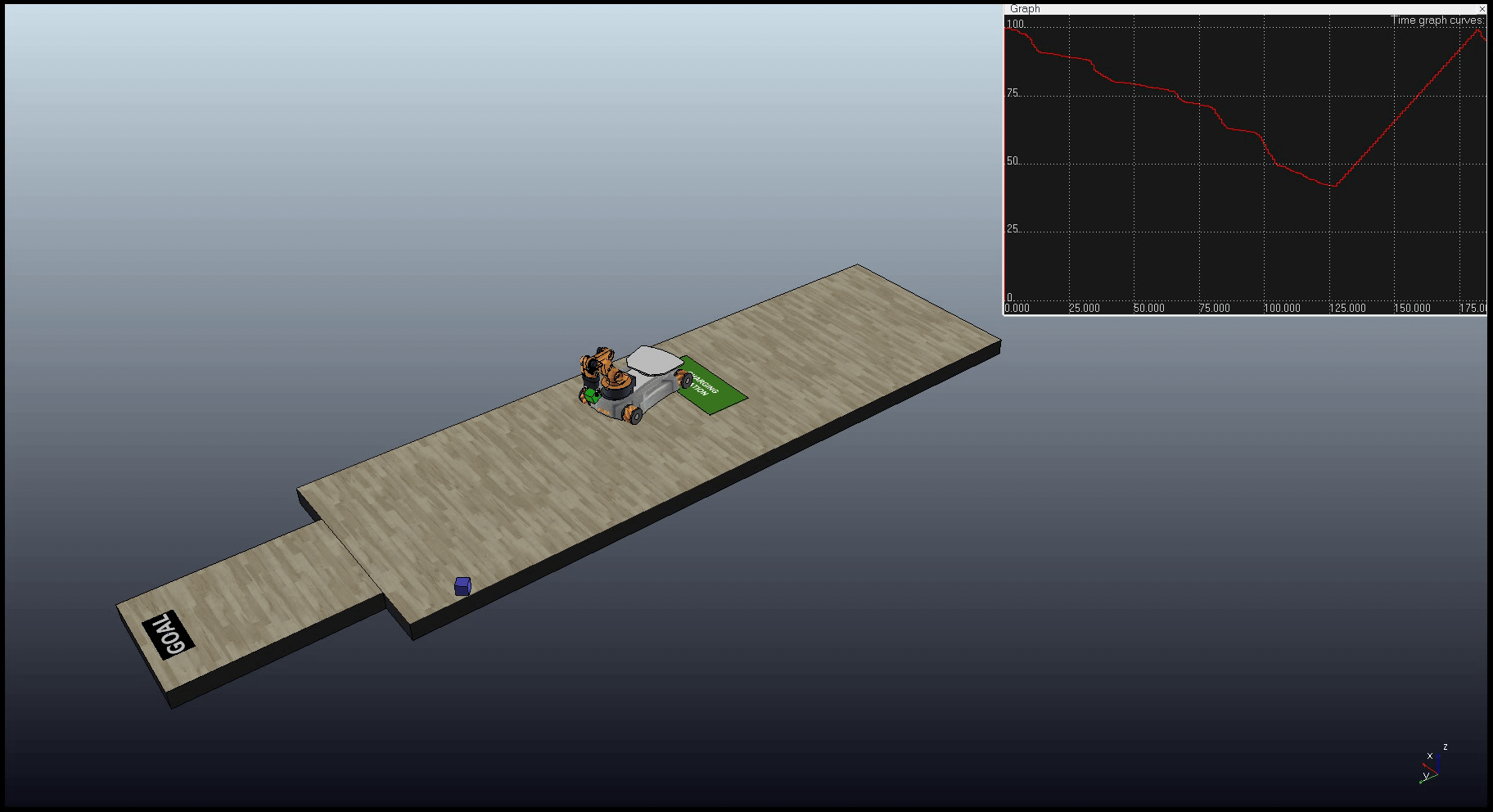}
                \caption{Once the battery is charged, the robot resumes its plan. }
                 \label{planning:SI.fig.youbotsafestep3}  
        \end{subfigure}   
        ~            
        \begin{subfigure}[b]{0.9\columnwidth}
                \centering
                \includegraphics[width=1\columnwidth,trim={5cm 5cm 28cm 15cm},clip]{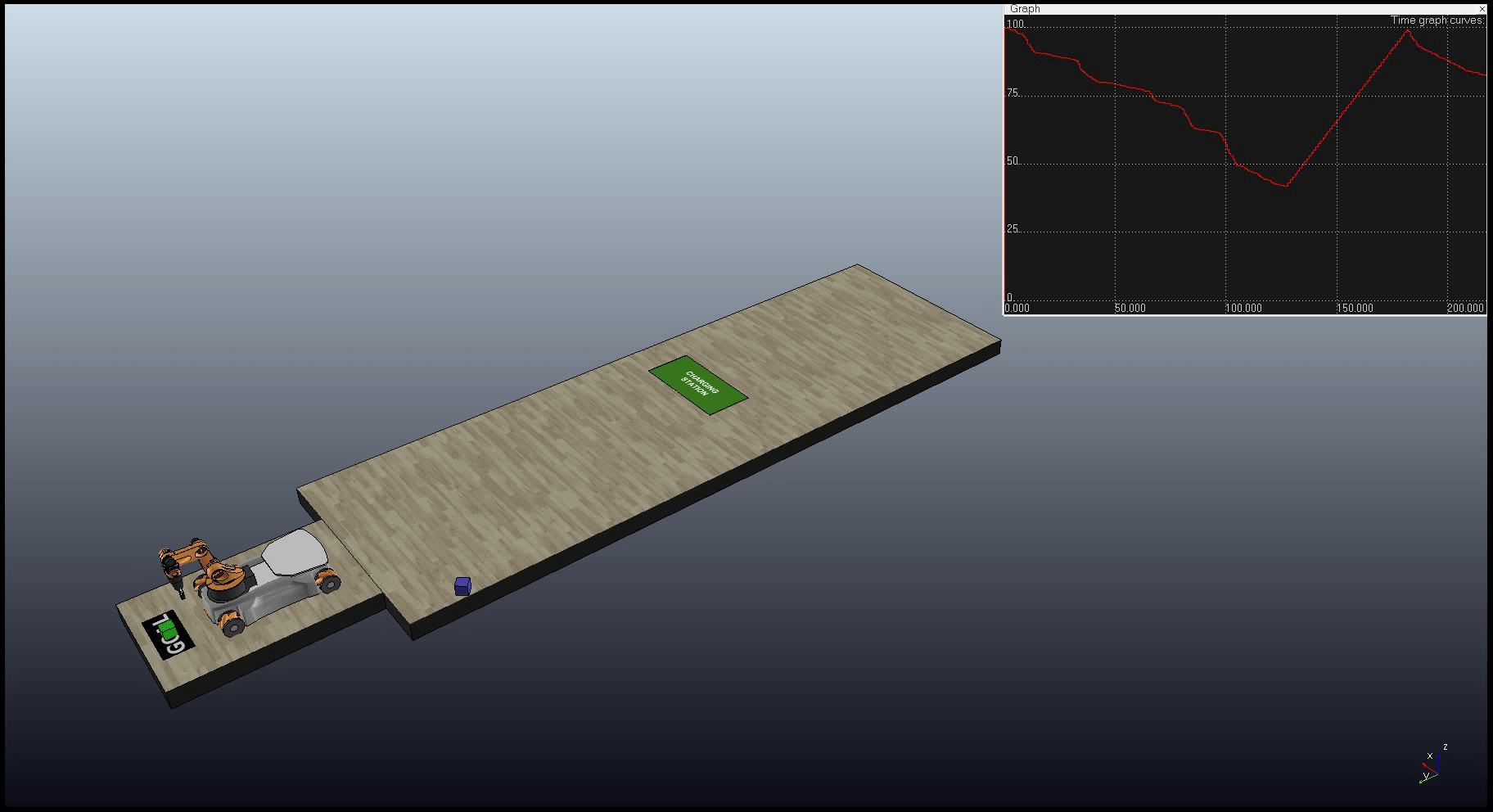}
                \caption{The robot places the object onto the desired location.}
                 \label{planning:SI.fig.youbotsafestep4}  
        \end{subfigure} 
        \caption{Execution of a KUKA Youbot  experiment illustrating safety.}
        \label{planning:SI.fig.yousafescreen2}
\end{figure}

\vspace*{\fill}

\clearpage

\begin{experiment}[Fault Tolerance]
In this experiment the robot is asked to perform the same task as in Experiment~\ref{planning:results.ex.1} with the main difference that the robot is equipped with an auxiliary arm and a fault can occur to either arm, causing the arm to stop functioning properly. The robot starts the execution as in the previous experiments (see Figure~\ref{planning:SI.fig.youbotftstep1}). However while the robot is approaching the goal area, the primary arm breaks, making the green cube fall on the ground (see Figure~\ref{planning:SI.fig.youbotftstep2}). The robot now tries to re-grasp the object with the primary arm, but this action fails since the grippers are no longer attached to the primary arm, hence the robot tries to grasp the robot with the auxiliary arm. However the auxiliary arm is too far from the object, and thus the robot has to move in a different position (see Figure~\ref{planning:SI.fig.youbotftstep3}) such that the object can  be grasped (see Figure~\ref{planning:SI.fig.youbotftstep4}) and the execution can continue.

\end{experiment}

\begin{figure}[h!]
      \begin{subfigure}[b]{1\columnwidth}
                \centering
                \includegraphics[width=0.95\columnwidth,trim={6cm 7cm 20cm 7cm},clip]{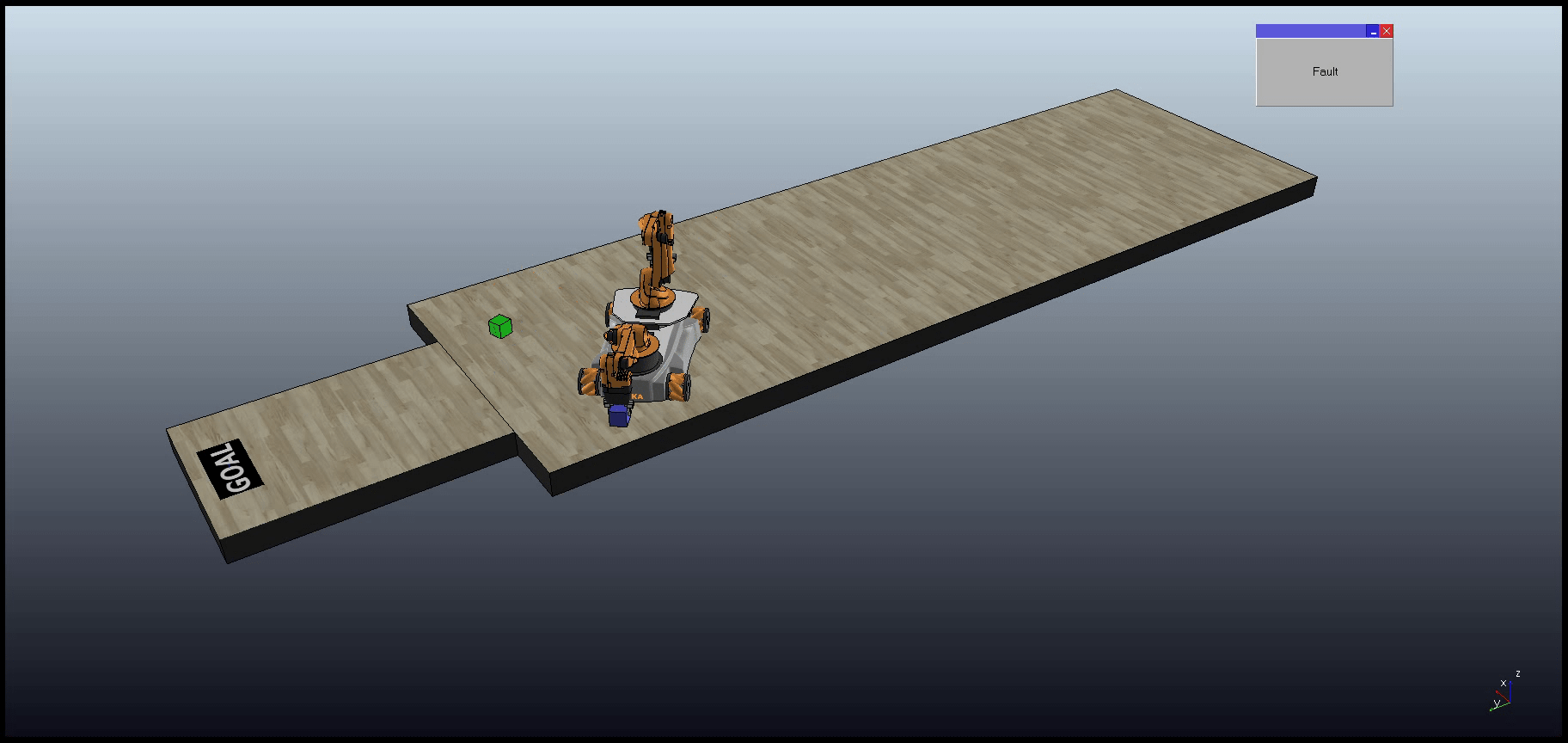}
                \caption{The robot moves the blue cube away from the path to the goal.}
                 \label{planning:SI.fig.youbotftstep1}  
        \end{subfigure}         
        ~
              \begin{subfigure}[b]{1\columnwidth}
                \centering
                \includegraphics[width=0.95\columnwidth,trim={6cm 7cm 20cm 7cm},clip]{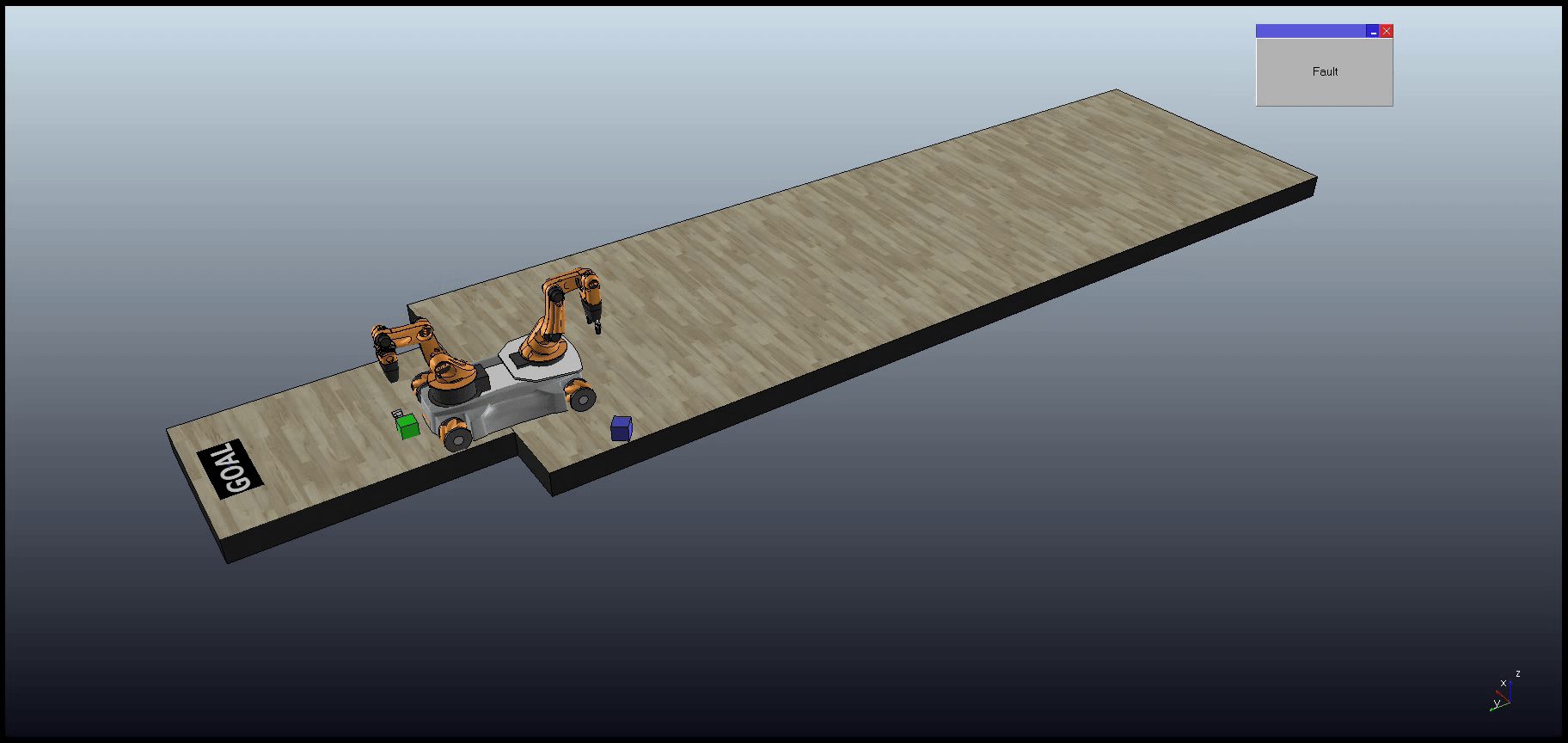}
                \caption{A fault occurs on the primary arm (the grippers break) and the green cube falls to the floor.}
                 \label{planning:SI.fig.youbotftstep2}  
        \end{subfigure} 
        \caption{Execution of a KUKA Youbot  experiment illustrating fault tolerance.}
        \label{planning:SI.fig.youftscreen2}
\end{figure}
\clearpage

\begin{figure}[t]
              \begin{subfigure}[b]{1\columnwidth}
                \centering
                \includegraphics[width=0.95\columnwidth,trim={6cm 7cm 20cm 7cm},clip]{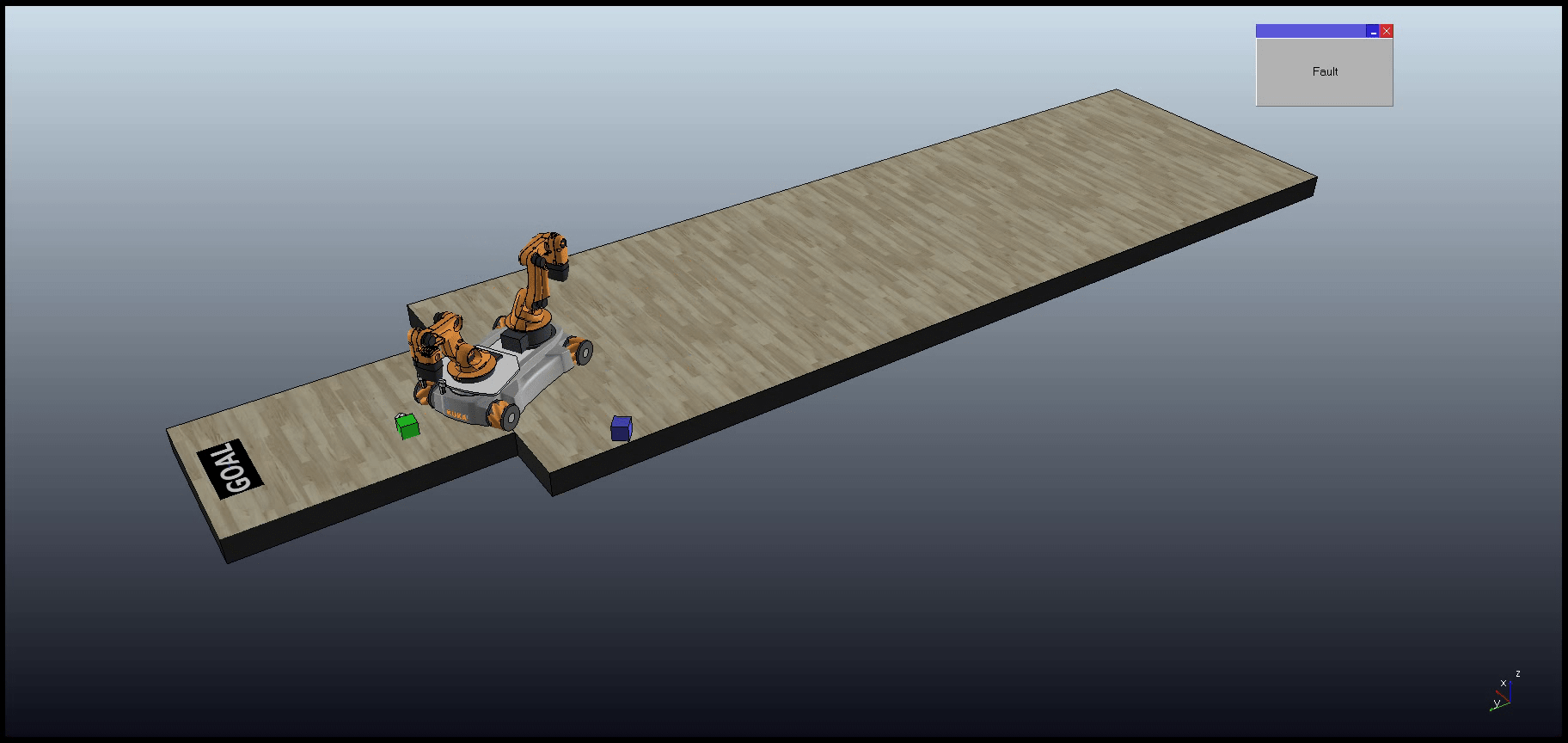}
                \caption{The robot rotates to have the object closer to the auxiliary arm.}
                 \label{planning:SI.fig.youbotftstep3}  
        \end{subfigure}      
        ~   
              \begin{subfigure}[b]{1\columnwidth}
                \centering
                \includegraphics[width=0.95\columnwidth,trim={6cm 7cm 20cm 7cm},clip]{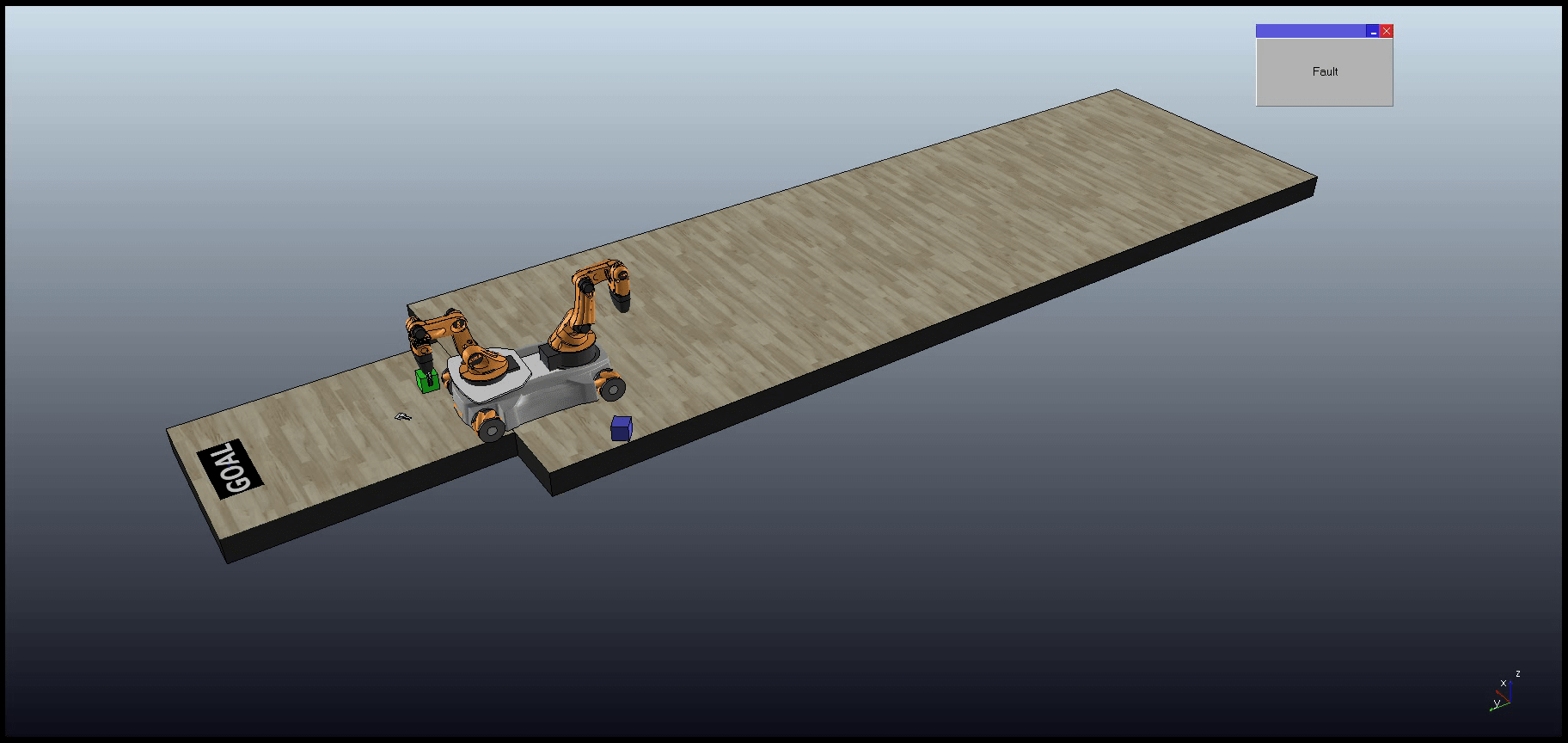}
                \caption{The robot grasps the object with the auxiliary arm.}
                 \label{planning:SI.fig.youbotftstep4}  
        \end{subfigure}        
        \caption{Execution of a KUKA Youbot  experiment illustrating fault tolerance.}
        \label{planning:SI.fig.youftscreen2bis}
\end{figure}

\vspace*{\fill}

\clearpage
\begin{experiment}[Dynamic Environment]
In this experiment the single armed version of the robot co-exists with other uncontrolled external robots. The robot is asked to place the green cube in the goal area on the opposite side of the room. The robot starts picking up the green cube and moves towards an obstructing object (a blue cube) to place it to the side (see Figure~\ref{planning:DSI.fig.youbotstep1}). Being single armed, the robot has to ungrasp the green cube (see Figure~\ref{planning:DSI.fig.youbotstep2}) to grasp the blue one (see Figure~\ref{planning:DSI.fig.youbotstep3}). While the robot is placing the blue cube to the side, an external robot places a new object between the controlled robot and the green cube (see Figure~\ref{planning:DSI.fig.youbotstep4}). The plan of the robot is then expanded to include the removal of this new object (see Figure~\ref{planning:DSI.fig.youbotstep5}). Then the robot can continue its plan by re-picking the green cube, without replaning. Now the robot approaches the yellow cube to remove it (see Figure~\ref{planning:DSI.fig.youbotstep6}), but before the robot is able to grasp the yellow cube, another external robot picks up the yellow cube (see Figure~\ref{planning:DSI.fig.youbotstep7}) and places it to the side. The subplan for removing the yellow cube is skipped (without replaning) and the robot  continues its task until the goal is reached (see Figure~\ref{planning:DSI.fig.youbotstep8}).	
\end{experiment}

\begin{figure}[h]
        \centering
        \begin{subfigure}[b]{1\columnwidth}
                \centering
                \includegraphics[width=0.85\columnwidth,trim={0 0 1cm 7cm},clip]{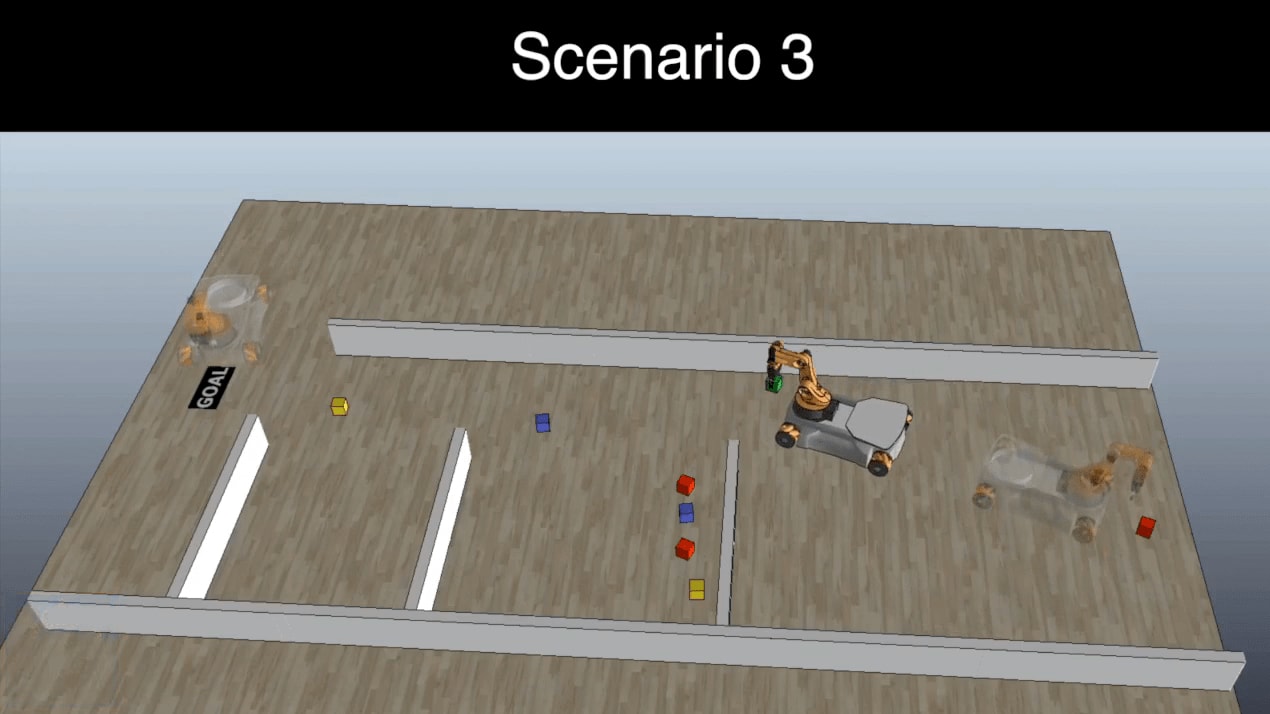}
                \caption{The robot picks the desired object, a green cube.}
                \label{planning:DSI.fig.youbotstep1}              
        \end{subfigure}       
               
        \begin{subfigure}[b]{1\columnwidth}
                \centering
                \includegraphics[width=0.85\columnwidth,trim={0 0 1cm 7cm},clip]{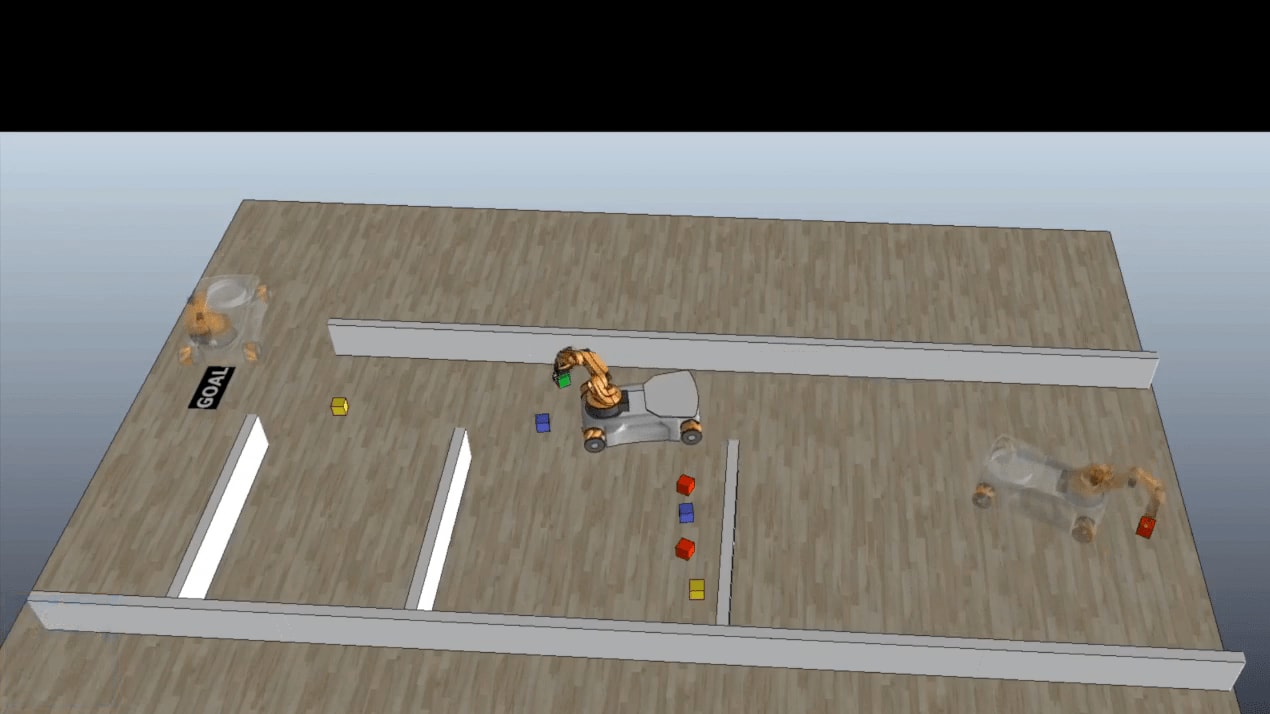}
                \caption{The blue cube obstructs the path to the goal region. The robot drops the green cube in order to pick the blue cube.}
 			\label{planning:DSI.fig.youbotstep2}  
         \end{subfigure}
        \caption{Execution of a complex KUKA Youbot experiment.}
        \label{planning:DSI.fig.youscreen}
\end{figure}
\clearpage

\begin{figure}[h]
\centering
        \begin{subfigure}[b]{1\columnwidth}
                \centering
                \includegraphics[width=0.85\columnwidth,trim={0 0 1cm 7cm},clip]{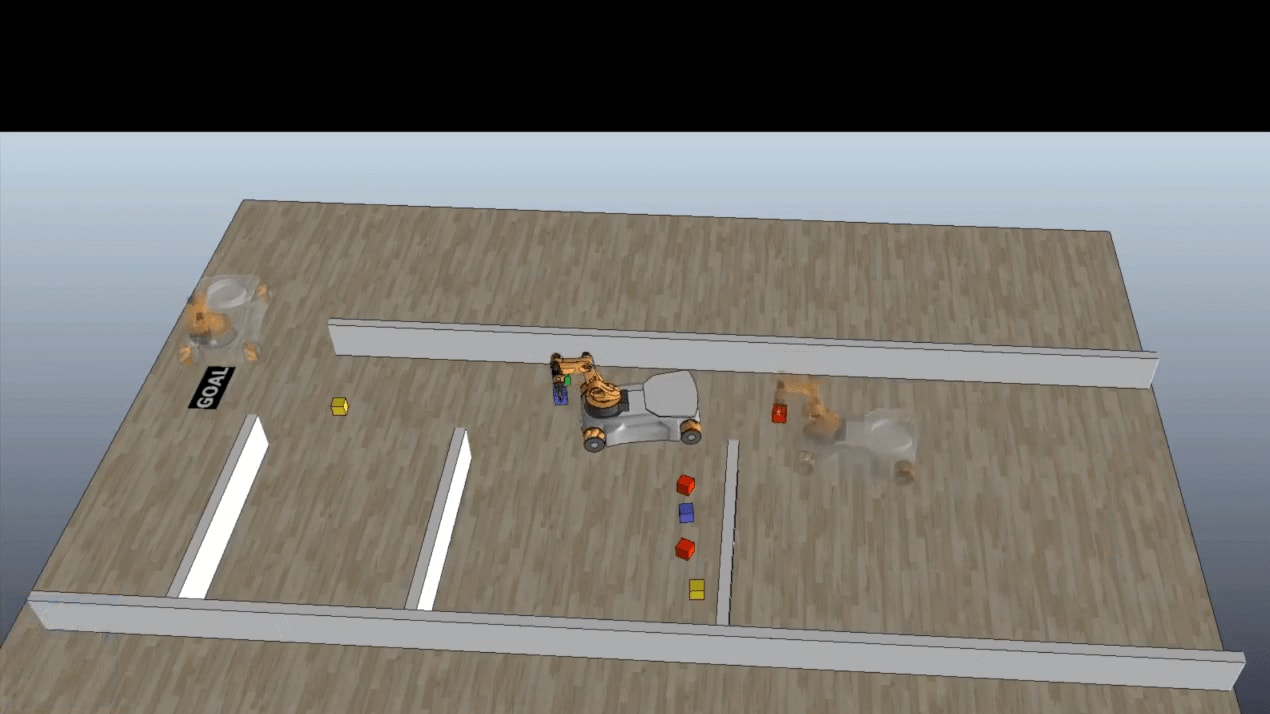}
                \caption{The robot picks the blue cube.}
                 \label{planning:DSI.fig.youbotstep3}  
        \end{subfigure}   
        ~            
        \begin{subfigure}[b]{1\columnwidth}
                \centering
                \includegraphics[width=0.85\columnwidth,trim={0 0 1cm 7cm},clip]{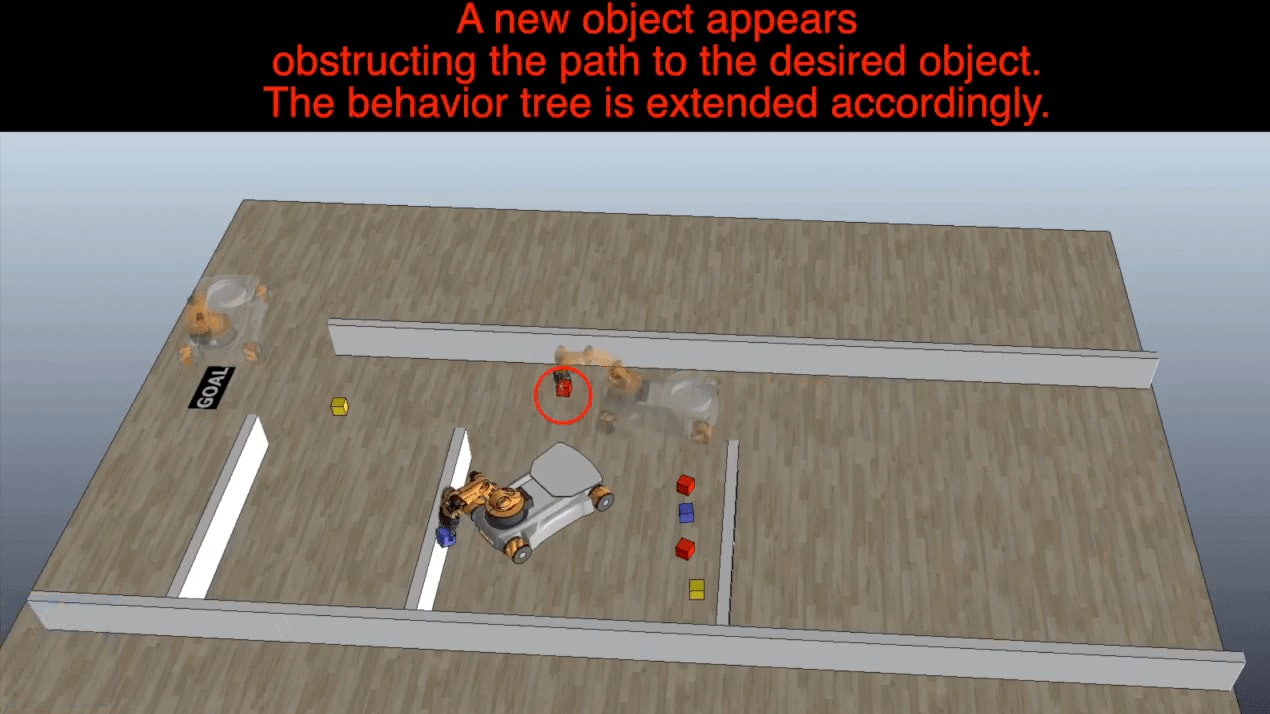}
                \caption{While the robot moves the blue cube away from the path to the goal, an external agent places a red cube between the robot and the green cube.}
                 \label{planning:DSI.fig.youbotstep4}  
        \end{subfigure} 
~

\centering

      \begin{subfigure}[b]{1\columnwidth}
                \centering
                \includegraphics[width=0.85\columnwidth,trim={0 0 1cm 7cm},clip]{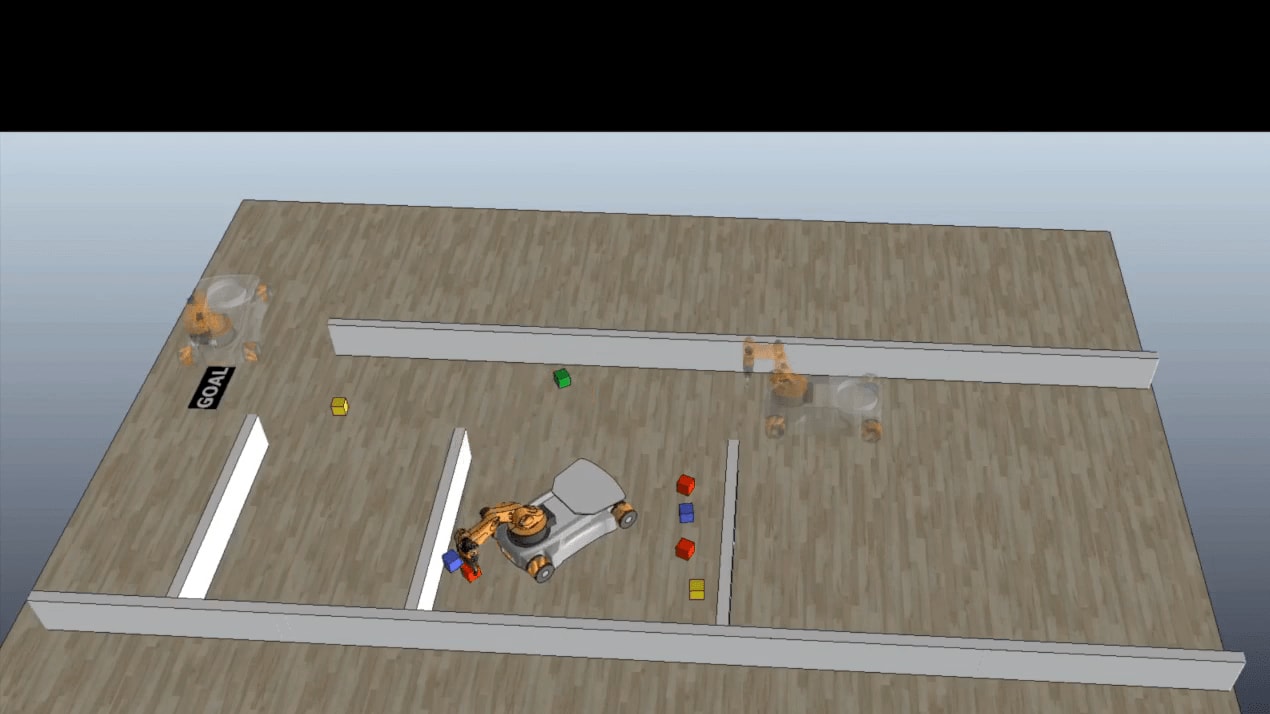}
                \caption{The robot moves the red cube away from the path to the goal.  }
                 \label{planning:DSI.fig.youbotstep5}  
        \end{subfigure}         
              \caption{Execution of a complex KUKA Youbot experiment.}
        \label{planning:DSI.fig.youscreen2}
\end{figure}
        \begin{figure}[h]
              \begin{subfigure}[b]{1\columnwidth}
                \centering
                \includegraphics[width=0.85\columnwidth,trim={0 0 1cm 7cm},clip]{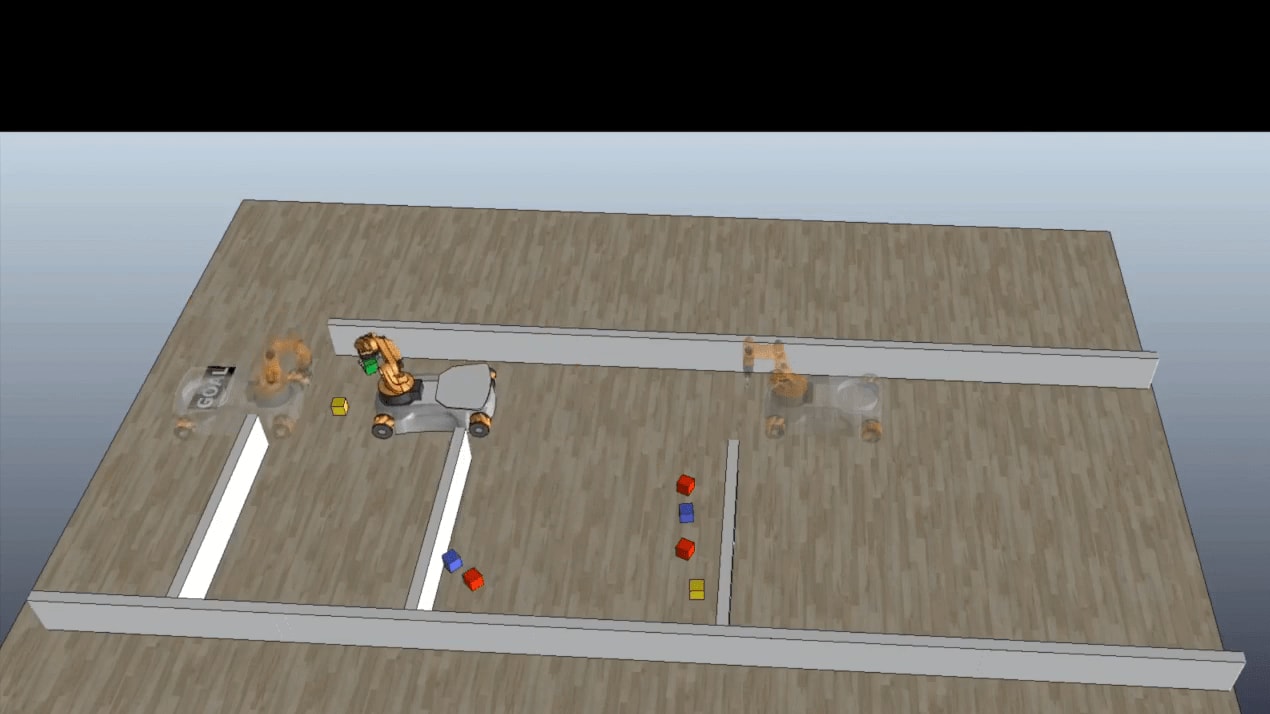}
                \caption{The yellow cube obstructs the path to the goal region. The robot drops the green cube in order to pick the yellow cube.}
                 \label{planning:DSI.fig.youbotstep6}  
        \end{subfigure} 

        \centering
              \begin{subfigure}[b]{1\columnwidth}
                \centering
                \includegraphics[width=0.85\columnwidth,trim={0 0 1cm 7cm},clip]{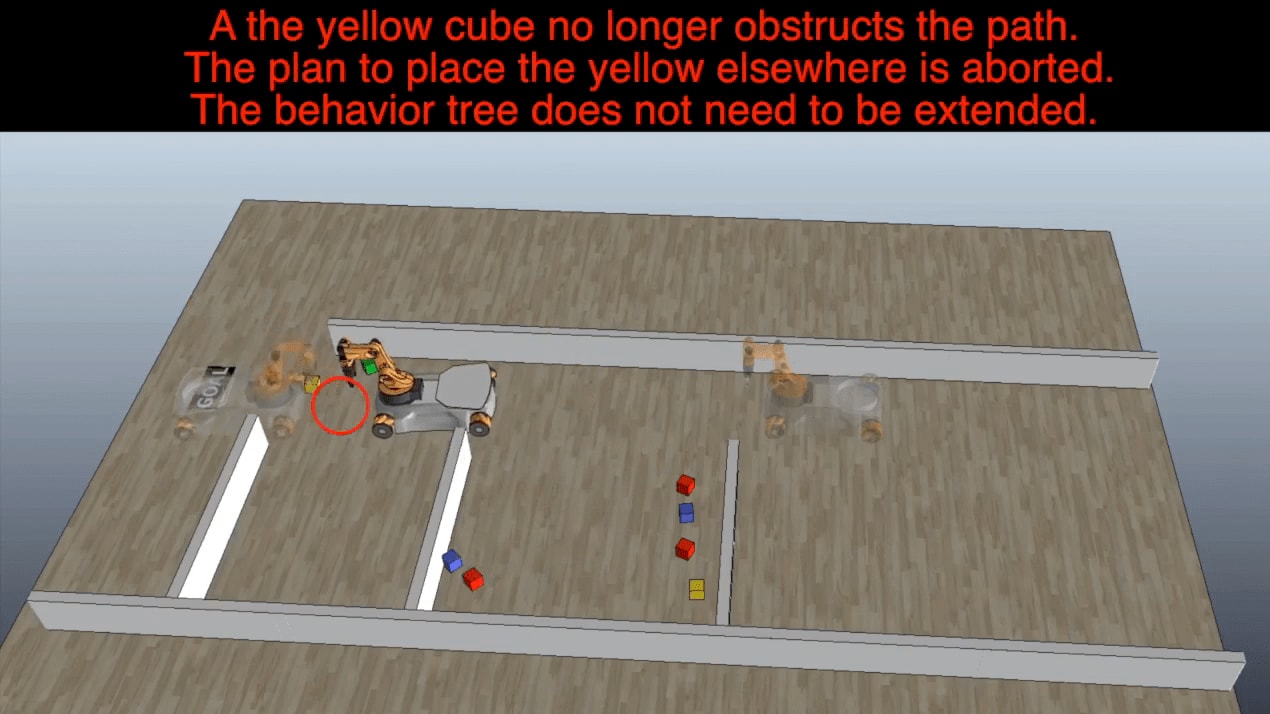}
                \caption{While the robot approaches the yellow cube, an external agent moves the yellow cube away.}
                 \label{planning:DSI.fig.youbotstep7}  
        \end{subfigure}      
        ~   
              \begin{subfigure}[b]{1\columnwidth}
                \centering
                \includegraphics[width=0.85\columnwidth,trim={0 0 1cm 7cm},clip]{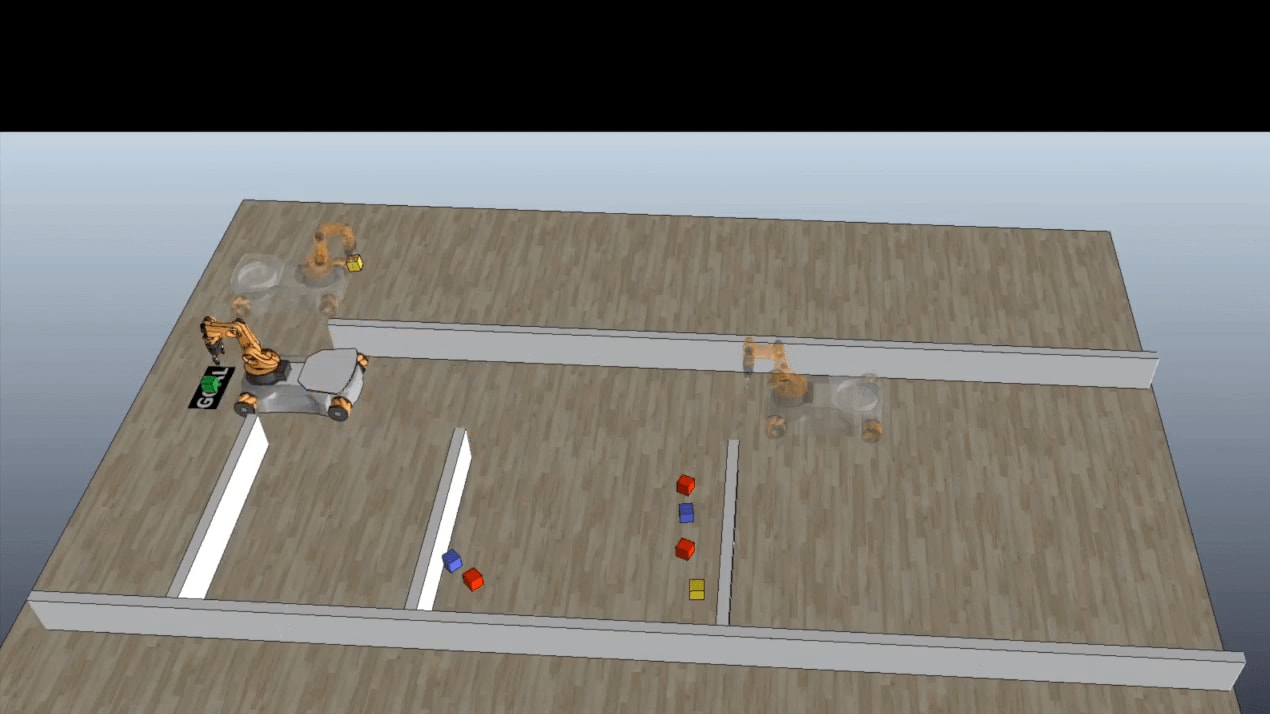}
                \caption{The robot picks the green cube and places it onto the goal region. }
                 \label{planning:DSI.fig.youbotstep8}  
        \end{subfigure}        
        \caption{Execution of a complex KUKA Youbot experiment.}
        \label{planning:DSI.fig.youscreen2bis}
\end{figure}
\clearpage

\subsubsection{ABB Yumi experiments}

In these scenarios, an ABB Yumi has to assemble a cellphone whose parts are scattered across a table, see Figure~\ref{planning:SI.fig.yumiscreen}. 
The robot is equipped with two arms with simple parallel grippers, which are not suitable for dexterous manipulation.
Furthermore, some parts must be grasped in a particular way to enable the assembly operation. 
\begin{experiment}
In this experiment, the robot needs to re-orient some cellphone's parts to expose them for assembly. Due to the gripper design, the robot must reorient the parts by performing multiple grasps transferring the part to the other gripper, see Figure~\ref{planning:SI.fig.Yumistep2}, effectively changing its orientation (see Figures~\ref{planning:SI.fig.Yumistep3}-\ref{planning:SI.fig.Yumistep4}).
\end{experiment}
}


\begin{figure}[h!]
        \centering
        \begin{subfigure}[b]{1\columnwidth}
                \centering
                \includegraphics[width=0.7\columnwidth,trim={11cm 4cm 16cm 7cm},clip]{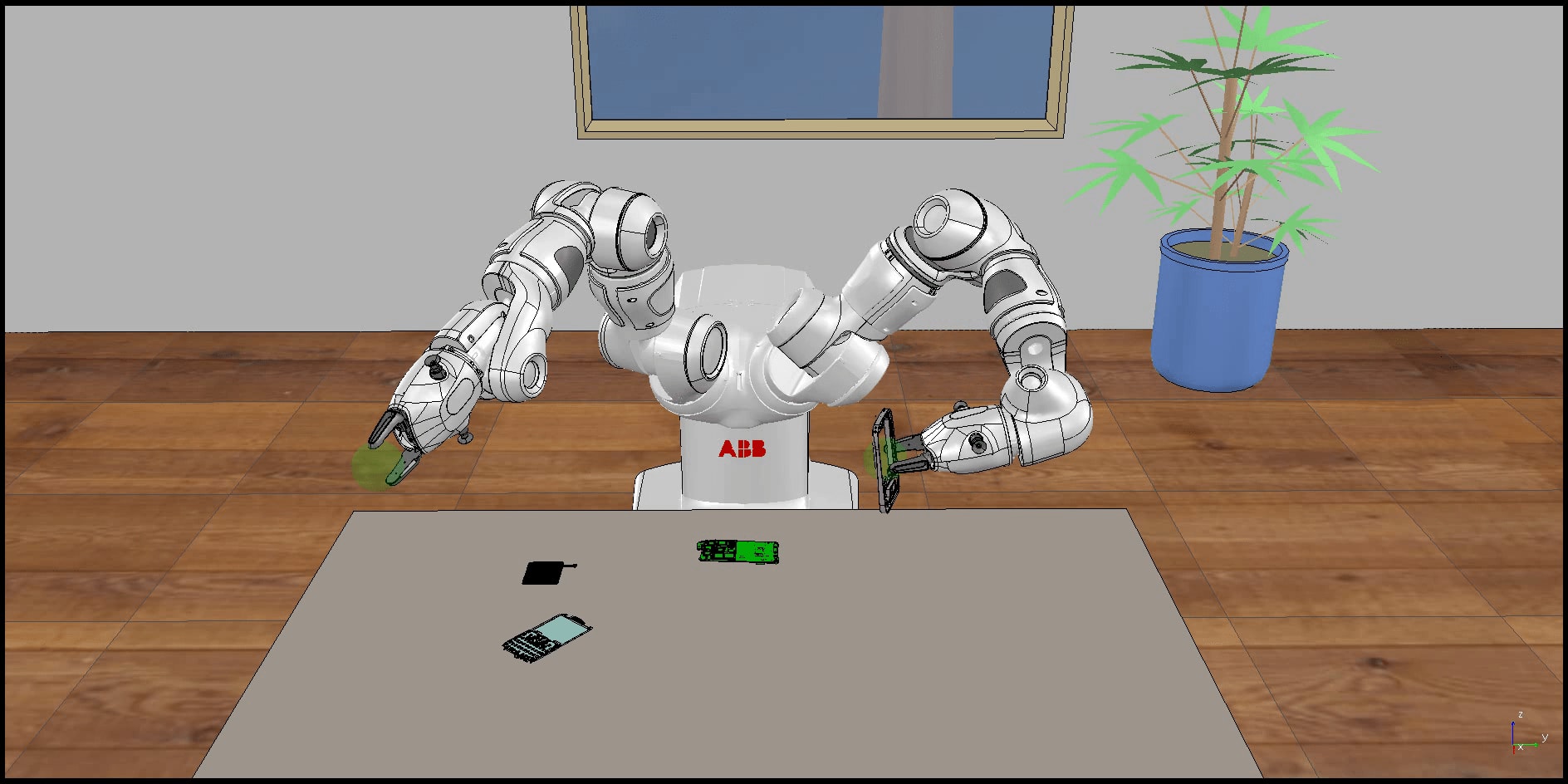}
                \caption{The robot picks the cellphone's chassis. The chassis cannot be assembled with this orientation. }
                \label{planning:SI.fig.Yumistep1}              
        \end{subfigure}       
        ~         
        \begin{subfigure}[b]{1\columnwidth}
                \centering
                \includegraphics[width=0.7\columnwidth,trim={11cm 4cm 16cm 7cm},clip]{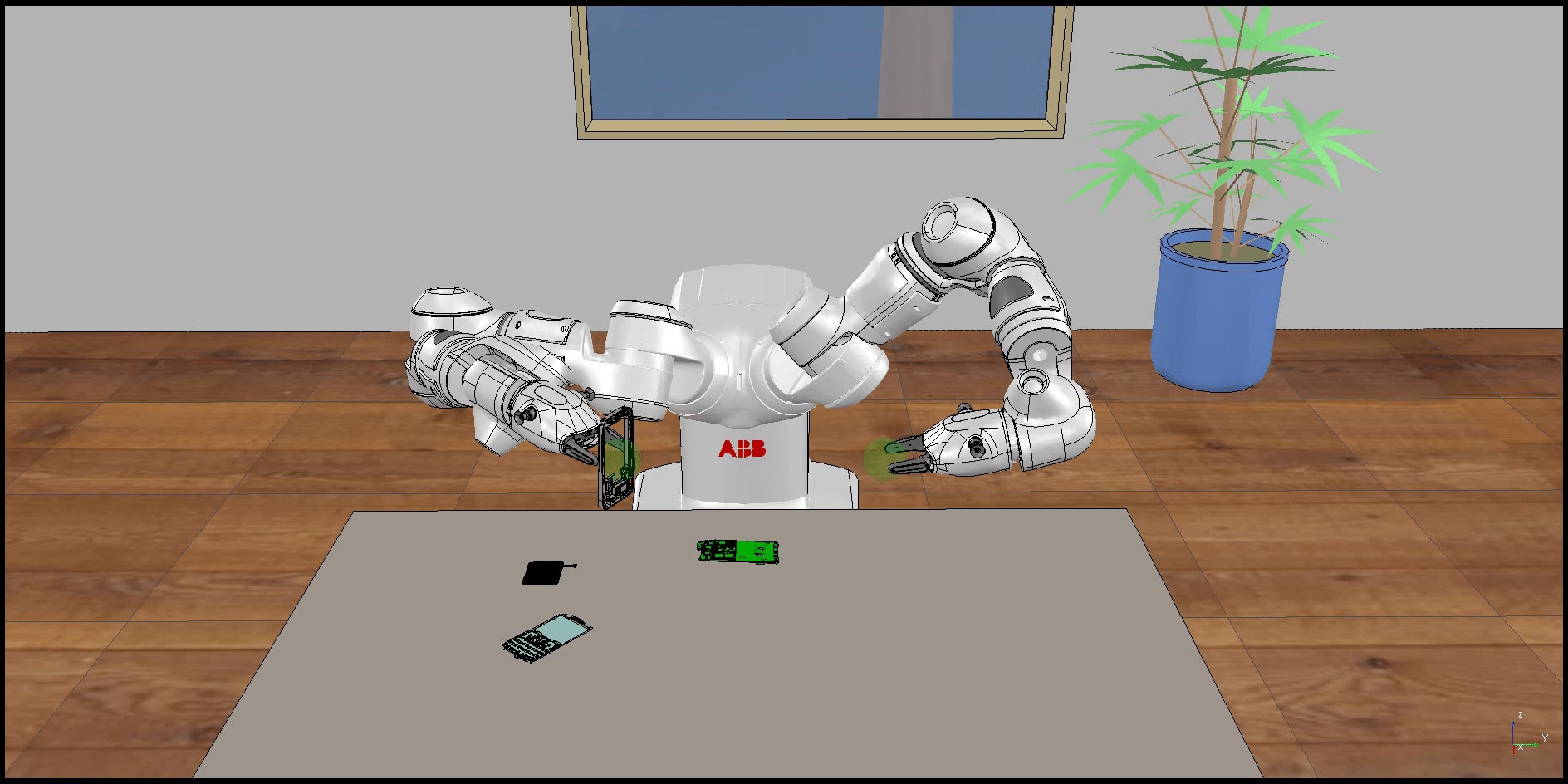}
                \caption{The chassis is handed over the other gripper. }
 			\label{planning:SI.fig.Yumistep2}  
         \end{subfigure}
        \caption{Execution of an ABB Yumi experiment.}
        \label{planning:SI.fig.yumiscreen}
\end{figure}
\clearpage
\begin{figure}[h]
\centering
        \begin{subfigure}[b]{1\columnwidth}
                \centering
                \includegraphics[width=0.7\columnwidth,trim={11cm 4cm 16cm 7cm},clip]{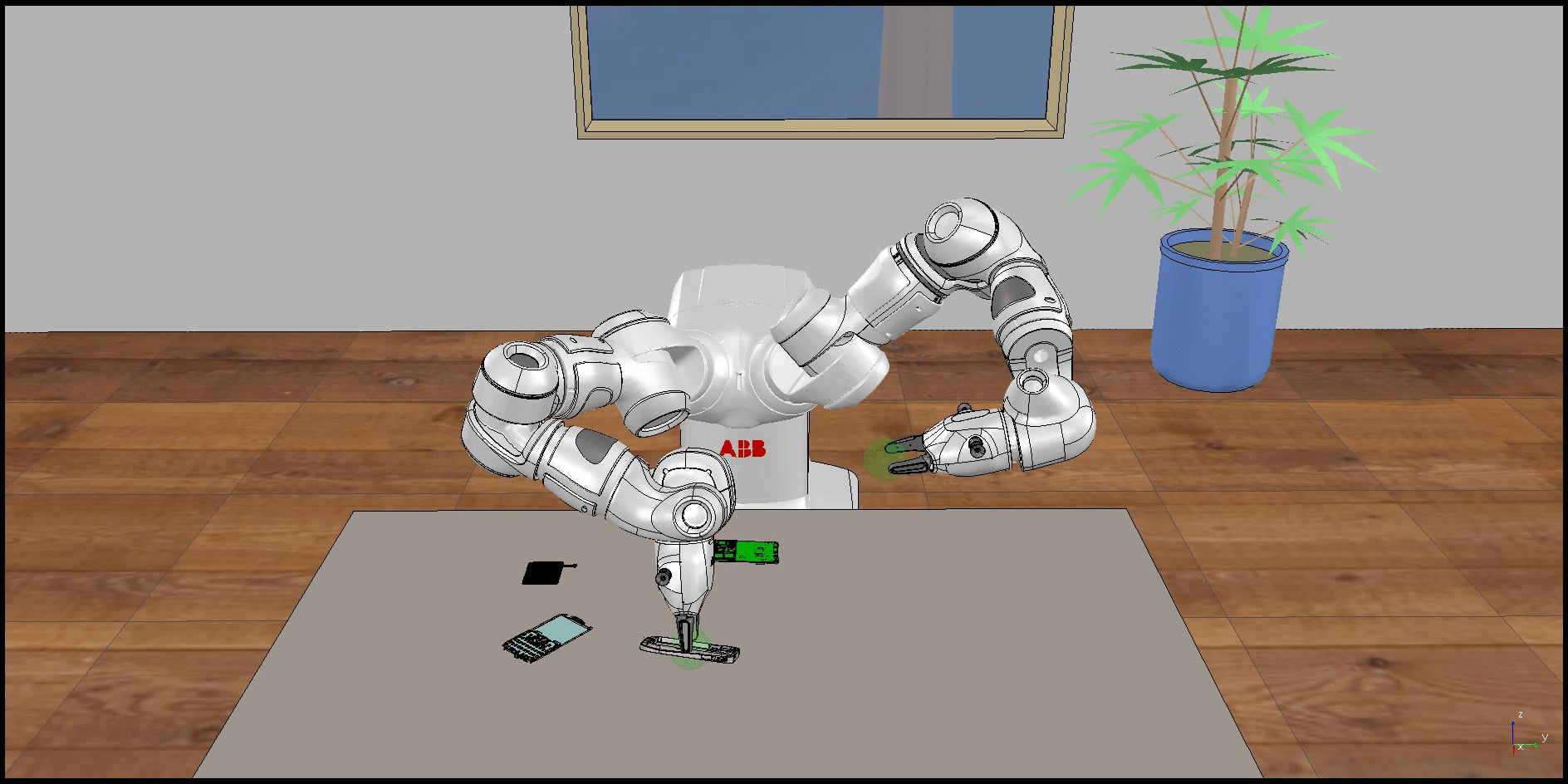}
                \caption{The chassis is placed onto the table with a different orientation than before (the opening part is facing down now).}
                 \label{planning:SI.fig.Yumistep3}  
        \end{subfigure}   
        ~            
        \begin{subfigure}[b]{1\columnwidth}
                \centering
                \includegraphics[width=0.7\columnwidth,trim={11cm 4cm 16cm 7cm},clip]{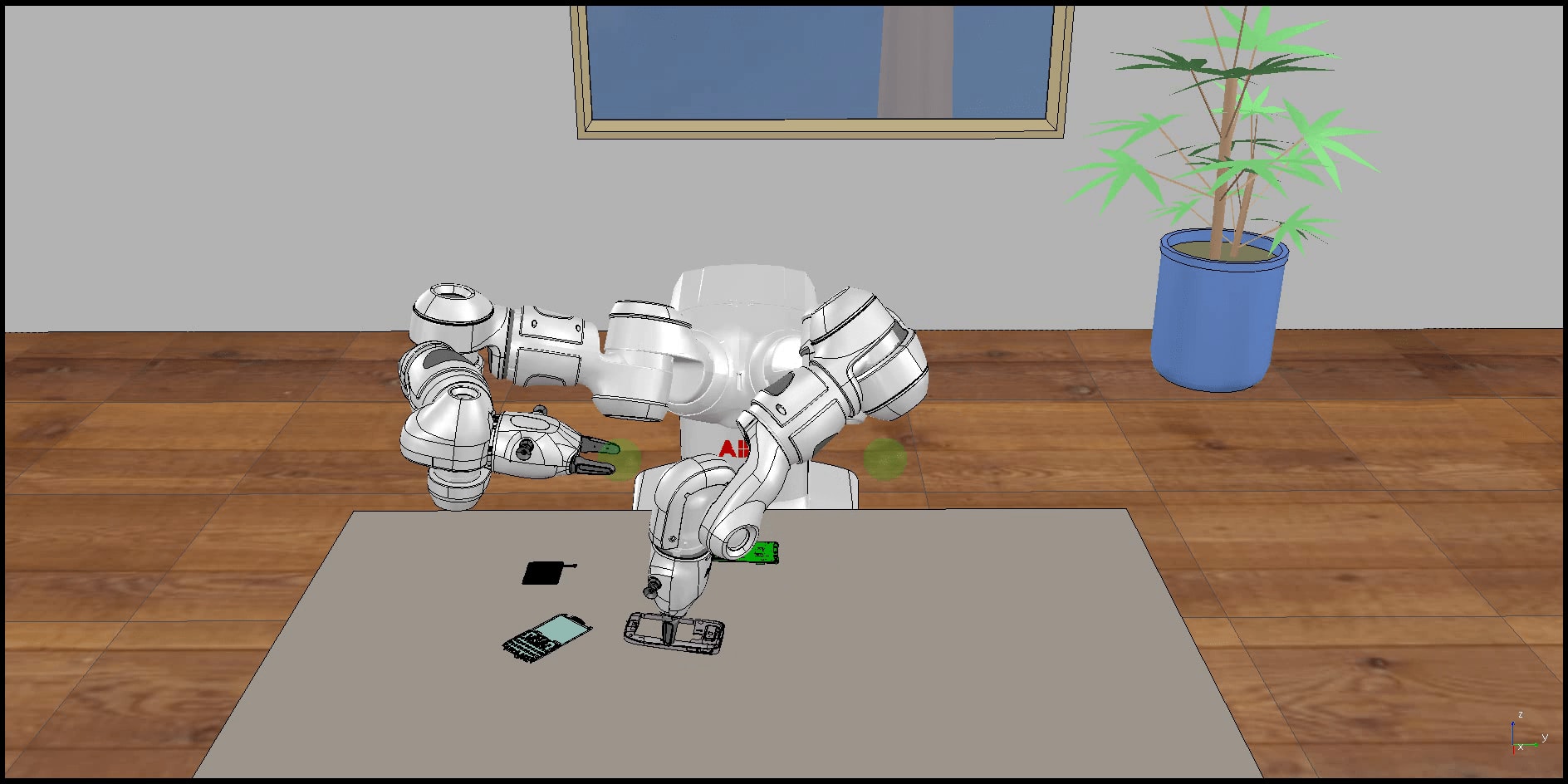}
                \caption{The robot picks the chassis with the new orientation. }
                 \label{planning:SI.fig.Yumistep4}  
        \end{subfigure} 
~

\centering

      \begin{subfigure}[b]{1\columnwidth}
                \centering
                \includegraphics[width=0.7\columnwidth,trim={11cm 4cm 16cm 7cm},,clip]{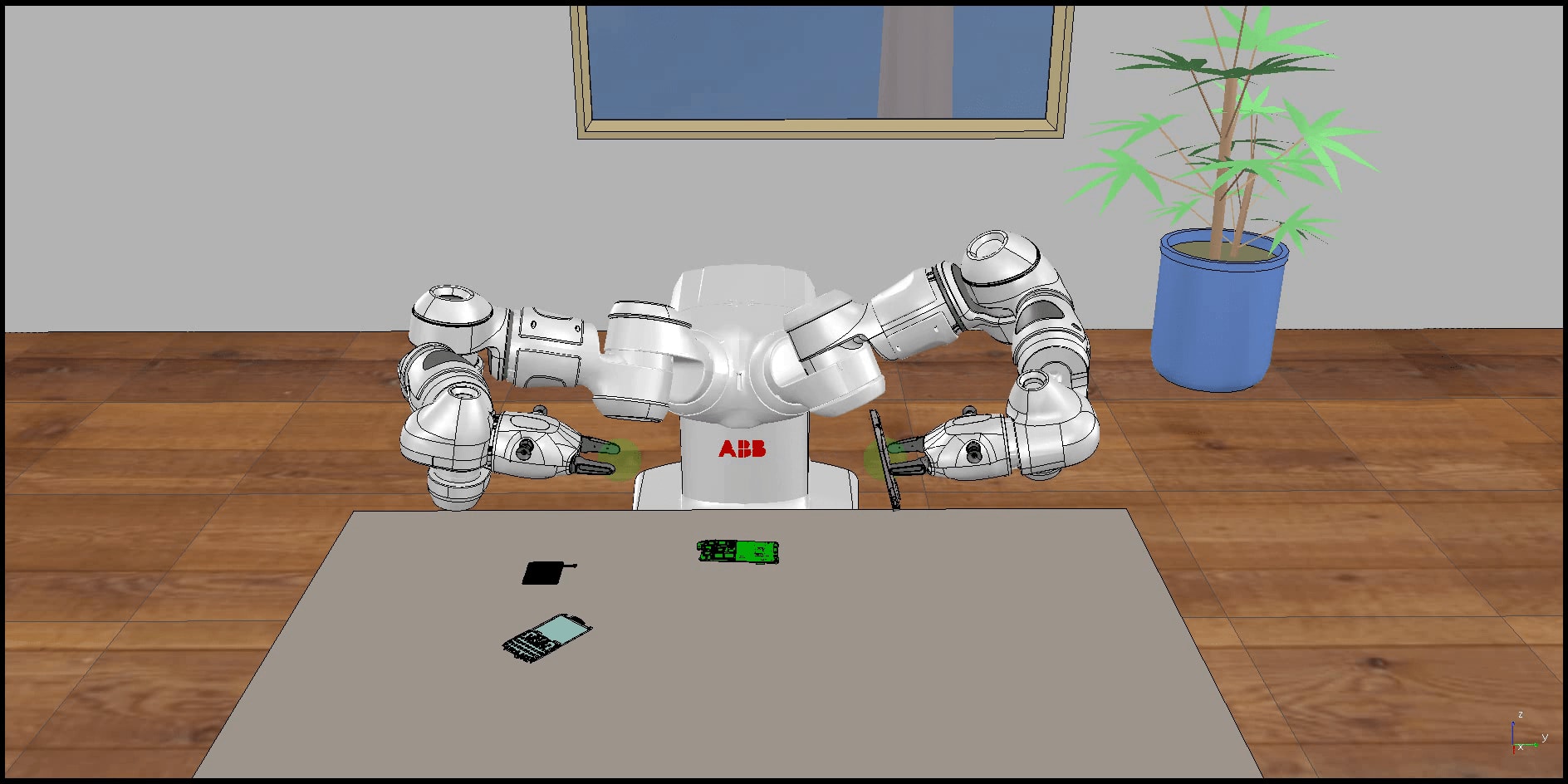}
                \caption{The chassis can be assembled with this orientation.}
                 \label{planning:SI.fig.Yumistep5}  
        \end{subfigure}         
                \label{planning:SI.fig.yumiscreen2}
                        \caption{Execution of an ABB Yumi experiment.}
\end{figure}
\clearpage
\begin{figure}[h]
              \begin{subfigure}[b]{1\columnwidth}
                \centering
                \includegraphics[width=0.7\columnwidth,trim={11cm 4cm 16cm 7cm},clip]{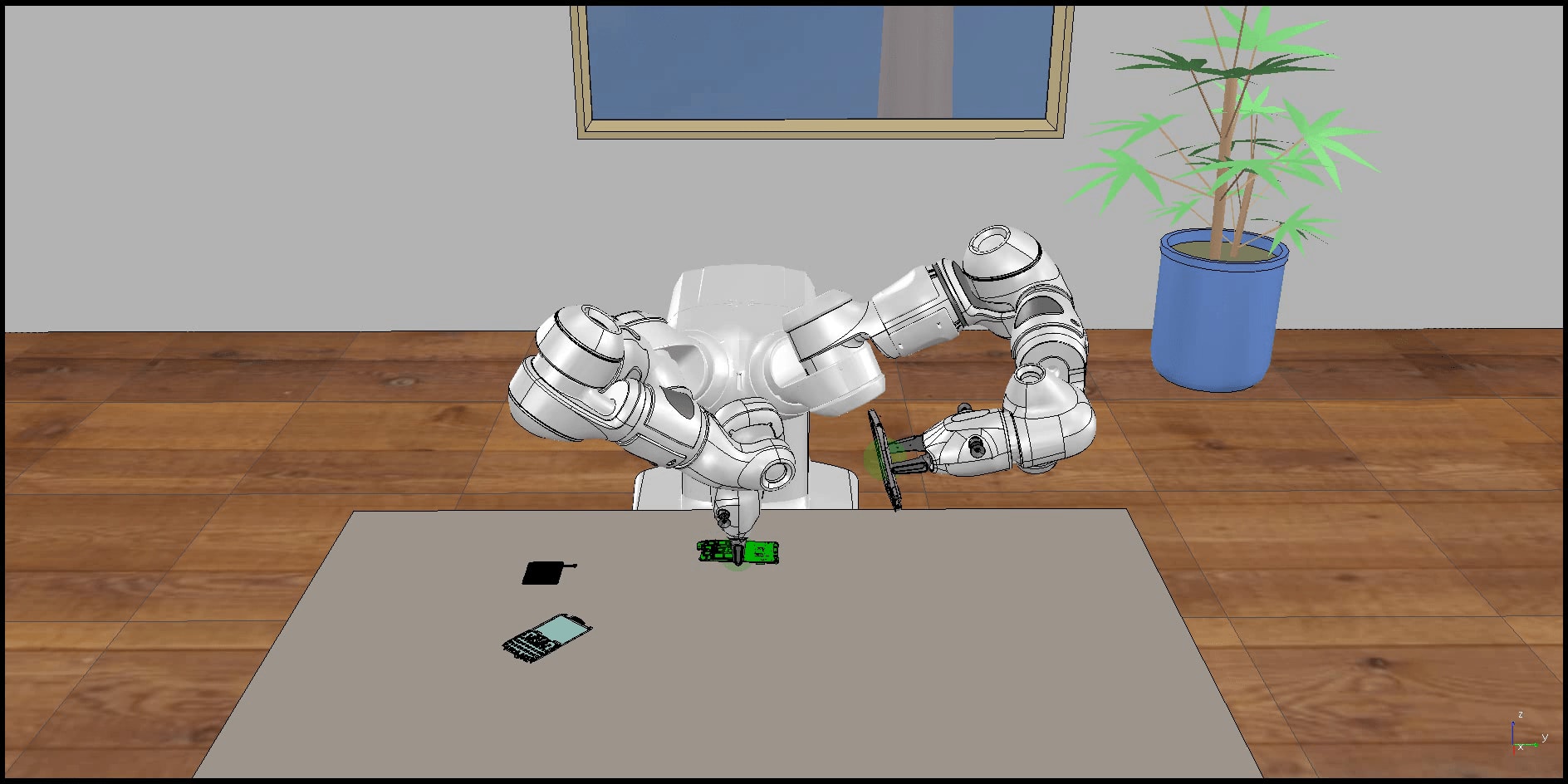}
                \caption{The robot pick the next cellphone's part to be assembled (the motherboard). }
                 \label{planning:SI.fig.Yumistep6}  
        \end{subfigure} 
        ~
        
        \centering
              \begin{subfigure}[b]{1\columnwidth}
                \centering
                \includegraphics[width=0.7\columnwidth,trim={11cm 4cm 16cm 7cm},,clip]{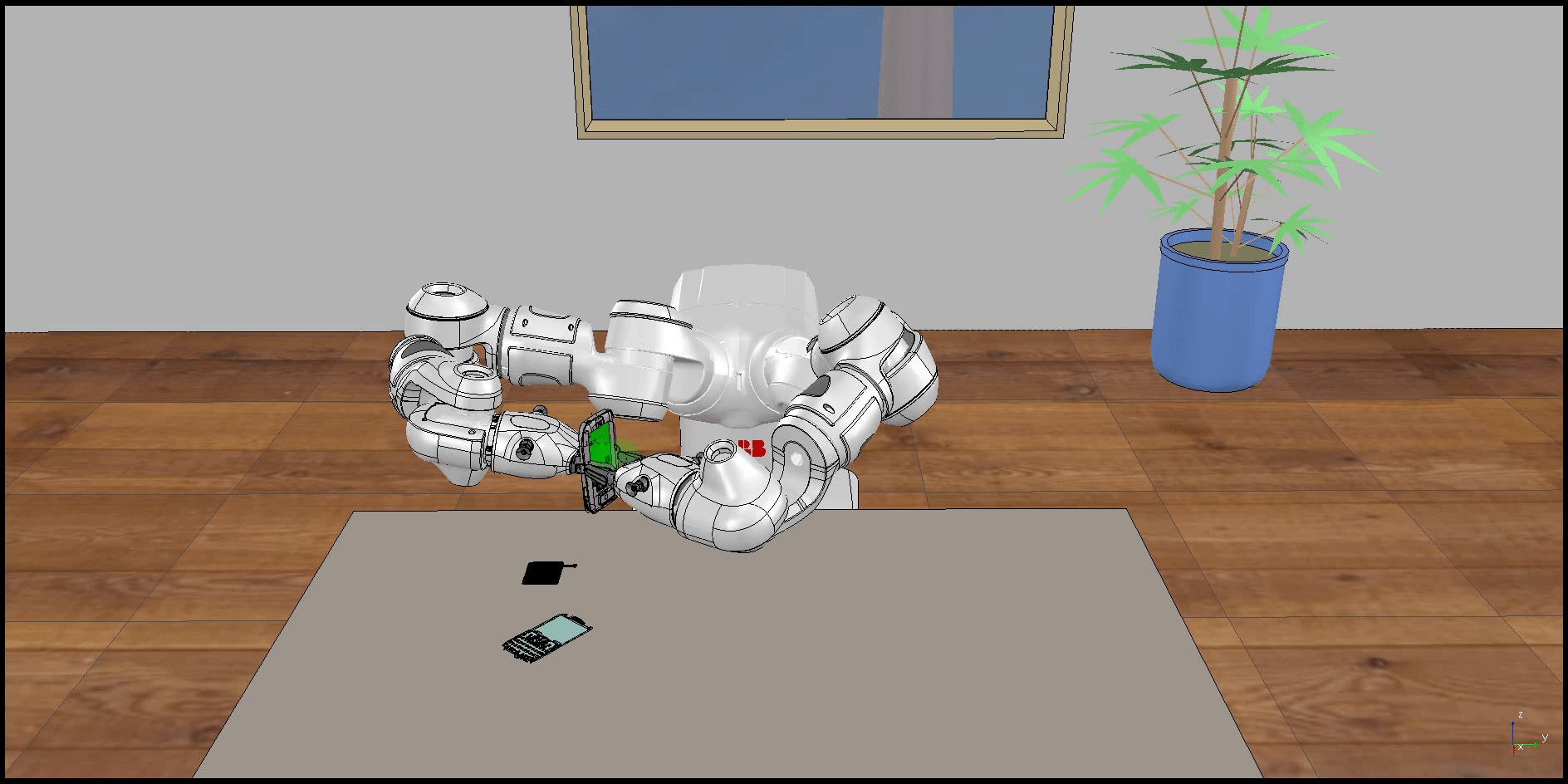}
                \caption{The motherboard and the chassis are assembled.}
                 \label{planning:SI.fig.Yumistep7}  
        \end{subfigure}      
        ~   
              \begin{subfigure}[b]{1\columnwidth}
                \centering
                \includegraphics[width=0.7\columnwidth,trim={11cm 4cm 16cm 7cm},clip]{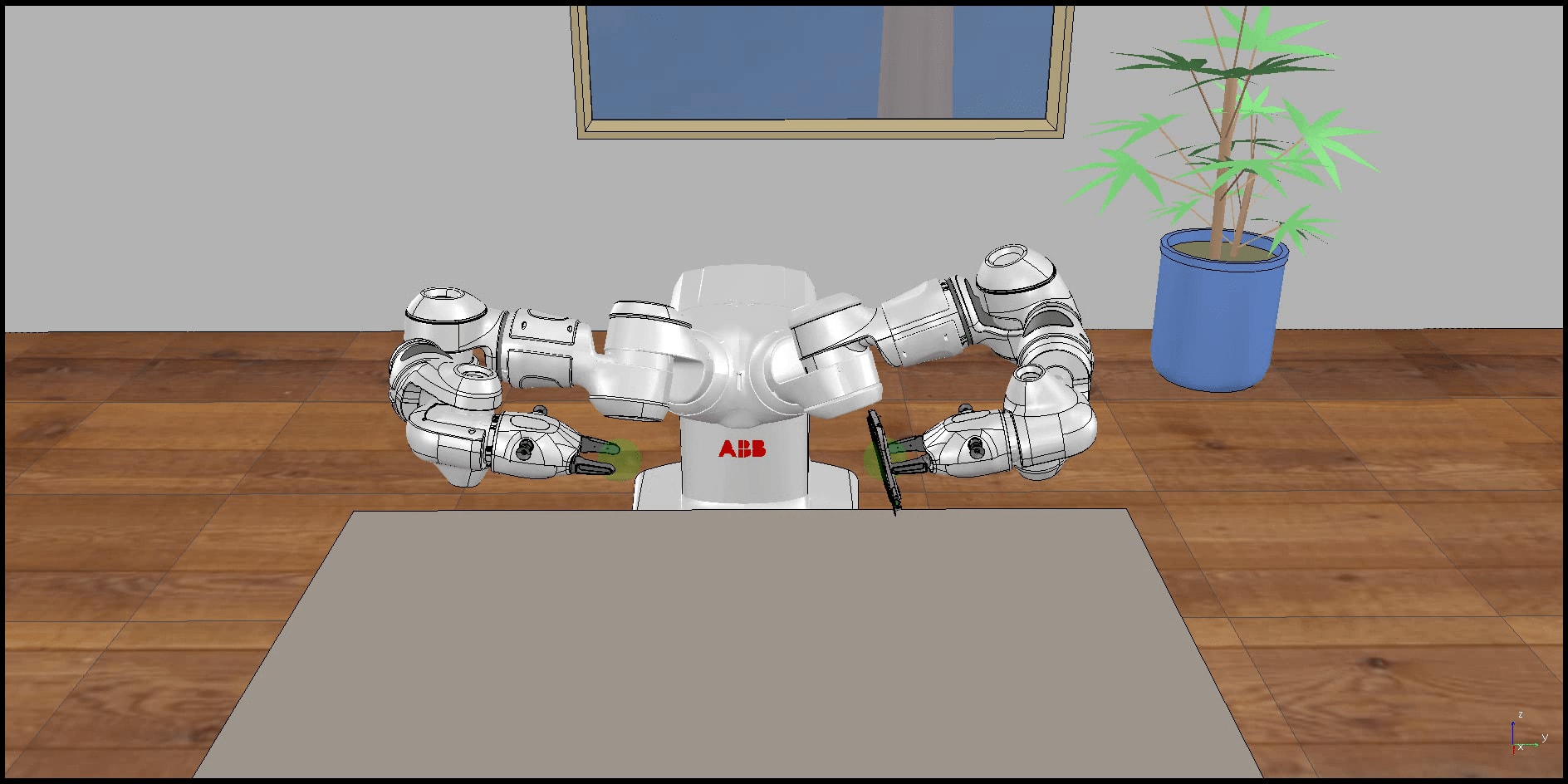}
                \caption{The robot assembles the cellphone correctly. }
                 \label{planning:SI.fig.Yumistep8}  
        \end{subfigure}        
        \caption{Execution of an ABB Yumi experiment.}
        \label{planning:SI.fig.yumiscreen2bis}
\end{figure}
\clearpage

\section{Planning using A Behavior Language (ABL)}
\label{planning.MS}

To contrast the PA-BT approach described above we will more briefly present planning using A Behavior Language
\footnote{In the first version of ABL, the tree structure that stores all the goals is called \say{Active Behavior Tree}. This tree 
is related to, but different from the BTs we cover in this book (e.g. no Fallbacks and no ticks).
  Later work used the classic BT formulation also for ABL.}
 (ABL, pronounced ``able")~\cite{weber2010reactive,weber2011building}. 

ABL was designed for the dialoge game Facade \cite{mateas2002abl}, but is also appreciated for its ability to handle the planning and acting on multiple scales 
that is often needed in both robotics and games, and in particular essential in so-called Real-Time Strategy games.
In such games, events take place both on a long term time scale, where strategic decisions has to be made regarding e.g., what buildings to construct
in the next few minutes, and where to locate them, and on a short term  time scale, where tactical decisions has to be made regarding e.g., what opponents to attack in the next few seconds.
Thus, performing well in multi-scale games requires the ability to execute short-term tasks while working towards long-term goals. 
In this section we will first use the the Pac-Man game as an example, 
where the short term decisions concern avoiding ghosts and the long term decisions concern eating all the available pills.

%


\subsection{An ABL Agent}

\begin{figure}[h]
\centering
\includegraphics[width=0.7\columnwidth]{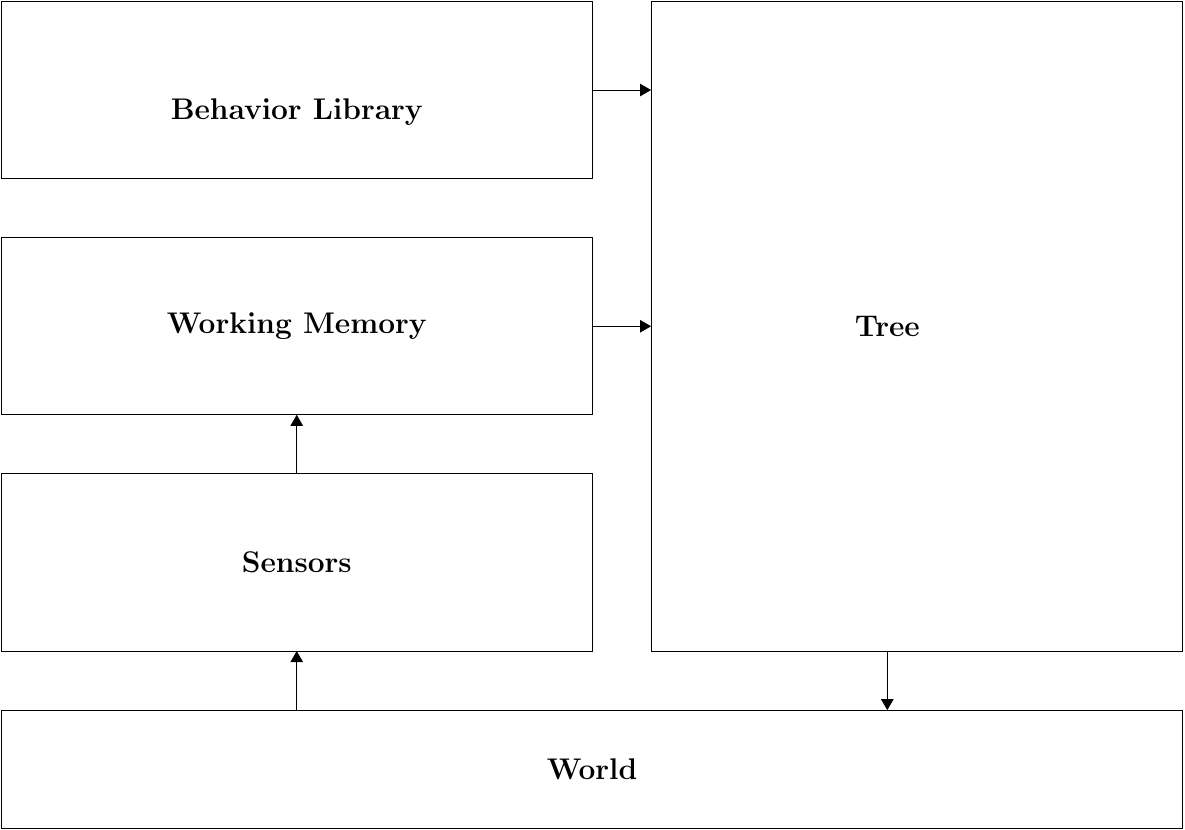}
\caption{Architecture of a ABL agent.}
\label{planning:rt.fig.ablarchitecture}
\end{figure}

An agent running an ABL planner is called an \emph{ABL Agent}. Figure~\ref{planning:rt.fig.ablarchitecture} depicts the architecture of an ABL agent. The \textbf{behavior library} is a repository of pre-defined behaviors where each behavior consists of a set of actions to execute to accomplish a goal (e.g. move to given location). There are two kinds of behaviors in ABL, \emph{sequential behaviors} and \emph{parallel behaviors}.
The \textbf{working memory} is a container for any information the agent must access during execution (e.g. unit's position on the map). The \textbf{sensors} report information about changes in
the world by writing that information into the working memory (e.g. when another agent is within sight). The \textbf{tree} (henceforth denoted ABL tree to avoid confusion) is an execution structure that describes how the agent will act, and it is dynamically extended. The ABL tree is initially defined as a collection of all the agent's goals, then it is recursively extended using a set of instructions that describe how to expand the tree.  Figure~\ref{planning:abl.fig.abtintro1} shows the initial ABL tree for the ABL Pac-Man Agent. Below we describe the semantic of  ABL tree instructions.

\begin{figure}[h]
\centering
\begin{lstlisting}
initial-tree{
	subgoal handleGhosts();
	subgoal eatAllPills();
	}
\end{lstlisting}
\caption{Example of an initial ABL tree instruction of the ABL agent for Pac-Man.}
\label{planning:abl.fig.abtintro1}
\end{figure}

A \textbf{subgoal} instruction establishes goals that must be accomplished in order to achieve
the main task.

An \textbf{act} instruction describes an action what will change the physical state of the agent or the environment.
A \textbf{mental act} instruction describes pure computation, as mathematical
computations or modifications to working memory. 

Act and mental act are both parts of behaviors, listed in  the behavior library, as the examples in Figures~\ref{planning:abl.fig.abtintro2}-\ref{planning:abl.fig.abtintro2bis}.

\begin{figure}[h]
\centering
\begin{lstlisting}
sequential behavior eatAllPills(){
	mental_act computeOptimalPath();
	act followOptimalPath();
}
\end{lstlisting}
\caption{Example of the content of the behavior library.}
\label{planning:abl.fig.abtintro2}
\end{figure}

\begin{figure}[h]
\centering
\begin{lstlisting}
parallel behavior eatAllPills(){
	mental_act recordData();
	act exploreRoom();
}
\end{lstlisting}
\caption{Example of the content of the behavior library.}
\label{planning:abl.fig.abtintro2bis}
\end{figure}

A \textbf{spawngoal} instruction is the key component for expanding the BT. It defines the subgoals that must be accomplished to achieve a behavior.

\begin{figure}[h]
\centering
\begin{lstlisting}
parallel behavior handleGhosts(){
	spawngoal handleDeadlyGhosts();
	spawngoal handleScaredGhosts();
}
\end{lstlisting}
\caption{Example of the content of the behavior library. }
\label{planning:rt.fig.abtintro3}
\end{figure}

\begin{remark}
The main difference between the instructions \emph{subgoal} and \emph{spawngoal} is that the spawngoal instruction is evaluated in a lazy fashion, expanding the tree only when the goal spawned is needed for the first time, whereas the subgoal instruction is evaluated in a greedy fashion, requiring the details on how to carry out the subgoal at design time.   
\end{remark}

\begin{figure}[h]
\centering
\begin{lstlisting}
sequential behavior handleDeadlyGhost(){
	precondition {
	(deadlyGhostClose);
	}
	act keepDistanceFromDeadlyGhost();
}
\end{lstlisting}
\caption{Example of a precondition instruction of the ABL agent for Pac-Man.}
\label{planning:abl.fig.abtintro4}
\end{figure}

A \textbf{precondition} instruction specifies
under which conditions the behavior can be selected. When all
of the preconditions are satisfied, the behavior
can be selected for execution or expansion, as in Figure~\ref{planning:abl.fig.abtintro4}.

\begin{figure}[h]
\centering
\begin{lstlisting}
conflict keepDistanceFromDeadlyGhost followOptimalPath; 
\end{lstlisting}
\caption{Example of a conflict instruction of the ABL agent for Pac-Man.}
\label{planning:rt.fig.abtintro5}
\end{figure}

A \textbf{conflict} instruction specifies priority order if two or more actions are scheduled for execution at the same time, using the same (virtual) actuator. 

\subsection{The ABL Planning Approach}

In this section, we present the ABL planning approach.
Formally, the  approach is described in Algorithms \ref{planning:ABL.alg.main}-\ref{planning:ABL.alg.Exec}.
First, will give an overview of the algorithms and see how they are applied to the problem described in \emph{Example~\ref{planning:ex:pacman}},
to iteratively create the BTs of Figure \ref{planning:abl.fig.abtpacman2BT}.
Then, we will discuss the key steps in more detail.

\begin{algorithm2e}[h!]
\caption{main loop - input(initial ABL tree)}
 \label{planning:ABL.alg.main}
\DontPrintSemicolon
%
\SetKwProg{myalg}{algorithm2e}{}{}
 $\mathcal{T} \gets$ ParallelNode \label{planning:ABL.alg.main.T} \\
\For{$subgoal$ in $initial-tree$} {
 $\mathcal{T}_g \gets$ \FuncSty{GetBT(subgoal)} \label{planning:ABL.alg.main.subgoal}
 
 $\mathcal{T}$.\FuncSty{AddChild}($\mathcal{T}_g$)\\

 }
\While{True\label{planning:ABL.alg.main.while}}{
    \FuncSty{Execute($\mathcal{T}$)} \label{planning:ABL.alg.main.fail2}}

\end{algorithm2e}


\begin{algorithm2e}[h!]
\caption{GetBT - input(goal)}
 \label{planning:ABL.alg.getbt}
\DontPrintSemicolon
%
\SetKwProg{myalg}{algorithm2e}{}{}

 $\mathcal{T}_g \gets \emptyset$\\

            \If{goal.behavior is sequential}{
 $\mathcal{T}_g \gets$ SequenceNode \label{planning:ABL.alg.main.sequence}\\
            	 }
            \Else{
             $\mathcal{T}_g \gets $ ParallelNode \label{planning:ABL.alg.main.parallel}\\
            }

 $Instructions \gets $\FuncSty{GetInstructions($goal$)}

 \For{$instruction$ in $Instructions$} {
\Switch{instruction}{

            \Case{act\label{planning:ABL.alg.main.act}}{
            	 $\mathcal{T}_g$.\FuncSty{AddChild} (ActionNode(act))\\
            } 
            \Case{mental act}{      
                $\mathcal{T}_g$.\FuncSty{AddChild} (ActionNode(mental act))\\ \label{planning:ABL.alg.main.mentalact}
                }
            \Case{spawngoal\label{planning:ABL.alg.main.spawngoal}}{
                $\mathcal{T}_g$.\FuncSty{AddChild} (PlaceholderNode(spawngoal))
            } 
 }
 }
 
             \If{goal.precondition is not empty \label{planning:ABL.alg.main.preconditions}}{
             	 $\mathcal{T}_{g'} \gets$ SequenceNode \\
				 \For{$proposition$ in $precondition$} {
                $\mathcal{T}_{g'}$.\FuncSty{AddChild}(ConditionNode($proposition$))\\
                }
            $\mathcal{T}_{g'}$.\FuncSty{AddChild}($\mathcal{T}_{g}$)\\
            \Return{$\mathcal{T}_{g'}$}
            }\Else{
 \Return{$\mathcal{T}_g$}
 }
\end{algorithm2e}

\begin{algorithm2e}[h!]
\caption{Execute - input(node)}
 \label{planning:ABL.alg.Exec}
\Switch{node.Type}{
            \Case{ActionNode}{
            \If{not in conflict}{
            	 \FuncSty{Tick}(node)
            	 }
            } 
            \Case{PlaceholderNode}{
                node $\gets$\FuncSty{GetBT}(node.goal)\label{planning:ABL.alg.placeholder} \\
                    \FuncSty{Execute}(node) 

            } 
            \Other{\FuncSty{Tick}(node)}
 }
\end{algorithm2e}

The execution of the algorithm is simple. It first creates a BT from the initial ABL tree~$t$, collecting all the subgoals in a BT Parallel composition (Algorithm~\ref{planning:ABL.alg.main}, Line~\ref{planning:ABL.alg.main.T}), then, the tree is extended by finding a BT for each subgoal (Algorithm~\ref{planning:ABL.alg.main}, Line~\ref{planning:ABL.alg.main.subgoal}). Each subgoal is translated in a corresponding BT node (Sequence or Parallel, according to the behavior in the behavior library) whose children are the instruction of the subgoal. If a behavior has precondition instructions, they are translated into BT Condition nodes, added first as children (Algorithm~\ref{planning:ABL.alg.getbt}, Line~\ref{planning:ABL.alg.main.preconditions}). If a behavior has act or mental act instruction, they are translated into BT Action nodes (Algorithm~\ref{planning:ABL.alg.getbt}, Lines~\ref{planning:ABL.alg.main.act}-\ref{planning:ABL.alg.main.mentalact}) and set as children. If a behavior has spawngoal instruction (Algorithm~\ref{planning:ABL.alg.getbt}, Line~\ref{planning:ABL.alg.main.spawngoal}), this is added as a \emph{placeholder node}, which, when ticked, extends itself as done for the subgoals (Algorithm~\ref{planning:ABL.alg.Exec}, Line~\ref{planning:ABL.alg.placeholder}). 


We are now ready to see how the algorithm is executed in a simple Pac-Man game.

\begin{example}[Simple Execution in Pac-Man]
\label{planning:ex:pacman}
While Pac-Man has to avoid being eaten by the ghosts, he has to compute the path to take in order to eat all Pills. The ABL tree Pac-Man agent is shown in Figure~\ref{planning:abl.fig.abtpacman1code}.
Running Algorithm~\ref{planning:ABL.alg.main}, the initial tree is translated into the BT in Figure~\ref{planning:abl.fig.abtpacman1BT}.
The subgoal \emph{eatAllPills} is expanded as the sequence of the two BT's Action nodes \emph{computeOptimalPath} and \emph{followOptimalPath}. The subgoal \emph{handleGhosts} is extended as a Sequence compostion of the Condition node \emph{ghostClose} (which is a precondition for handleGhosts) and a Parallel composition of placeholder nodes \emph{handleDeadlyGhosts} and \emph{handleScaredGhosts}. 
The BT is ready to be executed. Let's imagine that for a while Pac-Man is free to eat pills without being disturbed by the ghosts. For this time the condition  ghostClose is always false and the spawn of neither handleDeadlyGhosts nor handleScaredGhosts is invoked. Imagine now that a Ghost is close for the first time. This will trigger the expansion of  handleDeadlyGhosts and handleScaredGhosts. The expanded tree is shown in Figure~\ref{planning:abl.fig.abtpacman2BT}.
\begin{figure}[h]
\centering
\begin{lstlisting}

pacman_agent{
	initial-tree{
		subgoal handleGhosts();
		subgoal eatAllPills();
		}

	sequential behavior eatAllPills(){
		mental_act computeOptimalPath();
		act followOptimalPath();
	}
	parallel behavior handleGhosts(){
		precondition {
		    (ghostClose);
		}
		spawngoal handleDeadlyGhosts();
		spawngoal handleScaredGhosts();

	}
	sequential behavior handleScaredGhost(){
		precondition {
		    (scaredGhostClose);
		}
		act moveToScaredGhost();
	}
	sequential behavior handleDeadlyGhost(){
		precondition {
		    (deadlyGhostClose);
		}
		act keepDistanceFromDeadlyGhost();
	}

conflict keepDistanceFromDeadlyGhost moveToScaredGhost followOptimalPath; 
}
\end{lstlisting}
\caption{ABT tree for Pac-Man.}
\label{planning:abl.fig.abtpacman1code}
\end{figure}

\begin{figure}[h]
\centering
\includegraphics[width=\columnwidth]{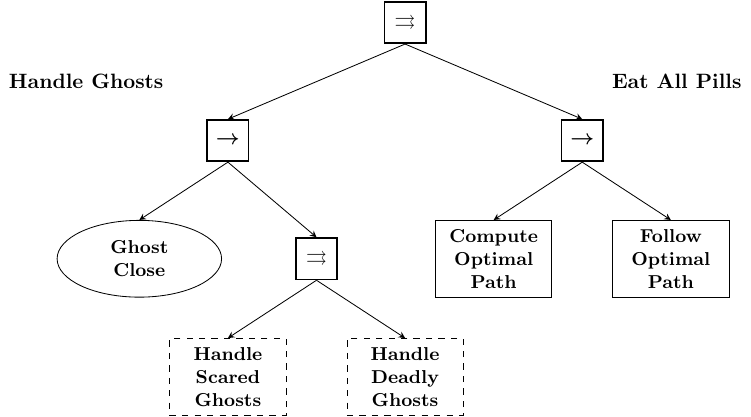}
\caption{Initial BT of Example~\ref{planning:ex:pacman}.}
\label{planning:abl.fig.abtpacman1BT}
\end{figure}

\begin{figure}[h]
\centering
\includegraphics[width=\columnwidth]{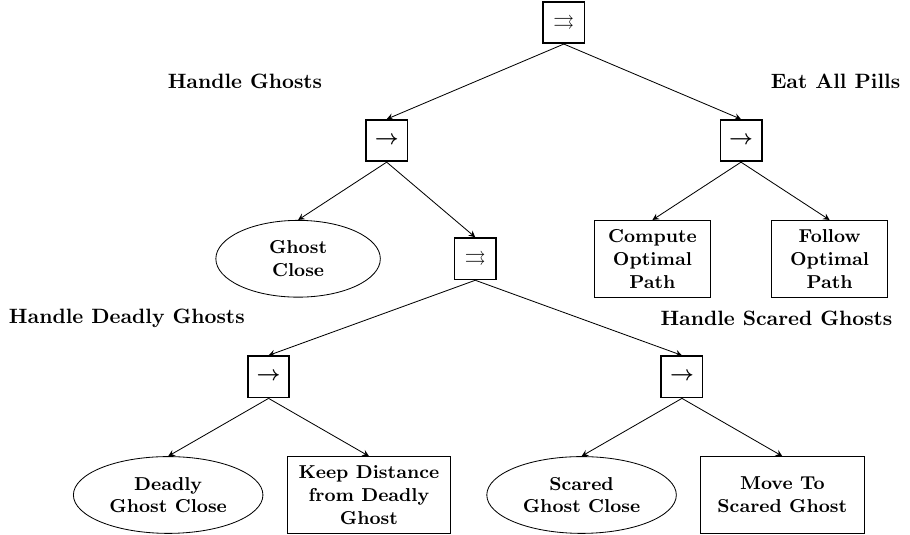}
\caption{Final BT of Example~\ref{planning:ex:pacman}.}
\label{planning:abl.fig.abtpacman2BT}
\end{figure}

\end{example}

\clearpage
\subsection{Brief Results of a Complex Execution in StarCraft}

\begin{figure}[h]
\centering
\includegraphics[width=0.6\columnwidth]{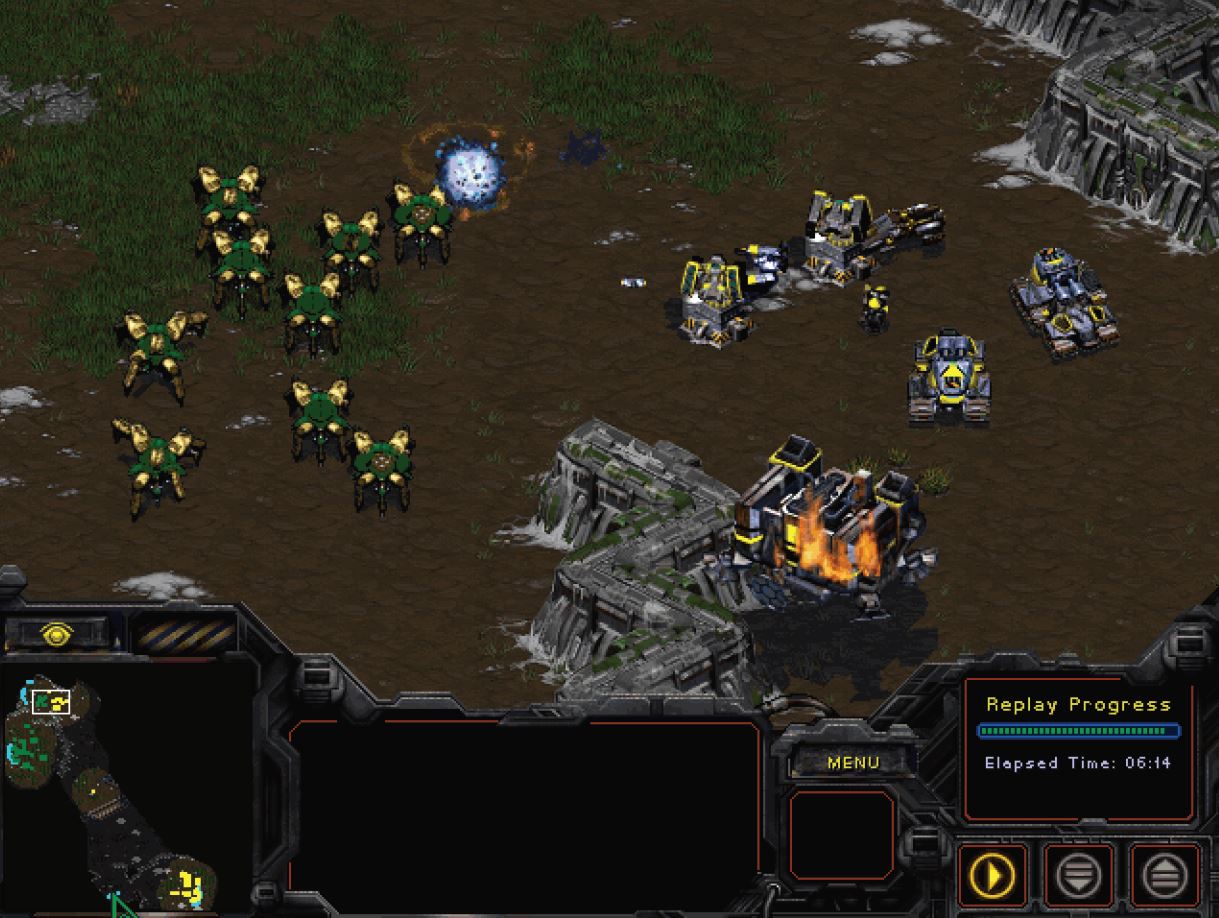}
\caption{A screenshot of StarCraft showing two players engaged in combat~\cite{weber2011building}.}
\label{planning:abl.fig.starcraft}
\end{figure}

We will now very briefly describe the results from a more complex scenario, from~\cite{weber2011building}.
One of the most well known strategy computer games that require
multi-scale reasoning is the real-time strategy game StarCraft. In StarCraft the players manage groups of units to compete for the control of the map by gathering resources
to produce buildings and units, and by researching
technologies that unlock more advanced abilities.
Building agents that perform well in this domain is challenging
due to the large decision space~\cite{aha2005learning}. StarCraft
is also a very fast-paced game, with top players performing above
300 actions per minute during peak intensity episodes~\cite{mccoy2008integrated}. This shows
that a competitive agent for StarCraft must reason quickly
at multiple scales in order to achieve efficient
game play.

\begin{example}
\label{planning:abl.ex.starcraft}

The StarCraft ABL agent is composed of three high-level managers: Strategy manager, responsible for the strategy selection and attack timing competencies; Production manager, responsible for the worker units, resource collection,
and expansion; and Tactics manager, responsible for the combat tasks and micromanagement
unit behaviors. The initial ABL tree takes the form of Figure~\ref{planning:abl.fig.starcrafttree}.

\begin{figure}[h]
\centering
\begin{lstlisting}
initial-tree{
	subgoal ManageTactic();
	subgoal ManageProduction();
	subgoal ManageStrategy();
	}
\end{lstlisting}
\caption{ABT tree for StarCraft.}
\label{planning:abl.fig.starcrafttree}
\end{figure}
Further discussions of the specific managers and behaviors  are available in~\cite{weber2011building}, and  a portion of the BT after some rounds of the game is shown in Figure~\ref{planning:abl.fig.starcraftbt}.

\begin{figure}[h]
\centering
\includegraphics[width=\columnwidth]{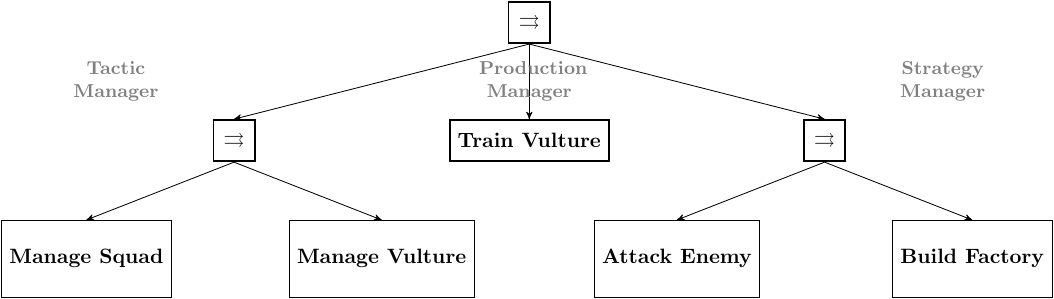}
\caption{A Portion of the tree of the ABL StarCraft agent.}
\label{planning:abl.fig.starcraftbt}
\end{figure}

\end{example}

\begin{table}
\begin{center}

   \begin{tabular}{| l | l | l | l |}
    \hline
    Map / Race &  Protoss & Terran & Zerg \\ \hline \hline
	Andromeda & 85\% & 55\% & 75\% \\
	Destination & 60\% & 60\% & 45\% \\
	Heartbreak Ridge & 70\% & 70\% & 75\% \\
	Overall & 72\% & 62\% & 65\% \\
    \hline
    \end{tabular}
    \end{center} 
    \caption{Win rate on different map/race combination over 20 trials~\cite{weber2010reactive}.}
    \label{planning:abl.tab.starcraftresult}
\end{table}

In~\cite{weber2010reactive} the ABL agent of Example~\ref{planning:abl.ex.starcraft} was evaluated against the build-in StarCraft AI. The agent was tested against three professional gaming maps: Andromeda, Destination, and Heartbreak Ridge; against all three races: Protoss, Terran, and Zerg over 20 trials. The result are shown in Table~\ref{planning:abl.tab.starcraftresult}. The ABL agent scored an overall win rate of over 60\%, additionally, the agent was capable to perform over 200 game actions per minute, highlighting the capability of the agent to combine low-level tactical task with high-level strategic reasoning.

\section{Comparison between PA-BT and ABL}
So, faced with a BT planning problem, should we choose PA-BT or ABL?
The short answer is that PA-BT is focused on creating a BT using a planning approach,
whereas ABL is a complete planning language in itself, using BT as an execution tool.
PA-BT is also better at exploiting the Fallback constructs of BTs, by iteratively expanding conditions into  PPAs,
that explicitly include fallback options for making a given condition true.

%% file: learning/learning.tex
\graphicspath{{./learning/figures/}}

\chapter{Behavior Trees and Machine Learning}
\label{ch:learning}

In this chapter, we describe how learning algorithms can be used to automatically create BTs, using ideas from~\cite{pereira2015framework,colledanchise2015learning}.
First, in Section~\ref{sec.gp}, we present a mixed learning strategy that combines a greedy element with Genetic Programming (GP) and show the result in a game and a robotic example. Then, in Section~\ref{sec:rl}, we present a Reinforcement Learning (RL) algorithm applied to BTs and show the results in a game example. Finally in Section~\ref{sec.dem} we overview the main approaches used to learn BTs from demonstration.

\section{Genetic Programming Applied to BTs}
\label{BG.GP}
The capability of an agent to learn from its own experience can be realized by imitating natural evolution. GP is an optimization algorithm that takes inspiration from biological evolution~\cite{rechenberg1994evolution} where a set of \emph{individual} policies are evolved until one of them solves a given optimization problem good enough.

In a GP approach a  particular set of individuals is called a \emph{generation}. At each GP iteration, a new generation is created from the previous one. First, a set of individuals are created by applying the operations \emph{cross-over} and \emph{mutation} to the previous generation. Then a subset of the individuals are chosen through \emph{selection} for the next generation based upon a user-defined reward. We will now describe how these three operators can be applied to BTs.

%

\textbf{Crossover of two BTs}
The crossover is performed by randomly swapping a subtree from one BT with a subtree of another BT at any level. Figure~\ref{PA.fig.CrossoverBefore} and Figure~\ref{PA.fig.CrossoverAfter} show two BTs before and after a cross-over operation.

\begin{figure}[h]
\centering
\begin{subfigure}[b]{0.7\columnwidth}
\includegraphics[width=\columnwidth]{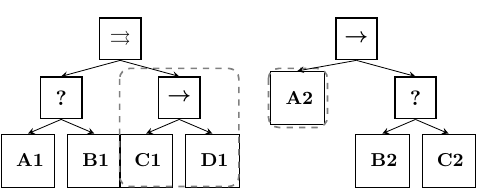}
\caption{BTs before the cross-over of the highlighted subtrees.}
\label{PA.fig.CrossoverBefore}
\end{subfigure}
~
\begin{subfigure}[b]{0.7\columnwidth}
\centering
\includegraphics[width=\columnwidth]{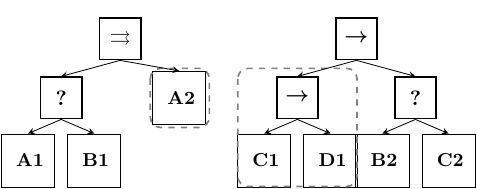}
\caption{BTs after the cross-over of the highlighted subtrees.}
\label{PA.fig.CrossoverAfter}
\end{subfigure}
\caption{Cross-over operation on two BTs.}
\label{PA.fig.Crossover}
\end{figure}

\begin{remark}
The use of BTs as the knowledge representation framework in the GP avoids the problem of logic violation during cross-over stressed in~\cite{fu2003genetic}. The logic violation occurs when, after the cross-over, the resulting individuals might not have a  consistent logic structure. One example of this is the crossover combination of two FSMs that might lead to a logic violation in terms of some transitions not leading to an existing state.
\end{remark}

\textbf{Mutation of a BT}
\label{bg.mutation}
The mutation is an unary operation that replaces a node in a BT with another node of the same type (i.e. it does not replace an execution node with a control flow node or vice versa). Mutations increase diversity, which is crucial in GP. To improve convergence properties it is common to use  so-called \emph{simulated annealing}, performing the mutation on a large number of nodes of the first generation of BTs and gradually reducing  the number of mutated nodes in each new generation. In this way we start with a very high diversity to avoid getting stuck in possible local minima of the objective function of the optimization problem, and  reduce diversity over time as we get closer to the goal.

\textbf{Selection of BTs}
In the selection step, a subset of the individuals created by mutation and crossover are selected for the next generation.
 The  process is  random,
giving each BT  a survival probability $p_i$. This probability is based upon  the \emph{reward function} which quantitatively measures the fitness of the agent, i.e., how close the agent gets to the goal. 
A common method to compute the survival probability of an individual is the Rank Space Method~\cite{mitchell1997machine},
where the designer first sets
 $P_c$ as the probability of the highest ranking individual, then we sort the BTs in  descending order w.r.t. the reward. Finally, the probabilities are defined as follows:
\begin{eqnarray}
p_k = &(1-P_c)^{k-1}P_c &\; \forall k \in \{1,2,\ldots,N-1 \} \\
p_N = &(1-P_c)^{N-1}&
\end{eqnarray}
where $N$ is the number of individuals in a generation.

\section{The GP-BT Approach}
\label{sec.gp}
In this section we outline the GP-BT approach \cite{colledanchise2015learning}. We begin with an example, describing the algorithm informally, and then  give a  formal description in Algorithm~\ref{PS.ALG.LA}.
GP-BT follows a mixed learning strategy,  trying a greedy algorithm first and then applying a GP algorithm when needed. This mixed approach reduces the learning time significantly, compared to using pure GP, while still resulting in a fairly compact BT that achieves the given objective.

We now provide an example to describe the algorithm informally.

\begin{figure}[h]
\centering
\includegraphics[width=0.3\columnwidth]{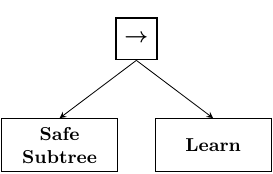}
\caption{The initial BT is a combination of the BT guaranteeing safety, and the BT \emph{Learn} that will be expanded during the learning.  }
\label{PA.fig.InitialBT}
\end{figure}

\begin{example}
Consider the case of the Mario AI setup in Figure \ref{res.fig.mario},  starting with the BT in Figure~\ref{PA.fig.InitialBT}.

The objective of Mario is to reach the rightmost end of the level.
The safety BT is optional, but motivated by the need 
to enable guarantees that the agent avoids some regions of the state space that are known by the user to be un-safe.
Thus, this is the only part of GP-BT that requires user input.

The un-safe regions must have a conservative margin to enable the safety action to act before it is too late (see Section~\ref{properties:sec:safety}).
Thus we cannot use the enemies as unsafe regions as
Mario needs to move very close to those to complete the level.
Therefore, for illustrative purposes, we let the safety action guarantee that Mario never reaches the leftmost wall of the level.

Mario starts really close to the left most wall, so the first thing that happens is that the safety action moves Mario a bit to the right. Then the Learn action is executed.

This action first checks all inputs and creates a BT, $\bt_{cond_\tau}$, of conditions that returns Success if and only if all inputs correspond to the current ``situation", as will be explained below.

Then the learning action executes all single actions available to Mario, e.g. go left, go right, jump, fire etc. and checks the resulting  reward.

All actions yielding an increase in the reward are collected in a Fallback composition BT, $\bt_{acts_i}$, sorted with respect to performance,  with the best performing action first.

If no single action results in a reward increase, the GP is invoked to find a combination of actions (arbitrary BTs are explored, possibly including parallel, Sequence and Fallback nodes) that produces an increase. Given some realistic assumptions, described in~\cite{colledanchise2015learning}, such a combination exists and will eventually be found by the GP according to previous results in \cite{rudolph1994convergence}, and stored in $\bt_{acts_i}$.

%
Then, the condition BT, $\bt_{cond_\tau}$, is composed with the corresponding action BT, $\bt_{acts_i}$, in a Sequence node and the result is added to the previously learned BT, with a higher priority than the learning action.

Finally, the new BT is executed until once again the learning action is invoked.
\end{example}

\begin{algorithm2e}[h]

  \ArgSty{$\bt$} $\gets$ \FuncSty{"Action Learn"}\\
  \Do{$\rho < 1$}
  {
  	\FuncSty{Tick(SequenceNode\ArgSty{($\bt_{safe}$,$\bt$)})}\\

    \If{\FuncSty{IsExecuted(Action Learn)}
}
    {
  		\ArgSty{$\bt_{cond}$} $\gets$ \FuncSty{GetSituation() \%Eq(\ref{eq:BTcond})}\\
  		\ArgSty{$\bt_{acts}$} $\gets$ \FuncSty{LearnSingleAction(\ArgSty{$\bt$}) \%Eq(\ref{eq:BTactSingle})}   \\
		\If{\ArgSty{$\bt_{acts}.NumOfChildren$} $=$ \ArgSty{0} }
		    {
			\ArgSty{$\bt_{acts}$} $\gets$ \FuncSty{GetActionsUsingGP(\ArgSty{$\bt$}) \%Eq(\ref{eq:BTactGP}) }\\
		    }
		    	\eIf{\FuncSty{IsAlreadyPresent(\ArgSty{$\bt_{acts}$})}}
		    {
		    \ArgSty{$\bt_{cond_{exist}}$} $\gets$ \FuncSty{GetConditions(\ArgSty{$\bt_{acts}$})}\\
		    \ArgSty{$\bt_{cond_{exist}}$} $\gets$ \FuncSty{Simplify}(FallbackNode((\ArgSty{$\bt_{cond_{exist}}$},\ArgSty{$\bt_{cond}$})) \\
		    }
		    {
		    \ArgSty{$\bt$} $\gets$ \FuncSty{FallbackNode(\FuncSty{SequenceNode(\ArgSty{$\bt_{cond}$},\ArgSty{$\bt_{acts}$})},\ArgSty{$\bt$})   \%Eq(\ref{PA.eq.newBT})}
		    }

    }
	  \ArgSty{$\rho$} $\gets$ \FuncSty{GetReward(SequenceNode\ArgSty{($\bt_{safe}$,$\bt$)})}\\

  }
  \Return{$\bt$}\\
  \caption{Pseudocode of the learning algorithm}
  \label{PS.ALG.LA}
  
\end{algorithm2e}

\vspace*{-2em}
\subsection{Algorithm Overview}
\label{LA.AO}

To describe the algorithm in detail, we need a set of definitions.
First we let a situation be a collection of all conditions, sorted on whether they are true or not, then we define the following:

\paragraph*{\textbf{Situation}}
$\mathcal{S}(t) = [C_{T}^{(t)},C_{F}^{(t)}]$ is the situation vector, where $C_{T}^{(t)}=\{C_{T1},\ldots, C_{TN}\}$ is the set of conditions that are true at time $t$ and $C_{F}^{(t)}=\{C_{F1},\ldots, C_{FM}\}$ is the set of conditions that are false at time $t$.

Then, using the analogy between AND-OR trees and BTs \cite{colledanchise2016behavior}, we create the BT that returns success only when a given situation occurs.
\paragraph*{$\mathbfcal{\bt}_\mathbf{cond_\tau}$}
$\bt_{cond_\tau}$ is the BT representation of $\mathcal{S}(\tau)$.
\begin{eqnarray}
 \bt_{cond_\tau}&\triangleq& \mbox{Sequence}(\mbox{Sequence}(C_{T1},\ldots, C_{TN}), \nonumber \\
 &&\mbox{Sequence
 }(\mbox{invert}(C_{F1}),\ldots,\mbox{invert} (C_{FM})) ) \hfill\label{eq:BTcond}
\end{eqnarray}
Figure~\ref{PA.fig.BTcond} shows a BT composition representing a situation $\mathcal{S}(\tau)$. 

\begin{figure}[h]
\centering
\includegraphics[width=0.6\columnwidth]{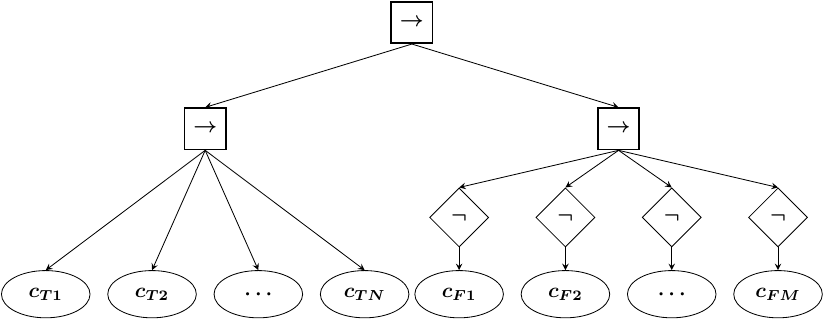}
\caption{Graphical representation of $\bt_{cond_\tau}$, $c_{Fi} \in C_F^\tau$, $c_{Tj} \in C_T^\tau$. The Decorator is the negation Decorator  (i.e it inverts the Success/Failure status). This BT returns success only when all $c_{Tj}$ are true and all $c_{Fi}$ are false. }
\label{PA.fig.BTcond}
\end{figure}

\paragraph*{$\mathbfcal{\bt}_\mathbf{acts_i}$}
Given a small $\epsilon > 0$,
if at least one action results in an increase in reward, $\Delta \rho > \epsilon$,
 let $A_{P1},\ldots, A_{P\tilde N}$ be the list of actions that result in such an improvement,
 sorted on the size of $\Delta \rho$,
then, $\bt_{acts_i}$ is defined as:
\begin{equation}
 \bt_{acts_i}= \mbox{FallbackWithMemory}(A_{P1},\ldots, A_{PN}) \label{eq:BTactSingle}
\end{equation}
 else, $ \bt_{acts_i}$ is defined as the solution of a GP algorithm that terminates when an improvement such that $\Delta \rho > \epsilon$ or $\rho(x)=1$ is found, i.e.:
 \begin{equation}
 \bt_{acts_i}= GP(\mathcal{S}(t)) \label{eq:BTactGP}
\end{equation}
\paragraph*{\textbf{New BT}}
 If $\bt_{acts_i}$ is not contained in the BT, the new BT learned is given as follows: 
\begin{equation}
\label{PA.eq.newBT}
\bt_i\triangleq \mbox{Fallback}(\mbox{Sequence}(\bt_{cond_i}, \bt_{acts_i}), \bt_{i-1})
\end{equation}
Else, if $\bt_{acts_i}$ is contained in the BT, i.e., there exists an index $j \neq i$ such that $\bt_{acts_i}=\bt_{acts_j}$,
this means there is already a situation identified where $\bt_{acts_i}$ is the appropriate response. 
Then, to reduce the number of nodes in the BT, 
we generalize the two special cases where this response is needed.
This is done by combining $\bt_{cond_i}$ with $\bt_{cond_j}$, that is, we find the BT, $\bt_{cond_{ij}}$, that returns success if and only if $\bt_{cond_i}$ or $\bt_{cond_j}$ return success. Formally, this can be written as
\begin{equation}
\label{PA.eq.newBT2}
\bt_i\triangleq \bt_{i-1}.\mbox{replace}(\bt_{cond_i},\mbox{simplify}((\bt_{cond_i},\bt_{cond_j})))
\end{equation}

Figure~\ref{PA.fig.merge} shows an example of this simplifying procedure. As can be seen this simplification generalizes the policy by iteratively removing conditions that are not relevant for the application of the specific action. It is thus central for keeping the number of nodes low, as seen below in Fig.~\ref{ER.fig.nodes}.
\begin{figure}[h]
        \centering
        \begin{subfigure}[b]{0.3\columnwidth}
                \centering
                \includegraphics[width=\columnwidth,clip]{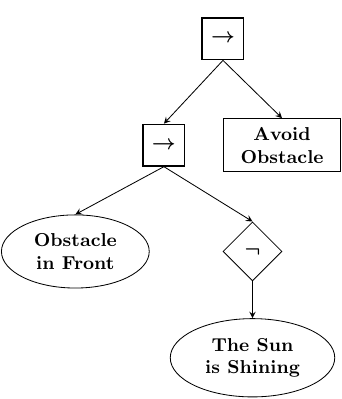}
                \caption{The action \emph{Avoid Obstacle} is executed if there is an obstacle in front and the sun is {\bf not} shining. }
        \end{subfigure}%
       ~ 
        \begin{subfigure}[b]{0.3\columnwidth}
                \centering
                \includegraphics[width=\columnwidth,clip]{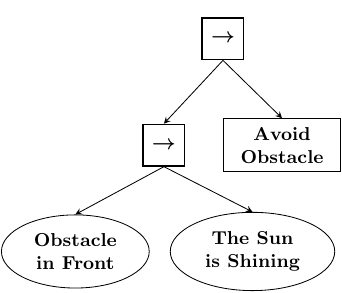}
                \caption{The action \emph{Avoid Obstacle} is executed if there is an obstacle in front and the sun is shining. \;\;\;\;}
        \end{subfigure}
        ~ 
        \begin{subfigure}[b]{0.3\columnwidth}
                \centering
                \includegraphics[width=\columnwidth,clip]{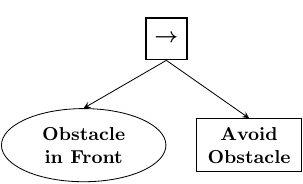}
                \caption{Simplified merged Tree. The action \emph{Avoid Obstacle} is executed if there is an obstacle in front. \;\; \;\;\;\;\;\;\;\;}
        \end{subfigure}
        \caption{Example of the simplifying procedure in (\ref{PA.eq.newBT2}). The two learned rules (a) and (b) are combined into (c): The important condition appears to be  \emph{Obstacle in Front}, and there is no reason to check the condition \emph{The Sun is Shining}. These simplifications generalize the policies, and keep the BT sizes down.  }
                        \label{PA.fig.merge}

\end{figure}

Given these definitions, we can go through the steps listed in Algorithm~\ref{PS.ALG.LA}.
Note that the agent runs $\bt_i$ until a new situation is encountered
which requires learning an expanded BT, or the goal is reached.

The BT $\bt$ is first initialized to be a single action, \emph{Action Learn}, which will be used to trigger the learning algorithm. Running Algorithm~\ref{PS.ALG.LA}  we execute the Sequence composition of the safe subtree $\bt_{safe}$ (generated manually, or using a non-learning approach) with the current tree $\bt$ (Algorithm~\ref{PS.ALG.LA} Line $3$). The execution of $\bt_{safe}$ overrides the execution of $\bt$ when needed to guarantee safety. If the action \emph{Action Learn} is executed, it means that the current situation in not considered in neither $\bt_{safe}$ nor $\bt$, hence a new action, or action composition, must be learned. The framework first starts with the greedy approach (Algorithm~\ref{PS.ALG.LA}, Line $6$) where it tries each action and stores the ones that increase the reward function, if no such actions are available (i.e. the reward value is a local maximum), then the framework starts learning the BT composition using the GP component (Algorithm~\ref{PS.ALG.LA}, Line $8$). Once the tree $\bt_{acts}$ is computed, by either the greedy or the GP component, the algorithm checks if $\bt_{acts}$ is already present in the BT as a response to another situation, and the new tree $\bt$ can be simplified using a generalization of the two situations (Algorithm~\ref{PS.ALG.LA}, Line $11$). Otherwise, the new tree $\bt$ is composed by the selector composition of the old $\bt$ with the new tree learned (Algorithm~\ref{PS.ALG.LA}, Line $13$). The algorithm runs until the goal is reached. The algorithm is guaranteed to lead the agent to the goal, under reasonable  assumptions~\cite{colledanchise2015learning}.

\subsection{The Algorithm Steps in Detail}
\label{LA.AS}
We now discuss Algorithm~\ref{PS.ALG.LA} in detail.
\subsubsection{GetSituation (Line 5)}
This function returns the tree $\bt_{cond}$ which represents the current situation. $\bt_{cond}$ is computed according to Equation~\eqref{eq:BTcond}.
\subsubsection{LearnSingleAction (Line 6)}
This function returns the tree $\bt_{acts}$ which represent the action to execute whenever the situation described by $\bt_{cond}$ holds. $\bt_{acts}$ is a Fallback composition with memory of all the actions that, if performed when $\bt_{cond}$ holds, increases the reward. The function \emph{LearnSingleAction} runs the same episode $N_a$ (number of actions) times executing a different action whenever $\bt_{cond}$ holds. When trying a new action, if the resulting reward increases, this action is stored. All the actions that lead to an increased reward are collected in a Fallback composition, ordered by the reward value. This Fallback composition, if any, is then returned to Algorithm~\ref{PS.ALG.LA}. 
\subsubsection{LearnActionsUsingGP (Line 8)}
If \emph{LearnSingleAction} has no children (Algorithm~\ref{PS.ALG.LA}, Line 7) then there exists no single action that can increase the reward value when
the situation described by $\bt_{cond}$ holds. In that case the algorithm learns a BT composition of actions and conditions that must be executed whenever $\bt_{cond}$ holds. This composition is derived as described in Section~\ref{BG.GP}.
\subsubsection{Simplify (Line 11)}
If the resulting $\bt_{acts}$ is present in $\bt$ (Algorithm~\ref{PS.ALG.LA}, Line 9) this means that there exist another situation $\mathcal{S}_{exist}$ described by the BT $\bt_{cond_{exist}}$, where the response in $\bt_{acts}$ is appropriate. To reduce the number of nodes in the updated tree, we create a new tree that captures both situations $\mathcal{S}$ and $\mathcal{S}_{exist}$. This procedure  removes from $\bt_{cond_{exist}}$ a single condition $c$ that is present in $C_F$  for one situation ($\mathcal{S}$ or $\mathcal{S}_{exist}$) and $C_S$ for the other situation. 

\begin{remark}
Note that the GP component is invoked exclusively whenever the greedy component fails to find a single action.
\end{remark}

\subsection{Pruning of Ineffective Subtrees}
\label{PA.removal}
Once obtained the BT that satisfies the goal, we can search for ineffective subtrees, i.e. those action compositions that are superfluous for reaching the goal. To identify the redundant or unnecessary subtrees, we enumerate the subtrees with a Breadth-first enumeration. We run the BT without the first subtree and checking whether the reward function has a lower value or not. In the former case the subtree is kept, in the latter case the subtree is removed creating a new BT without the subtree mentioned. Then we run the same procedure on the new BT. The procedure stops when there are no ineffective subtree found. This procedure is optional.

\subsection{Experimental Results}

\begin{figure}[t]
        \centering
        \begin{subfigure}[b]{0.5\columnwidth}
                \centering
                \includegraphics[width=\columnwidth,height=4cm,clip]{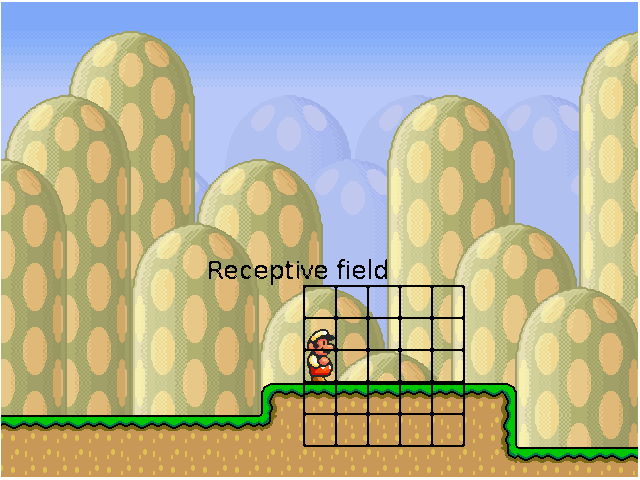}
                \caption{Mario AI benchmark. }
                \label{res.fig.mario}
        \end{subfigure}%
       ~ 
        \begin{subfigure}[b]{0.5\columnwidth}
                \centering
                \includegraphics[width=\columnwidth,height=4cm,clip]{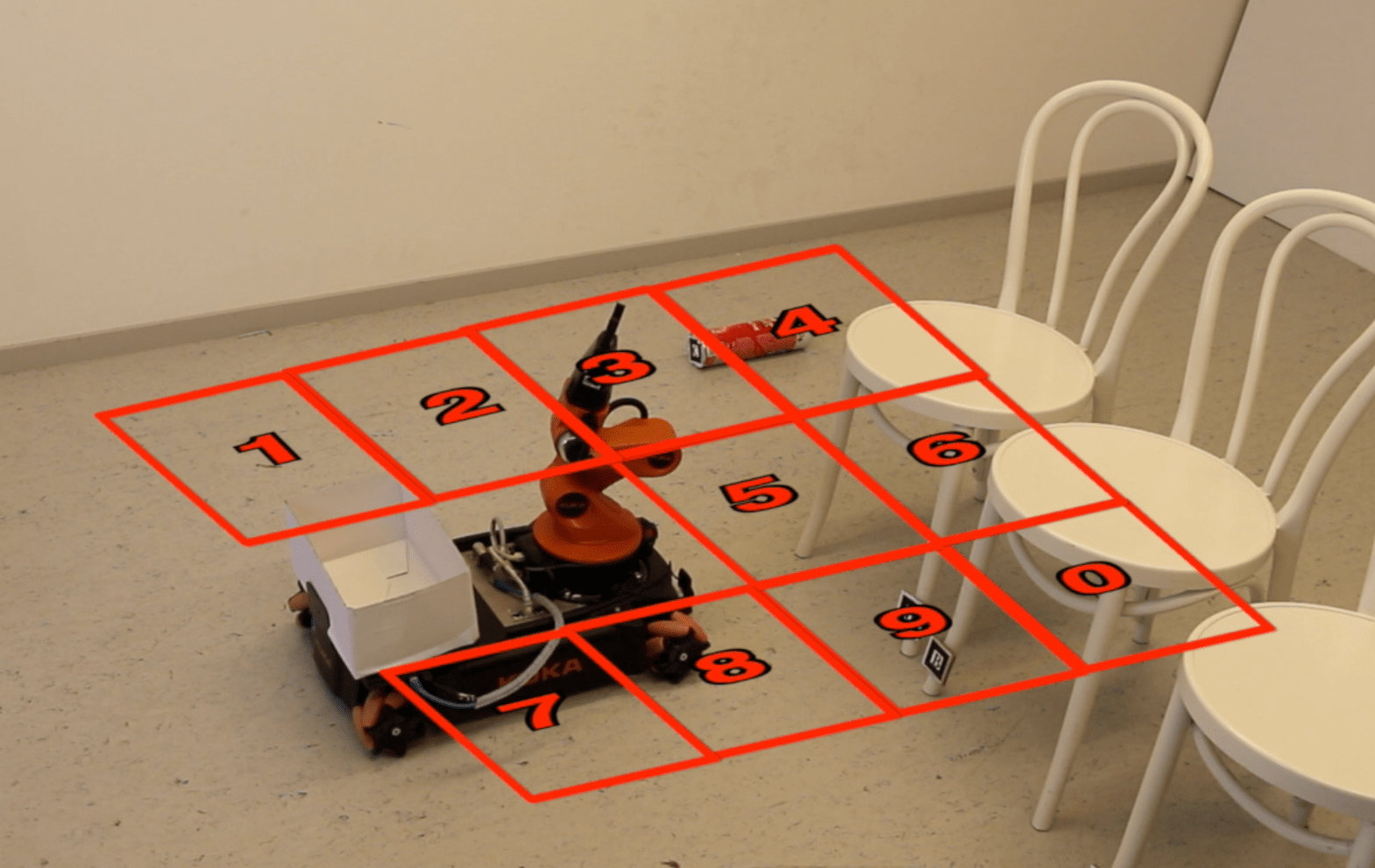}
                \caption{KUKA Youbot benchmark}
                \label{res.fig.youbot}              
        \end{subfigure}
        \caption{Benchmarks used to validate the framework.}
         \label{res.fig}        
\end{figure}

In this section we apply GP-BT to two problems. One is the \emph{Mario AI benchmark}  (Figure~\ref{res.fig.mario}) and one is  a real robot, the \emph{KUKA Youbot} (Figure~\ref{res.fig.youbot}). The results on the Mario AI benchmark shows the applicability of the GP-BT on a highly complex dynamic environment and
also allows us to compare the approach to the state-of-the-art. The results on the KUKA Youbot shows the applicability of GP-BT on a real robot.  

\subsubsection{Mario AI}
The Mario AI benchmark~\cite{karakovskiy2012mario} is an open-source software clone of Nintendo's Super Mario Bros used to test learning algorithms and game AI techniques. The task consists of moving the controlled character, Mario, through two-dimensional levels, which are viewed from the side. {Mario can walk and run to the right and left, jump, and (depending on the mode, explained below) shoot fireballs. Gravity acts on Mario, making it necessary to jump over gaps to get past them. Mario can be in one of three modes: \emph{Small}, \emph{Big}, and \emph{Fire} (can shoot fireballs).
The main goal of each level is to get to the end of the level, which means traversing it from left to right. Auxiliary goals include collecting as many coins as possible, finishing the level as fast as possible, and collecting the highest score, which in part depends on the number of enemies killed.
Gaps and moving enemies make the task more complex. If Mario falls down a gap, he loses a life. If he runs into an enemy, he gets hurt; this means losing a life if he is currently in the Small mode. Otherwise, his mode degrades from Fire to Big or from Big to Small. }

\textbf{Actions}
In the benchmark there are five actions available: \emph{Walk Right}, \emph{Walk Left}, \emph{Crouch}, \emph{Shoot}, and \emph{Jump}.

\textbf{Conditions}
In the benchmark there is a receptive field of observations as shown in Figure~\ref{res.fig.mario}. For each of the $25$ cells in the grid there are $2$ conditions available: \emph{the box is occupied by an enemy} and  \emph{the box is occupied by an obstacle}. There are two other conditions: \emph{Can Mario Shoot} and \emph{Can Mario Jump},
creating  a total of $52$ conditions.


\textbf{Reward Functions}
The reward function is given by a non linear function of the distance passed, enemies killed, number of times Mario is hurt, and time left when the end of the level is reached. The reward function is the same for every scenario.

{
\textbf{Cross-Validation}
To evaluate the learned BT in an episode, we run a different episode of the same complexity (in terms of type of enemies; height of obstacles and length of gaps). In the Mario AI framework, this is possible by choosing different so-called \emph{seeds} for the learning episode and the validating episode. The result shown below are cross-validated in this way.
}

\textbf{GP Parameters}
Whenever the GP part is invoked, it starts with $4$ random BTs composed by one random control flow node and $2$ random leaf nodes. The  number of individuals in a generation is set to $25$.

\textbf{Scenarios}
We ran the algorithm in five different scenarios of increasing difficulty. {The first scenario had no enemies and no gaps, thus only requiring motion to the right and jumping at the proper places. The resulting BT can be seen in Figure~\ref{ER.fig.BTs1} where the action \emph{Jump} is executed if an obstacle is in front of Mario and the action \emph{Go Right} is executed otherwise. The second scenario has no obstacles but it has gaps. The resulting BT can be seen in Figure~\ref{ER.fig.BTs2}  where the action \emph{Jump} is executed if Mario is close to a gap. The third scenario has high obstacles, gaps and walking enemies. The resulting BT can be seen in Figure~\ref{ER.fig.BTs3} which is similar to a combination of the previous BTs with the addition of the action \emph{Shoot} executed as soon as Mario sees an enemy (cell $14$), and to \emph{Jump} higher obstacles Mario cannot be too close. Note that to be able to show the BTs in a limited space, we used the \emph{Pruning} procedure mentioned in Section~\ref{PA.removal}}.
A video is available that shows the performance of the algorithm in all 5 scenarios.\footnote {\url{https://youtu.be/QO0VtUYkoNQ}}

\begin{figure}[h]
\centering
\includegraphics[width=0.4\columnwidth]{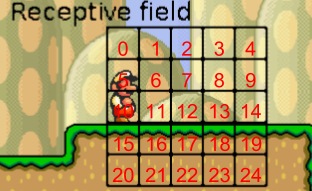}
\caption{Receptive field with cells' numbers.}
\label{ER.fig.receptiveannotated}
\end{figure}

\begin{figure}[h]
\centering
        \begin{subfigure}[t]{0.4\columnwidth}
\includegraphics[width=\columnwidth]{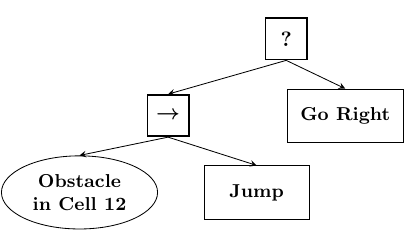}
\caption{Final BT for Scenario 1.}
\label{ER.fig.BTs1}
        \end{subfigure}%
\hfill
        \begin{subfigure}[t]{0.3\columnwidth}
        \centering
\includegraphics[width=\columnwidth]{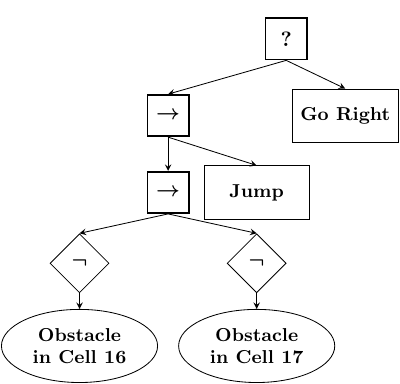}
\caption{Final BT for Scenario 2.}
\label{ER.fig.BTs2}
        \end{subfigure}%
        
        \begin{subfigure}[b]{0.6\columnwidth}
        \centering
\includegraphics[width=\columnwidth]{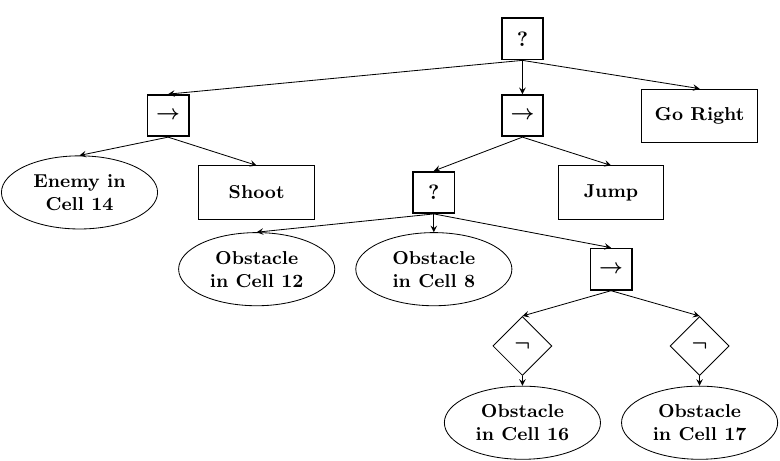}
\caption{Final BT for Scenario 3.}
\label{ER.fig.BTs3}
        \end{subfigure}%
        \caption{Final BTs learned for Scenario 1-3.}
\end{figure}
\begin{figure}[t]
        \centering
        \begin{subfigure}[b]{0.5\columnwidth}
                \centering
                \includegraphics[width=\columnwidth]{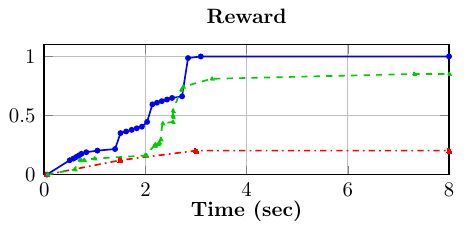}
                \caption{ Comparison for Scenario 1}
                \label{ER.fig.rs1}
        \end{subfigure}%
       ~ 
        \begin{subfigure}[b]{0.5\columnwidth}
                \centering
                \includegraphics[width=\columnwidth]{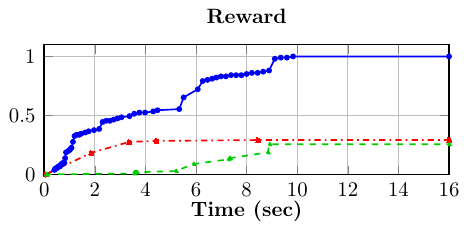}
                \caption{ Comparison for Scenario 5}
                \label{ER.fig.rs5}              
        \end{subfigure}
        \begin{subfigure}[b]{0.5\columnwidth}
                \centering
                \includegraphics[width=\columnwidth,clip]{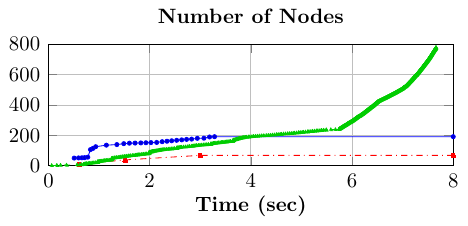}
                \caption{ Comparison for Scenario 1}
                \label{ER.fig.nns1}
        \end{subfigure}%
       ~ 
        \begin{subfigure}[b]{0.5\columnwidth}
                \centering
                \includegraphics[width=\columnwidth,clip]{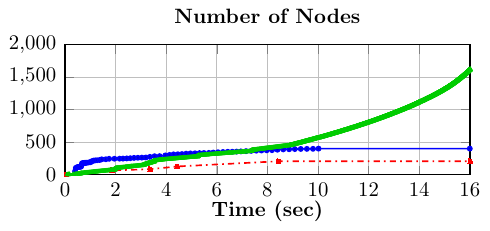}
                \caption{ Comparison for Scenario 5}
                \label{ER.fig.nns5}              
        \end{subfigure}
        \caption{ Reward value comparison (a and b) and nodes number comparison (c and d). The blue solid line refers to GP-BT. The red dash-dotted line refers to the pure GP-based algorithm. The green dashed line refers to the FSM-based algorithm.}
                \label{ER.fig.nodes}
\end{figure}

We compared the performance of the GP-BT approach to a FSM-based algorithm of the type described in~\cite{garcia2000integrated} and to a pure GP-based algorithm of the type described in~\cite{scheper2014}.

For all three algorithms, we measure the performance, in terms of the reward function, and the 
complexity of the learnt solution in terms of the number of nodes in the BT/FSM respectively. The data for the simplest and most complex scenario, 1 and 5, can be found in Figure~\ref{ER.fig.nodes}.
As can be seen in Figure~\ref{ER.fig.nodes}, the FSM approach generates a number of nodes that increases exponentially over time,
while the growth rate in GP-BT tends to decrease over time. We believe that this is due to the simplification step, described in Equation (\ref{PA.eq.newBT2}), where two different situations requiring the same set of actions are generalized by a merge in the BT.
The growth of the pure GP algorithm is very slow, as each iteration needs very long time to find a candidate improving the reward. {This is due to the fact that the number of conditions is larger than the number of actions, hence the pure GP approach often constructs BTs that check a large amount of conditions, while performing very few actions, or even none. Without a greedy component and the AND-OR-tree generalization with the conditions, a pure GP approach, like the one in~\cite{scheper2014},  is having difficulties  without any a-priory information.
}
Looking at the performance in Figure~\ref{ER.fig.nodes} we see that GP-BT is the only one who reaches a reward of $1$ (i.e. the task is completed) within the given execution time, not only for Scenario $5$, but also for the less complex Scenario $1$. 

 \begin{remark}
 Note that we do not compare GP-BT with the ones of the Mario AI challenge,
as we study problems with no a-priori information or model, whereas the challenge
provided full information about the task and environment.
 When the GP-BT learning procedure starts, the agent (Mario) does not even know that the enemies should be killed or avoided.
 \end{remark}

The other scenarios have similar result. We chose to depict the simplest and the most complex ones.

\subsubsection{KUKA Youbot}
As mentioned above, we use the KUKA Youbot to verify GP-BT on a real scenario. We consider three scenarios, one with a partially known environment and two with completely unknown environments.

Consider the Youbot in Fig.~\ref{res.fig.youbot}, the conditions are given in terms of the 10 receptive fields and binary conditions regarding a number of different objects, e.g. larger or smaller obstacles. The corresponding actions are: go left/right/forward, push object, pick object up etc. Again, the problem is to learn a switching policy mapping conditions to actions.

\textbf{Setup}
The robot is equipped with a wide range HD camera and uses markers to recognize the objects nearby. The recognized objects are mapped into the robot simulation environment V-REP~\cite{freese2010virtual}. The learning procedure is first tested on the simulation environment and then executed on the real robot.

\textbf{Actions}
Move Forward, Move Left, Move Right, Fetch Object, Slide Object to the Side, Push Object. 

\textbf{Conditions}
Wall on the Left, Wall on the Right, Glass in Front, Glass on the Left, Glass on the Right, Cylinder in Front, Cylinder on the Left, Cylinder on the Right, Ball in Front, Ball on the Left, Ball on the Right, Big Object in Front, Big Object on the Left, Big Object on the Right.

\textbf{Scenarios}
 In the first scenario, the robot has to traverse a corridor dealing with different objects that are encountered on the way.
The destination and the position of the walls is known a priori for simplicity. The other objects are recognized and mapped once they enter the field of view of the camera. 
The second scenario illustrates the reason why GP-BT performs the learning procedure for each different situation. 
The same type of cylinder is dealt with differently in two different situations.

In the third scenario, a single action is not sufficient to increase the reward. The robot has to learn an
action composition using GP to reach the goal.

A YouTube video is available that shows all three scenarios in detail\footnote{\url{https://youtu.be/P9JRC9wTmIE} and \url{https://youtu.be/dLVQOlKSqGU}}.

\subsection{Other Approaches using GP applied to BTs}
There exists several other approaches using GP applied to BTs. 

Grammatical Evolution (GE), a grammar-based form of GP, that specifies the syntax of possible solutions using a context-free grammar was used to synthesize a BT~\cite{perez2011evo}. The root node consists of a Fallback node, with a variable number
of subtrees called \emph{BehaviourBlocks}. Each BehaviourBlock consists of a sequence of one or more conditions created through a GE algorithm,
followed by a sequence of actions (possibly with some custom made Decorator) that are also created through a GE algorithm. The root node has as right most child a BehaviourBlocks called \emph{DefaultSequence}: a sequence with memory with only actions, again created through a GE algorithm. The approach works as follows: When the BT is executed, the root node will
route ticks to one BehaviourBlock; if none of those BehaviourBlocks execute actions, the ticks reach the DefaultSequence. The approach was used to compete at the 2010 Mario AI competition, reaching the fourth place of the gameplay track. 

Another similar approach was used to synthesize sensory-motor coordinations of a micro air vehicle using a pure GP algorithm on a initial population of randomly generated BTs~\cite{scheper2014}. The mutation operation is implemented using two methods: \emph{micro mutation} and \emph{macro mutation}. Micro mutation affects only leaf nodes and it is used to modify the parameter of a node. Macro mutation was used to replace a randomly selected node with a randomly selected tree.

\section{Reinforcement Learning applied to BTs}
\label{sec:rl}
In this section we will describe an approach proposed in \cite{pereira2015framework} for combining BTs with Reinforcement Learning (RL).

\subsection{Summary of Q-Learning}
A general class of problems that can be addressed by RL is the
Markov Decision Processes (MDPs). 
Let $s_t \in S$ be the system state at time step $t$, $a_t \in {A}_{s_t}$ the action executed at time $t$, where ${A}_s$ is a finite set of admissible actions for the state $s$ (i.e. the actions that can be performed by the agent while in $s$). The goal of the agent is to learn a policy $\pi : {S} \to {A}$ (i.e. a map from state to  action), where $A = \bigcup_{s \in {S}}{A}_s$, that maximizes the expected long-term reward from each state $s$,~\cite{sutton1999}.
 
One of the most used Reinforcement Learning techniques is \emph{Q-Learning}, which can handle problems with stochastic transitions and rewards. It uses a function called \emph{Q-function} $Q: S \times A \to \mathbb{R}$ to capture the expected future reward of each state-action combination (i.e. how good it is to perform a certain action when the agent is at a certain state). First, for each state-action pair, $Q$ is initialized to an arbitrary fixed value. Then, at each time step of the algorithm, the agent selects an action and observes a reward and a new state. $Q$ is updated according to a \emph{simple value iteration} function~\cite{mitchell1997machine}, using the weighted average of the old value and a value based on  the new information, as follows:

\begin{equation}
  Q_{k+1}(s, a) = (1 - \alpha_k) Q_k(s, a) + \alpha_k \left[ r + \gamma \max_{a' \in A_{s'}} Q_k(s', a') \right],
\end{equation}
where $k$ is the iteration index (increasing every time the agent receives an update), $r$ is the reward, 
$\gamma$ is the discount factor that trades off the influence of early versus late rewards, and $\alpha_k$ is the learning rate that trades off the influence of newly acquired information versus old information.

The algorithm converges to an optimal policy that maximizes the reward if all admissible state-action pairs are updated infinitely often 
\cite{sutton1998,barto2003}. At each state, the optimal policy selects  the action that maximizes the $Q$ value.

\textbf{Hierarchical Reinforcement Learning}

The vast majority of RL algorithms are based upon the MDP model above. However, Hierarchical RL (HRL) algorithms have their basis in Semi MDPs (SMDPs)~\cite{barto2003}. SMDPs enable a special treatment of the temporal aspects of the problem and, thanks to their hierarchical structure, reduce the impact of the curse of dimensionality by dividing the main problem into several subproblems.

%

%
%

Finally, the \emph{option} framework~\cite{sutton1999} is a SMDP-based HRL approach used to  efficiently compute the $Q$-function. In this approach, the options are a generalization of actions, that can call other options upon execution in a hierarchical fashion until a \emph{primitive option} (an action executable in the current state) is found. 

\subsection{The RL-BT Approach}
In this section, we outline the approach proposed in \cite{pereira2015framework}, that we choose to denote RL-BT. 
In order to combine the advantages of BTs and RL,
 the RL-BT starts from a manually designed BT where a subset of nodes are replaced with so-called \emph{Learning Nodes},
 which can be both actions and flow control nodes.
 
 In the \emph{Learning Action node}, the node encapsulates a Q-learning algorithm,
 with  states $s$, actions $A$, and reward $r$ defined by the user. 
 For example, imagine a robot with an Action node that need to ``pick an object''. This learning Action node uses Q-learning to learn how to grasp an object in different positions. Then, the state is defined as the pose of the object with respect to the robot's frame, the actions as the different grasp poses, and the reward function as the grasp quality.
 
In the \emph{Learning Control Flow Node}, the node is a standard Fallback with flexible node ordering, quite similar to the ideas described in Chapter~\ref{ch:extensions}. Here, the designer chooses the representation for the state $s$, while the actions in $A$ are the children of the Fallback node, which can be learning nodes or regular nodes. Hence, given a state $s$, the Learning Control Flow Node selects the order of its own children based on the reward. The reward function can be task-depended (i.e. the user finds a measure related to the specific task) or return value-dependent (i.e., it gives a positive reward for Success and negative reward for Failure). For example, consider a NPC that has three main goals: find resources, hide from stronger players, and attack weaker ones. 
The Learning Control Flow node has 3 subtrees, one for each goal. Hence the state $s$ can be a vector collecting information about players' position, weapons position, health level, etc. The reward function can be a combination of health level, score, etc. 

In \cite{pereira2015framework},
a formal definition of the learning node is made, with analogies to the Options approach for HRL to compute the $Q$-function,
and connections between BTs and the recursive nature of the approach.

\subsection{Experimental Results}
\label{sec:experiments}

In this section, we briefly describe the results of two experiments from \cite{pereira2015framework}. Both experiments execute the approach in 30 different episodes, each with 400 iterations. At the beginning of each experiment, the initial state is reset. 
The environment is composed of a set of rooms. In each room, the agent can perform $3$ actions: \emph{save victim}, \emph{use extinguisher X}, and \emph{change room}. Each room has a probability of $0.5$ to have a victim; if there is a victim, the agent must save it first. Moreover, each room has a probability of $0.5$ to have one of $3$ types of fire ($1, 2$ and $3$, with probability $1/3$ each); if there is a fire, the agent must extinguish it before leaving the room. The agent is equipped with $3$ extinguishers (types $A, B$ and $C$), and each extinguisher can successfully extinguish only one type of fire, randomly chosen at the beginning of the episode and unknown to the agent.

\textbf{Scenario 1}
\label{sec:experiments_scenario1}

\begin{figure}
  \centering
  \includegraphics[width=0.85\textwidth]{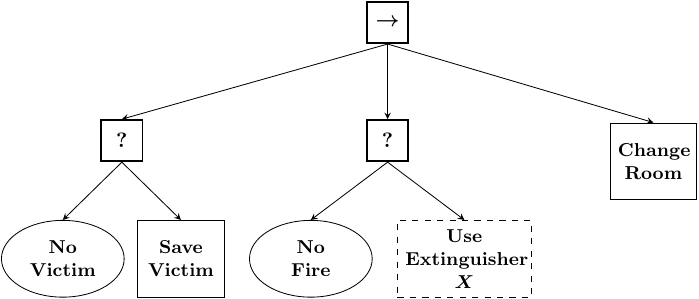}
  \caption{The BT model for the first agent. Only a single Action node (Use Extinguisher) has learning capabilities.}
  \label{fig:model_scenario1}
\end{figure}

In this scenario, we evaluate  the capability of  \emph{Use Extinguisher X} to learn the correct extinguisher to use for a specific type of fire.

Figure~\ref{fig:model_scenario1} depicts the BT modeling the behavior of the learning agent. In the learning Action node \emph{Use Extinguisher X} the state is defined as  $s = \langle \text{\emph{fire type}} \rangle$, where $\text{\emph{fire type}} = \{1, 2, 3\}$, and the available actions are as follows:\\ $A = \{\mbox{Use Extinguisher A}, \mbox{Use Extinguisher B}, \mbox{Use Extinguisher B}\}$. The reward is defined as  $10$ if the extinguisher can put out the fire and $-10$ otherwise.
%
The results in \cite{pereira2015framework} show that the accuracy converges to 100\%.

%

\textbf{Scenario 2}

This scenario is a more complex version of the one above, where we consider the time spent to execute an action.

The actions \emph{Save Victim} and \emph{Use Extinguisher X} now take time to complete, depending on the fire intensity. Any given fire has an intensity $\text{\emph{fire intensity}} \in \{1, 2, 3\}$, chosen randomly for each room. The fire intensity specifies the time steps needed to extinguish a fire. The fire is extinguished when its intensity is reduced to $0$.

The fire intensity is reduced by $1$ each time the agent uses the correct extinguisher. The action \emph{change room} is still executed instantly and the use of the wrong extinguisher makes the agent lose the room. 

\begin{figure}
  \centering
  \includegraphics[width=0.85\textwidth]{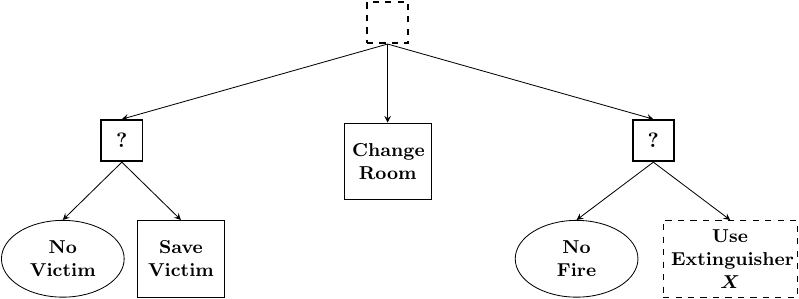}
  \caption{The BT model for the second agent using two nested Learning Nodes.}
  \label{fig:model_scenario2}
\end{figure}

Figure~\ref{fig:model_scenario2} shows the BT modeling the learning agent. 
It uses $2$ learning nodes. The first, similar to the one used in Scenario~1, is a learning Action node using. In that node, the state is defined as $s = \langle \text{\emph{fire type}} \rangle$, where $\text{\emph{fire type}} = \{1, 2, 3\}$,  the action set as $a = \{A, B, C\}$, and the reward as $\frac{10}{\text{\emph{fire intensity}}}$ if the extinguisher can extinguish the fire and $-10$ otherwise.


The second learning node is a learning control flow node. It learns the behavior that must be executed given the state $s = \langle \text{\emph{has victim?}}, \text{\emph{has fire?}}\rangle$. The node's children are the actions $A = \{\text{\emph{save victim}}, \text{\emph{use extinguisher X}}, \text{\emph{change room}}\}$. 
This learning node receives a cumulative reward, $-10$ if the node tries to save and there is no victim, $-1$ while saving the victim, and $+10$ when the victim is saved; $-10$ if trying to extinguish a non-existing fire, $-1$ while extinguishing it, and $+10$ if the fire is extinguished;  $+10$ when the agent leaves the room at the right moment and $-10$ otherwise.


The results in \cite{pereira2015framework} show that the accuracy converges to 97-99\%.
The deviation from 100\% is due to the fact that the
learning control flow node needs some steps of trial-and-error to learn the most effective action order.

%
%

\section{Comparison between GP-BT and RL-BT}
How do we choose between GP-BT and RL-BT for a given learning problem?

GP-BT and RL-BT operate on the same basic premise: they both try something, receive a reward depending on ``how good" the result is, and then try and do something different to get a higher reward. 

The key difference is that GP-BT operates on the BT itself, creating the BT from scratch and extending it when needed.
RL-BT on the other hand starts with a fixed BT structure that is not changed by the learning. Instead, the behaviors of 
a set of designated nodes are improved. If RL-BT was started from scratch with a BT consisting of a single action, there
 would be no difference between RL-BT and standard RL. Instead, the point of combining BTs with RL is to 
 address the curse of dimensionality, and do RL on smaller state spaces, while the BT connects these subproblems in a modular way,
 with the benefits described in this book.
 RL-BT thus needs user input regarding both the BT structure, and the actions, states and local reward functions to be considered in the subproblems addressed by RL. GP-BT on the other hand only needs user input regarding the single global reward function, and the conditions (sensing) and actions available to the agent.

%
%


\section{Learning from Demonstration applied to BTs}
\label{sec.dem}
Programming by demonstration  has a straightforward application in 
both robotics and 
the development of the AI for NPCs in games.

In robotics, the Intera5 software for the Baxter and Sawyer robots provides learning by demonstration support\footnote{\url{http://mfg.rethinkrobotics.com/intera/Building_a_Behavior_Tree_Using_the_Robot_Screen}}.

In computer games, one can imagine a game designer controlling the NPC during a training session in the game,  demonstrating the expected behavior for that character in different situations.
One such approach called \emph{Trained BTs (TBTs)} was  proposed in~\cite{olivenza2017trained}.

TBT applies the following approach:  
it first records traces of actions executed by the designer controlling the NPC to be programmed, then it processes the traces to generate a BT that can 
control the same NPC in the  game. The resulting BT can then be further edited with a  visual editor. 
The process starts with the creation of a minimal BT
using the provided BT editor. This initial BT contains a
special node called a Trainer Node (TN). Data for this node are
collected  in a game session where the designer simulates the intended behavior of the NPC. 
Then, the trainer node is
replaced by a machine-generated sub-behavior that, when
reached, selects the task that best fits the actual state of
the game and the traces stored  during the training session. 
The approach combines the advantages of programming
by demonstration with the ability to fine-tune the learned
BT.

Unfortunately, using learning from demonstration approaches the learned BT easily becomes  very large, as each trace is directly mapped into a new sub-BT. Recent approaches address this problem by finding common patterns in the BT and generalizing them~\cite{robertson2015building}.
First, it creates a maximally specific BT from the given traces. Then
iteratively reduces the BT in size by finding and combining common
patterns  of  actions.  The  process  stops when the BT has no common patters.    
Reducing  the  size  of  the  BT also improves readability.

%
%
%
%
%
%
%
%
%
%
%
%
%
%

%% file: stochastic/stochastic.tex

\graphicspath{{./stochastic/figures/}}

\chapter{Stochastic Behavior Trees}
\label{stochastic}
\label{ch:stochastic}


In this chapter, we study the reliability of reactive  plan executions, in terms of execution times and success probabilities. 
To clarify what we mean by these concepts, we consider the following minimalistic example: a robot is searching for an object, and can choose between the two subtasks \emph{searching on the table}, and \emph{opening/searching  the drawer}.
One possible plan is depicted in Figure \ref{stoch:mcEx1}. Here, the robot first searches the table and then, if the object was not found on the table, opens the drawer and searches it. In the figure,  each task has an execution time and a success probability. For example, searching the table has a success probability of 0.1 and an execution time of 5s. Given a plan like this, it is fairly straightforward to compute the reliability of the entire plan, in terms of execution time distribution and success probability.
In this chapter, we show how to compute such performance metrics for arbitrary complex plans encoded using BTs.
In particular, we will define Stochastic BTs in Section~\ref{stoch:sec:analysis}, transform them into Discrete Time Markov Chains (DTMCs) in Section~\ref{stoch:sec:SBTtoDTCM}, compute reliabilities in Section~\ref{stoch:sec:reli} and describe examples Section~\ref{properties:sec:example:reliability}.
Some of the results of this chapter were previously published in the 
 paper \cite{Colledanchise14}.

\begin{figure}[t]
\centering
\includegraphics[width=\columnwidth]{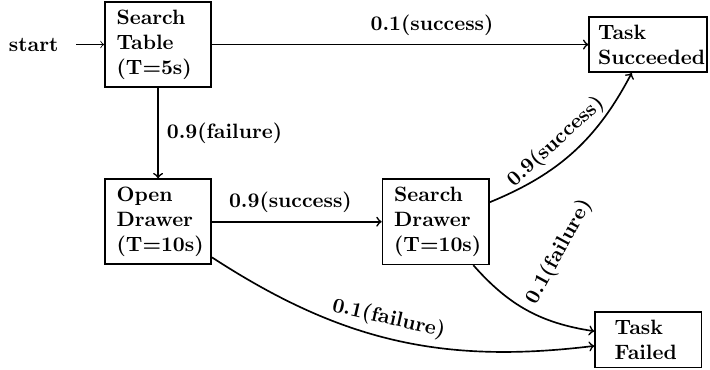}
\caption{A simple plan for a search task, modelled by a Markov Chain.
}
\label{stoch:mcEx1}
\end{figure}


Before motivating our study of BTs we will make a few more observations regarding the example above.
The ordering of the children of Fallback nodes (\emph{searching on the table} and \emph{opening/searching the drawer}) can in general be changed, whereas the ordering of the children of 
Sequence nodes (\emph{opening  the drawer} and \emph{searching  the drawer}) cannot.
Note also that adding subtasks to a Sequence generally decreases overall success probabilities, whereas adding Fallbacks generally increases overall success probabilities, as described in Section~\ref{sec:stochastic_extension}.

\begin{figure}[b]
\centering
\includegraphics[width=0.5\columnwidth]{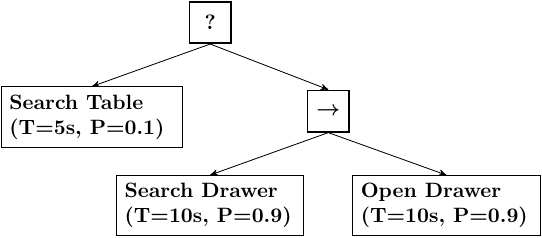}
\caption{The BT equivalent of the Markov chain in Figure \ref{stoch:mcEx1}. The atomic actions are the leaves of the tree, while the  interior nodes correspond to \emph{Sequence} compositions (arrow) and \emph{Fallbacks} compositions (question mark).
}
\label{stoch:btEx1}
\end{figure}

\section{Stochastic BTs}
\label{stoch:sec:analysis}

In this section we will show how some probabilistic measures, such as  Mean Time to Succeed (MTTS),  Mean Time to Fail (MTTF), and probabilities over time carry across modular compositions of BTs. The advantage of using BTs lie in their modularity and hierarchical structure, which provides good scalability, as explained above.  
To  address the questions above, we need to introduce some concepts from Markov theory.

\subsection{Markov Chains and Markov Processes}
Markov theory~\cite{norris1998markov} deals with memoryless processes. If a process is given by a sequence of actions that changes the system's state disregarding its history, a DTMC is suitable to model the plan execution. Whereas if a process is given by a transition rates between states, a Continuous Time Markov Chain (CTMC) it then suitable to model such plan execution.
A DTMC is given by a collection of states  $\mathcal{S}=\{s_1,s_2,\ldots,s_d\}$
and the transitions probabilities $p_{ij}$ between states $s_i$ and $s_j$. A CTMC is given by a collection of states $\mathcal{S}$ and the transition rates $q_{ij}^{-1}$ between states $s_i$ and $s_j$.

\begin{figure}[h]
\centering
\includegraphics[width=0.65\columnwidth ]{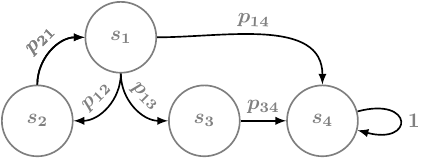}
\caption{Example of a DTMC with 4 states and 6 transitions.}
\label{stoch:bg.fig.mc}
\end{figure}

\begin{definition}
The stochastic sequence $\{X_n, n = 0,1,2,\ldots\}$ is a DTMC provided that:
\begin{equation}
\begin{split}
&P\{X_{n+1} = s_{n+1} \vert X_{n} = s_{n}, X_{n-1} = s_{n-1}, \ldots, X_0 = s_0\} =\\
&=P\{X_{n+1} = s_{n+1} \vert X_{n} = s_{n}\}
\end{split}
\label{stoch:BG.eq.DTMC}
\end{equation}
$\forall \; n\in \mathbb{N}$, and $\forall \; s\in \mathcal{S}$
\end{definition}
The expression on the right hand side of~\eqref{stoch:BG.eq.DTMC} is the so-called \emph{one step transition probability} of the chain and denotes the probability that the process goes from state $s_{n}$ to state $s_{n+1}$. We use the following notation:
\begin{equation}
p_{ij} = P\{X_{n+1} = s_j \vert X_{n} = s_i\}
\end{equation}
to denote the probability to jump from a state $s_i$ to a state $s_j$. Since we only consider homogeneous DTMC, the above probabilities do not change in time.
\begin{definition}
\label{stoch:bg.def.p}
The \emph{one-step transition matrix} $P$ is a $\left\vert{\mathcal{S}}\right\vert \times \left\vert{\mathcal{S}}\right\vert$ matrix in which the entries are the transition probabilities $p_{ij}$.
\end{definition} 
Let $\pi(k)=[\pi_1(k),\ldots, \pi_{\left\vert{\mathcal{S}}\right\vert}(k)]^\top$, where $\pi_i$ is the probability of being in state $i$, then the Markov process can be described as a discrete time system with the following time evolution:
\begin{equation}
\label{stoch:bg.eq.discrete}
\begin{cases}
\pi(k+1)=P^\top\pi(k) \\
\pi(0)= \pi_0.
\end{cases}.
\end{equation}
where $\pi_0$ is assumed to be known a priori.

\begin{definition}
The stochastic sequence $\{X(t), t\geq 0 \}$ is a CTMC provided that:
\begin{equation}
\begin{split}
P\{&X(t_{n+1}) = s_{n+1} \vert X(t_{n}) = s_{n}, X(t_{n-1}) = s_{n-1}, \ldots,\\ 
,&X(t_0) = s_0\} =P\{X(t_{n+1}) = s_{n+1} \vert X(t_{n}) = s_{n}\}
\end{split}
\label{stoch:BG.eq.CTMC}
\end{equation}
\end{definition}
$\forall \; n\in \mathbb{N}$, $\forall \; s\in \mathcal{S}$, and all sequences $\{t_0,t_1,\ldots,t_{n+1}\}$ such that $t_0 < t_1 < \ldots <  t_{n} <  t_{n+1}$.
We use the following notation:
\begin{equation}
p_{ij}(\tau) = P\{X(t+\tau) = s_j \vert X(\tau) = s_i\}
\end{equation}
to denote the probability to be in a state $s_j$ after a time interval of length $\tau$ given that at present time is into a state $s_i$. Since we only consider homogeneous CTMC, the above probabilities only depend on the time length $\tau$.

To study the continuous time behavior of a Markov process we define the so-called \emph{infinitesimal generator matrix} $Q$.
\begin{definition}
The infinitesimal generator of the transition probability matrix $P(t)$ is given by:
\begin{equation}
Q = \left[q_{ij}\right]
\end{equation}
where 

\begin{equation}
q_{ij}=
   \begin{cases}
   \displaystyle\lim_{\Delta t \to 0} \frac{p_{ij}(\Delta t)}{\Delta t}& \text{if } i \neq j\\
   - \displaystyle\sum_{k \neq i} q_{kj} & \text{otherwise}.
  \end{cases}
  \label{stoch:bg.eq.infGen}
\end{equation}
\end{definition}
Then, the continuous time behavior of the Markov process is described by the following ordinary differential equation, known as the Cauchy problem:
\begin{equation}
\begin{cases}
\dot \pi(t)=Q^\top\pi(t) \\
\pi(0)=\pi_0
\end{cases}
\label{stoch:bg.eq.cauchy}
\end{equation}
where the initial probability vector $\pi_0$ is assumed to be known a priori.
\begin{definition}
The average sojourn time $SJ_i$ of a state $s_i$ in a CTMC is the average time spent in that state. It is given by~\cite{stewart2009}:
\begin{equation}
SJ_i = -\frac{1}{q_{ii}}
\label{stoch:BG.eq.meansj}
\end{equation} 
\end{definition}
\begin{definition}
Considering the CTMC $\{X(t), t\geq 0 \}$, the stochastic sequence $\{Y_n, n = 0,1,2,\ldots\}$ is a DTMC and it is called Embedded MC (EMC) of the process $X(t)$~\cite{stewart2009}.
\end{definition}
The transition probabilities of the EMC $r_{ij}$ are defined as:
\begin{equation}
r_{ij} = P\{Y_{n+1} = s_j \vert Y_{n} = s_i\}
\end{equation}
and they can be easily obtained as a function of the transition rates $q_{ij}$:
\begin{equation}
r_{ij}=
   \begin{cases}
   -\frac{q_{ij}}{q_{ii}}& \text{if } i \neq j\\
    \displaystyle 1-\sum_{k \neq i} r_{kj} & \text{otherwise}.
  \end{cases}
\end{equation}
On the other hand, the infinitesimal generator matrix $Q$ can be reconstructed from the EMC as follows
\begin{equation}
q_{ij}=
   \begin{cases}
   \frac{1}{SJ_j }r_{ij} & \text{if } i \neq j\\
   - \displaystyle\sum_{k \neq i} r_{kj} & \text{otherwise}
  \end{cases}.
\end{equation}

\subsection{Formulation}
\label{stoch:sec:problem_formulation}

We are now ready to make some definitions and assumptions, needed to compute the performance estimates.

\begin{definition}
\label{stoch:def:stoc}
An action $\mathcal{A}$ in a BT, is called \emph{stochastic}
 if the following holds:
\begin{itemize}
\item
It first returns Running, for an amount of time that might be zero or non-zero, then consistently returns either Success or Failure for the rest of the execution of its parent node.\footnote{The execution of the parent node starts when it receives a tick and finishes when it returns either Success/Failure to its parent.}
\item The probability to succeed $ p_s$ and the probability to fail $ p_f$ are known a priori.
\item The probability to succeed  $p_s(t)$ and the probability to fail $p_f(t)$ are exponentially distributed with the following Probability Density Functions (PDFs):
\begin{eqnarray}
 \hat p_s(t)&=& p_s \mu e^{-\mu t}\\
  \hat p_f(t)&=& p_f \nu e^{-\nu t}
\end{eqnarray}
from which we can calculate the Cumulative Distribution Functions (CDFs)
\begin{eqnarray}
 \bar p_s(t)&=& p_s 	(1- e^{-\mu t})\\
  \bar p_f(t)&=& p_f (1- e^{-\nu t})
\end{eqnarray}
\end{itemize} 
\end{definition}
\begin{definition}
\label{stoch:def:det}
An action $\mathcal{A}$ in a BT, is called \emph{deterministic} (in terms of execution time, not outcome) if the following holds:
\begin{itemize}
\item
It first returns Running, for an amount of time that might be zero or non-zero, then consistently returns either Success or Failure for the rest of the execution of its parent node.
\item The probability to succeed $p_s$ and the probability to fail $p_f$ are known a priori.
\item The time to \emph{succeed} and the time to \emph{fail} are deterministic variables $\tau_s$ and $\tau_f$ known a priori.
\item The probability to succeed  $p_s(t)$ and the probability to fail  $p_f(t)$ have the following PDFs:
\begin{eqnarray}
\hat p_s(t) &=& p_s  \delta(t-\tau_s)\label{stoch:eq:d1}  \\
\hat p_f(t) &=&  p_f  \delta(t-\tau_f)  \label{stoch:eq:d2}
\end{eqnarray}
where $\delta(\cdot)$ is the Dirac's delta function. From the PDFs we can calculate the CDFs:
\begin{eqnarray}
\bar p_s(t) &=& p_s  H(t-\tau_s)\label{stoch:eq:h1}  \\
\bar p_f(t) &=&  p_f  H(t-\tau_f)  \label{stoch:eqh2}
\end{eqnarray}
where $H(\cdot)$ is the step function.

\end{itemize} 
\end{definition}

\begin{remark}
 Note that it makes sense to sometimes have $\tau_s \neq \tau_f$. Imagine a door opening task which takes 10s to complete successfully but fails 30\% of the time after 5s when the critical grasp phase fails.
\end{remark}

\begin{example}
 For comparison, given a deterministic action with $\tau_s$, we let the rates of a stochastic action have $\mu=\tau_s^{-1}$. Then the PDFs and CDFs are as seen in 
Figure~\ref{stoch:DE.fig.prob}.
\end{example}
%
\begin{figure}[h]
\centering
\begin{subfigure}[h]{0.8\columnwidth}
\includegraphics[width=0.8\columnwidth]{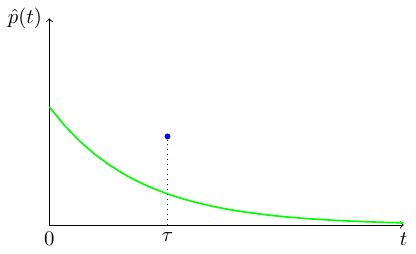}
\caption{PDFs.}
\end{subfigure}
\begin{subfigure}[h]{0.8\columnwidth}
\includegraphics[width=0.8\columnwidth]{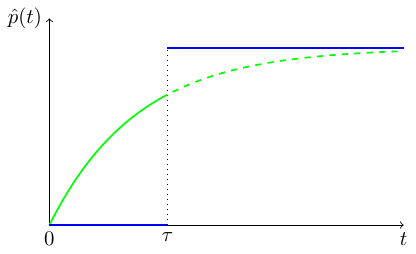}
\caption{CDFs.}
\end{subfigure}
\caption{Cumulative and probability density distribution function for a deterministic (dark straight lines) and stochastic action (bright curves).}
\label{stoch:DE.fig.prob}
\end{figure}

As we want to analyze the BT composition of actions, we must also define actions that include both stochastic and deterministic parts.

\begin{definition}
\label{stoch:def:hyb}
An action $\mathcal{A}$ in a BT, is called \emph{hybrid} if one of $p_s(t)$ and $p_f(t)$ is a random variable with exponential distribution, and the other one is deterministic.
\end{definition}

Thus  hybrid actions come in two different variations:
\textbf{Deterministic Success Time}
For this type of hybrid action, the following holds:    
\begin{itemize}
\item
It first returns Running, for an amount of time that might be zero or non-zero, then consistently returns either Success or Failure for the rest of the execution of its parent node.
\item The probability to succeed $ p_s$ is known a priori.
\item The time to \emph{succeed} is a deterministic variable $\tau_s$ known a priori.
\item The probability to fail has the following PDF:
\begin{equation}
\hat p_f(t)=\begin{cases} p_f (1- e^{-\nu t}) &\mbox{if } t < \tau_s \\ 
1-p_s &\mbox{if } t = \tau_s \\
0 & \mbox{otherwise } 
 \end{cases}.
\end{equation}

\end{itemize}
In this case the CDF and the PDF of the probability to succeed are discontinuous. In fact this hybrid action will return Failure if, after the success time $\tau_s$, it does not return Success.
Then, to have an analogy with stochastic actions we derive the PDF of the probability to succeed:
\begin{equation}
\hat p_s(t) = p_s  \delta(t-\tau_s)  
\end{equation}
and the CDFs as follows:
\begin{equation}
\bar p_s(t) = p_s  H(t-\tau_s)  
\end{equation}
\begin{equation}
\bar p_f(t)=\begin{cases} p_f (1- e^{-\nu t}) &\mbox{if } t < \tau_s \\ 
1 - \bar p_s(t)& \mbox{otherwise }  \end{cases}.
\end{equation}
Thus, the probability of Running is zero after $\tau_s$ i.e. after $\tau_s$ it either fails or succeeds.
Moreover, the success rate is set to $\mu=\tau_s^{-1}$.

\textbf{Deterministic Failure Time}
For this type of hybrid action, the following holds:    
\begin{itemize}
\item It first returns Running, for an amount of time that might be zero or non-zero, then consistently returns either Success or Failure for the rest of the execution of its parent node.
\item The probability to fail $p_f$ is known a priori.
\item The time to \emph{succeed} is a random variables with exponential distribution with rate $\mu$ known a priori. 
\item The probability to succeed has the following PDF:
\begin{equation}
\hat p_s(t)=\begin{cases} p_s (1- e^{-\mu t}) &\mbox{if } t < \tau_f \\ 
1-p_f &\mbox{if } t = \tau_f \\
0 & \mbox{otherwise } 
 \end{cases}.
\end{equation}
\end{itemize}
To have an analogy with stochastic actions we derive the PDF of the probability to fail:
\begin{equation}
\hat p_f(t) = p_f  \delta(t-\tau_f)  
\end{equation}
and the CDFs as follows:
\begin{equation}
\bar p_f(t) = p_f  H(t-\tau_f)  
\end{equation}
\begin{equation}
\bar p_s(t)=\begin{cases} p_s (1- e^{-\mu t}) &\mbox{if } t < \tau_f \\ 
1 - \bar p_f(t)& \mbox{otherwise }  \end{cases}.
\end{equation}

Moreover, the failure rate is set to $\nu=\tau_f^{-1}$

\begin{remark}
Note that the addition of deterministic execution times makes~\eqref{stoch:bg.eq.cauchy} discontinuous on the right hand side, but it still has a unique solution in the Carath\'eodory sense~\cite{filippov1988differential}.
\end{remark}

We will now give an example of how these concepts transfer over BT compositions.

\begin{figure}[h]
\centering
\includegraphics[width=0.3\columnwidth]{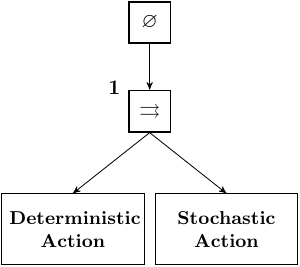}
\caption{Parallel node of Example \ref{stoch:ex:par}.}
\label{stoch:DE.fig.parexbt}
\end{figure}

\begin{example}
\label{stoch:ex:par}
Consider the BT in Figure~\ref{stoch:DE.fig.parexbt}. The Parallel node is set to returns Success as soon as one child returns Success, and the two children are of different kinds,
one deterministic and the other stochastic.
Note that the MTTS and MTTF of this BT has to account for the heterogeneity of its children. 
The deterministic child can succeed only at time $\tau_s$. The CDF of the Parallel node is given by the sum of the CDFs of its children. The PDF has a jump at time $\tau_s$ accounting for the fact that the Parallel node is more likely to return Success after that time.
Thus, the PDF and the CDF of a Success return status are shown in Figure~\ref{stoch:DE.fig.parexpd}. 

\begin{figure}[h]
\centering
\begin{subfigure}[h]{0.8\columnwidth}
\includegraphics[width=0.8\columnwidth]{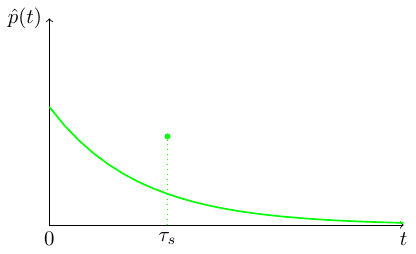}
\caption{PDF.}
\end{subfigure}
\begin{subfigure}[h]{0.8\columnwidth}
\includegraphics[width=0.8\columnwidth]{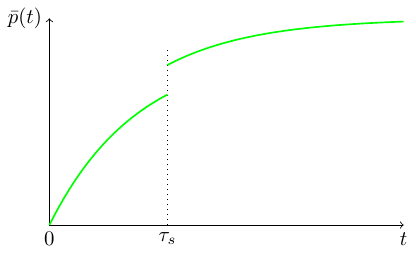}
\caption{CDF.}
\end{subfigure}
\caption{Cumulative and probability density distribution function of Success of the Parallel node in Figure~\ref{stoch:DE.fig.parexbt}.}
\label{stoch:DE.fig.parexpd}
\end{figure}
\end{example}

\begin{definition}
A BT $\bt_1$ and a BT $\bt_2$ are said \emph{equivalent} if and only if 
$\bt_1$ can be created from $\bt_2$ by permutations of the children of Fallbacks and Parallel compositions.
\label{stoch:PS.def.eq}
\end{definition}
An example of two equivalent BTs is shown in Figure~\ref{stoch:fig:search}.

\begin{assumption}
\label{stoch:ass:SBT:ac}
For each action $\mathcal{A}$ in the BT, one of the following holds
\begin{itemize}
\item The action $\mathcal{A}$ is a stochastic action.
\item The action $\mathcal{A}$ is a deterministic action. 
\item The action $\mathcal{A}$ is a hybrid action. 
\end{itemize} 
\end{assumption}

\begin{assumption}
\label{stoch:ass:SBT:cond}
For each condition $\mathcal{C}$ in the BT, the following holds
\begin{itemize}
\item It consistently returns the same value (Success or Failure) throughout the execution of its parent node.
\item The probability to succeed at any given time $p_s(t)$ and the probability to fail at any given time $p_f(t)$ are known a priori. 
\end{itemize} 
\end{assumption}
We are now ready to define a Stochastic BT (SBT).
\begin{definition}
A SBT is a BT satisfying Assumptions~\ref{stoch:ass:SBT:ac} and~\ref{stoch:ass:SBT:cond}.
\end{definition}

Given a SBT, we want to use the probabilistic descriptions of its actions and conditions, $p_s(t)$, $p_f(t)$, $\mu$ and $\nu$, to recursively compute analogous descriptions for every subtrees and finally the whole tree. 

%
%

To illustrate the investigated problems and SBTs we take a look at the following example.
\begin{example}
\label{stoch:ex:search}
Imagine a robot that is to search for a set of keys on a table and in a drawer. The robot knows that the keys are often located in the drawer, so that location is more likely than the table. However, searching the table takes less time, since the drawer must be opened first. Two possible plans are conceivable: searching the table first, and then the drawer, as in Figure~\ref{stoch:fig:searcht}, or the other way around as in Figure~\ref{stoch:fig:searchd}. These two plans can be formulated as SBTs and analyzed through the scope of Problem 1 and 2, using the results of Section~\ref{stoch:sec:analysis} below. Depending on the user requirements in terms of available time or desired reliability at a given time, the proper SBT can be chosen.
\begin{figure}
        \centering
        \begin{subfigure}[b]{0.5\columnwidth}
                \centering
                \includegraphics[width=0.8\columnwidth]{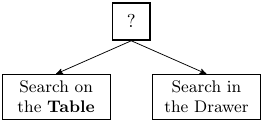}
                \caption{}
                 \label{stoch:fig:searcht}

        \end{subfigure}%
       ~ 
        \begin{subfigure}[b]{0.5\columnwidth}
                \centering
                \includegraphics[width=0.8\columnwidth]{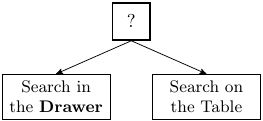}
                \caption{}
                \label{stoch:fig:searchd}
        \end{subfigure}
        ~ 
        
                        \caption{BT modeling of two plan options. In (a), the robot searches on the table first, and in the drawer only if the table search fails. In (b), the table is searched only if nothing is found in the drawer.}    
                        \label{stoch:fig:search}
\end{figure}
\end{example}
\begin{remark}
 Note that Assumptions~\ref{stoch:ass:SBT:ac} corresponds to the return status  of the search actions in Example \ref{stoch:ex:search} behaving in a reasonable way, e.g., not switching between Success and Failure.
\end{remark}

\section{Transforming a SBT into a DTMC}
\label{stoch:sec:SBTtoDTCM}
The first step is to define, for each control flow node in $\mathcal{V}$, a vector representation of the children's outcomes and a description of its execution policy, then we map the execution into a DTMC with a direct representation of the one-step transition matrix, and finally we compute the probability of success and failure over time for each node.  

Note that the modularity of BTs comes from the recursive tree structure, any BT can be inserted as subtree in another BT.
This modularity allows us to do the analysis in a recursive fashion, beginning with the leaves of the BT, i.e. the actions and conditions which have known probabilistic parameters according to Assumptions~\ref{stoch:ass:SBT:ac} and~\ref{stoch:ass:SBT:cond}, and then progressing upwards in a scalable fashion.

 To keep track of the execution of a given flow control node,
the children outcomes  are collected in a vector state called the \emph{marking} of the node, and the transitions between markings are defined according to the execution policy of the node. 
In detail, let $\mathbf{m}(k)=[m_1(k),m_2(k),\ldots,m_N(k)]$ be a marking of a given BT node with $N$ children at time step $k$ with
\begin{equation}
m_i(k) =
	 \begin{cases} -1 &\mbox{if child $i$ returns Failure  at k}  \\
 		\hspace{8pt}1 & \mbox{if child $i$ returns Success  at k}\\
 		\hspace{8pt}0  & \mbox{otherwise} \\
 	\end{cases} 
\end{equation}
\begin{example}
\label{stoch:ex:Marking}
 Consider the BT in Figure\ref{stoch:fig:searcht}. If the first child (Search Table) has failed, and the second (Search Drawer) is currently running, the marking would be $\mathbf{m}(k)=[-1,0]$.
\end{example}
We define an \emph{event} related to a BT node when one of its children returns either Success or Failure.
Defining $\mathbf{e}_i(k)$ to be the vector associated to the event of the $i$-th running child, all zeros except the $i$-th entry which is equal to $e_i(k) \in \{-1,1\}$:
 \begin{equation}
e_i(k) = \begin{cases} -1 &\mbox{if child $i$ has failed at k}  \\
\hspace{8pt}1 & \mbox{if child $i$ has succeeeded at k.}
 \end{cases} 
\end{equation}

We would like to describe the time evolution of the node marking due to an event associated with the child $i$ as follows:
\begin{equation}
 \mathbf{m}(k+1)=\mathbf{m}(k)+\mathbf{e}_i(k)
\end{equation}
 with the event $\mathbf{e}_i(k)$ restricted to the feasible set of events at $\mathbf{m}(k)$, i.e. 
 $$\mathbf{e}_i(k) \in \mathcal{F}(\mathbf{m}(k)).$$ 
In general, $\mathcal{F}(\mathbf{m}(k)) \subset \mathcal{F}_0$, with
\begin{equation}
 \mathcal{F}_0=\{\mathbf{e}_i: \mathbf{e}_i\in \{-1,0,1\}^N, ||\mathbf{e}_i||_2=1 \},
\end{equation}
i.e. events having only one nonzero element, with value $-1$ or $1$. We will now describe the set $\mathcal{F}(\mathbf{m}(k))$ for the three different node types.

\subsubsection*{Feasibility condition in the Fallback node}
\begin{equation}
\begin{split}
\mathcal{F}_{FB}(\mathbf{m}(k))=\{\mathbf{e}_i\in  \mathcal{F}_0: & \exists i: m_i(k)=0, e_i \neq 0, \\ & m_j(k) = -1, \forall j, 0<j<i  \},
\end{split}
\end{equation}
i.e. the event of a child returning Success or Failure is only allowed if it was ticked, which only happens if it is the first child, or if all children before it have returned Failure.

\subsubsection*{Feasibility condition in the Sequence node}
\begin{equation}
\begin{split}
\mathcal{F}_{S}(\mathbf{m}(k))=\{\mathbf{e}_i\in  \mathcal{F}_0: & \exists i: m_i(k)=0, e_i \neq 0,  \\ & m_j(k) = 1, \forall j, 0<j<i  \},
\end{split}
\end{equation}
i.e. the event of a child returning Success or Failure is only allowed if it was ticked, which only happens if it is the first child, or if all children before it have returned Success.

\subsubsection*{Feasibility condition in the Parallell node}
\begin{equation}
\begin{split}
\mathcal{F}_{P}(\mathbf{m}(k))=\{\mathbf{e}_i\in  \mathcal{F}_0: 
& \exists i: m_i(k)=0, e_i \neq 0, \\
& \Sigma_{j:m_j(k)>0}m_j(k)<M \\
&  \Sigma_{j:m_j(k)<0}m_j(k)<N-M+1 \},
\end{split}
\end{equation}
i.e. the event of a child returning Success or Failure is only allowed it if has not returned yet, and the conditions for Success ($<M$ successful children) and Failure ($<N-M-1$ failed children) of the Parallel node are not yet fulfilled.

\begin{example}
Continuing Example \ref{stoch:ex:Marking} above, $\mathcal{F}(\mathbf{m}(k))=\mathcal{F}_{FB}([-1,0])=\{(0,1), (0, -1)\}$, i.e. the second child returning Success or Failure.
Note that if the first child would have returned Success, the feasible set would be empty $\mathcal{F}_{FB}([1,0])=\emptyset$.
\end{example}

The \emph{Marking Reachability Graph} (MRG)\label{definition:RG}, see Figure \ref{stoch:app.fig.seq},
 of a BT node can now be computed starting from the initial marking $\mathbf{m}(0)=\mathbf{m}_0=\mathbf{0}^\top$, taking into account all the possible event combinations that satisfy the feasibility condition.

\begin{definition}
A marking $\mathbf{m}_{i}$ is reachable from a marking $\mathbf{m}_j$ if there exists a sequence of feasible events $\boldsymbol{\sigma}=[\sigma_1,\sigma_2,\ldots,\sigma_g]$ such that $\mathbf{m}_{i}= \mathbf{m}_{j} + \sum_{h=1}^g \sigma_h$.
\end{definition}
\begin{remark}
Note that $\mathbf{m}(k)=\mathbf{m}_i$ when $\mathbf{m}_i$ is the marking at time $k$.
\end{remark}
\begin{figure}[h]
\centering
\includegraphics[width=1\columnwidth ]{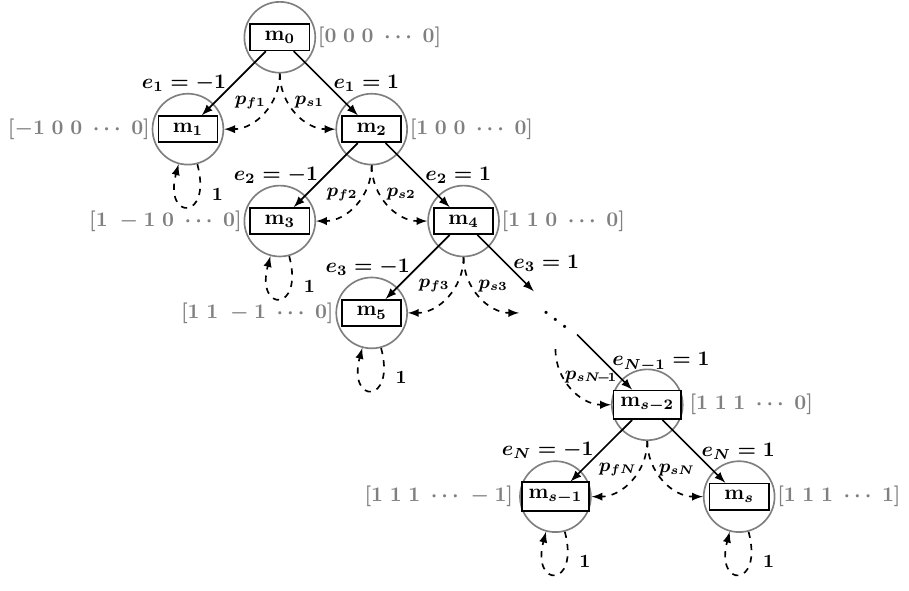}
\caption{MRG of the Sequence node (rectangles) with $N$ children and its DTMC representation (circles).}
\label{stoch:app.fig.seq}
\end{figure}
\begin{figure}[h]
\centering
\includegraphics[width=1\columnwidth ]{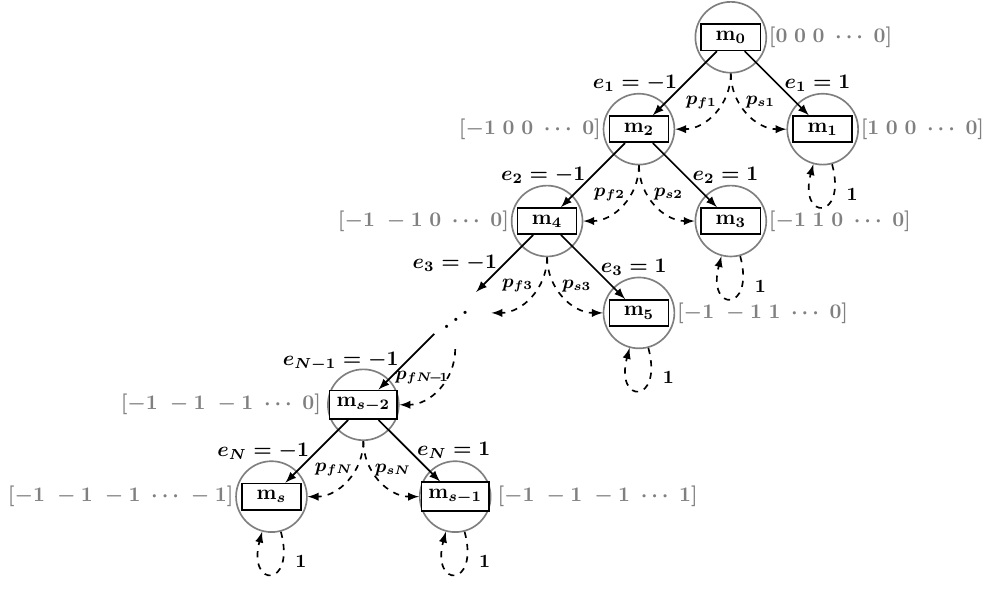}
\caption{MRG of the Fallback node (rectangles) with $N$ children and its DTMC representation (circles).}
\label{stoch:app.fig.sel}
\end{figure}

\begin{figure}[h]
\centering
\includegraphics[width=1\columnwidth ]{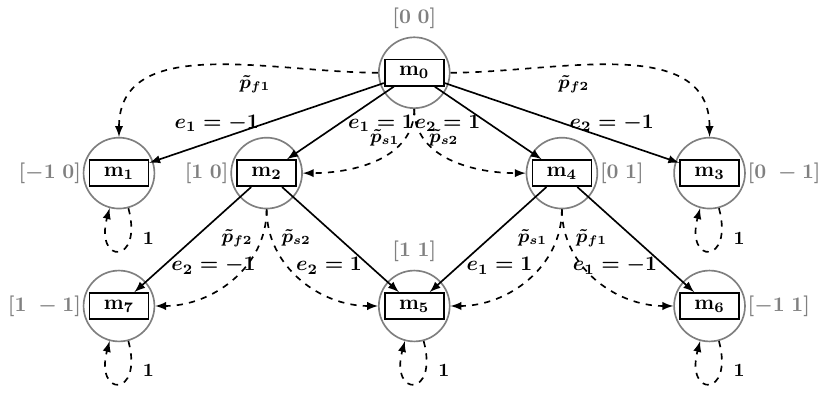}
\caption{MRG of the Parallel node (rectangles) with $2$ children and its DTMC representation (circles).}
\label{stoch:app.fig.par}
\end{figure}

\subsection{Computing Transition Properties of the DTMC}
The MRG of a BT node comprises all the reachable markings, the transitions between them describe events which have a certain success/failure probability.  We can then map the node execution to a DTMC where the states are the markings in the MRG and the one-step transition matrix $P$ is given by the probability of jump between markings, with off diagonal entries defined as follows:

\begin{equation}
p_{ij}=
   \begin{cases}
   \tilde p_{sh} & \!\!\!\text{if } \mathbf{m}_j-\mathbf{m}_i\in \mathcal{F}(\mathbf{m}_i) \land  e_h\mathbf{e}_h^T(\mathbf{m}_j-\mathbf{m}_i)>0\\
   \tilde p_{fh} & \!\!\!\text{if } \mathbf{m}_j-\mathbf{m}_i\in \mathcal{F}(\mathbf{m}_i) \land  e_h\mathbf{e}_h^T(\mathbf{m}_j-\mathbf{m}_i)<0\\
   0 & \!\!\!\text{otherwise}
  \end{cases} 
\end{equation}
and diagonal entries defined as:
\begin{equation}
p_{ii}=1-\sum_{j}p_{ij}. 
\end{equation}
 with:
\begin{equation}
\tilde p_{sh} = \frac{p_{sh} \mu_h \nu_h}{p_{fh} \mu_h + p_{sh}\nu_h} \cdot \left( \sum_{j: \mathbf{e}_j \in \mathcal{F}(\mathbf{m}_i)} \frac{\mu_j \nu_j}{p_{fj} \mu_j + p_{sj}\nu_j} \right)^{-1}
\end{equation} 
and

\begin{equation}
\tilde p_{fh} = \frac{p_{fh} \mu_h \nu_h}{p_{fh} \mu_h + p_{sh}\nu_h} \cdot \left( \sum_{j: \mathbf{e}_j \in \mathcal{F}(\mathbf{m}_i)} \frac{\mu_j \nu_j}{p_{fj} \mu_j + p_{sj}\nu_j} \right)^{-1}
\end{equation} 
where $p_{sj}$ and $p_{fj}$ is the $p_{s}$ and $p_{f}$ of child $j$.

\begin{remark}
For Sequence and Fallback nodes the following holds: $\tilde p_{sh}=p_{sh}$ and $\tilde p_{fh}=p_{fh}$.
\end{remark}
In Figures.~\ref{stoch:app.fig.seq} and~\ref{stoch:app.fig.sel} the mapping from MRG to a DTCM related to a Sequence node and a Fallback node are shown. In Figure~\ref{stoch:app.fig.par} the mapping for a Parallel node with two children and $M=2$ is shown. We choose not to depict the mapping of a general Parallel node, due to its large amount of states and possible transition between them.

To obtain the continuous time probability vector $\pi(t)$ we need to compute the infinitesimal generator matrix $Q$ associated with the BT node. For doing so we construct a CTMC for which the EMC is the DTMC of the BT node above computed. According to~\eqref{stoch:bg.eq.infGen} the map from the EMC and the related CTMC is direct, given the average sojourn times $SJ_i$. 
\section{Reliability of a SBT}
\label{stoch:sec:reli}
\subsection{Average sojourn time}
We now compute the average sojourn time of each marking $\mathbf{m}_i$ of a BT node. 
\begin{lemma}
For a BT node with $p_{si},p_{fi}, \mu_i,\nu_i$ given for each child, the average sojourn time of in a marking $\mathbf{m}_i$ is:
\begin{equation}
SJ_i =\left(   \displaystyle\sum_{h:\mathbf{e}_h\in \mathcal{F}(\mathbf{m}_i)} \left(\frac{p_{sh}}{\mu_h}+\frac{p_{fh}}{\nu_h}\right)^{-1}     \right)^{-1}
\label{stoch:app.eq.sj}
\end{equation}  
with $h: \mathbf{e}_h \in \mathcal{F}(\mathbf{m}_i)$.

\end{lemma}
\begin{proof}

In each marking one of the following occur: the running child $h$ fails or succeeds. To take into account both probabilities and time rates, that influence the average sojourn time, we describe the child execution using an additional CTMC, depicted in Figure~\ref{stoch:app.fig.sjmc}

\begin{figure}[h]
\centering
\includegraphics[width=0.6\columnwidth ]{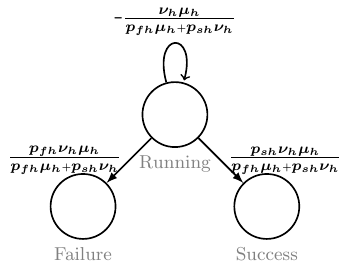}
\caption{CTMC of a child's execution.}
\label{stoch:app.fig.sjmc}
\end{figure}
According to~\eqref{stoch:BG.eq.meansj} the average sojourn time is: 
\begin{equation}
\tau_i=\frac{p_{fh} \mu_h + p_{sh}\nu_h } {\nu_h\mu_h } = \frac{p_{sh}}{\mu_h}+\frac{p_{fh}}{\nu_h}
\end{equation}
and the rate of leaving that state is ${\tau_i}^{-1}$.
Now to account all the possible running children outcome, e.g. in a Parallel node, we consider all the rates associate to the running children. The rate of such node is the sum of all the rates associated to the running children ${\tau_i}^{-1}$. Finally, the average sojourn time of a marking $\mathbf{m}_i$ is given by the inverse of the combined rate:
\begin{equation}
\frac{1}{ SJ_i }= \displaystyle \sum_{h:\mathbf{e}_h\in \mathcal{F}(\mathbf{m}_h)} \frac{1}{\frac{p_{sh}}{\mu_h}+\frac{p_{fh}}{\nu_h}}
\end{equation}
from which we obtain~\eqref{stoch:app.eq.sj}.
\end{proof}
\begin{remark}
The EMC associated with the CTMC in Figure~\ref{stoch:app.fig.sjmc} is depicted in Figure~\ref{stoch:app.fig.dsjmc}. It describes the child's execution as a DTMC.
\end{remark}
\begin{figure}[h]
\centering
\includegraphics[width=0.45\columnwidth ]{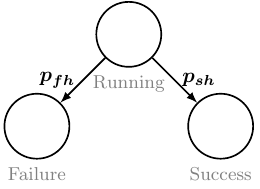}
\caption{DTMC of a child's execution.}
\label{stoch:app.fig.dsjmc}
\end{figure}


\subsection{Mean Time To Fail and Mean Time To Succeed}
To derive a closed form of the mean time to fail MTTF/MTTS of a BT node, we take the probability to reach a success/failure state from the DTCM and the average time spent in each state visited before reaching this state obtained from~\eqref{stoch:app.eq.sj}.
We rearrange the state space of the DTMC so that the initial state is first, the other transient states are second, the failure states are second last and the success states are last:
\begin{equation}
P_c^\top=\left[
\begin{array}{ccc}
T & {0}& {0} \\ 
R_F & \mathbb{I} & {0} \\
R_S & {0} & \mathbb{I} 
\end{array} \right]
\label{stoch:Example.eq.MTTSform}
\end{equation}
where $T$ is the matrix describing the one-step transition from a transit state to another one, $R_F$ is a the matrix describing the one-step transition from a transit state to a failure state, and $R_S$ is the matrix describing the one-step transition from a transit state to a success state. We call this rearrangement the \emph{canonization} of the state space. 


\begin{lemma}
Let $A$ be a matrix with the $ij$-th entry defined as $\exp(t_{ij})$ where $t_{ij}$ is the time needed to transit from a state $j$ to a state $i$ if $j,i$ are neighbors in the MRG, $0$ otherwise. 
The MTTF and MTTS of the BT node can be computed as follows
\begin{equation}
MTTF=\frac{\sum_{i=1}^{\left\vert{\mathcal{S}_F}\right\vert} u^F_{i1}\mbox{log}(h^F_{i1})}{\sum_{i=1}^{\left\vert{\mathcal{S}_F}\right\vert}u^F_{i1}}
\label{stoch:main.eq.mttf}
\end{equation}
where:
\begin{equation} 
H^F\triangleq A_F \sum_{i=0}^{\infty} A_T^i.
\label{stoch:main.eq.hf}
\end{equation}
and
\begin{equation}
MTTS=\frac{\sum_{i=1}^{\left\vert{\mathcal{S}_S}\right\vert} u^S_{i1}\mbox{log}(h^S_{i1})}{\sum_{i=1}^{\left\vert{\mathcal{S}_S}\right\vert}u^S_{i1}}
\label{stoch:main.eq.mtts}
\end{equation}
where:
\begin{equation} 
H^S \triangleq A_S \sum_{i=0}^{\infty} A_T^i
\label{stoch:main.eq.hs}
\end{equation}
where $A_T$, $A_F$, and $A_S$ are the submatrices of $A$ corresponding to the canonization described in~\eqref{stoch:Example.eq.MTTSform}, for which the following holds:
\begin{equation}
 A=\left[
\begin{array}{ccc}
A_T & {0}& {0} \\ 
A_F & {0} & {0} \\
A_S & {0} & {0} 
\end{array} \right].
\label{stoch:main.eq.A}
\end{equation}

\end{lemma}

\begin{proof}
Failure and success states are absorbing, hence we focus our attention on the probability of leaving a transient state, described by the matrix $U$, defined below:

\begin{equation}
U=\sum_{k=0}^\infty  T^i,
\label{stoch:bg.eq.u}
\end{equation}
Thus, considering $i$ as the initial transient state, the entries $u_{ij}$ is the mean number of visits of $j$ starting from $i$ before 
being absorbed, we have to distinguish the case in which the absorbing state is a failure state from the case in which it is a success state:
\vspace{-8pt}
\begin{eqnarray} 
U^F\triangleq R_F U
\label{stoch:mtt.eq.uf} \\
 U^S\triangleq R_S U .
\label{stoch:mtt.eq.us} 
\end{eqnarray}
Equations~\eqref{stoch:mtt.eq.uf} and~\eqref{stoch:mtt.eq.us} represent the mean number of visits before being absorbed in a failure or success state respectively.

To derive MTTF/MTTS we take into account the mean time needed to reach every single failure/success state with its probability, normalized over the probability of reaching any failure (success) state, starting from the initial state. Hence we sum the probabilities of reaching a state starting from the initial one, taking into account only the first column of the matrices obtaining \eqref{stoch:main.eq.mttf} and \eqref{stoch:main.eq.mtts}.
\end{proof}
\begin{remark}
Since there are no self loops in the transient state of the DTMC above, the matrix $T$ is nilpotent. Hence $u_{ij}$ is finite $\forall i,j$.
\end{remark}

\subsection{Probabilities Over Time} 
Since all the marking of a BT node have a non null corresponding average sojourn time, the corresponding DTMC is a EMC of a CTMC with infinitesimal generator matrix $Q(t)$ as defined in \eqref{stoch:bg.eq.infGen}. Hence, we can compute the probability distribution over time of the node according to~\eqref{stoch:bg.eq.cauchy} with the initial condition $\pi_0=[1 \;\mathbf{0}]^\top$ that represents the state in which none of the children have returned Success/Failure yet.


\subsection{Stochastic Execution Times}

\begin{proposition}
Given a SBT, with known probabilistic parameters for actions and conditions, we can compute probabilistic measures for the rest of the tree as follows:
For each node whose children have known probabilistic measures we compute the related DTMC. Now the probability of a node to return Success $p_s(t)$ (Failure $p_f(t)$) is given by the sum of the probabilities of the DTMC of being in a success (failure) state.
Let $\mathcal{S}_S\subset \mathcal{S}_A$, and $\mathcal{S}_F\subset \mathcal{S}_A$ be the set of the success and failure states respectively of a DTMC related to a node, i.e. those states representing a marking in which the node returns Success or Failure, with $\mathcal{S}_F \cup \mathcal{S}_S =\mathcal{S}_A$ and $\mathcal{S}_F \cap \mathcal{S}_S =\emptyset$. 

Then we have
\begin{eqnarray}
\bar p_s(t)=\displaystyle \sum_{i:s_i\in \mathcal{S}_S} \pi_i(t) \label{stoch:main.eq.ps} \\
\bar p_f(t)=\displaystyle \sum_{i:s_i\in \mathcal{S}_F} \pi_i(t)
\end{eqnarray}
where $\pi(t)$ is the probability vector of the DTMC related to the node (i.e. the solution of~\eqref{stoch:bg.eq.cauchy}). 
The time to succeed (fail) for a node is given by a random variable with exponential distribution 
and rate given by the inverse of the MTTS (MTTF) since for such random variables the mean time is given by the inverse of the rate.
\begin{eqnarray}
\mu=MTTS^{-1}\\
\nu=MTTF^{-1}
\label{stoch:main.eq.nu}
\end{eqnarray}
\end{proposition}
\begin{remark}
Proposition~\ref{stoch:main.eq.nu} holds also for deterministic and hybrid BTs, as \eqref{stoch:bg.eq.cauchy} has a unique solution in the Carath\'eodory sense~\cite{filippov1988differential}.

\end{remark}
\subsection{Deterministic Execution Times}
\label{stoch:PS.numdet}


As the formulation of the deterministic case involves Dirac delta functions, see Equation (\ref{stoch:eq:d1})-(\ref{stoch:eq:d2}), the approach described above might lead to computational difficulties.
As an alternative, we can take advantage of the fact that we know the exact time of possible transitions.
Thus, the success and failure probabilities of a deterministic node are unchanged in the intervals between the $MTTF$ and $MTTS$ of its children. 

\begin{example}
 Consider the BT 
\begin{equation}
\mathcal{T}=\mbox{Fallback}(\mathcal{A}_1,\mathcal{A}_2)
\end{equation}
depicted in Figure~\ref{stoch:res.fig.BTdet} and let $\tau_{Fi}$ ($\tau_{Si}$) be the $MTTF$ ($MTTS$) of action $i$ and $p_{fi}$ ($p_{si}$) its probability to fail (succeed). The success/failure probability over time of the tree $\mathcal{T}$ is a discontinuous function depicted in Figure~\ref{stoch:res.fig.probdet}. 
\end{example} 

\begin{figure}[h]
\centering
\includegraphics[width=0.5\columnwidth]{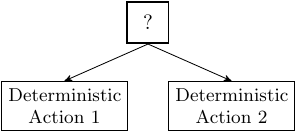}
\caption{Example of a Fallback node with two deterministic actions/subtrees.}
\label{stoch:res.fig.BTdet}
\end{figure}

\begin{figure}[h]
\centering
\includegraphics[width=0.8\columnwidth]{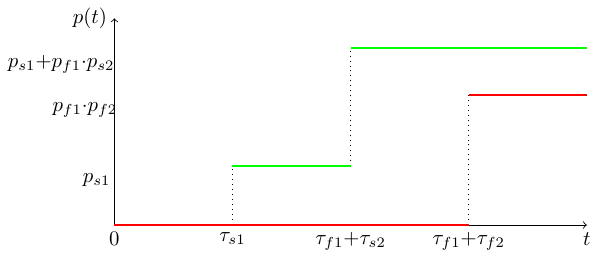}
\caption{Failure (red, lower) and success (green, upper) probability of the deterministic node of example. The running probability is the complement of the other two (not shown). 
}
\label{stoch:res.fig.probdet}
\end{figure}

Hence the success and failure probability have discrete jumps over time. These piece-wise continuous functions can be described by the discrete time system~\eqref{stoch:bg.eq.discrete} introducing the information of the time when the transitions take place, which is more tractable than directly solving~\eqref{stoch:bg.eq.cauchy}. Then, the calculation of $\pi(t)$ is given by a zero order hold of the discrete solution.

\begin{proposition}
\label{stoch:ps.prop.det}
Let $P$ be the one-step transition matrix given in Definition~\ref{stoch:bg.def.p} and let $\tau_{Fi}$ ($\tau_{Si}$) be the time to fail (succeed) of action $i$ and $ p_{fi}$ ($  p_{si}$) its probability to fail (succeed). Let $\tilde \pi(\tau)=[\tilde \pi_1(\tau),\ldots, \tilde \pi_{\left\vert{\mathcal{S}}\right\vert}(\tau)]^\top$, where $ \tilde \pi_i(\tau)$ is the probability of being in a marking $\mathbf{m}_i$ at time $\tau$ of a MRG representing a deterministic node with $N$ children, let $\tilde P(\tau)$ be a matrix which entries $\tilde p_{ij}(\tau)$ are defined as:
\begin{equation}
\tilde p_{ij}(\tau) = \begin{cases} p_{ij} \cdot \delta(\tau-(log (\tilde a_{j1})) &\mbox{if } i\neq j \\ 
1- \displaystyle\sum_{k \neq i} \tilde p_{ik} & \text{otherwise}  \end{cases} 
\label{stoch:ps.eq.ptilde}
\end{equation}
with $\tilde a_{ij}$ the $ij$-th entry of the matrix $\tilde A$ defined as:
\begin{equation}
\tilde A \triangleq \sum_{i=0}^{\infty} A^i
\label{stoch:ps.eq.atilde}
\end{equation}
with $A$ as defined in~\eqref{stoch:main.eq.A}.

Then the evolution of $ \tilde \pi(k)$ process can be described as a discrete time system with the following time evolution:
\begin{equation}
\label{stoch:tildepi}
\tilde \pi(\tau +{\Delta \tau}) = \tilde P(\tau)^\top\tilde \pi(\tau) 
\end{equation}
where $\Delta \tau$ is the common factor of $\{\tau_{F1},\tau_{S1},\tau_{F2},\tau_{S2},\ldots,\tau_{FN},\tau_{SN}\}$
Then for, deterministic nodes, given $\tilde \pi(\tau)$ the probability over time is given by:
\begin{equation}
\label{stoch:ps.eq.pdet}
\pi(t) = \mbox{ZOH}(\tilde \pi(\tau))
\end{equation}
where ZOH is the zero order hold function.
\end{proposition}
\begin{proof}
The proof is trivial considering that \eqref{stoch:tildepi} is a piece-wise constant function and $\Delta \tau$ is the common faction of all the step instants.
\end{proof}

\section{Examples}
\label{properties:sec:example:reliability}
In this section, we present three examples.
The first  example is the BT in Figure~\ref{stoch:Example.Tree}, which is fairly small and allows us to show the details of each step.
The second example is the deterministic time version of the same BT, illustrating the differences between the two cases.
The third example involves a more complex BT, shown in Figure \ref{stoch:res.fig.BT}.
This example will be used to verify the approach numerically, by performing Monte Carlo simulations and comparing the numeric results to the analytical ones, see Table \ref{stoch:res.tab.param} and Figure \ref{stoch:res.fig.comparisondet}. It is also used to illustrate the difference in performance metrics, between two equivalent BTs, see Figure \ref{stoch:res.fig.comparison}.

We will now carry out the computation of probabilistic parameters for an example SBT.
\label{stoch:Example.computed}
\begin{figure}
        \centering
        \begin{subfigure}[b]{0.4\columnwidth}
                \centering
                \includegraphics[width=\columnwidth]{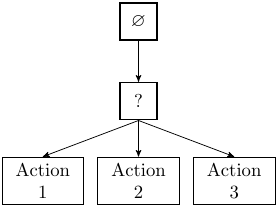}
                \caption{BT of example.}
                \label{stoch:Example.Tree}
        \end{subfigure}%
        ~ 
        \begin{subfigure}[b]{0.6\columnwidth}
                \centering
                \includegraphics[width=\columnwidth]{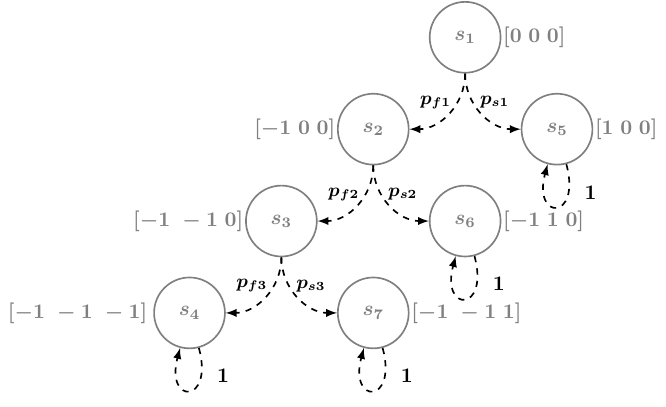}
                \caption{Markov Chain.}   
                \label{stoch:Example.MC}           
        \end{subfigure}
        ~ 
        \caption{BT and related DTMC modeling the plan of Example~\ref{stoch:Example.computed}.}
\end{figure}
\begin{example}
Given the tree shown in Figure~\ref{stoch:Example.Tree}, its probabilistic parameters are given by evaluating the Fallback node, since it is the child of the root node. The given PDF of the $i$-th action are:
\begin{eqnarray}
 \hat p_s(t)&=& p_{s_i} \mu e^{-\mu_i t}\\
  \hat p_f(t)&=& p_{f_i} \nu e^{-\nu_i t}
\end{eqnarray}
where:
\begin{itemize}
\item $p_{f_i}$ probability of failure
\item $p_{s_i}$ probability of success
\item $\nu_i$ failure rate
\item $\mu_i$ success rate
\end{itemize}
The DTMC related as shown in Figure~\ref{stoch:Example.MC} has $\mathcal{S}=\{s_1,s_2,s_3,s_4,s_5,s_6,s_7\}$, $\mathcal{S}_F =\{s_4\}$ and $\mathcal{S}_S =\{s_5,s_6,s_7\}$. 
 
According to the canonization in \eqref{stoch:Example.eq.MTTSform}, the one-step transition matrix is:
\begin{equation}
\label{stoch:app.eq.MTTSform}
 P_c^\top=\left[
\begin{array}{ccccccc}
0 & 0 & 0 & 0 & 0 & 0 & 0 \\ 
p_{f_1} & 0 & 0 & 0 & 0 & 0 & 0 \\ 
0 & p_{f_2} & 0 & 0 & 0 & 0 & 0 \\ 
0 & 0 & p_{f_3} & 1 & 0 & 0 & 0 \\ 
p_{s_1} & 0 & 0 & 0 & 1 & 0 & 0 \\ 
0 & p_{s_2} & 0 & 0 & 0 & 1 & 0 \\ 
0 & 0 & p_{s_3} & 0 & 0 & 0 & 1
\end{array} \right]
\end{equation}

According to \eqref{stoch:app.eq.sj} the average sojourn times are collected in the following vector
\begin{equation}
\label{stoch:app.eq.exSj}
SJ =\left[\begin{array}{c c c}
\frac{p_{s_1}}{\mu_1}+\frac{p_{f_1}}{\nu_1},&\frac{p_{s_2}}{\mu_2}+\frac{p_{f_2}}{\nu_2},&\frac{p_{s_3}}{\mu_3}+\frac{p_{f_3}}{\nu_3} 
\end{array}\right]
\end{equation}

The infinitesimal generator matrix is defined, according to~\eqref{stoch:bg.eq.infGen}, as follows:
\begin{equation}
\label{stoch:app.eq.exQ}
\setlength\arraycolsep{3pt}
Q=\left[
\begin{array}{ccccccc}
\frac{-\mu_1 \nu_1 }{p_{s_1}\nu_1+p_{f_1}\mu_1} & \scriptstyle{0} & \scriptstyle{0}  & \scriptstyle{0}  & \scriptstyle{0}  & \scriptstyle{0}& \scriptstyle{0}  \\[1ex]
\frac{\mu_1 \nu_1 p_{f_1} }{p_{s_1}\nu_1+p_{f_1}\mu_1} & \frac{-\mu_2 \nu_2 }{p_{s_2}\nu_2+p_{f_2}\mu_2} & \scriptstyle{0}  & \scriptstyle{0}  & \scriptstyle{0}  & \scriptstyle{0}  & \scriptstyle{0}  \\ [1ex]
\scriptstyle{0}  & \frac{\mu_2 \nu_2 p_{f_2} }{p_{s_2}\nu_2+p_{f_2}\mu_2} & \frac{-\mu_3 \nu_3 }{p_{s_3}\nu_3+p_{f_3}\mu_3} & \scriptstyle{0}  & \scriptstyle{0}  & \scriptstyle{0}  & \scriptstyle{0}  \\ [1ex]
\scriptstyle{0}  & \scriptstyle{0}  & \frac{\mu_3 \nu_3 p_{f_3} }{p_{s_3}\nu_3+p_{f_3}\mu_3} & \scriptstyle{0}  & \scriptstyle{0}  & \scriptstyle{0}  & \scriptstyle{0}  \\ 
\frac{\mu_1 \nu_1 p_{s_1} }{p_{s_1}\nu_1+p_{f_1}\mu_1} & \scriptstyle{0}  & \scriptstyle{0}  & \scriptstyle{0}  & \scriptstyle{0}  & \scriptstyle{0}  & \scriptstyle{0}  \\ 
\scriptstyle{0}  & \frac{\mu_2 \nu_2 p_{s_2} }{p_{s_2}\nu_2+p_{f_2}\mu_2} & \scriptstyle{0}  & \scriptstyle{0}  & \scriptstyle{0}  & \scriptstyle{0}  & \scriptstyle{0}  \\ 
\scriptstyle{0}  & \scriptstyle{0}  & \frac{\mu_3 \nu_3 p_{s_3} }{p_{s_3}\nu_3+p_{f_3}\mu_3} & \scriptstyle{0}  & \scriptstyle{0}  & \scriptstyle{0}  & \scriptstyle{0} 
\end{array} \right].
\end{equation}
The probability vector, according to~\eqref{stoch:bg.eq.cauchy}, is given by:
\begin{equation}
\pi(t)=\left[\begin{array}{ccccccc}
\pi_1(t)&\!\!\!\pi_2(t)&\!\!\!\pi_3(t)&\!\!\!\pi_4(t)&\!\!\!\pi_5(t)&\!\!\!\pi_6(t)&\!\!\!\pi_7(t)
\end{array}\right]^\top
\end{equation}
We can now derive closed form expression for MTTS and MTTF.
Using the decomposition in~\eqref{stoch:Example.eq.MTTSform}, the matrices computed according \eqref{stoch:mtt.eq.us} and~\eqref{stoch:mtt.eq.uf} are:
\begin{eqnarray}
 U^S&=\left[
\begin{array}{ccc}
p_{s_1} &     0	 & 0 \\ 
p_{f_1}  p_{s_2} &    p_{s_2}&   0\\ 
p_{f_1} p_{f_2} p_{s_3}& p_{f_2} p_{s_3} & p_{s_3}
\end{array} \right]
\\
 U^F&=\left[
\begin{array}{ccc}
 p_{f_1}p_{f_2}p_{f_3}& p_{f_2}p_{f_3} & p_{f_3}
\end{array} \right]
\end{eqnarray}
Note that $U^S$ is a  $3 \times 3$ matrix and $U^S$ is a $1 \times 3$ matrix since there are $3$ transient states, $3$ success state and $1$ failure state.
For action $i$ we define $t_{f_i}= \nu_i^{-1}$ the time to fail and $t_{s_i}= \mu_i^{-1}$ the time to succeed. The non-zero entries of the matrix given by~\eqref{stoch:main.eq.A} are:
\begin{equation}
\begin{split}
a_{2,1}={e}^{t_{f_1}}\;\;\;a_{3,2}={e}^{t_{f_2}}\;\;\;a_{4,3}={e}^{t_{f_3}} \\
a_{5,1}={e}^{t_{s_1}}\;\;\;a_{6,2}={e}^{t_{s_2}}\;\;\;a_{7,3}={e}^{t_{s_3}}
\end{split}
\end{equation}
from which we derive~\eqref{stoch:main.eq.hf} and~\eqref{stoch:main.eq.hs} as:
\begin{eqnarray}
  H^S&=\left[
\begin{array}{ccc}
e^{t_{s_1}}&                0 &        0 \\
e^{t_{f_1}}e^{t_{s_2}}&          e^{t_{s_2}}&        0\\
 e^{t_{f_1}}e^{t_{f_2}}e^{t_{s_3}}& e^{t_{f_2}}e^{t_{s_3}}& e^{t_{s_3}}
\end{array} \right]
\\
  H^F&=\left[
\begin{array}{ccc}
 e^{t_{f_1}}e^{t_{f_2}}e^{t_{f_3}}& e^{t_{f_2}}e^{t_{f_3}}& e^{t_{f_3}}
\end{array} \right]
\end{eqnarray}
Using \eqref{stoch:main.eq.mttf} and~\eqref{stoch:main.eq.mtts} we obtain the MTTS and MTTF. 
Finally, the probabilistic parameters of the tree are expressed in a closed form according to \eqref{stoch:main.eq.ps}-\eqref{stoch:main.eq.nu}:
\begin{eqnarray}
\bar p_s(t)=\!\!\!&\pi_5(t)+\pi_6(t)+\pi_7(t) \\
\bar p_f(t)=\!\!\!&\pi_4(t) \\
\mu=\!\!\!&\frac{p_{s_1}+p_{f_1}p_{s_2}+p_{f_1}p_{f_2}p_{s_3}}{p_{s_1}t_{s_1}+p_{f_1}p_{s_2}(t_{f_1}+t_{s_2})+p_{f_1}p_{f_2}p_{s_3}(t_{f_1}+t_{f_2}+t_{s_3})}\\
\nu=\!\!\!&\frac{1}{t_{f_1}+t_{f_2}+t_{f_3}}
\end{eqnarray}

\end{example}

\begin{example}
Consider the BT given in Example~\ref{stoch:Example.computed}, we now compute the performances in case when the actions are all deterministic.

The computation of MTTF and MTTS follows from Example~\ref{stoch:Example.computed}, whereas the computation of $\pi(t)$ can be made according to Proposition~\ref{stoch:ps.prop.det}.

According to~\eqref{stoch:ps.eq.atilde} the matrix $\tilde A$ takes the form below

\begin{equation}
 \tilde A=\left[
\begin{array}{ccccccc}
0 & 0 & 0 & 0 & 0 & 0 & 0 \\ 
 e^{t_{f_1}} & 0 & 0 & 0 & 0 & 0 & 0 \\ 
 e^{t_{f_1}} e^{t_{f_2}}& e^{t_{f_2}} & 0 & 0 & 0 & 0 & 0 \\ 
 e^{t_{f_1}}e^{t_{f_2}} e^{t_{f_3}} & e^{t_{f_2}} e^{t_{f_3}} & e^{t_{f_3}} & 0 & 0 & 0 & 0 \\ 
e^{t_{s_1}} & 0 & 0 & 0 & 0 & 0 & 0 \\ 
 e^{t_{f_1}}e^{t_{s_2}} & e^{t_{s_2}} & 0 & 0 & 0 & 0 & 0 \\ 
 e^{t_{f_1}}e^{t_{f_2}}e^{t_{s_3}} & e^{t_{f_2}}e^{t_{s_3}} & e^{t_{s_3}} & 0 & 0 & 0 & 0
\end{array} \right]
\end{equation}
thereby, according to~\eqref{stoch:ps.eq.ptilde}, the modified one step transition matrix takes the form of
Figure~\ref{stoch:eq:largefig},
\begin{landscape}
\begin{figure}\centering
\setlength{\tabcolsep}{1em}
\vfill\hfill $\tiny
\left[
\begin{array}{ccccccc}
1-(p_{f_1}\delta(t\minus{}t_{f_1})\plus{}p_{s_1}\delta(t\minus{}t_{s_1})) & 0 & 0 & 0 & 0 & 0 & 0 \\ 
p_{f_1}\delta(t\minus{}t_{f_1}) & 1-(p_{f_2}\delta(t\minus{}(t_{f_1}\plus{}t_{f_2}))\plus{}p_{s_2}\delta(t\minus{}(t_{f_1}\plus{}t_{s_2}))) & 0 & 0 & 0 & 0 & 0 \\ 
0 & p_{f_2}\delta(t\minus{}(t_{f_1}\plus{}t_{f_2})) & 1-(p_{f_3}\delta(t\minus{}(t_{f_1}\plus{}t_{f_2}\plus{}t_{f_3}))\plus{}p_{s_3}\delta(t\minus{}(t_{f_1}\plus{}t_{f_2}\plus{}t_{s_3}))) & 0 & 0 & 0 & 0 \\ 
0 & 0 & p_{f_3}\delta(t\minus{}(t_{f_1}\plus{}t_{f_2}\plus{}t_{f_3})) & 1 & 0 & 0 & 0 \\ 
p_{s_1}\delta(t\minus{}t_{s_1}) & 0 & 0 & 0 & 1 & 0 & 0 \\ 
0 & p_{s_2}\delta(t\minus{}(t_{f_1}\plus{}t_{s_2})) & 0 & 0 & 0 & 1 & 0 \\ 
0 & 0 & p_{s_3}\delta(t\minus{}(t_{f_1}\plus{}t_{f_2}\plus{}t_{s_3})) & 0 & 0 & 0 & 1
\end{array} \right]$ 
\vfill
 \captionsetup{justification=centering}
\caption{Modified one step transition matrix $\tilde P^\top$.}
 \label{stoch:eq:largefig}
\end{figure}
\end{landscape}
and the probability vector $\pi(t)$ is given by~\eqref{stoch:ps.eq.pdet}.
\end{example}

Below we present a more complex example, extending Example~\ref{stoch:Example.computed} above.
We use this example for two purposes, first, to verify the correctness of the proposed approach  using Monte Carlo simulations, and second, to illustrate how changes in the SBT lead to different performance metrics.

\begin{example}
\label{stoch:PS.ex.complex}
The task given to a two armed robot is to find and collect objects which can be found either on the floor, in the drawers or in the closet. The time needed to search for a desired object on the floor is less than the time needed to search for it in the drawers, since the latter has to be reached and opened first. On the other hand, the object is more likely to be in the drawers than on the floor, or in the closet. Moreover, the available policies for picking up objects are the one-hand and the two-hands grasps. The one-hand grasp most likely fails, but it takes less time to check if it has failed or not. Given these options, the task can be achieved in different ways, each of them corresponding to a different performance measure.
The plan chosen for this example is modeled by the SBT shown in Figure~\ref{stoch:res.fig.BT}. 

The performance estimates given by the proposed approach for the whole BT, as well as for two sub trees can be seen in Figures~\ref{stoch:res.fig.7}-\ref{stoch:res.fig.35} .
\end{example}

\begin{figure}[h]
\centering
\includegraphics[width=\columnwidth]{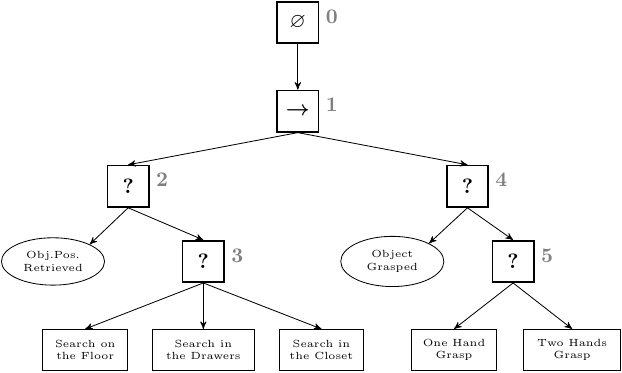}
\caption{BT modeling the search and grasp plan. The leaf nodes are labeled with a text, and the control flow nodes are labeled with a number, 
for easy reference.}
\label{stoch:res.fig.BT}
\end{figure}

We also use the example above to verify the correctness of the analytical estimates, and the results  can be seen in Table~\ref{stoch:res.tab.param}.
We compared the analytical solution derived using our approach with  numerical results given by a massive Monte Carlo simulation carried out using a BT implementation in the Robot Operative System (ROS)~\cite{Marzinotto14} where actions and conditions are performed using ROS nodes with outcomes computed using the \verb!C++! 
random number generator with exponential distribution. The BT implementation in ROS was run approximately 80000 times to have enough samples to get numerical averages close to the true values.
For each run we stored if the tree (and some subtrees) succeeded or failed and how long it took, allowing us to estimate $\mu$, $\nu$, $p_s(t)$, $p_f(t)$ experimentally. The match is reported in Figures~\ref{stoch:res.fig.7}-\ref{stoch:res.fig.35} and in Table~\ref{stoch:res.tab.times}. 
As can be seen, all estimates are within  0.18 \% of the analytical results.

\begin{table}[h]
\centering
  \begin{tabular}{| l | c | c | c |}\hline
    Measure & Analytical & Numerical & Relative Error \\ \hline \hline 
      $\mu_{0}$ & $\num{5.9039e-3}$ & $\num{5.8958e-3}$ & 0.0012 \\ \hline 
      $\nu_{0}$ & $\num{4.4832e-3}$ & $\num{4.4908e-3}$ & 0.0017 \\ \hline 
      $\mu_3$ &  $\num{6.2905e-3}$ & $\num{6.2998e-3}$  & 0.0014\\ \hline 
      $\nu_3$ & $\num{2.6415e-3}$ & $\num{2.6460e-3}$ & 0.0017 \\ \hline 
      $\mu_5$ & $\num{9.6060e-2}$ & $\num{9.5891e-2}$ &0.0018  \\ \hline 
      $\nu_5$ & $\num{4.8780e-2}$ & $\num{4.8701e-2}$ &0.0016 \\  
               \hline
  \end{tabular}
  \caption{Table comparing numerical and experimental results of MTTF and MTTS. The labels of the subscripts are given in Figure~\ref{stoch:res.fig.BT}}
  \label{stoch:res.tab.times}	
  \end{table}
\begin{table}[h]
\centering
  \begin{tabular}{| l | c | c | c | c |}\hline
    Label & $\mu$ & $\nu$ & $p_s(t)$ & $p_f(t)$\\ \hline \hline
    Obj. Pos. Retrieved & $-$ & $-$ & $p_{s5}(t)$ & $p_{f5}(t)$\\ \hline
    Object Grasped & $-$ & $-$ & $p_{s4}(t)$ & $p_{f4}(t)$\\ \hline
    Search on the Floor & $0.01$ & $0.0167$ & $0.3$ & $0.7$\\ \hline
    Search in the Drawer & $0.01$ & $0.01$ & $0.8$ & $0.2$\\ \hline
    Search in the Closet  & $0.005$ & $0.0056$ & $0.2$ & $0.8$\\ \hline
    One Hand Grasp & $0.1$ & $20$ & $0.1$ & $0.9$\\ \hline
    Two Hands Grasp & $0.1$ & $0.05$ & $0.5$ & $0.5$\\ 
    \hline
  \end{tabular}
  \caption{Table collecting given parameters, the labels of the control flow nodes are given in Figure~\ref{stoch:res.fig.BT}.}
  \label{stoch:res.tab.param}	
  \end{table}

\begin{figure}[!h]
\centering
\includegraphics[width=1.0\columnwidth]{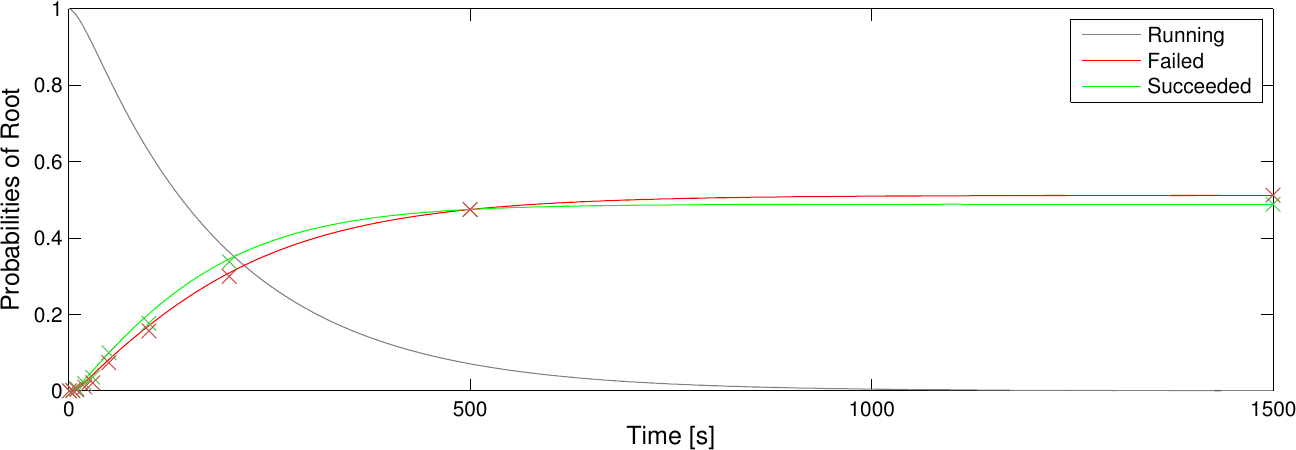}
\caption{Probability distribution over time for the Root node of the larger BT in Figure~\ref{stoch:res.fig.BT}.  Numerical results are marked with an 'x' and analytical results are drawn using solid lines. Note how the failure probability is initially lower, but then becomes higher than the success probability after $t=500$.}
\label{stoch:res.fig.7}
\end{figure}

\begin{figure}[h]
\centering
        \begin{subfigure}[h]{0.8\columnwidth}
                \includegraphics[width=0.8\columnwidth]{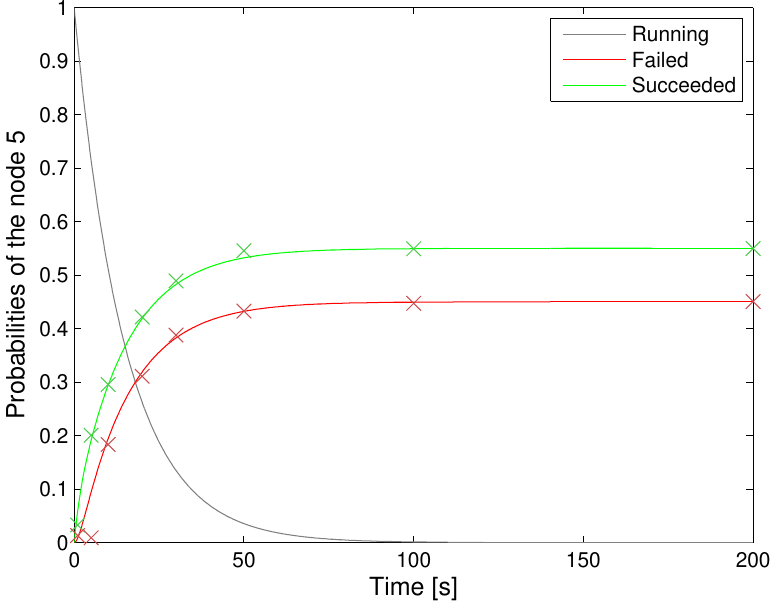}
                \caption{Node 5}
                \vspace{0.3cm}
                \label{stoch:res.fig.5}
        \end{subfigure}
        \begin{subfigure}[h]{0.8\columnwidth}
                \includegraphics[width=0.8\columnwidth]{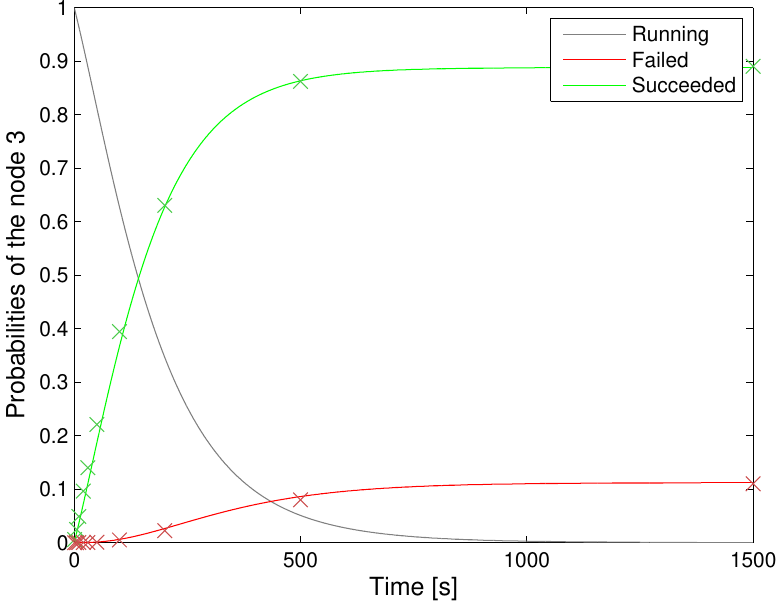}
                \caption{Node 3}   
                \label{stoch:res.fig.3}           
        \end{subfigure}
        ~ 
        \caption{Comparison of probability distribution over time related to  Node 5 (a) and Node 3 (b).
        Numerical results are marked with an 'x' and analytical results are drawn using solid lines. The failure probabilities are lower in both plots.}
        \label{stoch:res.fig.35}
\end{figure}

To further illustrate the difference between modeling the actions as deterministic and stochastic,
we again use the BT in Figure \ref{stoch:res.fig.BT} and
compute the
 accumulated Succes/Failure/Running probabilities for the two cases.
Defining the time to succeed and fail as the inverse of the given rates and computing the probabilities as described in Section~\ref{stoch:PS.numdet}
we get the results depicted in 
 Figures~\ref{stoch:res.fig.comparisondet} and \ref{stoch:res.fig.35det}.   
 As can be seen, the largest deviation is found in the Failure probabilities. In the stochastic case the CDF rises instantly, whereas in the deterministic case it becomes non-zero only after all the Fallbacks in at least one of the two subtrees have failed. 

\begin{figure}[!h]
\centering
\includegraphics[width=1.0\columnwidth]{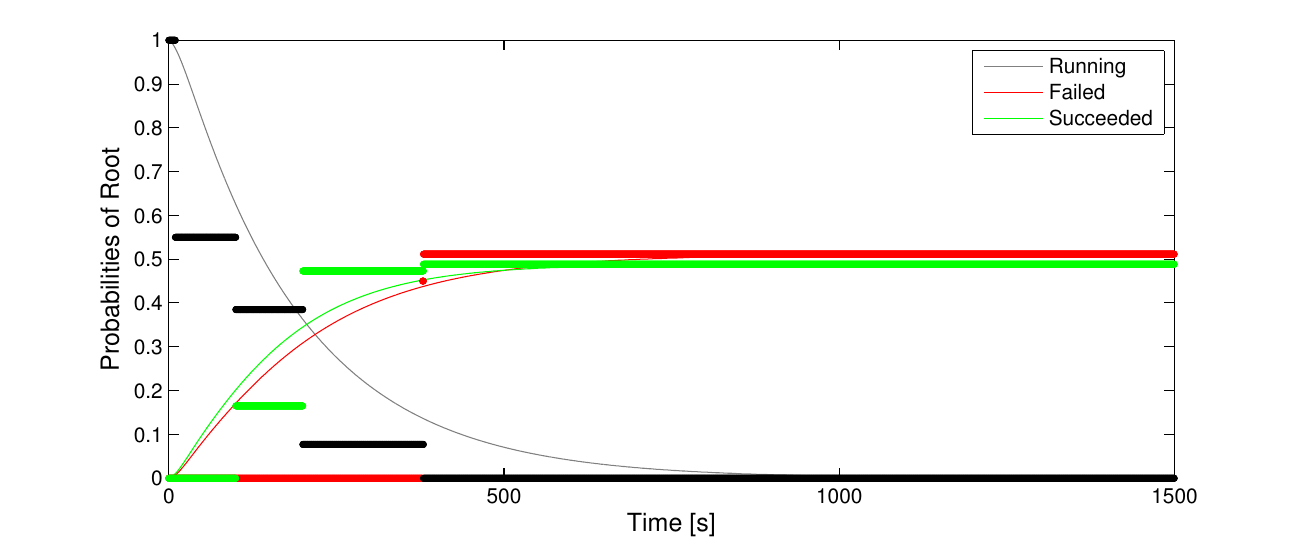}
\caption{Comparison of Success/Failure/Running probabilities of the root node in the case of deterministic times (thick) and stochastic times (thin). }
\label{stoch:res.fig.comparisondet}
\end{figure}

\begin{figure}[h]
\centering
        \begin{subfigure}[h]{0.8\columnwidth}
                \centering
                \includegraphics[width=0.75\columnwidth]{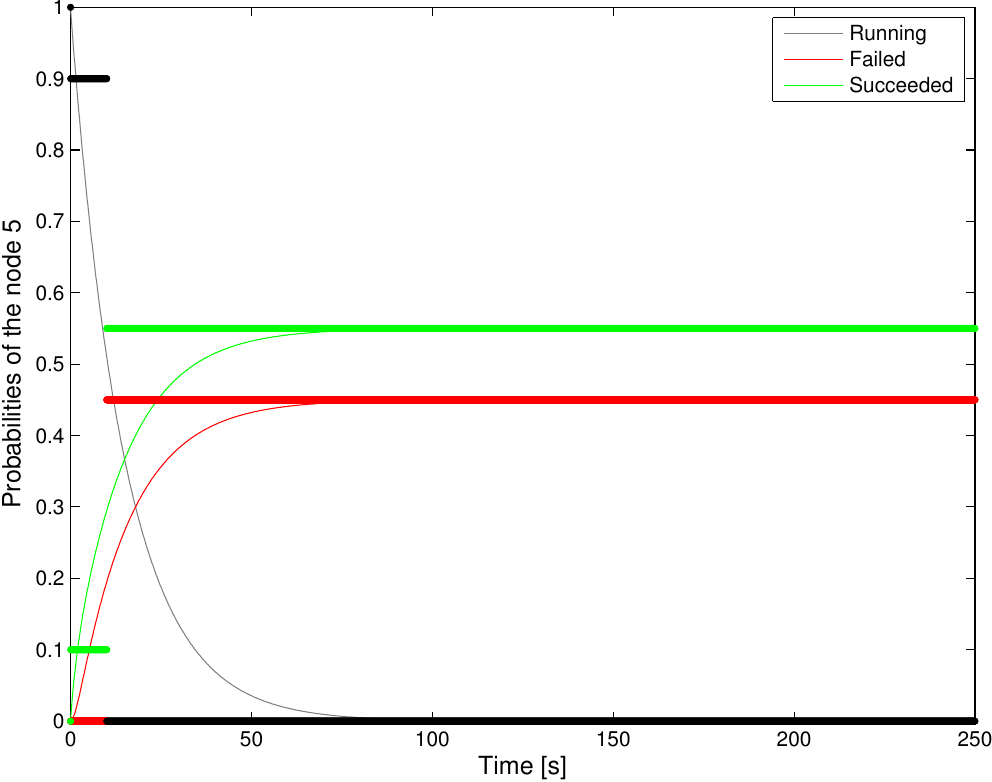}
                \caption{Node 5}
                 \vspace{0.3cm}
                \label{stoch:res.fig.5compar}
        \end{subfigure}%
        \vspace{0.1cm}
        \begin{subfigure}[h]{0.8\columnwidth}
                \centering
                \includegraphics[width=0.75\columnwidth]{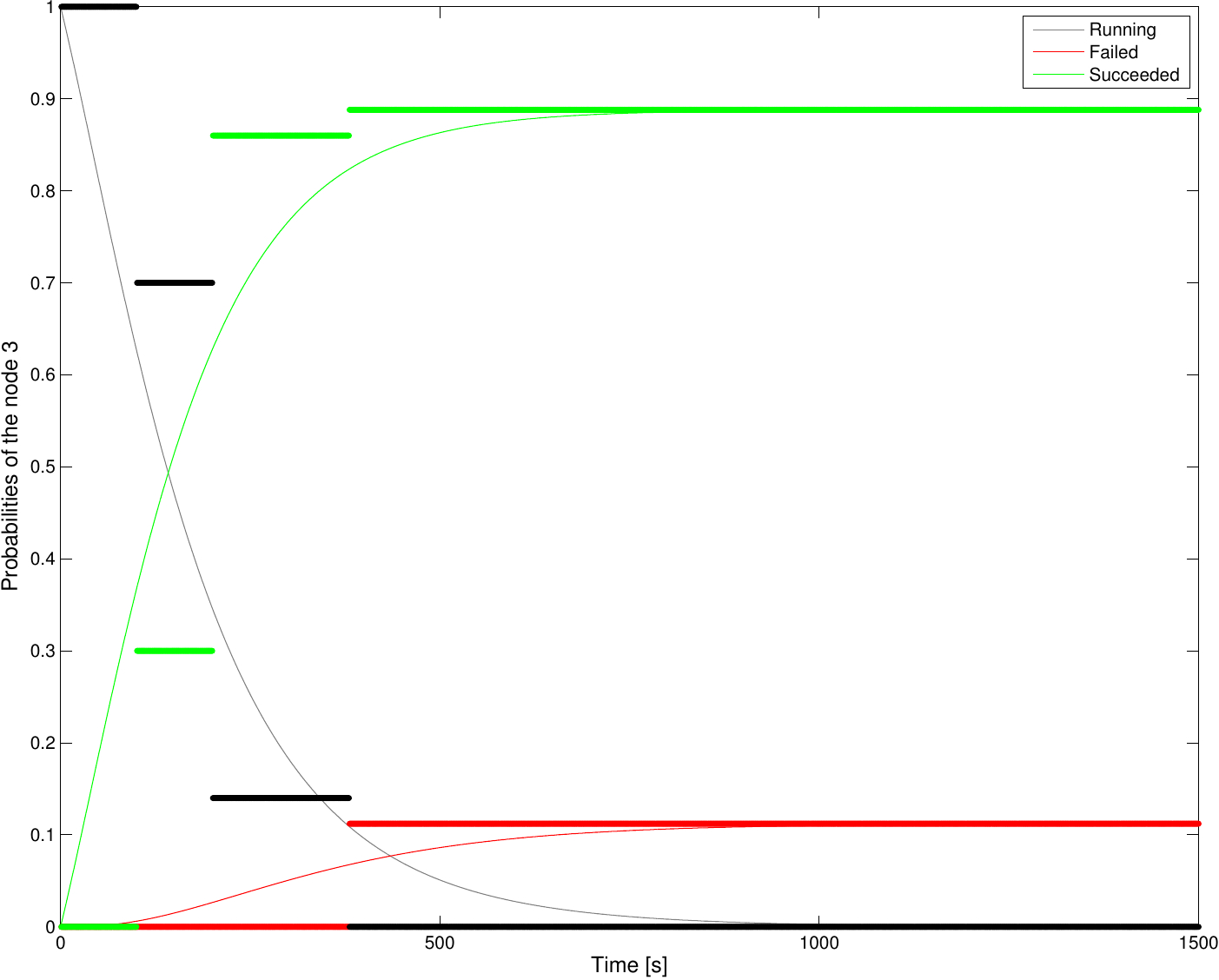}
                \caption{Node 3}   
                \label{stoch:res.fig.3compar}           
        \end{subfigure}
        ~ 
        \caption{Comparison of Success/Failure/Running probabilities of the node 5 (a) and node 3 (b) in the case of deterministic times (thick) and stochastic times (thin).}
        \label{stoch:res.fig.35det}
\end{figure}


In Figure~\ref{stoch:res.fig.comparison}
the results of swapping the order of ``Search on the Floor" and ``Search in the Drawers" are shown in.
As can be seen, the success probability after 100s is about 30\% when starting with the drawers, and about 20\% when starting with the floor.
Thus the optimal solution is a new BT, with the drawer search as the first option.
Note that
the asymptotic probabilities are always the same for equivalent BT, see Definition \ref{stoch:PS.def.eq}, as the  changes considered are only permutations of Fallbacks.

\begin{figure}[h]
\centering
\includegraphics[width=1.0\columnwidth]{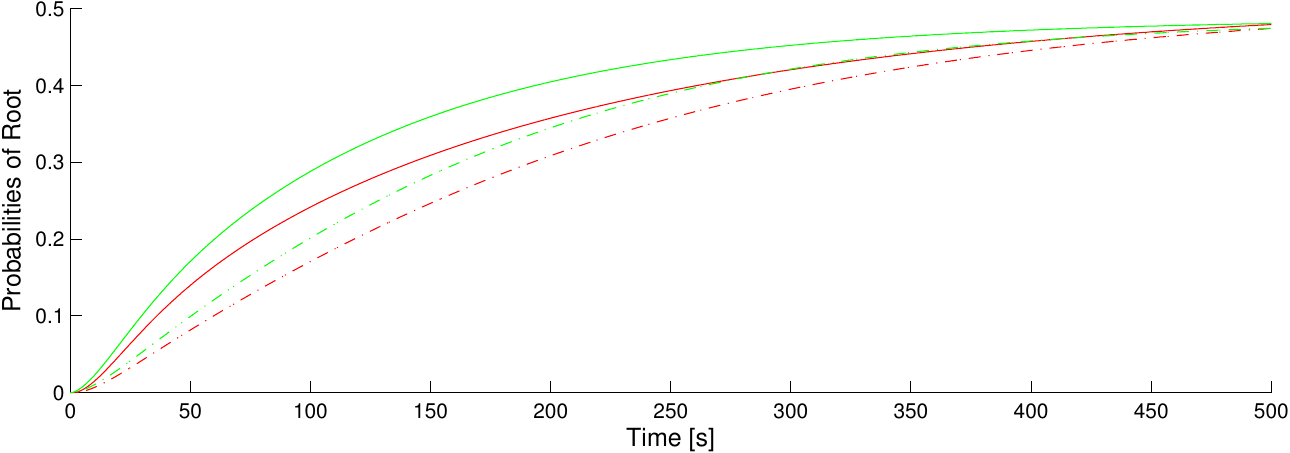}
\caption{Success/Failure probabilities in the case of searching on the floor first (dashed) and searching in the drawer first (solid). Failure probabilities are lower in both cases.}
\label{stoch:res.fig.comparison}
\end{figure}

%% file: conclusions/conclusions.tex
\chapter{Concluding Remarks}
\thispagestyle{empty}

In this book, we have tried to present a broad, unified picture of BTs.
We have covered the classical formulation of BTs, its extensions and its relation to other approaches.
We have provided theoretical results on efficiency, safety and robustness, using a new state space formalism,
as well as estimates on execution time and success probabilities using a stochastic framework.
We have described a number of practical design principles as well as connections between BTs and the important areas of planning and learning.

We believe that modularity is the main reason behind the huge success of BTs in the computer game AI community, and the growing popularity of BTs in robotics.
It is well known that modularity is a key enabler when designing
 complex, maintainable and reusable systems. Clear interfaces 
 reduce dependencies between components and makes development, testing, and reuse much simpler.
 BTs have such  interfaces, as each level of the tree has the same interface as a single action,
 and the internal nodes of the tree makes
 the implementation of an action  independent of the context and order in which the action is to be used. 
 Finally, these simple interfaces provide structures that are equally beneficial for both humans and machines. 
 In fact, they are vital to the ideas of all chapters, from state-space formalism and planning to design principles and machine learning.

Thus, BTs represent a promising control architecture in both computer game AI and robotics. However, the parallel development in the field has given rise to a set of different formulations and variations on the theme.  This book is an attempt to provide a unified view of a breadth of ideas, algoritms and applications.
There is still lots of work to be done, and we hope the reader has found this book helpful,
and perhaps inspiring, when continuing on the journey towards building better virtual agents and robots.